\documentclass[11pt]{article}

\usepackage{amsmath,amssymb,amsthm,mathtools}
\usepackage{bbm}
\usepackage{graphicx}
\usepackage{booktabs}
\usepackage{algorithm}
\usepackage{algpseudocode}
\usepackage{xurl}
\usepackage{arxiv}
\usepackage{subcaption}
\usepackage{float}
\usepackage{enumitem}
\newcommand{\sS}{\mathcal{S}}
\newcommand{\sA}{\mathcal{A}}
\newcommand{\sX}{\mathcal{X}}

\newcommand{\sH}{\mathcal{H}}

\newcommand{\sQ}{\mathcal{Q}}
\newcommand{\sW}{\mathcal{W}}
\newcommand{\mhat}{\widehat m}

\newcommand{\E}{\mathbb{E}}
\newcommand{\1}{\mathbf{1}}
\newcommand{\dd}{\mathrm{d}}

\newcommand{\nup}{\nu_\pi^+}
\newcommand{\cpi}{c_\pi}
\newcommand{\omegazero}{\omega_0}
\newcommand{\omegastar}{\omega_{\pi,\gamma}}
\newcommand{\omegahatzero}{\widehat{\omega}_0}
\newcommand{\Bpig}{\mathsf{B}_\gamma^\pi}
\newcommand{\Bpigcov}{\mathsf{B}_{\gamma,\mathrm{cov}}^\pi}
\newcommand{\Bpigcap}[1]{\mathsf{B}_{\gamma,\mathrm{cov},#1}^\pi}
\newcommand{\Pback}{P_{\pi,\nu}^{\leftarrow}}
\newcommand{\Mpi}{\Pback}
\newcommand{\Ppi}{P_\pi}
\newcommand{\dinit}{d_0}
\newcommand{\omegahat}{\widehat{\omega}}

\newcommand{\chat}{\widehat{c}_\pi}

\newcommand{\norm}[2][]{\left\lVert #2 \right\rVert_{#1}}
\newcommand{\LoneNu}[1]{\norm[L^1(\nu)]{#1}}
\newcommand{\LoneNuPlus}[1]{\norm[L^1(\nu_\pi^+)]{#1}}

\newcommand{\abs}[1]{\left\lvert #1 \right\rvert}

\DeclareMathOperator*{\argmin}{arg\,min}

\theoremstyle{plain}
\newtheorem{theorem}{Theorem}[section]
\newtheorem{corollary}[theorem]{Corollary}
\newtheorem{proposition}[theorem]{Proposition}
\newtheorem{lemma}[theorem]{Lemma}

\theoremstyle{definition}

\theoremstyle{remark}

\makeatletter
\AddToHook{env/theorem/after}{\@afterindentfalse\@afterheading}
\AddToHook{env/corollary/after}{\@afterindentfalse\@afterheading}
\AddToHook{env/proposition/after}{\@afterindentfalse\@afterheading}
\AddToHook{env/lemma/after}{\@afterindentfalse\@afterheading}
\AddToHook{env/definition/after}{\@afterindentfalse\@afterheading}
\AddToHook{env/remark/after}{\@afterindentfalse\@afterheading}
\AddToHook{env/proof/after}{\@afterindentfalse\@afterheading}
\AddToHook{env/enumerate/after}{\@afterindentfalse\@afterheading}
\AddToHook{env/itemize/after}{\@afterindentfalse\@afterheading}
\makeatother

\title{Fitted Occupancy-Ratio Evaluation\\ without Bellman Completeness}

\author{
  Lars van der Laan \\
 Stanford University \\
   \texttt{vdlaan@stanford.edu}
  \And
  Nathan Kallus \\
  Netflix and Cornell University
}

\begin{document}

\maketitle

\begin{abstract}
 
Occupancy ratios correct distribution shift in offline reinforcement learning
and are central to off-policy evaluation. Existing primal-dual and minimax
methods typically estimate these ratios by enforcing occupancy-balance moments over a critic class. We propose fitted occupancy-ratio
evaluation (\textsc{FORE}), a fitted fixed-point method that characterizes the
discounted occupancy ratio through an adjoint Bellman recursion. At each
iteration, \textsc{FORE} solves a single-level density-ratio objective on
one-step-transition data, thereby projecting the adjoint Bellman image onto a log-ratio
class in Kullback--Leibler (KL) divergence. Unlike analyses of fitted
\(Q\)-evaluation, which typically require value-function realizability together
with Bellman completeness or projected-operator stability, our central
approximation condition is just realizability of the discounted occupancy ratio
itself. Under this condition, the population KL-projected recursion contracts
in relative entropy toward the true ratio by virtue of the adjoint Bellman operator being a KL-contraction. For the empirical recursion, we
establish finite-sample regret bounds that yield convergence in KL up to approximation error and a statistical error governed by the
complexity of the ratio hypothesis class. 
When full coverage fails, we introduce coverage-stopped \textsc{FORE}, which
targets the discounted occupancy accumulated before the first uncovered
state--action pair and yields a conservative lower bound on target-policy value
for nonnegative rewards.  The fitted ratio supports direct value
estimation by reward reweighting, occupancy-weighted fitted \(Q\)-evaluation, and doubly robust estimation that combines the fitted ratio with
a fitted \(Q\)-function. Together, these results identify discounted occupancy-ratio realizability as a sufficient condition for offline policy evaluation
without any completeness assumptions.

\end{abstract}

\section{Introduction}

Offline policy evaluation must correct the mismatch between the distribution of
observed transitions and the discounted occupancy distribution induced by a
target policy. Value-based and ratio-based methods address this mismatch by
estimating complementary objects. On the value-function side, fitted
$Q$-evaluation (FQE) is the canonical fitted-iteration approach: it repeatedly regresses Bellman targets onto a value-function class to estimate the $Q$-function
\citep{lagoudakisParr2003LSPI,ernstEtAl2005TreeBatchRL,
antosEtAl2007FQI,munosSzepesvari2008FVI,riedmiller2005neural,
tosatto2017boosted,leEtAl2019BatchPolicyConstraints}.
This approximates the iteration of the composition of a projection operator with
the Bellman operator. Unlike the Bellman operator alone, this composition need
not be contractive: the Bellman operator and the projection are naturally
controlled under \textit{different} measures
\citep{vanDerLaanKallus2025StationaryWeightedFQE}. Contractivity can be
recovered if the Bellman image already lies in the projection image, so that the
composition reduces to the Bellman operator itself. This condition is called Bellman
completeness, a key assumption in analyses of FQE; alternatives include
directly assuming projected-Bellman stability or a small inherent Bellman error
\citep{antosEtAl2007FQI,munosSzepesvari2008FVI,chen2019information,
foster2021offline,amortila2020variant,wang2021exponential,
wang2021statistical,chang2022learning}.

The discounted occupancy ratio is the density of the target policy's discounted
state--action occupancy measure relative to the offline data distribution. It
converts offline-distribution averages into target-occupancy averages and is a
central ingredient in marginalized importance sampling, doubly robust
estimation, stationary-distribution correction, and occupancy-weighted value
methods 
\citep{thomasBrunskill2016DataEfficientOPE,jiangLi2016DoublyRobust,
xieEtAl2019MarginalizedIS,yinWang2020EfficientTabularOPE,
kallusUehara2020DoubleRL,kallusUehara2020DRL_ICML,kallusUehara2022BreakingHorizonDRL,liuEtAl2018InfiniteHorizonOPE,
hallakMannor2017COPTD,suttonEtAl2016EmphaticTD,
geladaBellemare2019CovariateShift,
vanDerLaanKallus2025StationaryWeightedFQE,
vanDerLaanKallus2025SoftFQI}.
This ratio satisfies an adjoint Bellman equation, equivalently a family of
balance conditions indexed by state-action test functions. This perspective
gives rise to primal-dual and minimax methods that seek a ratio function
minimizing empirical imbalance over a class of critic functions
\citep{liuEtAl2018InfiniteHorizonOPE,nachumEtAl2019DualDICE,
ueharaEtAl2021FiniteSampleMinimax,ueharaEtAl2020MWLMQL}. Their guarantees consequently shift the approximation burden to the coupled ratio and critic classes and require critic richness, dual realizability, or completeness conditions. The same approaches also apply to $Q$-learning by minimizing empirical deviations of the (non-adjoint) Bellman equation, with similar requirements for guarantees.

We propose fitted occupancy-ratio evaluation (\textsc{FORE}), a fitted-iteration
method for discounted occupancy-ratio estimation. \textsc{FORE} is analogous to
FQE, but at each iteration it approximates the composition of a projection in
Kullback--Leibler (KL) divergence with the adjoint Bellman operator.
Importantly, unlike FQE, this composition combines compatible contractions,
yielding convergence without any completeness condition. The method requires no
separate critic class, and each iteration can be implemented with standard
supervised learners, such as gradient-boosted trees or neural networks.

The key approximation condition for \textsc{FORE} is mere realizability (or approximation) of the true discounted occupancy ratio by the hypothesis class. It does not
require an adjoint Bellman completeness condition: adjoint Bellman images of
arbitrary candidate ratios need not remain in, or be well approximated by, the
hypothesis class. The reason is that the adjoint Bellman recursion is aligned with
the KL projection geometry. The unprojected adjoint Bellman map contracts
relative entropy toward the true ratio by the discount factor, and
\textsc{FORE} projects in the same geometry. Thus, when the target ratio is
realizable, the projected population operator inherits this contraction. The
analysis therefore replaces closure of an entire sequence of Bellman images
with direct approximation of the target fixed point.

This contrasts with standard FQE. The Bellman evaluation operator is naturally
stable in the target-policy occupancy norm, whereas FQE uses a projection norm
under the offline data distribution. Without Bellman completeness or projected-operator stability,
this norm mismatch can make the projected Bellman recursion underlying FQE
unstable
\citep{pattersonEtAl2022GeneralizedPBE,
vanDerLaanKallus2025StationaryWeightedFQE}. This helps explain why
value-function realizability alone does not ensure stable FQE
\citep{wang2021exponential,wang2021statistical,foster2021offline,
amortila2020variant}.

Besides realizability, the coverage requirements for full-ratio guarantees are
standard in offline policy evaluation: the target-relevant distributions must
be absolutely continuous with respect to the offline data distribution, and
the finite-sample theory requires boundedness or subexponential tail conditions
for empirical-process control. Because \textsc{FORE} models density ratios
through their log-ratios, the full-coverage analysis also assumes that the
target discounted occupancy ratio is positive on the support of the offline
data distribution.

\medskip

\noindent\textbf{Contributions.}
We make five contributions.
\begin{enumerate}
\item We characterize the discounted occupancy ratio as the unique fixed point
of an adjoint Bellman recursion and introduce \textsc{FORE}, a fitted method
that estimates this recursion through repeated single-level KL density-ratio
objectives.

\item We develop a population approximation theory based directly on
occupancy-ratio realizability. The KL-projected adjoint Bellman operator
contracts in relative entropy toward the true ratio, up to an approximation
error determined by the log-ratio class. This result does not require an
adjoint Bellman completeness condition.

\item We prove finite-sample guarantees for the empirical fitted recursion. The
generalized KL error decomposes into a geometrically decaying initialization
term, a population KL approximation term, and a statistical term governed
by local Rademacher critical radii for the log-ratio class and the induced
multiplier class.

\item We develop three policy-evaluation applications. The fitted ratio yields
reward-reweighted value estimates, a doubly robust estimator with a
product-form error bound, and an occupancy-weighted FQE procedure. For the last
application, we derive bounds that separate ratio error, value-function
approximation error, and finite-iteration error, without imposing Bellman
completeness on the value-function class.

\item We extend \textsc{FORE} to limited-coverage settings via coverage-stopped
\textsc{FORE}, which targets the discounted occupancy accumulated before the
first uncovered state--action pair. The method uses a learned coverage
classifier to remove unsupported target-policy occupancy, while clipping
regularizes estimation in weakly covered regions and preserves population
contraction. The resulting occupancy mass diagnoses effective coverage, and
the corresponding value is a conservative lower bound on the target-policy
value for nonnegative rewards.
\end{enumerate}

\subsection{Related Work}
\label{sec:relatedwork}

\paragraph{Off-policy evaluation and occupancy corrections.}
Classical off-policy evaluation uses trajectory-level or per-decision
importance ratios, while doubly robust estimators combine importance weighting
with value-function estimates
\citep{thomasBrunskill2016DataEfficientOPE,jiangLi2016DoublyRobust}.
Marginalized importance sampling avoids products of trajectory ratios by
correcting marginal state or state--action occupancies
\citep{xieEtAl2019MarginalizedIS,yinWang2020EfficientTabularOPE,
liuEtAl2018InfiniteHorizonOPE}. Semiparametric theory likewise identifies the
occupancy ratio, together with a value function, as the pair of nuisance
functions underlying efficient and doubly robust off-policy evaluation
\citep{kallusUehara2019IntrinsicOPE,kallusUehara2020DRL_ICML,
kallusUehara2020DoubleRL,ueharaShiKallus2022OPEReview,
kallusUehara2020PolicyGradients,kallusUehara2020DeterministicPolicyDR,
kallusUehara2022BreakingHorizonDRL,kallusUehara2024NaturalPolicies,
vanDerLaanEtAl2025AutomaticDRL,
vanDerLaanKallusBibaut2025ClassificationIRL,
vanDerLaanBibautKallus2025EfficientIRL}. Occupancy and
stationary-distribution corrections also appear in off-policy temporal
difference learning, emphatic TD, generalized projected Bellman-error
objectives, stationary-weighted FQE, and stationary-reweighted soft fitted
\(Q\)-iteration
\citep{suttonEtAl2016EmphaticTD,hallakMannor2017COPTD,
geladaBellemare2019CovariateShift,pattersonEtAl2022GeneralizedPBE,
vanDerLaanKallus2025StationaryWeightedFQE,
vanDerLaanKallus2025SoftFQI}.

\paragraph{Primal-dual and minimax ratio estimation.}
These methods estimate occupancy corrections by enforcing balance or
stationarity restrictions through saddle-point, minimax, or
temporal-difference objectives. DualDICE estimates discounted distribution
corrections without behavior-policy probabilities or trajectory products
\citep{nachumEtAl2019DualDICE}. GenDICE and GradientDICE extend this
perspective to stationary-distribution correction and off-policy evaluation
\citep{zhangEtAl2020GenDICE,zhangEtAl2020GradientDICE}. Related work develops
infinite-horizon density-ratio estimators, minimax weight and value-function
learners, regularized-Lagrangian formulations, confidence intervals, and
regression-based variants
\citep{liuEtAl2018InfiniteHorizonOPE,ueharaEtAl2020MWLMQL,
ueharaEtAl2021FiniteSampleMinimax,yangEtAl2020RegularizedLagrangianDICE,
daiEtAl2020CoinDICE,cheEtAl2025AVGDICE}. Other occupancy-correction methods
use successor representations, state abstractions, or distribution matching
for policy optimization, constrained control, and imitation learning
\citep{fujimotoEtAl2021SRDICE,pavseHanna2023ScalingMIS,
nachumEtAl2019AlgaeDICE,kostrikovEtAl2020ValueDICE,
leeEtAl2021OptiDICE,leeEtAl2022COptiDICE,maEtAl2022SMODICE,
kimEtAl2022LobsDICE}. These procedures can also be viewed as minimax
estimators of the Riesz representer of the policy-value functional
\citep{dikkala2020minimax,bennett2023source,bennett2025inference,
vanDerLaanEtAl2025AutomaticDRL}. In contrast, \textsc{FORE} uses the adjoint
Bellman identity recursively and estimates each projected image through a
single-level density-ratio objective, without introducing a separate critic
class.

\paragraph{Fitted occupancy-ratio estimation under completeness.}
A closely related finite-horizon precursor to \textsc{FORE} is the FORC
estimator of \citet{huangChenJiang2023DensityFeatures}; see also
\citet{huangJiang2024OccupancyPG}. FORC recursively fits stagewise occupancy
ratios by squared-loss regression and is closest in spirit to the
regression-based variant of \textsc{FORE} in
Appendix~\ref{app:strong-form-regression-fori}, which combines iterative
regression with preliminary density-ratio estimates. Both approaches rely on a
density-ratio analogue of Bellman completeness; specifically, our
regression-based variant assumes adjoint Bellman completeness. By contrast,
the main \textsc{FORE} estimator iteratively targets adjoint Bellman images and
projects each image onto the fitted ratio class in KL divergence, under which
the population adjoint Bellman recursion is contractive. Its guarantees
therefore require only approximation of the target discounted occupancy ratio,
rather than closure of the fitted class under intermediate adjoint Bellman
images.

\paragraph{Policy evaluation and Bellman completeness.}
A central difficulty in offline policy evaluation with function approximation
is that value-function realizability alone does not ensure stable Bellman
regression under distribution shift. Analyses of FQE and fitted value
iteration typically control this instability through Bellman completeness,
small inherent Bellman error, or contraction of the projected Bellman operator,
together with coverage or concentrability conditions relating target-policy
distributions to the offline data
\citep{antosEtAl2007FQI,munosSzepesvari2008FVI,chen2019information,foster2021offline}. In tabular
models, state aggregation, and discretized representations, the required
stability is built into the approximation architecture
\citep{puterman1994MDP,munosSzepesvari2008FVI}. Linear function approximation provides another route when the
features are compatible with the reward, transition, and coverage structure
\citep{lagoudakisParr2003LSPI}.

Minimax and adversarial Bellman-error methods replace direct Bellman
regression with moment conditions evaluated against an auxiliary critic class.
These approaches can avoid Bellman completeness of the value class, but their
analyses instead require the critic to detect the relevant residuals, through
conditions such as critic richness, dual realizability or completeness,
coverage, and control of saddle-point optimization
\citep{ueharaEtAl2020MWLMQL,ueharaEtAl2021FiniteSampleMinimax,
dikkala2020minimax,bennett2023source,bennett2025inference}. Weighted TD and
FQE-style methods take a complementary route: they retain fitted Bellman
regression but change the projection norm using emphatic, covariate-shift, or
stationary occupancy weights
\citep{suttonEtAl2016EmphaticTD,hallakMannor2017COPTD,geladaBellemare2019CovariateShift,
pattersonEtAl2022GeneralizedPBE,
vanDerLaanKallus2025StationaryWeightedFQE,
vanDerLaanKallus2025SoftFQI}. These results show that an appropriate occupancy
weight can stabilize projected Bellman regression. With nonlinear function
approximation, however, the weight must itself be estimated, and DICE-style or
minimax estimators can reintroduce critic-side realizability or completeness
conditions at this first stage.

\textsc{FORE} addresses this ratio-estimation problem directly. Its guarantees
require neither Bellman completeness of a value class, adjoint Bellman
completeness of a ratio class, nor dual completeness of a critic class. The
central approximation condition is instead that the discounted occupancy ratio
be realizable, or well approximated, by the chosen log-ratio class. This is
made possible by KL projection, which aligns with the relative-entropy
contraction of the adjoint Bellman recursion. We use the resulting ratio in
Section~\ref{sec:fori-weighted-fqe} to construct an end-to-end
\textsc{FORE}-weighted FQE procedure.

\section{Setup and Adjoint Bellman Identification}
 
\subsection{MDP and target occupancy}
\label{sec:mdp-target-occupancy}

Let \((\sS,\sA,P,\mu_0)\) be an MDP with target initial state distribution
\(\mu_0\). Write \(\sX=\sS\times\sA\) and \(X=(S,A)\). We study estimation of
the discounted occupancy ratio of a target policy \(\pi\) relative to an
offline state--action distribution \(\nu\), focusing on the setting
\(\gamma<1\); the undiscounted analogue is treated in
Appendix~\ref{app:mixed-contraction}. The method requires samples from \(\nu\)
and their one-step target-policy successors, but not an explicit behavior
policy. For an ergodic trajectory under a stationary behavior policy, \(\nu\)
may be the stationary distribution of the induced state--action process. For
pooled finite trajectories, \(\nu\) may instead be the time-averaged
state--action distribution over the sampled time points.

For an action distribution \(\eta\), let \(P_\eta\) denote the induced state--action
transition kernel,
\[
    P_\eta(\dd s',\dd a'\mid s,a)
    =
    P(\dd s'\mid s,a)\eta(\dd a'\mid s'),
\]
and write \(\Ppi\) for the kernel induced by \(\pi\). For a finite signed
measure \(\mu\) on \(\sX\), let \(\mu\Ppi\) and \(\mu\Ppi^t\) denote its
one-step and \(t\)-step pushforwards. Let
\[
    \dinit(\dd s,\dd a)=\mu_0(\dd s)\pi(\dd a\mid s)
\]
be the target initial state--action distribution. For \(X\sim\nu\), let
\(X^+\mid X\sim\Ppi(\cdot\mid X)\), and write \(\nup=\nu\Ppi\) for the
marginal distribution of \(X^+\). We make the following coverage condition.

\begin{enumerate}[label=\textbf{(A\arabic*)}, ref=A\arabic*, series=forecond]
\item \label{ass:overlap}
\textit{One-step target coverage.} The measures \(\dinit\) and \(\nup\) are
absolutely continuous with respect to \(\nu\).
\end{enumerate}
For a discount factor \(\gamma\in[0,1)\), the normalized discounted target
occupancy measure is
\[
    d_{\pi,\gamma}
    =
    (1-\gamma)\sum_{t=0}^{\infty}\gamma^t \dinit \Ppi^t .
\]
Under Condition~\ref{ass:overlap}, \(d_{\pi,\gamma}\ll\nu\). Our goal is to
estimate the state--action occupancy ratio
\[
    \omegastar=\frac{\dd d_{\pi,\gamma}}{\dd\nu}.
\]
This ratio is useful for policy evaluation because, for any integrable reward
function \(r\),
\[
    V_\pi(r)
    :=
    E_{d_0,P_\pi}\left\{(1-\gamma)\sum_{t=0}^\infty
    \gamma^t r(X_t)\right\}
    =
    E_{d_{\pi,\gamma}}\{r(X)\}
    =
    E_\nu\{\omegastar(X)r(X)\},
\]
where the first expectation is over trajectories \(\{X_t\}_{t\ge0}\) with
\(X_0\sim d_0\) evolving according to \(P_\pi\).

When Condition~\ref{ass:overlap} fails, the full occupancy ratio is not
identified. Section~\ref{sec:coverage-truncated-fore} instead targets a
coverage-stopped occupancy, estimated using a learned coverage classifier,
with clipping used to regularize estimation.

\subsection{Adjoint Bellman identification}
\label{sec:adjoint-identification}

The discounted occupancy measure satisfies the Bellman equation
\begin{equation}
\label{eqn::occbell}
    d_{\pi,\gamma}
    =
    (1-\gamma)\dinit
    +
    \gamma d_{\pi,\gamma}\Ppi .
\end{equation}
Taking Radon--Nikodym derivatives in \eqref{eqn::occbell} gives
\[
    \omegastar
    =
    (1-\gamma)\omegazero
    +
    \gamma
    \frac{\dd\{(\omegastar\nu)\Ppi\}}{\dd\nu}; \qquad     \omegazero:=\frac{\dd\dinit}{\dd\nu}.
\]
Thus, for any \(\omega\) such that \((\omega\nu)\Ppi\ll\nu\), define the
adjoint Bellman operator \citep{ueharaEtAl2021FiniteSampleMinimax}
\[
    \Bpig\omega
    =
    (1-\gamma)\omegazero
    +
    \gamma
    \frac{\dd\{(\omega\nu)\Ppi\}}{\dd\nu}.
\]
Then the occupancy ratio is characterized by the fixed-point equation
\begin{equation}
\label{eq:adjoint-bellman-fixed-point}
    \omegastar=\Bpig\omegastar .
\end{equation}

Although \(\Bpig\omega\) is generally not available pointwise, its action
against critic functions can be evaluated from one-step transitions. For any
measurable \(f\) for which the expectations exist,
\begin{equation}
\label{eq:bellman-moment-update}
    E_\nu\{(\Bpig\omega)(X)f(X)\}
    =
    (1-\gamma)E_{\dinit}\{f(X)\}
    +
    \gamma E_\nu\{\omega(X)f(X^+)\},
\end{equation}
where \(X^+\sim\Ppi(\cdot\mid X)\). At the fixed point, this becomes the
occupancy Bellman moment identity
\begin{equation}
\label{eq:weak-occupancy-balance}
    E_\nu\!\left[
        \omegastar(X)\{f(X)-\gamma f(X^+)\}
    \right]
    =
    (1-\gamma)E_{\dinit}\{f(X)\}.
\end{equation}

Minimax
occupancy-balancing methods \citep{liuEtAl2018InfiniteHorizonOPE,nachumEtAl2019DualDICE,
ueharaEtAl2020MWLMQL,ueharaEtAl2021FiniteSampleMinimax} estimate a ratio function by making violations of
\eqref{eq:weak-occupancy-balance} small uniformly over a critic class:
\begin{equation}
    \label{eq:minimax}
    \argmin_{\omega\in\mathcal W}
    \sup_{f\in\mathcal F}
    \left\{
        (1-\gamma)E_{\dinit}\{f(X)\}
        -
        E_\nu\!\left[
            \omega(X)\{f(X)-\gamma f(X^+)\}
        \right]
    \right\}.
\end{equation}
Here \(f\in\mathcal F\) acts as a critic for violations of the adjoint Bellman balance equations. To make this into an estimator, one replaces true expectations with empirical ones and also regularizes $\omega$ and/or $f$.
This approach generally requires that the critic class contain witnesses for the
adjoint Bellman residuals generated by candidate weights. For example,
\citet{ueharaEtAl2021FiniteSampleMinimax} require a
completeness condition of the form \(\{\mathsf B_\gamma^\pi \omega-\omega:\omega\in\mathcal W\}\subseteq c\mathcal F\) for some scaling $c>0$, which, combined with $\omegastar\in\mathcal W$, ensures that $\omegastar$ is in the argmin set in \eqref{eq:minimax}. This completeness condition is similar to Bellman completeness, but using the adjoint Bellman operator and with the critic class allowed to be different from the hypothesis class. Nonetheless it can be quite restrictive.

\textsc{FORE} takes a complementary fixed-point view. Instead of minimizing a
worst-case balance residual, it iterates the adjoint Bellman map to reach the fixed point. The
moment identity \eqref{eq:bellman-moment-update} provides an estimable loss for each
KL-projected update.

\section{FORE: Fitted Occupancy-Ratio Evaluation}
\label{sec:kl-fori}

The adjoint Bellman identification suggests estimating the discounted occupancy
ratio by iterating the adjoint Bellman map. In general state-action spaces,
however, the exact image \(\Bpig\omega\) is not available as a pointwise density
ratio. \textsc{FORE} addresses this by replacing each exact Bellman image with
its KL projection onto a tractable normalized ratio class. The construction
rests on two population facts: the exact adjoint Bellman map is contractive in
relative entropy, and the corresponding KL projection can be written using only
initial-state moments and one-step target-policy transitions.

\subsection{KL contraction of the adjoint Bellman operator}
\label{sec:kl-contraction}

We first establish the population stability that makes the fitted iteration
well posed. Let
\[
    \Delta_\nu=\{\omega\geq 0:E_\nu\omega=1\},
    \qquad
    D_\nu(\omega\|\widetilde\omega)
    =
    E_\nu\!\left[
        \omega(X)\log\frac{\omega(X)}{\widetilde\omega(X)}
    \right].
\]
Thus \(D_\nu(\omega\|\widetilde\omega)=D_{\rm KL}(\omega\nu\|\widetilde\omega\nu)\),
where \(D_{\rm KL}\) denotes the KL divergence between measures.
Starting from any \(\omega^{(0)}\in\Delta_\nu\), the exact
adjoint Bellman iteration is
\[
    \omega^{(k+1)}=\Bpig\omega^{(k)}, \qquad k=0,1,\ldots .
\]
The next lemma shows that this exact adjoint Bellman iteration contracts relative entropy to the target ratio.

\begin{lemma}[KL contraction of the adjoint Bellman operator]
\label{lem:adjoint-bellman-kl-contraction}
Suppose Condition~\ref{ass:overlap} holds and \(\gamma\in[0,1)\). Then, for any
\(\omega,\widetilde\omega\in\Delta_\nu\),
\[
    D_\nu(\Bpig\omega\|\Bpig\widetilde\omega)
    \leq
    \gamma D_\nu(\omega\|\widetilde\omega).
\]
Consequently, for any \(\omega\in\Delta_\nu\) with
\(D_\nu(\omega\|\omegastar)<\infty\),
\[
    D_\nu((\Bpig)^k\omega\|\omegastar)
    \leq
    \gamma^k D_\nu(\omega\|\omegastar),
    \qquad k\geq 0 .
\]
\end{lemma}

\begin{proof}[Proof sketch]
By joint convexity of KL and the data processing inequality for the
Markov kernel \(\Ppi\) \citep{coverThomas2006Elements},
\begin{align*}
    D_\nu(\Bpig\omega\|\Bpig\widetilde\omega)
    &=
    D_{\rm KL}\!\left\{
        (1-\gamma)\dinit+\gamma(\omega\nu)\Ppi
        \,\middle\|\,
        (1-\gamma)\dinit+\gamma(\widetilde\omega\nu)\Ppi
    \right\} \\
    &\leq
    (1-\gamma)D_{\rm KL}(\dinit\|\dinit)
    +
    \gamma D_{\rm KL}\{(\omega\nu)\Ppi\|(\widetilde\omega\nu)\Ppi\} \\
    &=
    \gamma D_{\rm KL}\{(\omega\nu)\Ppi\|(\widetilde\omega\nu)\Ppi\} \\
    &\leq
    \gamma D_{\rm KL}(\omega\nu\|\widetilde\omega\nu) \\
    &=
    \gamma D_\nu(\omega\|\widetilde\omega).\qedhere
\end{align*}
\end{proof}

\noindent The exact adjoint Bellman iteration is infeasible because we do not know $\Bpig$ and we cannot approximate it uniformly well over the unrestricted ratio space $\Delta_\nu$.

\subsection{KL-projected Bellman updates}
\label{sec:kl-projected-bellman-updates}

We now restrict to a hypothesis class of weights: given a hypothesis class \(\mathcal H\) of log-ratios, let
\[
    \mathcal W
    =
    \{\omega_h:h\in\mathcal H\},
    \qquad
    \omega_h(x)=\exp\{h(x)-\Lambda_\nu(h)\},
    \qquad
    \Lambda_\nu(h)=\log E_\nu e^{h(X)} .
\]
The log-partition term \(\Lambda_\nu(h)\) normalizes each candidate so that
\(E_\nu\{\omega_h(X)\}=1\) and \(\sW\subset\Delta_\nu\).

For \(u\in\Delta_\nu\), define the KL projection onto \(\sW\) by
\[
    \Pi_{\sW}^{\rm KL}u
    \in
    \argmin_{v\in\sW}D_\nu(u\|v),
\]
whenever the minimizer exists. The population-level KL-projected adjoint
Bellman operator is the composition
\[
    \mathsf T_{\sW}^{\rm KL}\omega
    =
    \Pi_{\sW}^{\rm KL}\Bpig\omega
    \in
    \argmin_{\tilde\omega\in\sW} D_\nu(\Bpig\omega\|\tilde\omega).
\]
Starting from any \(\omega^{(0)}\in\sW\), the exact KL-projected adjoint
Bellman iteration is
\begin{equation}\label{eq:pop-fori}
    \omega^{(k+1)}
    =
    \mathsf T_{\sW}^{\rm KL}\omega^{(k)},
    \qquad
    k=0,\ldots,K-1 .
\end{equation}

Although \(\Bpig\omega\) is generally not available pointwise, the following
lemma shows that its KL projection onto \(\sW\), \(\mathsf T_{\sW}^{\rm KL}\omega\),
can be learned from transition-data moments using supervised learning.

\begin{lemma}[KL projection loss]
\label{lem:bellman-kl-loss}
Suppose Condition~\ref{ass:overlap} holds. For any
\(\omega\in\Delta_\nu\), the KL projection of \(\Bpig\omega\) onto
\(\mathcal W\) is obtained by solving
\[
    \argmin_{h\in\mathcal H}
    D_\nu(\Bpig\omega\|\omega_h)
    =
    \argmin_{h\in\mathcal H}
    \left\{
        \Lambda_\nu(h)
        -
        (1-\gamma)E_{\dinit}\{h(X)\}
        -
        \gamma E_\nu\{\omega(X)h(X^+)\}
    \right\},
\]
where \(X^+\sim\Ppi(\cdot\mid X)\).
\end{lemma}

\noindent The contraction argument in Lemma~\ref{lem:adjoint-bellman-kl-contraction} is not specific to KL: by joint convexity and data processing \citep{coverThomas2006Elements,raginsky2014StrongDataProcessing}, the adjoint Bellman map is a contraction with respect to any \(f\)-divergence. KL is used because, for the normalized exponential ratio class, its projection reduces to the single-level loss in Lemma~\ref{lem:bellman-kl-loss}.

\subsection{Empirical FORE algorithm}
\label{sec:empirical-fori}

We now turn the exact KL-projected adjoint Bellman iteration into an estimator.
Lemma~\ref{lem:bellman-kl-loss} provides the bridge: each KL-projected step
depends only on the initial moment \(E_{\dinit}\{h(X)\}\), the one-step moment
\(E_\nu\{\omega(X)h(X^+)\}\), and the normalizing log-partition function. Each
quantity has a direct sample analogue.

Suppose we observe one-step transitions
\(X_i=(S_i,A_i)\sim\nu\), \(S_i'\sim P(\cdot\mid X_i)\),
\(i=1,\ldots,n\). For each transition, draw
\(A_i^+\sim\pi(\cdot\mid S_i')\) and set \(X_i^+=(S_i',A_i^+)\).\footnote{If \(\mathcal A\) is discrete, one can replace \(h(X_i^+)\) everywhere by \(\sum_{a\in\mathcal A}h(S_i',a)\pi(a\mid S_i')\). Alternatively, one can sample \(A_i^{+,j}\sim\pi(\cdot\mid S_i')\) multiple times, for \(j=1,\dots,m\), and replace \(h(X_i^+)\) everywhere by
\(\frac1m\sum_{j=1}^m h(X_i^{+,j})\), or by any other unbiased estimator of the conditional expectation. This reduces the conditional Monte Carlo variance but does not affect the rates.} Let \(\widehat P_0 h\) denote an estimator of
\(E_{\dinit}\{h(X)\}\), where \(\dinit\) is the target initial state-action distribution,
with \(\widehat P_0 1=1\). For example, if initial samples
\(X_1^0\sim\dinit,\dots,X_m^0\sim\dinit\) are available, one may take
\(\widehat P_0 h = m^{-1}\sum_{i=1}^m h(X_i^0)\). Alternatively, if the state
marginal of \(\nu\) equals the target initial state distribution \(\mu_0\), one
may draw \(A_i^\pi\sim\pi(\cdot\mid S_i)\) and take
\(\widehat P_0 h = n^{-1}\sum_{i=1}^n h(S_i,A_i^\pi)\).

\begin{algorithm}[!htb]
\caption{\textsc{FORE}: Fitted Occupancy-Ratio Evaluation}
\label{alg:fori}
\begin{algorithmic}[1]
\Require Offline transitions \(\{X_i=(S_i,A_i),S_i'\}_{i=1}^n\), initial-moment estimator \(\widehat P_0\),
target policy \(\pi\), discount \(\gamma\), function class \(\mathcal H\),
iteration count \(K\)
\State Draw \(A_i^+\sim\pi(\cdot\mid S_i')\) and set
\(X_i^+=(S_i',A_i^+)\), \(i=1,\ldots,n\)
\State Initialize \(\widehat\omega^{(0)}(x)\equiv 1\)
\For{\(k=0,\ldots,K-1\)}
    \State Compute
    \[
        \widehat h_{k+1}
        \in
        \argmin_{h\in\mathcal H}
        \left\{
            \log\left(
                \frac{1}{n}\sum_{i=1}^n e^{h(X_i)}
            \right)
            -
            (1-\gamma)\widehat P_0 h
            -
            \gamma
            \frac{
                n^{-1}\sum_{i=1}^n
                \widehat\omega^{(k)}(X_i)h(X_i^+)
            }{
                n^{-1}\sum_{i=1}^n
                \widehat\omega^{(k)}(X_i)
            }
        \right\}.
    \]
    \State Set
    \[
        \widehat\omega^{(k+1)}(x)
        =
        \frac{e^{\widehat h_{k+1}(x)}}
        {\frac{1}{n}\sum_{i=1}^n e^{\widehat h_{k+1}(X_i)}}.
    \]
\EndFor
\Ensure Occupancy-ratio estimate \(\widehat\omega^{(K)}\)
\end{algorithmic}
\end{algorithm}

The resulting estimator replaces the population moments in
Lemma~\ref{lem:bellman-kl-loss} by sample averages and normalizes each
exponential update empirically. Algorithm~\ref{alg:fori} states this fitted
recursion, initialized at \(\widehat\omega^{(0)}\equiv1\). The objective function 
\[
    \widehat L(h;\omega)
    =
    \widehat\Lambda_\nu(h)
    -
    (1-\gamma)\widehat P_0 h
    -
    \gamma
    \frac{n^{-1}\sum_{i=1}^n\omega(X_i)h(X_i^+)}
    {n^{-1}\sum_{i=1}^n\omega(X_i)},      \qquad \widehat\Lambda_\nu(h)
    =
    \log\left\{
        \frac{1}{n}\sum_{i=1}^n e^{h(X_i)}
    \right\}
\]
is convex in \(h\). Thus, for a
linear hypothesis class, the objective remains convex in the linear
coefficients. For nonlinear classes, one can use batched stochastic gradients by
writing the empirical log partition in variational form,
\[
    \widehat\Lambda_\nu(h)
    =
    \inf_{a\in\mathbb R}
    \left\{
        a-1+\frac1n\sum_{i=1}^n e^{h(X_i)-a}
    \right\}.
\]
For a parametrized class \(\{h_\theta:\theta\in\Theta\}\), a transition batch
\(\{(X_i,X_i^+)\}_{i=1}^b\), and an initial-state batch
\(\{X_i^0\}_{i=1}^b\), the corresponding stochastic gradient in \((\theta,a)\)
is
\[
\begin{pmatrix}
\displaystyle
\frac1b\sum_{i=1}^b e^{h_\theta(X_i)-a}\nabla_\theta h_\theta(X_i)
-(1-\gamma)\frac1b\sum_{i=1}^b\nabla_\theta h_\theta(X_i^0)
-\gamma
\frac{\sum_{i=1}^b
    \widehat\omega^{(k)}(X_i)\nabla_\theta h_\theta(X_i^+)}
{\sum_{i=1}^b \widehat\omega^{(k)}(X_i)}
\\[1ex]
\displaystyle
1-\frac1b\sum_{i=1}^b e^{h_\theta(X_i)-a}
\end{pmatrix}.
\]
One may also update \(\widehat\omega^{(k)}\), after one or a few batched gradient
steps, to the current empirically normalized ratio
\(\widehat\omega_{h_{\widehat\theta}}(x):=
e^{h_{\widehat\theta}(x)}/\{n^{-1}\sum_{i=1}^n
e^{h_{\widehat\theta}(X_i)}\}\), rather than waiting for convergence of each
iteration, as in practical neural fitted value iteration
\citep{mnih2013playing}. Regularization, such as a Tikhonov penalty, can be
added to the same objective.

\paragraph{Poisson-loss implementation.}
\textsc{FORE} can equivalently be formulated using a generalized-KL objective
for finite measures, which has the form of a Poisson loss. When the log-ratio
class is closed under additive constants, replacing
\(\widehat\Lambda_\nu(h)
=\log\{n^{-1}\sum_{i=1}^n e^{h(X_i)}\}\) with
\(n^{-1}\sum_{i=1}^n e^{h(X_i)}\) yields the same empirically normalized fitted
ratio. This formulation corresponds to
Algorithm~\ref{alg:coverage-truncated-fore} in
Section~\ref{sec:coverage-clipped-fitted} with \(\mathcal C=\{1\}\) and
clipping disabled.

\section{Ratio-approximation Guarantees for FORE}
\label{sec:main-guarantees}

This section establishes guarantees for how well FORE approximates $\omegastar$ in KL-divergence.

\subsection{KL-projected fixed-point recursion}
\label{sec:kl-fori-theory}

We begin by analyzing the idealized population version of \textsc{FORE} where we iterate $\mathsf T_{\sW}^{\rm KL}$, as given by \eqref{eq:pop-fori}. Then we compose this analysis with the sample-based approximation errors. We leverage the following conditions in addition to  \ref{ass:overlap}:
 
\begin{enumerate}[label=\textbf{(A\arabic*)}, ref=A\arabic*, resume=forecond]
\item \label{ass:kl-class}
\textit{Closed convex log-ratio class.} The class \(\sH\) is
convex, closed, and totally bounded as a subset of \(L^2(\nu)\).

\item \label{ass:kl-population-integrability}
\label{ass:kl-target-integrability}
\label{ass:kl-bellman-entropy}
\textit{Population positivity and finite entropy.} The target ratio satisfies
\(\omegastar>0\) \(\nu\)-a.e. With
\(\cpi=\mathrm d\nup/\mathrm d\nu\),
\[
    E_\nu\{\omegazero\log\omegazero\}
    +
    E_\nu\{\cpi\log\cpi\}<\infty .
\]

\item \label{ass:kl-bounded}
\textit{Bounded centered log class.} There exist a measurable set
\(\mathcal X_R\) with \(\nu(\mathcal X_R)=1\) and a finite constant \(R\) such that
\[
    \sup_{h\in\sH}
    \left|h(x)-E_\nu\{h(X)\}\right|
    \le R,
    \qquad x\in\mathcal X_R.
\]
\end{enumerate}
Condition~\ref{ass:kl-class} is mild for finite-dimensional linear classes:
if \(\theta\mapsto h_\theta\) is continuous into \(L^2(\nu)\) and the parameter
space is compact and convex, then \(\sH\) is convex, closed, and totally
bounded. Condition~\ref{ass:kl-population-integrability} requires
positivity for the target occupancy ratio and ensures KL divergences and entropies are well-defined and finite. Positivity is relative to \(\nu\); when the target
occupancy support is known, one may restrict \(\nu\) to that support.   Finally, Condition~\ref{ass:kl-bounded} uniformly bounds
the normalized ratios, controls empirical-process envelopes, and ensures local
quadratic curvature of the excess KL loss.

Define the KL approximation error by
\[
    \varepsilon_{\rm KL}
    :=
    \inf_{v\in\sW}
    D_\nu(v\|\omegastar).
\]
In particular, if we just have realizability \(\omegastar\in\sW\), then
\(\varepsilon_{\rm KL}=0\).

\begin{theorem}[KL-projected fixed-point recursion]
\label{thm:kl-fori-realizable}
Let \(\gamma\in[0,1)\). Suppose Conditions~\ref{ass:overlap},
\ref{ass:kl-class} and~\ref{ass:kl-population-integrability} hold, and suppose
Condition~\ref{ass:kl-bounded} holds with constant \(R\). Then there is a
finite constant \(C_{\rm app}\), depending only on \(R\), such that, for every
\(\omega\in\sW\),
\[
    D_\nu(\mathsf T_{\sW}^{\rm KL}\omega\|\omegastar)
    \le
    \gamma D_\nu(\omega\|\omegastar)
    +
    C_{\rm app}\varepsilon_{\rm KL}.
\]
Consequently, for any \(\omega^{(0)}\in\sW\), the iterates defined by \( \omega^{(k+1)}
    =
    \mathsf T_{\sW}^{\rm KL}\omega^{(k)},\) \(k=0,\ldots,K-1,\)
satisfy
\[
    D_\nu(\omega^{(K)}\|\omegastar)
    \le
    \gamma^K D_\nu(\omega^{(0)}\|\omegastar)
    +
    C_{\rm app}\frac{1-\gamma^K}{1-\gamma}\varepsilon_{\rm KL}.
\]
\end{theorem}
\begin{proof}[Proof sketch in the realizable case]
Assume \(\omegastar\in\sW\). Since \(\sW\) is a normalized exponential family with convex natural-parameter
space, the KL projection satisfies the Pythagorean inequality for information
projections \citep{csiszar1975divergence,banerjee2005clustering}. Thus
\begin{align*}
    D_\nu(\mathsf T_{\sW}^{\rm KL}\omega\|\omegastar)
    &\le
    D_\nu(\Bpig\omega\|\omegastar)
    -
    D_\nu(\Bpig\omega\|\mathsf T_{\sW}^{\rm KL}\omega) \\
    &\le
    D_\nu(\Bpig\omega\|\omegastar) \\
    &=
    D_\nu(\Bpig\omega\|\Bpig\omegastar) \\
    &\le
    \gamma D_\nu(\omega\|\omegastar),
\end{align*}
where the last inequality is Lemma~\ref{lem:adjoint-bellman-kl-contraction}.
\end{proof}
 
Theorem~\ref{thm:kl-fori-realizable} gives a KL error recursion for the
population KL-projected fixed-point iteration. Each step decomposes into the
\(\gamma\)-contraction of the exact adjoint Bellman map toward the target
occupancy ratio and a KL projection error. Importantly, this projection error is
controlled by the best KL approximation error for the fixed point
\(\omegastar\), not by an approximation error for Bellman images of arbitrary
candidate ratios. In particular, if \(\omegastar\in\sW\), then the projection
error vanishes and the bound reduces to
\(D_\nu(\omega^{(K)}\|\omegastar)
\le
\gamma^K D_\nu(\omega^{(0)}\|\omegastar)\).

This is the main distinction from standard FQE/FVI analyses. In those settings,
the standard value Bellman operator is contractive in the stationary target-policy
norm, but the population algorithm composes this operator with a projection map,
typically an \(L^2\) projection under the offline data distribution
\citep{munosSzepesvari2008FVI,vanDerLaanKallus2025StationaryWeightedFQE}.
Because of this norm mismatch, the projection can prevent the projected
recursion from inheriting the stability of the Bellman fixed point. Stability
therefore typically requires Bellman completeness, approximate Bellman
completeness, or a small inherent Bellman error; these conditions ensure that
Bellman images of functions in the approximation class remain close to the class
\citep{munosSzepesvari2008FVI,chen2019information,foster2021offline}. By
contrast, \textsc{FORE} composes the adjoint Bellman operator with a KL
projection map. The adjoint Bellman step contracts relative entropy toward
\(\omegastar\), and the projection step uses the same KL loss. The projection
error is controlled by how well the ratio class approximates \(\omegastar\),
rather than by a global inherent adjoint Bellman error such as
\(
    \sup_{\omega\in\mathcal W}
    \inf_{\widetilde\omega\in\mathcal W}
    D_\nu\!\left(\Bpig\omega\,\middle\|\,\widetilde\omega\right).
\)
Thus, the analysis does not require the ratio class to be closed under adjoint
Bellman updates.

\subsection{Finite-sample error bounds}
\label{sec:fitted-kl-fori}

We next incorporate sampling error into Algorithm~\ref{alg:fori}. The theorem
below analyzes the exact-ERM fitted recursion with the initial moment estimated
from an independent sample.

For the finite-sample statement, let \(X_1^0,\ldots,X_n^0\) be i.i.d. samples
from \(\dinit\), independent of a transition sample with
\(X_i\stackrel{\mathrm{i.i.d.}}{\sim}\nu\) and
\(S_i'\mid X_i\sim P(\cdot\mid X_i)\).
Apply Algorithm~\ref{alg:fori} with
\(\widehat P_0h=n^{-1}\sum_{i=1}^n h(X_i^0)\) and exact ERM at each fitted
step. This produces iterates
\(\widehat\omega^{(0)},\ldots,\widehat\omega^{(K)}\). Each output is normalized
under the empirical offline data distribution. Since it need not integrate to one under
\(\nu\), we measure its error using the generalized KL divergence
\[
    D_\nu^{\rm gen}(f\|g)
    =
    E_\nu\left[
        f(X)\log\frac{f(X)}{g(X)}-f(X)+g(X)
    \right].
\]
This reduces to \(D_\nu(f\|g)\) when both arguments integrate to one under
\(\nu\), that is, $f,g\in\Delta_\nu$.

\begin{enumerate}[label=\textbf{(A\arabic*)}, ref=A\arabic*, resume=forecond]
\item \label{ass:fitted-kl-coverage-bounded}
\textit{Subexponential initial coverage and one-step smoothing.} There are
constants \(0<K_0,K_+<\infty\) such that
\[
    \|\omegazero\|_{\psi_1} \le K_0,
    \qquad
    \sup_{\omega\in\sW}
    \left\|
        \frac{\mathrm d\{(\omega\nu)\Ppi\}}{\mathrm d\nu}
    \right\|_{\psi_1}
    \le K_+,
\]
where
\(\|Z\|_{\psi_1}:=\inf\{s>0:E_\nu\exp(|Z|/s)\le2\}\).

\item \label{ass:fitted-kl-lower-tail}
\textit{Lower-tail margin.} There exist constants \(0<A<\infty\) and
\(\alpha>0\) such that, for every \(t\in(0,1]\),
\[
    \nu\{x:0<\omegastar(x)\le t\}\le A t^\alpha .
\]
\end{enumerate}
Condition~\ref{ass:fitted-kl-coverage-bounded} ensures that fitted-loss control under $\nu$ extends to the initial and successor distributions. It holds, for example, when the initial density ratio and the transition density relative to \(\nu\) are uniformly bounded, and more generally allows unbounded induced densities with uniformly exponential tails; see Lemma~\ref{lem:bounded-kernel-smoothing}. Condition~\ref{ass:fitted-kl-lower-tail} is a mild soft-margin condition that
allows \(\omegastar\) to approach zero, provided the \(\nu\)-mass of near-zero
regions decays polynomially. It holds automatically for any $\alpha > 0$ under the hard-margin condition
\(\omegastar(x)\ge m_\star>0\) almost surely.

The statistical error is governed by the local complexity of the log-ratio
class and by the multiplier class induced by the Bellman moment terms. For a
class \(\mathcal G\) of square-integrable functions under a distribution \(P\), define
the local Rademacher complexity
\citep{bartlettEtAl2005LocalRademacher}
\[
    \mathcal R_n(\mathcal G,r;P)
    =
    \E_{Z,\sigma}
    \sup_{\substack{g\in\mathcal G:\\
        \|g\|_{L^2(P)}\le r}}
    \left|
        \frac{1}{n}\sum_{i=1}^n
        \sigma_i g(Z_i)
    \right|,
\]
where \(Z_1,\ldots,Z_n \sim P\) are independent draws, and
\(\sigma_1,\ldots,\sigma_n\) are independent Rademacher variables. Let
\[
    \sH^\circ=\{h-E_\nu\{h(X)\}:h\in\sH\},
    \qquad
    \mathcal H_\Delta=\{h_1-h_2:h_1,h_2\in\sH^\circ\}.
\]
Let \(Q_{\nu,\Delta}\) denote the distribution of \((X,X)\) with \(X\sim\nu\), and let
\(Q_{\nu,\pi}\) denote the distribution of \((X,X^+)\) with \(X\sim\nu\) and
\(X^+\mid X\sim\Ppi(\cdot\mid X)\). Define the multiplier class
\[
    \mathcal G_\times
    =
    \left\{
    (x,x^+)\mapsto
    f(x)h_\Delta(x^+):
    f\in\sW,\ h_\Delta\in\mathcal H_\Delta
    \right\}.
\]
Define $ \mathfrak C_n(r)
    =
    \max\left\{
        \mathcal R_n(\mathcal H_\Delta,r;\nu),
        \mathcal R_n(\mathcal H_\Delta,r;\dinit),
        \mathcal R_n(\mathcal G_\times,r;Q_{\nu,\Delta}),
        \mathcal R_n(\mathcal G_\times,r;Q_{\nu,\pi})
    \right\}.$
Define the critical radius \citep{wainwright2019HighDimensionalStatistics}
\begin{equation}
\label{eq:rad-critical-radius}
    \mathfrak r_{n,\rm fit}
    =
    n^{-1/2}
    \vee
    \inf\left\{r>0:
    \mathfrak C_n(r)
    \le
    r^2
    \right\}.
\end{equation}
\begin{theorem}[Fitted \textsc{FORE} with empirical normalization]
\label{thm:fitted-kl-fori}
Let \(\gamma\in[0,1)\). Assume Conditions~\ref{ass:overlap},
\ref{ass:kl-class}, \ref{ass:kl-population-integrability},
\ref{ass:kl-bounded},
\ref{ass:fitted-kl-coverage-bounded}, and~\ref{ass:fitted-kl-lower-tail},
and assume \(0\in\sH\). Let
\(\{\widehat\omega^{(k)}\}_{k=0}^K\) be the fitted \textsc{FORE} iterates
defined by Algorithm~\ref{alg:fori}, initialized at
\(\widehat\omega^{(0)}\equiv1\).
Then, for every \(0<\delta<1\), with probability at least \(1-\delta\),
\[
\begin{aligned}
    D_\nu^{\rm gen}(\widehat\omega^{(K)}\|\omegastar)
    &\le
    C_{\rm env}\left(\frac{1+\gamma}{2}\right)^K
    D_\nu^{\rm gen}(\widehat\omega^{(0)}\|\omegastar)
    +
    \frac{C_{\rm env}}{1-\gamma}\varepsilon_{\rm KL}
    \\
    &\quad+
    \frac{C_{\rm env}}{(1-\gamma)^2}
    \log^2(en)
    \left\{
        \mathfrak r_{n,\rm fit}^2
        +
        \frac{\log(1/\delta)}{n}
    \right\}.
\end{aligned}
\]
Here, for universal finite exponents \(p,q\) and a finite
\(C_0=C_0(A,\alpha)\), \(C_{\rm env}<\infty\) may be chosen so that $C_{\rm env}
    \le
    C_0(1+K_0+K_+)^q(1+e^{2R})^p.$
\end{theorem}
 
\paragraph{Bound terms.}
This bound decomposes the error into three terms: the fixed-point error of $K$ iterations of the idealized population iteration, the best KL approximation error for \(\omegastar\), and the statistical error of estimating the idealized population iteration using data. The first term is negligible even for moderate $K$. For example, any $K\geq \log n/\log (2/(1+\gamma))$ ensures this term is $O(1/n)$. The second term crucially only depends on how well our hypothesis class $\mathcal W$ approximates $\omegastar$, \textit{not} how well we approximate every iteration. This is exactly how our bounds are distinguished from (approximate) Bellman completeness. The third term is a standard statistical error for empirical risk minimization and we instantiate bounds on it for specific function classes below. Notice that unlike some analyses of fitted iterations \citep[e.g.][]{munosSzepesvari2008FVI,chang2022learning} we avoid splitting the data into $K$ samples and needing to balance the number of iterations and the amount of data available for statistical estimation. Following \citet{vanDerLaanKallus2025StationaryWeightedFQE,hu2025fast}, we control this by using a uniform statistical error, which is the reason for introducing the function class $G_\times$.

\paragraph{Horizon dependence.} In long-horizon value estimation, the powers of \((1-\gamma)^{-1}\) determine
how the generalized-KL error bound scales with the effective horizon. The
deterministic approximation term \(\varepsilon_{\rm KL}/(1-\gamma)\) retains
the population horizon factor from Theorem~\ref{thm:kl-fori-realizable},
whereas the finite-sample error
\[
    \frac{\log^2(en)}{(1-\gamma)^2}
    \left\{
        \mathfrak r_{n,\rm fit}^2
        +
        \frac{\log(1/\delta)}{n}
    \right\}
\]
pays one additional factor of \((1-\gamma)^{-1}\) due to the propagation of
statistical error across iterations. As shown in the next
section, the policy-value bounds depend on the
square root of the generalized-KL ratio error. Thus, the statistical term has
the familiar \((1-\gamma)^{-1}\) value-level horizon dependence of FQE, while
the deterministic approximation term for \textsc{FORE} has the more favorable
\(\sqrt{\varepsilon_{\rm KL}/(1-\gamma)}\) value-scale contribution. This
favorable dependence for the approximation term contrasts with standard FQE
bounds under approximate Bellman completeness, where inherent Bellman error is
propagated through the Bellman recursion and appears with \((1-\gamma)^{-1}\)
dependence at the value-error scale \citep{munosSzepesvari2008FVI}.

\paragraph{Bounds on statistical error for specific function classes.}

The key statistical term in Theorem~\ref{thm:fitted-kl-fori} is the critical radius $\mathfrak r_{n,\rm fit}$. We next discuss bounds for specific choices of our hypothesis class $\mathcal H$.

\begin{itemize}

\item \textit{Linear function classes}.
If $\mathcal H\subseteq\{x\mapsto\beta_0+\beta^\top \phi(x)\}$ where $\phi:\mathcal X\to\mathbb R^d$, then Corollary~\ref{cor:kl-finite-dimensional} in the appendix establishes that
\[
    \mathfrak r_{n,\rm fit}^2
    \lesssim
    \frac{d\log(n)}{n}.
\]

\item \textit{Nonparametric function classes}.
If the hypothesis class \(\mathcal H\) has a finite uniform
entropy integral, then Corollary~\ref{cor:kl-entropy-integral} bounds
\(\mathfrak r_{n,\rm fit}\) by the entropy-based critical radius associated
with \(\mathcal H\)
\citep{vanDerVaartWellner2011LocalMaximal,van2026researcher}. For bounded
H\"older log-ratio balls and Sobolev balls in dimension \(d\) and smoothness
\(s>d/2\), this gives \citep{nicklPotscher2007BracketingMetricEntropy}
\[\mathfrak r_{n,\rm fit}^2\lesssim n^{-2s/(2s+d)}.\]

\end{itemize}


\section{Applications to Policy Evaluation}
\label{sec:applications-policy-evaluation}

The preceding sections focus on estimating the discounted occupancy ratio
\(\omegastar=\dd d_{\pi,\gamma}/\dd\nu\). This ratio can be used to evaluate
bounded target-occupancy functionals. For any bounded measurable \(g\),
\[
    \Psi_\pi(g)
    :=
    E_{d_{\pi,\gamma}}\{g(X)\}
    =
    E_\nu\{\omegastar(X)g(X)\}.
\]
Thus, a single ratio fit can evaluate rewards, costs, feature moments, and
visitation probabilities under the target discounted occupancy.

Let \(\omega_{\rm fit}:=\widehat\omega^{(K_\omega)}\) denote the fitted
\textsc{FORE} estimate from Section~\ref{sec:fitted-kl-fori}, and define
\[
\begin{aligned}
    \mathcal E_{\rm FORE}^2
    :=
    \left(\frac{1+\gamma}{2}\right)^{K_\omega}
    D_\nu^{\rm gen}(\widehat\omega^{(0)}\|\omegastar)
    +
    \frac{\varepsilon_{\rm KL}}{1-\gamma} +
    \frac{\log^2(en)}{(1-\gamma)^2}
    \left\{
        \mathfrak r_{n,\rm fit}^2
        +
        \frac{\log(1/\delta)}{n}
    \right\},
\end{aligned}
\]
where \(\mathcal E_{\rm FORE}
    :=
    \{\mathcal E_{\rm FORE}^2\}^{1/2}\).

\begin{corollary}[Bounded target-functional bound]
\label{cor:fori-target-functional}
Suppose the conditions of Theorem~\ref{thm:fitted-kl-fori} hold.
Then, with probability at least \(1-\delta\), there is a finite constant
\(C_{\rm eval}\), depending only on the constants in
Conditions~\ref{ass:kl-bounded},
\ref{ass:fitted-kl-coverage-bounded}, and
\ref{ass:fitted-kl-lower-tail}, such that
\[
    \sup_{\|g\|_\infty\le 1}
    \left|
        E_\nu\{\omega_{\rm fit}(X)g(X)\}
        -
        \Psi_\pi(g)
    \right|
    \le
    C_{\rm eval}\mathcal E_{\rm FORE}.
\]
\end{corollary}
\noindent Thus the estimated functional
\(g\mapsto E_\nu\{\omega_{\rm fit}(X)g(X)\}\) converges to the target occupancy
functional \(g\mapsto E_{d_{\pi,\gamma}}\{g(X)\}\) uniformly over bounded
test functions, at rate \(\mathcal E_{\rm FORE}\).

The remainder of this section specializes
Corollary~\ref{cor:fori-target-functional} to policy-value estimation. We
first combine \textsc{FORE} with a fitted \(Q\)-function to obtain a doubly
robust estimator, and then use the \textsc{FORE} ratio as the projection weight
in fitted \(Q\)-evaluation.

 \subsection{Doubly robust policy-value estimation}
\label{sec:fori-drl}

Let \(Y\) be a reward observed with \(X\), and define \(r(x) := E(Y\mid X=x)\).
For policy-value estimation, assume \(\|r\|_\infty\le R_{\max}<\infty\). Taking
\(g=r\), the normalized discounted value is
\[
    V_\pi(r)
    =
    E_{d_{\pi,\gamma}}\{r(X)\}
    =
    E_\nu\{\omegastar(X)r(X)\}
    =
    E\{\omegastar(X)Y\},
\]
where the last expectation is under the offline reward distribution. The plug-in
estimator based on \(\omega_{\rm fit}\) uses the sample analogue of this identity.
We can also combine an estimated ratio with an estimated \(Q\)-function through
the standard doubly robust Bellman-residual correction
\citep{jiangLi2016DoublyRobust,kallusUehara2020DoubleRL,
kallusUehara2022BreakingHorizonDRL,vanDerLaanEtAl2025AutomaticDRL}. 

Define the policy-evaluation Bellman operator by
\[
    \mathcal T^\pi Q
    =
    r+\gamma \Ppi Q,
    \qquad
    (\Ppi Q)(x)=E\{Q(X^+)\mid X=x\},
\]
where \(X^+=(S^+,A^+)\) is generated by the transition distribution and target policy
\(\pi\). The target \(Q\)-function is the fixed point
\(Q^\pi=\mathcal T^\pi Q^\pi\), and
\(V_\pi(r)=(1-\gamma)E_{\dinit}\{Q^\pi(X)\}\). For any weight \(\omega\) and
function \(Q\), define the doubly robust functional
\[
    \Psi_{\rm DR}(\omega,Q)
    =
    (1-\gamma)E_{\dinit}\{Q(X)\}
    +
    E_\nu\!\left[
        \omega(X)\{\mathcal T^\pi Q(X)-Q(X)\}
    \right].
\]
The one-sided estimators are recovered by setting \(Q=0\), which gives
\(\Psi_{\rm DR}(\omega,0)=E_\nu\{\omega(X)r(X)\}\), or by setting
\(\omega=0\), which gives
\(\Psi_{\rm DR}(0,Q)=(1-\gamma)E_{\dinit}\{Q(X)\}\). In what follows, we denote
\[
    \|g\|_\star^2
    :=
    E_{d_{\pi,\gamma}}\{g(X)^2\}.
\]

\begin{enumerate}[label=\textbf{(B\arabic*)}, ref=B\arabic*, series=appcond]
\item \label{ass:value-kl-lower}
\textit{Hard margin.} There exists
\(m_\star>0\) such that \(\omegastar(x)\ge m_\star\) for \(\nu\)-almost every
\(x\).
\end{enumerate}
Condition~\ref{ass:value-kl-lower} yields sharper dependence on the
\(Q\)-function estimation error by bounding the target-weighted chi-square
ratio error in terms of the generalized-KL ratio error. It can be relaxed to
the soft-margin condition in Condition~\ref{ass:fitted-kl-lower-tail}, at the
cost of less favorable dependence on the estimation errors.

\begin{theorem}[Doubly robust value bound]
\label{thm:fori-drl-main}
Suppose the conditions of Corollary~\ref{cor:fori-target-functional} and
Condition~\ref{ass:value-kl-lower} hold, and let
\(Q\in L^2(d_{\pi,\gamma})\). Then, with probability at least \(1-\delta\),
there is a finite constant \(C_\chi\), depending only on the constants in
Conditions~\ref{ass:kl-bounded}, \ref{ass:fitted-kl-coverage-bounded},
and~\ref{ass:value-kl-lower}, such that
\[
    \left|\Psi_{\rm DR}(\omega_{\rm fit},Q)-V_\pi(r)\right|
    \le
    C_\chi\mathcal E_{\rm FORE}
    \|\mathcal T^\pi Q-Q\|_\star .
\]
\end{theorem}
The identity yields double robustness: the value error vanishes if either
\(\omega=\omega^\pi\) or \(Q=Q^\pi\), and otherwise it is bounded by the product
of the ratio error and the Bellman residual. By
Lemma~\ref{lem:discounted-occupancy-q-contraction},
\(\|\mathcal T^\pi Q-Q\|_\star
\le (1+\sqrt{\gamma})\|Q-Q^\pi\|_\star\). Hence the doubly robust error is
controlled, up to constants, by
\(\mathcal E_{\rm FORE}\|Q-Q^\pi\|_\star\).

In practice, the population averages in \(\Psi_{\rm DR}\) are replaced by
sample averages. A plug-in estimator evaluates the fitted nuisance functions in
the empirical functional:
\[
    \widehat \Psi_{\rm DR}
    =
    (1-\gamma)\frac{1}{n_0}\sum_{i=1}^{n_0}\widehat Q(X_i^0)
    +
    \frac{1}{n}\sum_{i=1}^n
    \widehat\omega(X_i)
    \left\{
        Y_i+\gamma (\pi\widehat Q)(S_i')-\widehat Q(X_i)
    \right\},
\]
where \(X_i^0\sim\dinit\), \(X_i=(S_i,A_i)\), \(S_i'\) is the observed next
state, and \(Y_i\) is the observed reward. Here
\((\pi Q)(s)=\int Q(s,a)\pi(\dd a\mid s)\) denotes the target-policy average.

A natural approach is to estimate \(Q\) by fitted \(Q\)-evaluation. With
nonlinear function approximation, however, standard convergence guarantees for
FQE typically require Bellman completeness or related projected-operator
stability conditions \citep{munosSzepesvari2008FVI}. The next subsection uses
the \textsc{FORE} ratio to stabilize FQE and obtain guarantees for the fitted
\(Q\)-function.

\subsection{Occupancy-weighted FQE without Bellman completeness}
\label{sec:fori-weighted-fqe}

Occupancy-weighted FQE first estimates the discounted occupancy ratio using FORE and then uses
the fitted ratio \(\omega_{\rm fit}\) as a fixed projection weight in fitted
\(Q\)-evaluation. The resulting Bellman regressions are carried out in an
estimated target-occupancy norm, rather than a projection norm under the offline data distribution. This
gives a discounted analogue of stationary-weighted FQE and can restore
contraction of the projected Bellman equation without Bellman completeness
\citep{vanDerLaanKallus2025StationaryWeightedFQE,
vanDerLaanKallus2025SoftFQI,pattersonEtAl2022GeneralizedPBE}.

Let \(\sQ\) be a closed convex subset of \(L^2(d_{\pi,\gamma})\).
For a nonnegative weight \(\omega\), define
\[
    \mathcal T_{\sQ,\omega}Q
    :=
    \Pi_{\sQ,\omega}\mathcal T^\pi Q,
    \qquad
    \Pi_{\sQ,\omega}g
    \in
    \argmin_{q\in\sQ}
    E_\nu\{\omega(X)(g(X)-q(X))^2\}.
\]
Let
\(\mathcal T_{\sQ,\star}:=\mathcal T_{\sQ,\omegastar}\) denote the oracle
projected Bellman operator based on the discounted occupancy ratio
\(\omegastar\). Appendix~\ref{app:discounted-occupancy-contraction} shows that
\(\mathcal T_{\sQ,\star}\) is a \(\sqrt{\gamma}\)-contraction in
\(\|\cdot\|_\star\). Let \(Q_{\sQ,\star}\) denote its unique fixed point.

FORE-weighted FQE replaces the oracle weight \(\omegastar\) by the fitted ratio
\(\omega_{\rm fit}\) and iterates
\[
    Q^{(j+1)}
    =
    \mathcal T_{\sQ,\omega_{\rm fit}}Q^{(j)}
    =
    \Pi_{\sQ,\omega_{\rm fit}}\mathcal T^\pi Q^{(j)},
    \qquad
    j=0,\ldots,K_Q-1.
\]
Following \citet{vanDerLaanKallus2025StationaryWeightedFQE}, the effect of this
replacement is controlled by the Bellman-projection error
\[
    \varepsilon_{\rm Bell}
    :=
    \sup_{Q\in\sQ}
    \sup_{\substack{h\in\sQ-\sQ:\ \|h\|_\star\le 1}}
    \left\|
        \{\mathcal T^\pi Q-\mathcal T_{\sQ,\star}Q\}h
    \right\|_\star .
\]
This error is zero under Bellman completeness: if
\(\mathcal T^\pi Q\in\sQ\), then
\(\mathcal T_{\sQ,\star}Q=\mathcal T^\pi Q\) for every \(Q\in\sQ\).

\begin{enumerate}[label=\textbf{(B\arabic*)}, ref=B\arabic*, resume=appcond]
\item \label{ass:fori-target-upper}
\textit{Bounded target occupancy ratio.} There exists \(M_\star<\infty\) such
that \(\|\omegastar\|_\infty\le M_\star\).
\end{enumerate}

\begin{theorem}[FORE-weighted projected FQE]
\label{thm:fori-weighted-fqe-main}
Let
\(Q^{(0)}\in\sQ\). Suppose the conditions of
Corollary~\ref{cor:fori-target-functional} and
Conditions~\ref{ass:value-kl-lower} and~\ref{ass:fori-target-upper} hold.
Then, with probability at least \(1-\delta\), there is a finite constant
\(C_\chi\), depending only on the constants in
Conditions~\ref{ass:kl-bounded}, \ref{ass:fitted-kl-coverage-bounded},
\ref{ass:value-kl-lower}, and~\ref{ass:fori-target-upper}, such that
\[
\begin{aligned}
    \|Q^{(K_Q)}-Q^\pi\|_\star
    \le\;&
    \gamma^{K_Q/2}\|Q^{(0)}-Q_{\sQ,\star}\|_\star  +
    \frac{1-\gamma^{K_Q/2}}{1-\sqrt{\gamma}}
    C_\chi\varepsilon_{\rm Bell}\mathcal E_{\rm FORE}  +
    \frac{1}{1-\sqrt{\gamma}}
    \inf_{q\in\sQ}\|q-Q^\pi\|_\star .
\end{aligned}
\]
\end{theorem}

The bound separates three sources of error:
finite FQE iteration, use of the fitted ratio \(\omega_{\rm fit}\) rather than
the oracle occupancy ratio, and approximation bias of the oracle projected
Bellman fixed point. The plug-in weight error is controlled by the product
\(\varepsilon_{\rm Bell}\mathcal E_{\rm FORE}\), where
\(\mathcal E_{\rm FORE}\) is the fitted \textsc{FORE} error and
\(\varepsilon_{\rm Bell}\) is the Bellman-projection error. Under Bellman
completeness, this term vanishes; otherwise, the effect of ratio estimation is
attenuated by the size of \(\varepsilon_{\rm Bell}\). The oracle approximation bias is controlled by
\((1-\sqrt{\gamma})^{-1}\inf_{q\in\sQ}\|q-Q^\pi\|_\star\). For linear or affine value classes,
Lemma~\ref{lem:discounted-fqe-linear-projection-bias} improves this to $(1-\gamma)^{-1/2}\inf_{q\in\sQ}\|q-Q^\pi\|_\star$ \citep{tsitsiklisVanRoy1997Analysis}, matching the value-level horizon
dependence of \textsc{FORE} in Theorem~\ref{thm:fitted-kl-fori}. A fully empirical implementation incurs
an additional statistical error term controlled by the complexity of the
optimization class \(\sQ\); see
\citet{vanDerLaanKallus2025StationaryWeightedFQE} for details. Combining this
FQE bound with Theorem~\ref{thm:fori-drl-main} yields the following value bound.

\begin{corollary}[FORE-weighted doubly robust value bound]
\label{cor:fori-weighted-dr-main}
Under the conditions of Theorems~\ref{thm:fori-drl-main}
and~\ref{thm:fori-weighted-fqe-main}, suppose \(K_Q\) is chosen so that the
finite-iteration term in Theorem~\ref{thm:fori-weighted-fqe-main} is negligible.
Then, with probability at least \(1-\delta\),
\[
\begin{aligned}
    \left|\Psi_{\rm DR}(\omega_{\rm fit},Q^{(K_Q)})-V_\pi(r)\right|
    \le\;&
    C_{\rm DR}\mathcal E_{\rm FORE}
    \Bigg[
        \frac{\varepsilon_{\rm Bell}\mathcal E_{\rm FORE}}
        {1-\sqrt{\gamma}}
        +
        \frac{1}{1-\sqrt{\gamma}}
        \inf_{q\in\sQ}\|q-Q^\pi\|_\star
    \Bigg],
\end{aligned}
\]
with a finite constant \(C_{\rm DR}\) depending only on the constants in
Conditions~\ref{ass:kl-bounded}, \ref{ass:fitted-kl-coverage-bounded},
\ref{ass:value-kl-lower}, and~\ref{ass:fori-target-upper}.
\end{corollary}

Thus, once the fitted-\(Q\) iteration term is negligible, the value error is the
sum of a second-order ratio-estimation term,
\(\varepsilon_{\rm Bell}\mathcal E_{\rm FORE}^2/(1-\sqrt{\gamma})\), and the
product of the FORE ratio error with the value-class approximation error, $ \mathcal E_{\rm FORE} \inf_{q\in\sQ}\|q-Q^\pi\|_\star /(1-\sqrt{\gamma})$. If
\(Q^\pi\in\sQ\), the approximation term vanishes; if the class is Bellman
complete, then \(\varepsilon_{\rm Bell}=0\).

\section{Coverage-stopped \textsc{FORE} under insufficient data coverage}
\label{sec:coverage-truncated-fore}

The preceding guarantees target the full discounted occupancy ratio
\(\omegastar\) and therefore require the target discounted occupancy to be
absolutely continuous with respect to the offline data distribution \(\nu\).
When this condition fails, the full ratio is not identified from the offline
data. Inspired by the recursively clipped occupancy construction of
\citet{huangChenJiang2023DensityFeatures} for finite-horizon MDPs, we define
an infinite-horizon coverage-stopped discounted occupancy. At each adjoint
Bellman update, we discard the component singular with respect to \(\nu\).
The resulting fixed point \(\omega_{\rm cov}\nu\) is a subprobability
discounted occupancy measure with density \(\omega_{\rm cov}\) relative to
\(\nu\) and total mass equal to the coverage-stopped occupancy mass.

A distinctive feature of our fitted construction is that it avoids estimating
the adjoint Bellman density pointwise. Instead, a learned coverage classifier
identifies the portion of each Bellman update represented under the offline
distribution, while finite clipping regularizes estimation in weakly covered
regions. As with standard \textsc{FORE}, we do not require the ratio class to
satisfy adjoint Bellman completeness. Instead, the classifier class must
uniformly approximate the oracle retention rules encountered along the
recursion.

\subsection{Coverage-stopped occupancy and conservative policy-value bounds}
\label{sec:coverage-clipped-target}

For any nonnegative \(\omega\in L^1(\nu)\), define the coverage-stopped
adjoint Bellman operator by
\[
    (\Bpigcov\omega)(x)
    =
    \frac{\dd
    \left\{(1-\gamma)\dinit+\gamma(\omega\nu)\Ppi\right\}_{\rm ac}}
    {\dd\nu}(x),
\]
where \(\{\mu\}_{\rm ac}\) denotes the \(\nu\)-absolutely continuous component
in the Lebesgue decomposition of a finite measure \(\mu\)
\citep{bogachev2007measure}. The coverage-stopped discounted occupancy
ratio is the fixed point
\begin{equation}
\label{eq:coverage-retained-fixed-point}
    \omega_{\rm cov}=\Bpigcov\omega_{\rm cov},
\end{equation}
which exists because \(\Bpigcov\) is a \(\gamma\)-contraction on
the nonnegative cone of \(L^1(\nu)\); see
Lemma~\ref{lem:coverage-retained-fixed-point} in
Appendix~\ref{app:coverage-truncated-fore-proofs}.

This fixed point admits an interpretation as the discounted occupancy generated
by following the target policy until an uncovered state--action pair is reached;
see Lemma~\ref{lem:coverage-truncated-retained-flow}. Let \(C_{\rm cov}\) be a
measurable set on which the \(\nu\)-absolutely continuous component of
\[
    (1-\gamma)\dinit+\gamma(\omega_{\rm cov}\nu)\Ppi
\]
is concentrated and on whose complement its \(\nu\)-singular component is
concentrated. 
Define the state--action-dependent coverage indicator
\[
    a_{\rm cov}(x)
    :=
    \mathbf 1\{x\in C_{\rm cov}\}.
\]
At each time \(t\), conditional on the trajectory having continued through time
\(t-1\), continue through \(X_t\) if \(a_{\rm cov}(X_t)=1\); otherwise, stop the
trajectory before collecting the reward at \(X_t\). Let
\[
    T_{\rm cov}
    :=
    \inf\{t\geq 0:a_{\rm cov}(X_t)=0\}
\]
denote the first uncovered time. Then, for every measurable \(B\),
\[
    (\omega_{\rm cov}\nu)(B)
    =
    (1-\gamma)\mathbb E
    \left[
        \sum_{t=0}^\infty
        \gamma^t
        \mathbf 1\{T_{\rm cov}>t,\ X_t\in B\}
    \right].
\]
Thus, \(\omega_{\rm cov}\nu\) is the discounted occupancy accumulated before the
trajectory is stopped. Consequently, for any bounded reward \(r\), define the
stopped discounted return
\[
    V_{\pi,{\rm cov}}(r)
    :=
    \mathbb E_\nu\{\omega_{\rm cov}(X)r(X)\}
    =
    (1-\gamma)\mathbb E
    \left[
        \sum_{t=0}^\infty
        \gamma^t r(X_t)\mathbf 1\{T_{\rm cov}>t\}
    \right].
\]
This is the discounted return accumulated before coverage-dependent stopping,
with rewards from the first uncovered state onward set to zero. For
nonnegative rewards, it is therefore a conservative lower bound on the full
target-policy value:
\[
    V_{\pi,{\rm cov}}(r)\le V_\pi(r).
\]
More generally, if
\(\underline r\le r(x)\le\overline r\) for all \(x\), then
\[
    V_{\pi,{\rm cov}}(r)+\underline r(1-m_{\rm cov})
    \le
    V_\pi(r)
    \le
    V_{\pi,{\rm cov}}(r)+\overline r(1-m_{\rm cov}),
    \qquad
    m_{\rm cov}:=\mathbb E_\nu\omega_{\rm cov}.
\]
The coverage-stopped occupancy mass has the stopping-time representation
\[
    m_{\rm cov}
    =
    1-\mathbb E\!\left[\gamma^{T_{\rm cov}}\right],
\]
where \(\gamma^\infty:=0\) and, when \(\gamma=0\), \(0^0:=1\). Thus,
\(1-m_{\rm cov}\) is exactly the discounted
occupancy mass removed because of insufficient coverage and controls the
resulting value uncertainty.

\subsection{Projected population recursion and approximation error}
\label{sec:coverage-clipped-population}

\noindent We use a finite clipping level to regularize estimation of the
coverage-stopped update. For \(\tau_u\in[1,\infty)\), define
\[
    (\Bpigcap{\tau_u}\omega)(x)
    =
    \left[
        \frac{\dd
        \left\{(1-\gamma)\dinit+\gamma(\omega\nu)\Ppi\right\}_{\rm ac}}
        {\dd\nu}(x)
    \right]
    \wedge\tau_u.
\]
By Lemma~\ref{lem:coverage-clipped-fixed-point}, this operator has a unique
fixed point,
\begin{equation}
\label{eq:coverage-truncated-fixed-point}
    \omega_{\tau_u}=\Bpigcap{\tau_u}\omega_{\tau_u}.
\end{equation}
We approximate this clipped update by generalized-KL projection. Fix a lower
envelope \(\tau_\ell\) and an upper clipping level \(\tau_u\) satisfying
\(0<\tau_\ell\le1\le\tau_u<\infty\), and write
\[
    \sH_{\rm clip}
    =
    \{h\in\sH:\log\tau_\ell\le h\le\log\tau_u\},
    \qquad
    \sW_{\rm clip}=\{e^h:h\in\sH_{\rm clip}\}.
\]
For any bounded nonnegative \(u\), define its generalized-KL projection
onto \(\sW_{\rm clip}\) and the corresponding projected Bellman operator by
\[
    \Pi_{\sW_{\rm clip}}^{\rm genKL}u
    \in
    \argmin_{v\in\sW_{\rm clip}}D_\nu^{\rm gen}(u\|v),
    \qquad
    \mathsf T_{\sW_{\rm clip}}^{\rm genKL}\omega
    =
    \Pi_{\sW_{\rm clip}}^{\rm genKL}(\Bpigcap{\tau_u}\omega).
\]
Finally, define the ratio-class approximation error for the clipped target,
analogous to \(\varepsilon_{\rm KL}\), by
\[
    \varepsilon_{\rm ratio}(\tau_u)
    =
    \inf_{v\in\sW_{\rm clip}}
    D_\nu^{\rm gen}(v\|\omega_{\tau_u}).
\]

\begin{theorem}[Clipped population recursion]
\label{thm:coverage-truncated-recursion}
Fix a lower model envelope \(\tau_\ell\) and an upper clipping level \(\tau_u\)
such that \(0<\tau_\ell\le1\le\tau_u<\infty\), and suppose that
Condition~\ref{ass:kl-class} holds. Let \(\omega_{\tau_u}\) denote the fixed
point defined in \eqref{eq:coverage-truncated-fixed-point}. If
\(\omega^{(k+1)}=\mathsf T_{\sW_{\rm clip}}^{\rm genKL}\omega^{(k)}\) for
\(k\ge0\), with \(\omega^{(0)}\in\sW_{\rm clip}\), then
\[
    D_\nu^{\rm gen}(\omega^{(K)}\|\omega_{\tau_u})
    \le
    \gamma^K
    D_\nu^{\rm gen}(\omega^{(0)}\|\omega_{\tau_u})
    +
    \frac{\tau_u}{\tau_\ell}
    \frac{1-\gamma^K}{1-\gamma}
    \varepsilon_{\rm ratio}(\tau_u).
\]
Moreover,
\[
    \|\omega^{(K)}-\omega_{\rm cov}\|_{L^1(\nu)}
    \le
    \left\{
        2(\tau_u+1)
        D_\nu^{\rm gen}
        (\omega^{(K)}\|\omega_{\tau_u})
    \right\}^{1/2}
    +
    \frac{E_\nu\{(\omega_{\rm cov}-\tau_u)_+\}}{1-\gamma}.
\]
\end{theorem}
\noindent If \(\omega_{\tau_u}\in\sW_{\rm clip}\), then
\(\varepsilon_{\rm ratio}(\tau_u)=0\), and the projected recursion converges
geometrically to \(\omega_{\tau_u}\). More generally,
\[
    \limsup_{K\to\infty}
    D_\nu^{\rm gen}(\omega^{(K)}\|\omega_{\tau_u})
    \le
    \frac{\tau_u}{\tau_\ell}
    \frac{\varepsilon_{\rm ratio}(\tau_u)}{1-\gamma}.
\]
The final term in the \(L^1(\nu)\) bound is the clipping bias relative to the
coverage-stopped occupancy. Under absolute continuity,
\(\omega_{\rm cov}=\omegastar\), so the displayed \(L^1(\nu)\) bound applies
directly to the full occupancy ratio.

\subsection{Moment identification and fitted coverage-stopped \textsc{FORE}}
\label{sec:coverage-clipped-fitted}

To express the projection of the clipped update in terms of initial-distribution
and offline-transition moments, let
\[
    \mu_\omega
    :=
    (1-\gamma)\dinit+\gamma(\omega\nu)\Ppi
    =
    (\Bpigcov\omega)\nu+\mu_{\omega,\perp}
\]
be the Lebesgue decomposition of \(\mu_\omega\) relative to \(\nu\).
Define the population retention indicator by
\[
    c_{\omega,\tau_u}^\star
    \in
    \argmin_{\substack{c:\sX\to\{0,1\}\\ c\ \mathrm{measurable}}}
    \left\{
        (1-\gamma)E_{\dinit}\{c(X)\}
        +
        \gamma E_\nu\{\omega(X)c(X^+)\}
        +
        \tau_u E_\nu\{1-c(X)\}
    \right\}.
\]
The criterion is separable across state--action pairs. Any minimizer satisfies
\[
    c_{\omega,\tau_u}^\star
    =
    \mathbf 1\{\Bpigcov\omega<\tau_u\}
    \quad \nu\text{-a.e. outside the tie set},
    \qquad
    c_{\omega,\tau_u}^\star=0
    \quad \mu_{\omega,\perp}\text{-a.e.},
\]
with either value allowed \(\nu\)-a.e. on
\(\{\Bpigcov\omega=\tau_u\}\). Thus, the retention indicator removes the
singular component of the adjoint Bellman measure and, within its absolutely
continuous component, identifies where the density is retained and where
clipping is active.

\begin{proposition}[Moment identification for the generalized-KL projection]
\label{prop:coverage-truncated-gate}
Suppose Condition~\ref{ass:kl-class} holds and
\(\sH_{\rm clip}\ne\varnothing\). For any nonnegative
\(\omega\in L^1(\nu)\),
\(\mathsf T_{\sW_{\rm clip}}^{\rm genKL}\omega=e^{h_\omega}\), where
\[
\begin{aligned}
    h_\omega
    & \in
    \argmin_{h\in\sH_{\rm clip}}
    \bigg\{
        E_\nu\left[
            e^{h(X)}
            -
            \tau_u\{1-c_{\omega,\tau_u}^\star(X)\}h(X)
        \right]\\
    &\hspace{5em}
        -
        (1-\gamma)E_{\dinit}\{
            c_{\omega,\tau_u}^\star(X)h(X)
        \} 
        -
        \gamma E_\nu\{
            \omega(X)c_{\omega,\tau_u}^\star(X^+)h(X^+)
        \}
    \bigg\}.
\end{aligned}
\]
\end{proposition}

\noindent Algorithm~\ref{alg:coverage-truncated-fore} gives the corresponding
fitted procedure.

\begin{algorithm}[H]
\caption{Coverage-stopped \textsc{FORE} with a learned coverage classifier}
\label{alg:coverage-truncated-fore}
\begin{algorithmic}[1]
\Require Offline transitions \(\{X_i=(S_i,A_i),S_i'\}_{i=1}^n\),
initial-moment estimator \(\widehat P_0\), target policy \(\pi\), discount
\(\gamma\), lower model envelope \(\tau_\ell\) and upper clipping level
\(\tau_u\) satisfying \(0<\tau_\ell\le1\le\tau_u<\infty\),
log-ratio class \(\mathcal H\),
classifier class
\(\mathcal C\subseteq\{c:\sX\to\{0,1\}\}\), and iteration count \(K\)
\State Draw \(A_i^+\sim\pi(\cdot\mid S_i')\) and set
\(X_i^+=(S_i',A_i^+)\), \(i=1,\ldots,n\)
\State Initialize \(\widehat\omega^{(0)}(x)\equiv 1\)
\For{\(k=0,\ldots,K-1\)}
    \State Fit the retention indicator
    \[
    \widehat c_k
    \in
    \argmin_{c\in\mathcal C}
    \left\{
        (1-\gamma)\widehat P_0 c
        +
        \gamma\frac1n\sum_{i=1}^n
            \widehat\omega^{(k)}(X_i)c(X_i^+)
        +
        \tau_u\frac1n\sum_{i=1}^n\{1-c(X_i)\}
    \right\}.
    \]
    \State Fit the generalized-KL projection of the clipped update
    \[
    \begin{aligned}
        \widehat h_{k+1}
        \in
        \argmin_{\substack{h\in\mathcal H\\
            \log\tau_\ell\le h\le\log\tau_u}}
        \Bigg\{
        &\frac1n\sum_{i=1}^n e^{h(X_i)}
        -
        (1-\gamma)\widehat P_0(\widehat c_k h) \\
        &-
        \gamma\frac1n\sum_{i=1}^n
            \widehat\omega^{(k)}(X_i)
            \widehat c_k(X_i^+)h(X_i^+)  -
        \tau_u\frac1n\sum_{i=1}^n
            \{1-\widehat c_k(X_i)\}h(X_i)
        \Bigg\}.
    \end{aligned}
    \]
    \State Set
    \(\widehat\omega^{(k+1)}(x)=e^{\widehat h_{k+1}(x)}\).
\EndFor
\Ensure \(\widehat\omega^{(K)}\) and coverage-stopped occupancy-mass diagnostic
\(\widehat m_K=n^{-1}\sum_{i=1}^n\widehat\omega^{(K)}(X_i)\)
\end{algorithmic}
\end{algorithm}

\noindent At each iteration, the retention indicator
\(\widehat c_k\) implements the clipped update: it retains the absolutely
continuous Bellman density where it does not exceed \(\tau_u\) and rejects
both the region where clipping is active and the singular component. Separately,
the constraint \(\log\tau_\ell\le h\le\log\tau_u\) restricts the fitted ratio
to the envelope \([\tau_\ell,\tau_u]\). Both steps depend only on moments of
the initial distribution and the offline transition law.

The coverage classifier can be fit using a smooth, weighted logistic
surrogate. Choose a score class \(\mathcal F\), set
\[
    \mathcal C
    =
    \{x\mapsto\mathbf 1\{f(x)\ge0\}:f\in\mathcal F\},
\]
and let \(\ell(t,y)=\log\{1+\exp(t)\}-yt\). At iteration \(k\), fit
\[
\begin{aligned}
    \widehat f_k
    \in
    \argmin_{f\in\mathcal F}
    \bigg[
        &\tau_u\frac1n\sum_{i=1}^n \ell\{f(X_i),1\}
        +(1-\gamma)\widehat P_0\ell\{f(X),0\} \\
        &\qquad
        +\gamma\frac1n\sum_{i=1}^n
        \widehat\omega^{(k)}(X_i)\ell\{f(X_i^+),0\}
    \bigg],
\end{aligned}
\]
and set $ \widehat c_k(x)=\mathbf 1\{\widehat f_k(x)\ge0\}.$
If
\(\widehat P_0g=m^{-1}\sum_{j=1}^m g(X_{0,j})\), this objective is a
weighted binary classification loss on the pooled sample. After multiplying
all weights by the common factor \(n\), the offline observations receive
class-one weight \(\tau_u\), while the initial and successor observations
receive class-zero weights \(n(1-\gamma)/m\) and
\(\gamma\widehat\omega^{(k)}(X_i)\), respectively.

Because the oracle retention rule depends on the current ratio iterate, the
classifier is refit at each iteration. Reusing the same data
therefore requires uniform control of its approximation and estimation errors
over the ratio class. In practice, one may use flexible classifiers, such as
neural networks or gradient-boosted trees, together with regularization or
early stopping.

\subsection{Finite-sample theory}
\label{sec:coverage-clipped-finite-sample}

We now establish a finite-sample guarantee for the exact-ERM version of
Algorithm~\ref{alg:coverage-truncated-fore}. The result parallels
Theorem~\ref{thm:fitted-kl-fori}, with a single critical radius controlling
estimation of both the retention indicator and the generalized-KL projection.

Let \(X_1^0,\ldots,X_n^0\) be i.i.d. draws from \(\dinit\), independent of the
transition sample, and apply Algorithm~\ref{alg:coverage-truncated-fore} with $\widehat P_0 g
    :=
    \frac{1}{n}\sum_{i=1}^n g(X_i^0).$
Fix \(A>0\), and throughout this subsection set
\[
    \tau_{u,n}=1\vee A\log(en).
\]
The classes
\(\sH_{\rm clip}\) and \(\sW_{\rm clip}\) are understood to use this clipping
level.
We impose the following conditions:

\begin{enumerate}[label=\textbf{(C\arabic*)}, ref=C\arabic*, series=clipcond]
\item \label{ass:clip-finite-cov-upper-tail}
\textit{Subexponential coverage-stopped occupancy ratio.}
There exists \(0<K_{\rm cov}<\infty\) such that
\[
    \|\omega_{\rm cov}\|_{\psi_1}\le K_{\rm cov}.
\]

\item \label{ass:clip-finite-lower-tail}
\textit{Coverage-stopped occupancy-ratio lower tail.}
There exist constants \(0<A_{\rm cov}<\infty\) and \(\alpha_{\rm cov}>0\)
such that
\[
    \omega_{\rm cov}>0
    \quad \nu\text{-a.e.},
    \qquad
    \nu\{x:0<\omega_{\rm cov}(x)\le t\}
    \le
    A_{\rm cov}t^{\alpha_{\rm cov}},
    \qquad
    0<t\le1.
\]

\item \label{ass:clip-finite-selector-margin}
\textit{Uniform threshold margin.}
There exist constants \(0<A_{\rm mar}<\infty\) and
\(\alpha_{\rm mar}>0\) such that, for every
\(n\), \(\omega\in\sW_{\rm clip}\), and \(0<s\le1\),
\[
    \nu\{x:|(\Bpigcov\omega)(x)-\tau_{u,n}|\le s\}
    \le
    A_{\rm mar}s^{\alpha_{\rm mar}}.
\]

\item \label{ass:clip-finite-projection-compact}
\textit{Projection compactness.}
For every \(n\), the class \(\sH_{\rm clip}\) is compact in
$ L^2(\bar\nu_\pi)$, where $ \bar\nu_\pi
    :=
    \frac{1}{3}(\nu+\dinit+\nup).$
\end{enumerate}
Condition~\ref{ass:clip-finite-cov-upper-tail} permits unbounded,
subexponential coverage-stopped Bellman images while controlling clipping
error. If these images are uniformly bounded over the fitted ratio class, a
fixed threshold above that bound makes clipping inactive; otherwise, the
logarithmic schedule accommodates their unbounded tails.
Condition~\ref{ass:clip-finite-lower-tail} allows the coverage-stopped ratio
to approach zero while controlling its lower-tail mass.
Condition~\ref{ass:clip-finite-selector-margin} controls the \(\nu\)-mass near
the clipping threshold, ensuring that the oracle retention rule is well
separated; larger values of \(\alpha_{\rm mar}\) imply stronger separation and
faster rates. Finally,
Condition~\ref{ass:clip-finite-projection-compact} ensures that the population
projection objective is attained for every candidate retention indicator and
ratio iterate.

Beyond the population-iteration and model-approximation terms, our
finite-sample bound depends on two quantities: a joint statistical rate for
estimating the retention indicator and generalized-KL projection, and a
uniform approximation error for the classifier class. Let
\(\mathfrak r_{n,\rm clip}\) denote the critical radius of the corresponding
loss-difference classes, as defined in Appendix~\ref{app:clip-finite}, and set
\[
    \mathcal E_{n,\rm stat}(\delta)
    :=
    \left\{
        \mathfrak r_{n,\rm clip}
        +
        \left\{\frac{\log(1/\delta)}{n}\right\}^{1/2}
    \right\}
    \left\{
        \mathfrak r_{n,\rm clip}^{
            \alpha_{\rm mar}/(\alpha_{\rm mar}+2)}
        +
        \left\{\frac{\log(1/\delta)}{n}\right\}^{
            \alpha_{\rm mar}/\{2(\alpha_{\rm mar}+2)\}}
    \right\}.
\]
For each Lebesgue decomposition
\(\mu_\omega=(\Bpigcov\omega)\nu+\mu_{\omega,\perp}\), define the uniform
coverage-classifier approximation error
\[
    \varepsilon_{\rm cls}(\tau_{u,n})
    :=
    \sup_{\omega\in\sW_{\rm clip}}
    \inf_{c\in\mathcal C}
    \left[
        E_\nu\left\{
            |(\Bpigcov\omega)(X)-\tau_{u,n}|
            1\{c(X)\ne c_{\omega,\tau_{u,n}}^\star(X)\}
        \right\}
        +
        \int c\,\dd\mu_{\omega,\perp}
    \right].
\]
This quantity weights classifier disagreements with the oracle retention
indicator by their distance from the clipping threshold and penalizes the
classifier for incorrectly retaining Bellman mass that is singular with
respect to \(\nu\).

\begin{theorem}[Fitted coverage-stopped \textsc{FORE}]
\label{thm:clip-finite-fore}
Let \(\gamma\in[0,1)\), and fix a lower model envelope \(\tau_\ell\) and an
upper clipping schedule \(\tau_{u,n}=1\vee A\log(en)\), where
\(0<\tau_\ell\le1\) and \(A>K_{\rm cov}/2\). Assume
Conditions~\ref{ass:kl-class},
\ref{ass:clip-finite-cov-upper-tail},
\ref{ass:clip-finite-lower-tail},
\ref{ass:clip-finite-selector-margin}, and
\ref{ass:clip-finite-projection-compact}. Suppose also that
\(1\in\sW_{\rm clip}\). Let
\(\{\widehat\omega^{(k)}\}_{k=0}^K\) be the exact-ERM fitted coverage-stopped
\textsc{FORE} iterates of
Algorithm~\ref{alg:coverage-truncated-fore}. Then,
for every \(0<\delta<1\), with probability at least \(1-\delta\), for
\(\rho=(1+\gamma)/2\),
\[
\begin{aligned}
\|\widehat\omega^{(K)}-\omega_{\rm cov}\|_{L^1(\nu)}
\le C_n\Bigg[
&\rho^{K/2}
\left\{
D_\nu^{\rm gen}
(\widehat\omega^{(0)}\|\omega_{\tau_{u,n}})
\right\}^{1/2}
+
\frac{
\left\{
(1-\gamma)\varepsilon_{\rm ratio}(\tau_{u,n})
+\varepsilon_{\rm cls}(\tau_{u,n})
\right\}^{1/2}
}{1-\gamma}\\
&+
\frac{
\{\mathcal E_{n,\rm stat}(\delta)\}^{1/2}
}{1-\gamma}
\Bigg],
\end{aligned}
\]
where finite constants \(C_0,q\), independent of \(n\), may be chosen so that
\[
    C_n
    \le
    C_0\{1+\log(en)+\tau_\ell^{-1}\}^{q}
    \left\{1+\log\frac{e}{1-\gamma}\right\}^{1/2}.
\]
\end{theorem}

\paragraph{Bound terms.}
As in Theorem~\ref{thm:fitted-kl-fori}, the first term is the geometrically decaying iteration error. The second combines the ratio-class approximation error \(\varepsilon_{\rm ratio}(\tau_{u,n})\) for the clipped target with the coverage-classifier approximation error \(\varepsilon_{\rm cls}(\tau_{u,n})\). Although the ratio class need not be closed under the clipped Bellman operator, the classifier class \(\mathcal C\) must uniformly approximate the oracle retention indicators, as measured by \(\varepsilon_{\rm cls}(\tau_{u,n})\). This condition is analogous in spirit to completeness but concerns a binary retention rule rather than Bellman images of the ratio class. The third term is the statistical error accumulated across iterations from jointly estimating the retention indicator and projection, with critical radius \(\mathfrak r_{n,\rm clip}\) capturing the complexity of both problems. If \(\mathcal C\) has VC dimension \(d_{\mathcal C}\) and \(\sH_{\rm clip}\) has VC-subgraph dimension \(d_{\mathcal H}\), Lemma~\ref{lem:clip-finite-vc-radius} yields
\[
    \mathfrak r_{n,\rm clip}^2
    \le
    C_n\frac{(d_{\mathcal C}+d_{\mathcal H})\log(en)}{n}.
\]
For fixed class dimensions and confidence level, \(\mathcal E_{n,\rm stat}(\delta)\) is, up to logarithmic factors, of order \(n^{-2/3}\) when \(\alpha_{\rm mar}=1\), \(n^{-3/4}\) when \(\alpha_{\rm mar}=2\), and approaches \(n^{-1}\) as \(\alpha_{\rm mar}\to\infty\). In this limit,
\[
    \mathcal E_{n,\rm stat}(\delta)
    \lesssim
    \mathfrak r_{n,\rm clip}^2
    +
    \frac{\log(1/\delta)}{n}.
\]

\paragraph{Adaptivity to coverage.}
The estimator automatically adapts to favorable coverage by incurring little
or no classifier-approximation error. If \(1\in\mathcal C\), choosing
\(c\equiv1\) gives
\[
    \varepsilon_{\rm cls}(\tau_{u,n})
    \le
    \sup_{\omega\in\sW_{\rm clip}}
    \left[
        E_\nu\{((\Bpigcov\omega)(X)-\tau_{u,n})_+\}
        +
        \mu_{\omega,\perp}(\sX)
    \right].
\]
This bound is small when, uniformly over \(\omega\in\sW_{\rm clip}\), the
excess of \(\Bpigcov\omega\) above the clipping threshold and the total mass
of \(\mu_{\omega,\perp}\) are small. If
\(\dinit\ll\nu\) and \(\Ppi(x,\cdot)\ll\nu\) for \(\nu\)-a.e. \(x\), then
\(\mu_{\omega,\perp}=0\) for every \(\omega\in\sW_{\rm clip}\). If, in
addition, \(\Bpigcov\omega\le\tau_{u,n}\) \(\nu\)-a.e. for every
\(\omega\in\sW_{\rm clip}\), then the oracle retention rule is identically one
and \(\varepsilon_{\rm cls}(\tau_{u,n})=0\). In this regime, the guarantee
essentially reduces to the corresponding \textsc{FORE} bound, up to the
additional statistical cost of jointly estimating the retention rule and the
generalized-KL projection.

\section{Numerical Experiments}
\label{sec:experiments}

The theory separates two requirements in offline policy evaluation:
realizability of \(Q^\pi\) in a value class and representability of the
discounted occupancy ratio in a density-ratio class. Our first two examples
isolate this distinction. In both, \(Q^\pi\) belongs to the fitted value class,
but linear FQE can be unstable because its Bellman update is projected in the
offline data norm and the class is not Bellman complete. By contrast,
log-linear \textsc{FORE} remains stable when the ratio class contains the true
discounted occupancy ratio, even though the class is not adjoint Bellman
complete. As occupancy-estimation baselines, we compare \textsc{FORE} with DualDICE
\citep{nachumEtAl2019DualDICE} and minimax weight learning (MWL)
\citep{ueharaEtAl2020MWLMQL}, using the same ratio class and favorable tuning
of the critic classes. We also use the fitted \textsc{FORE} ratio to construct
a \textsc{FORE}-reweighted FQE baseline, which changes only the projection distribution
in FQE. A third experiment evaluates coverage-stopped \textsc{FORE} when full
coverage fails.

\subsection{Baird-style finite MRP}
\label{sec:experiments-baird}

Our first example is a Baird-style finite MRP based on the star-shaped
off-policy counterexample of \citet{baird1995residual}. The state space has six
symmetric upper states and one lower state. We specify a target transition
kernel, an offline data distribution, and a one-dimensional feature \(\phi\), with
\(\phi(x)=0.1\) on each upper state and \(\phi(x)=1\) on the lower state. The
discounted occupancy ratio is exactly represented by a one-parameter normalized
log-linear class:
\[
    \omegastar(x)=
    \begin{cases}
        0.2211, & x\text{ upper},\\
        15.7987, & x\text{ lower}.
    \end{cases}
\]
To illustrate the role of the occupancy ratio in stabilizing FQE, rewards are
chosen from the Bellman equation \(r=\phi-\gamma P\phi\). Hence the target
value function is realizable in the scalar class \(q_\beta(x)=\beta\phi(x)\),
with \(q^\pi=\phi\), and the policy value is \(0.1\).

Figure~\ref{fig:baird-population} illustrates the population recursions. The
population \textsc{FORE} KL recursion converges to the true ratio. In contrast,
under the offline data distribution, the projected linear FQE recursion has scalar
multiplier \(2.103\), so coefficient errors are amplified across iterations.
Using the \textsc{FORE} ratio as the FQE projection weight changes this
multiplier to \(0.801\). Tabular FQE is included as a Bellman-complete
benchmark, for which the projected Bellman operator has contraction multiplier
\(\gamma=0.95\).

\begin{figure}[H]
\centering
\includegraphics[width=0.92\linewidth]{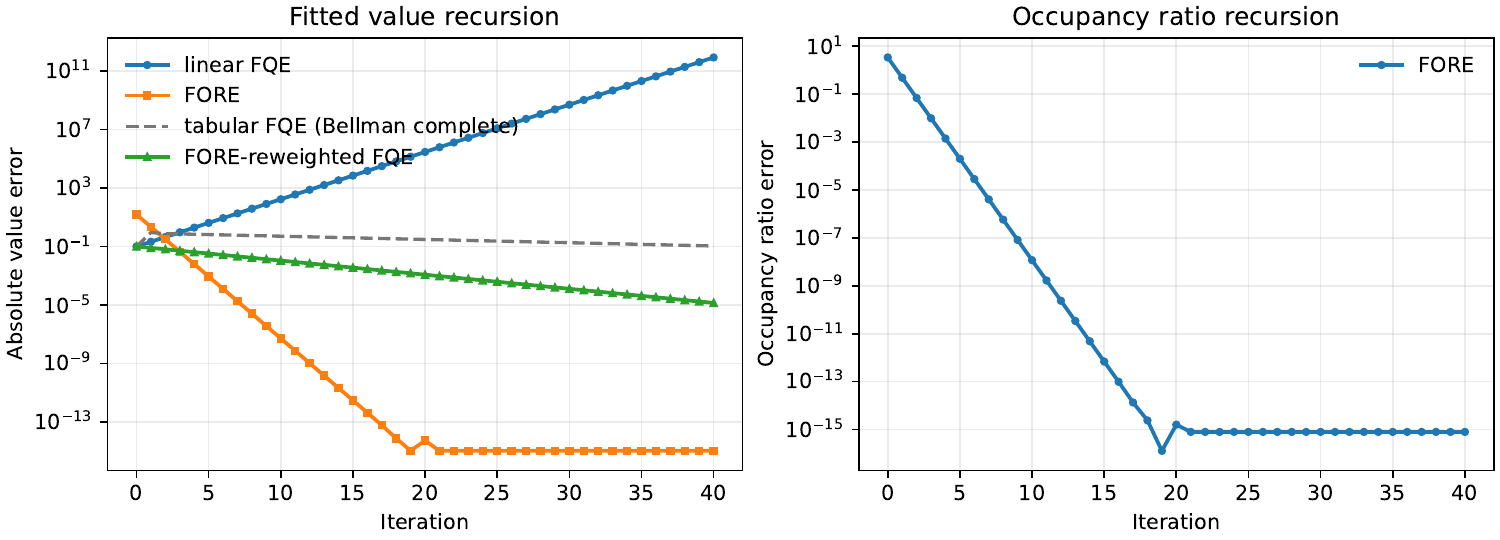}
\caption{Baird-style finite MRP. The left panel shows policy-value error for
linear FQE, direct \textsc{FORE}, tabular FQE, and \textsc{FORE}-reweighted FQE. The right panel
shows the \textsc{FORE} occupancy-ratio error in \(L^1(\nu)\).}
\label{fig:baird-population}
\end{figure}

\subsection{Linear-Gaussian policy evaluation}
\label{sec:experiments-gaussian}

Our second example is a continuous linear-Gaussian policy-evaluation problem with
\(X=(S,A)\in\mathbb R^2\). Offline samples are drawn from
\(\nu=N(0,\Sigma_b)\), where \(\Sigma_b=\operatorname{diag}(1.5,0.4)\). Under
the target policy,
\[
    S^+=0.7S+0.5A+\varepsilon_s,\qquad
    A^+=-0.8S^+ + \varepsilon_a,
\]
with Gaussian noise. The initial distribution is the target stationary distribution, so the
discounted occupancy distribution is Gaussian and the true density ratio is
exponential quadratic. We write \(h_\star\) for the log-density ratio, up to an
additive constant, and use the normalized log-linear class with sufficient
statistics \((h_\star,s,a)\) for \textsc{FORE}, MWL, and DualDICE. This class
contains the target ratio but is not closed under the target transition or the
corresponding adjoint Bellman update.

Rewards are chosen from the Bellman equation. Specifically, we take
\(r=q-\gamma Pq\), with \(q\) quadratic in \(a\), so that \(Q^\pi=q\) belongs
to a three-dimensional value class of the form
\[
    \{ \beta_0 q+\beta_1 s+\beta_2 a:\beta\in\mathbb R^3\}.
\]
Thus the value function is realizable. However, the class is not Bellman
complete, because \(Pq\) contains the quadratic directions \(s^2\) and \(sa\),
which are missing from the value class. At the population level, the projected
linear FQE recursion under the offline data distribution is expansive, with dominant
iteration multiplier \(1.22\). By contrast, the \textsc{FORE} ratio recursion
is contractive, with multiplier \(0.086\). Reweighting FQE by the resulting
occupancy ratio also makes the projected FQE recursion contractive, with
multiplier \(0.68\).

Finite-sample runs use \(n\in\{500,1000,2000,5000,10000\}\) offline
transitions and \(300\) independent repetitions at each sample size.
\textsc{FORE}, MWL, and DualDICE use the same three-dimensional normalized
log-linear ratio class. Linear FQE, \textsc{FORE}-reweighted FQE, and MQL use
the same three-dimensional value class. Thus the direct ratio and value
estimators are compared using classes of the same size. MWL, MQL, and DualDICE
use the same random-Fourier RBF critic class, with \(128\) features and an
intercept term. Additional numerical constants, tuning parameters, and
implementation details are reported in Appendix~\ref{app:gaussian-details}.

\begin{figure}[H]
\centering
\includegraphics[width=0.92\linewidth]{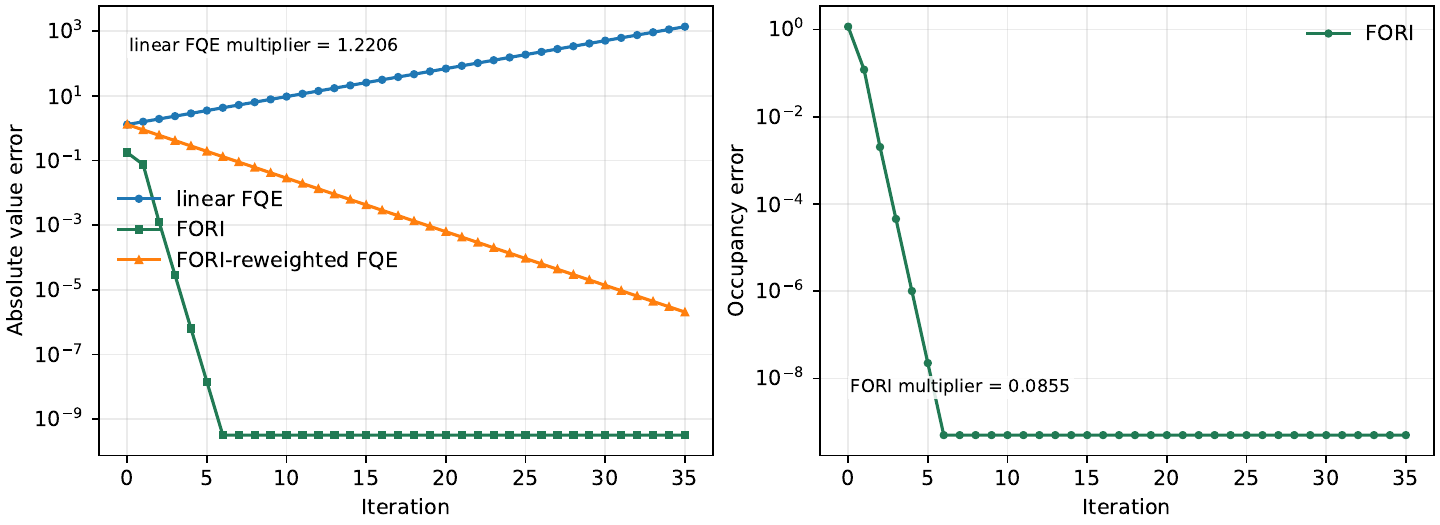}
\caption{Linear-Gaussian population recursions. The left panel shows
policy-value error for linear FQE, direct \textsc{FORE}, and
\textsc{FORE}-reweighted FQE. The right panel shows \textsc{FORE} occupancy
error.}
\label{fig:gaussian-population}
\end{figure}

\begin{figure}[H]
\centering
\includegraphics[width=0.92\linewidth]{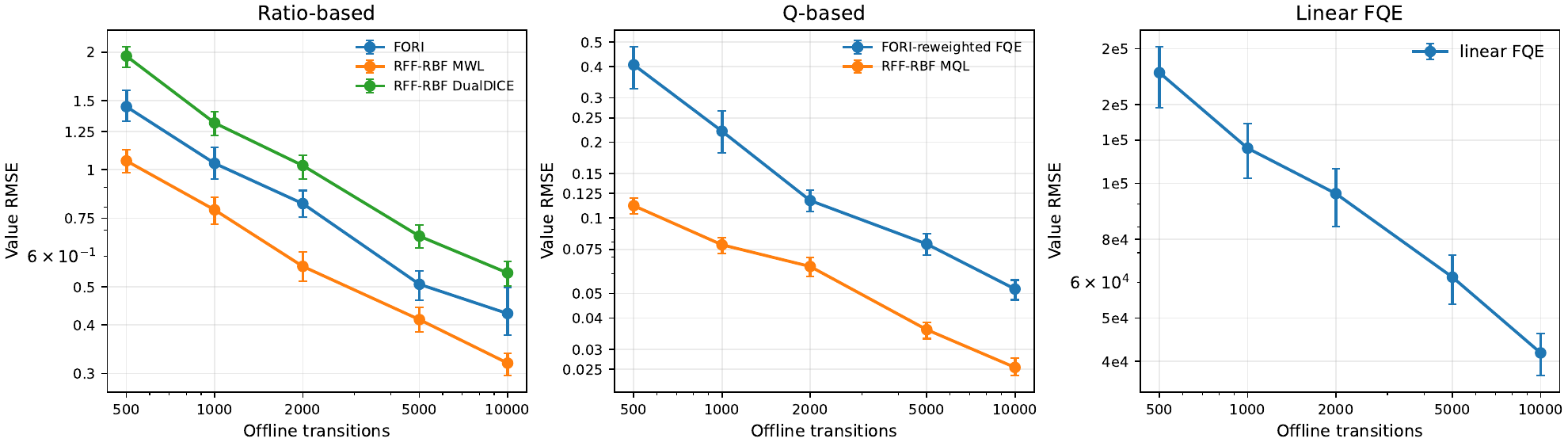}
\caption{Linear-Gaussian finite-sample value error. Curves report value RMSE
over \(300\) repetitions; vertical bars indicate Monte Carlo uncertainty.}
\label{fig:gaussian-finite}
\end{figure}

\begin{figure}[H]
\centering
\includegraphics[width=0.88\linewidth]{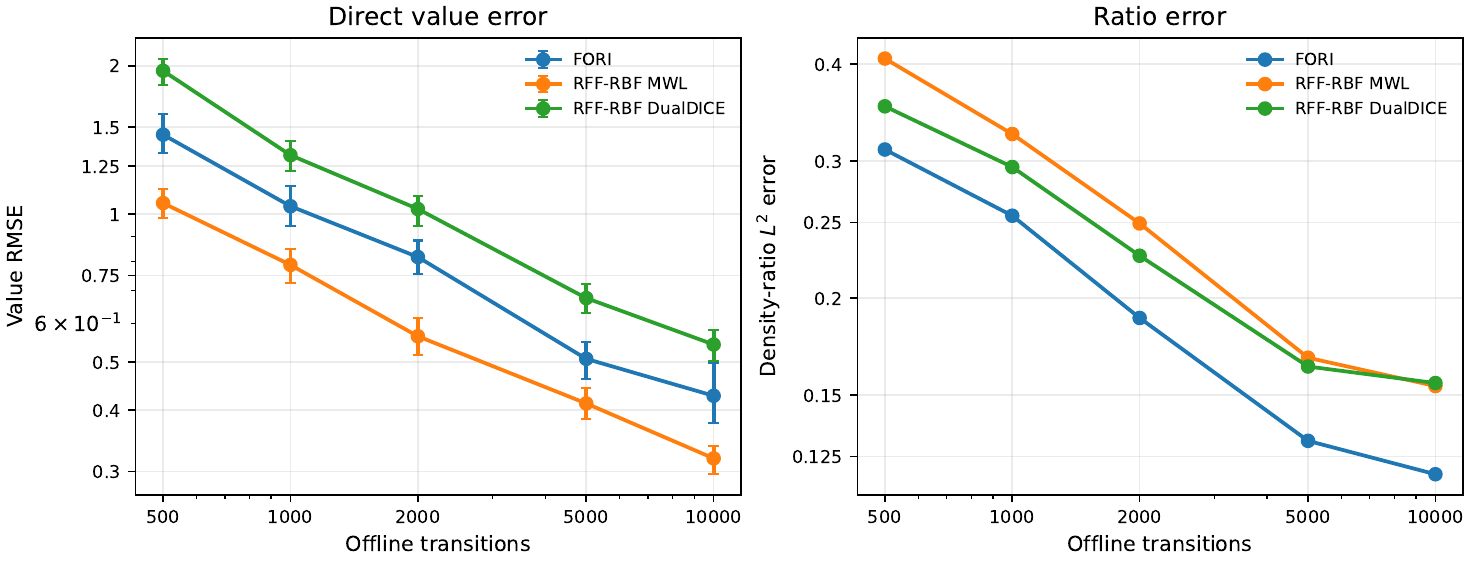}
\caption{Direct ratio estimators in the linear-Gaussian experiment. The left
panel reports value RMSE from direct reward reweighting; the right panel
reports empirical \(L^2(\nu)\) error of the fitted density ratio.}
\label{fig:gaussian-ratio-value}
\end{figure}

Figure~\ref{fig:gaussian-ratio-value} separates density-ratio error from the
error of direct reward reweighting. In this design, \textsc{FORE} has the
smallest density-ratio error across the reported sample sizes, while MWL has
the smallest direct value RMSE among the ratio estimators. At \(n=10000\), MQL
and \textsc{FORE}-reweighted FQE have value RMSEs \(0.025\) and \(0.052\),
respectively, compared with \(4.18\times 10^4\) for linear FQE. Direct reward
reweighting has value RMSEs \(0.428\), \(0.319\), and \(0.543\) for
\textsc{FORE}, MWL, and DualDICE, respectively.

The \textsc{FORE}-reweighted FQE results show that recovering the target
occupancy distribution can stabilize the projected Bellman recursion even when the value
class is not Bellman complete for the Bellman projection under the offline data distribution. MQL uses the
same value class as \textsc{FORE}-reweighted FQE, so their difference reflects
the fitted criterion and critic weighting rather than the size of the
\(Q\)-model.

We also vary the discount factor at fixed sample size \(n=5000\). For each
\(\gamma\), the reward is redefined as \(r=q-\gamma Pq\), so the value class
remains correctly specified. Figure~\ref{fig:gaussian-gamma-sweep} plots value
RMSE against the effective horizon \((1-\gamma)^{-1}\). The direct ratio
estimators grow approximately linearly on this scale, consistent with the
value-level horizon dependence in Theorem~\ref{thm:fitted-kl-fori}. The
\(Q\)-based estimators, MQL and \textsc{FORE}-reweighted FQE, are less
sensitive to the discount in this example. Linear FQE is run for the same fixed
number of fitted updates at every discount, including settings in which the
empirical projected Bellman recursion is noncontractive.

\begin{figure}[H]
\centering
\includegraphics[width=0.92\linewidth]{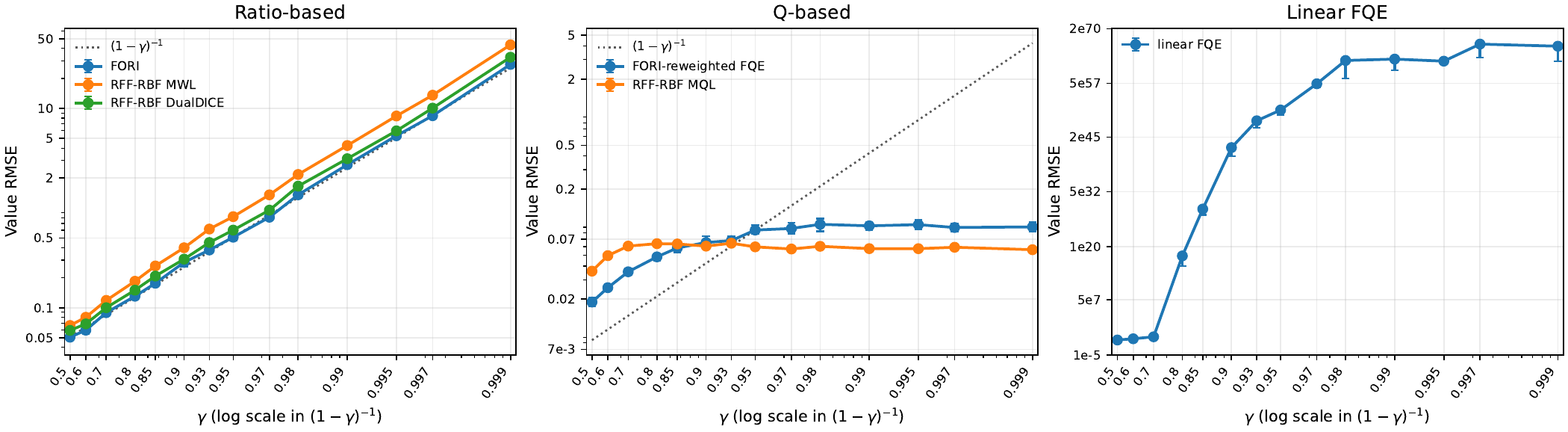}
\caption{Linear-Gaussian value error as the discount varies. Curves report value
RMSE over \(500\) repetitions at \(n=5000\), with the horizontal axis scaled by
the effective horizon \((1-\gamma)^{-1}\). Linear FQE is evaluated after the
same fixed number of fitted updates at every discount.}
\label{fig:gaussian-gamma-sweep}
\end{figure}

\subsection{Coverage-stopped occupancy under insufficient data coverage}
\label{sec:experiments-stopped-fore}

Our third experiment evaluates whether coverage-stopped \textsc{FORE} recovers
the subprobability occupancy induced by following the target policy until the
first unsupported state--action pair. The coverage classifier is estimated
from the data. We vary the fraction of covered contexts over
\(p\in\{0,0.25,0.5,0.75,1\}\) and introduce the support failure either at the
initial stage or at a recurrent successor state. We compare coverage-stopped
\textsc{FORE} with standard \textsc{FORE} normalized to unit mass and with the
standard ratio clipped post hoc at \(20\), evaluating all methods against the
same coverage-stopped occupancy ratio. For \(p<1\), the two standard variants
do not target this subprobability occupancy and test whether normalization or
post-hoc clipping can approximate the removal of unsupported occupancy mass.
For each support-failure location and covered-context fraction, we use \(30\)
independent repetitions and training samples of size \(2{,}000\) and
\(10{,}000\). Appendix~\ref{app:stopped-fore-details} gives the full data-generating
process, analytical occupancy masses, and implementation details.

Figure~\ref{fig:stopped-fore} reports the ratio and value errors. At each sample
size, coverage-stopped \textsc{FORE} has lower \(L^1(\nu)\) ratio error than
both comparison methods in all \(240\) runs with \(p<1\). At \(n=10{,}000\),
its median ratio error is \(0.032\), compared with \(1.154\) for standard
\textsc{FORE} and \(0.626\) for post-hoc clipping. The corresponding median
absolute coverage-stopped value errors are \(0.0067\), \(0.315\), and \(0.124\).
Under full support, coverage-stopped and standard \textsc{FORE} perform
similarly, with median ratio errors of \(0.0537\) and \(0.0504\), respectively.
Thus, the gains from coverage stopping arise specifically under support
failure rather than from uniformly stronger regularization.

\begin{figure}[H]
\centering
\includegraphics[width=0.94\linewidth]{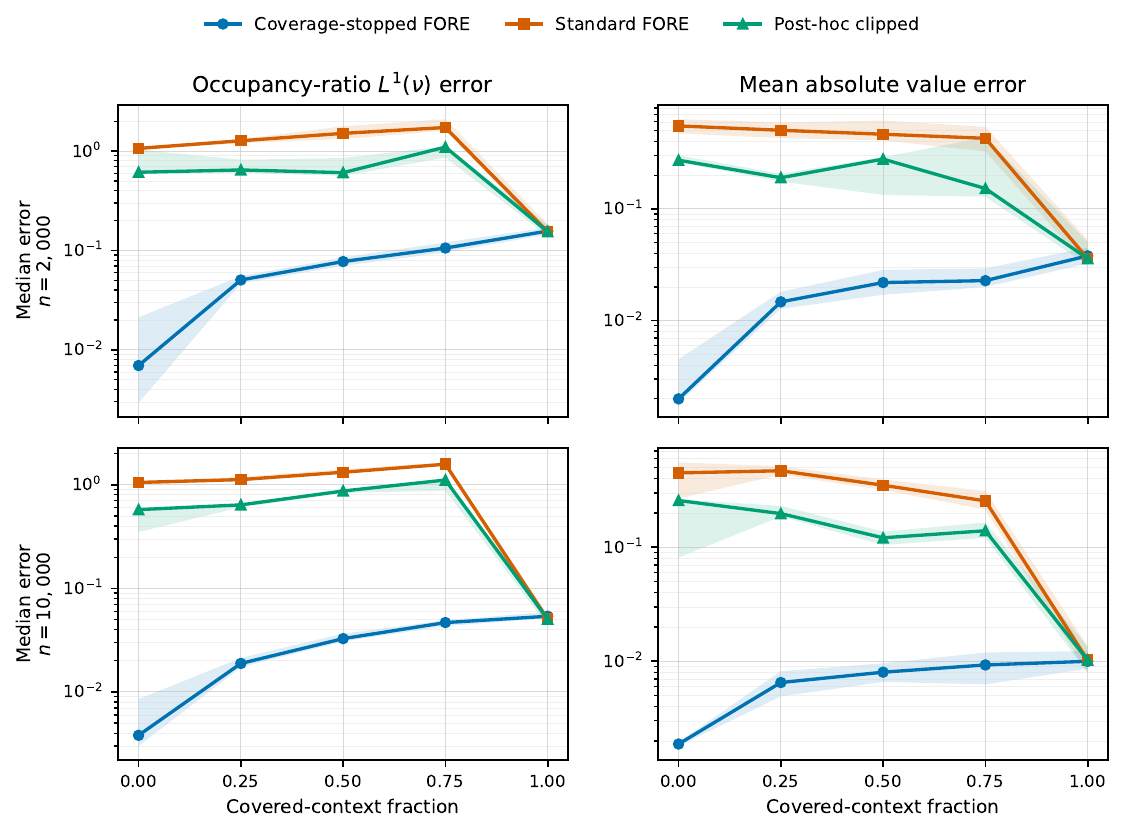}
\caption{Coverage-stopped occupancy under insufficient data coverage. Columns
show the \(L^1(\nu)\) error of the coverage-stopped occupancy ratio and the
mean absolute coverage-stopped value error; rows correspond to the training
sample size. At each covered-context fraction, points show the median over
\(60\) runs, with \(30\) repetitions for each support-failure location. Bands
are pointwise \(95\%\) percentile-bootstrap intervals for the median. }
\label{fig:stopped-fore}
\end{figure}

\section{Conclusion}

\textsc{FORE} formulates discounted occupancy-ratio estimation as a fitted
adjoint Bellman problem. Rather than solving a ratio--critic saddle point, it
iterates adjoint Bellman updates and KL-projects each update onto a class of
positive, normalized ratios. The KL geometry gives the main stability result:
the population update contracts in relative entropy toward the target
occupancy ratio, and the projected recursion converges up to the approximation
error of the reverse-KL projection onto the ratio class. The fitted analysis adds a statistical error
governed by the complexity of the same class.

The approximation requirement is therefore placed directly on the occupancy
ratio. Instead of requiring Bellman completeness of a value class, adjoint
Bellman completeness of a ratio class, or dual completeness of a critic class,
the fixed-policy theory requires that the distribution shift from the offline
distribution to the target discounted occupancy be representable and
estimable. If \(\log\omegastar\) is poorly approximated by the chosen class,
\textsc{FORE} converges only up to the corresponding KL approximation error.

This perspective contrasts with value-function realizability results in
offline reinforcement learning. Realizability of \(Q^\pi\) alone does not
control distribution shift, so finite-sample value guarantees typically require
additional coverage, concentrability, completeness, or weighting conditions
\citep{jiangLi2016DoublyRobust,munosSzepesvari2008FVI,
xieEtAl2019MarginalizedIS,yinWang2020EfficientTabularOPE,
pattersonEtAl2022GeneralizedPBE,vanDerLaanKallus2025StationaryWeightedFQE}.
Given support overlap, however, the structural realizability condition in
\textsc{FORE} is on the discounted occupancy ratio, not on a value or critic
class.

\medskip
 
\noindent\textbf{Limitations.}
Full-ratio recovery requires coverage: target-induced state--action
distributions must be supported by the offline data. Under weak overlap, in
high-dimensional continuous state--action spaces, or for near-deterministic
target policies, occupancy-ratio estimation may be as difficult as, or more
difficult than, value-function estimation. The guarantees also require
boundedness and complexity control of the log-ratio class. Because
\textsc{FORE} models ratios on the log scale, the target discounted occupancy
ratio must be positive on the support of the offline distribution, and the
finite-sample theory requires a lower-tail margin condition. The coverage-stopped extension in
Section~\ref{sec:coverage-truncated-fore} avoids extrapolation into unsupported
regions by targeting the coverage-stopped discounted occupancy. This changes
the estimand: when target occupancy is removed, the resulting value is the
coverage-stopped value \(V_{\pi,{\rm cov}}(r)\), not the full target-policy
value.

\bibliographystyle{plainnat}
\bibliography{refs}

@InProceedings{jiangLi2016DoublyRobust,
  title = 	 {Doubly Robust Off-policy Value Evaluation for Reinforcement Learning},
  author = 	 {Jiang, Nan and Li, Lihong},
  booktitle = 	 {Proceedings of The 33rd International Conference on Machine Learning},
  pages = 	 {652--661},
  year = 	 {2016},
  editor = 	 {Balcan, Maria Florina and Weinberger, Kilian Q.},
  volume = 	 {48},
  series = 	 {Proceedings of Machine Learning Research},
  address = 	 {New York, New York, USA},
  month = 	 {20--22 Jun},
  publisher =    {PMLR},
  url = 	 {https://proceedings.mlr.press/v48/jiang16.html}
}

@InProceedings{thomasBrunskill2016DataEfficientOPE,
  title = 	 {Data-Efficient Off-Policy Policy Evaluation for Reinforcement Learning},
  author = 	 {Thomas, Philip and Brunskill, Emma},
  booktitle = 	 {Proceedings of The 33rd International Conference on Machine Learning},
  pages = 	 {2139--2148},
  year = 	 {2016},
  editor = 	 {Balcan, Maria Florina and Weinberger, Kilian Q.},
  volume = 	 {48},
  series = 	 {Proceedings of Machine Learning Research},
  address = 	 {New York, New York, USA},
  month = 	 {20--22 Jun},
  publisher =    {PMLR},
  url = 	 {https://proceedings.mlr.press/v48/thomasa16.html}
}

@inproceedings{xieEtAl2019MarginalizedIS,
 author = {Xie, Tengyang and Ma, Yifei and Wang, Yu-Xiang},
 booktitle = {Advances in Neural Information Processing Systems},
 editor = {H. Wallach and H. Larochelle and A. Beygelzimer and F. d\textquotesingle Alch\'{e}-Buc and E. Fox and R. Garnett},
 pages = {},
 publisher = {Curran Associates, Inc.},
 title = {Towards Optimal Off-Policy Evaluation for Reinforcement Learning with Marginalized Importance Sampling},
 url = {https://proceedings.neurips.cc/paper_files/paper/2019/file/4ffb0d2ba92f664c2281970110a2e071-Paper.pdf},
 volume = {32},
 year = {2019}
}

@article{van2026researcher,
  title={A Researcher's Guide to Empirical Risk Minimization},
  author={van der Laan, Lars},
  journal={arXiv preprint arXiv:2602.21501},
  year={2026}
}

@inproceedings{chen2019information,
  title={Information-theoretic considerations in batch reinforcement learning},
  author={Chen, Jinglin and Jiang, Nan},
  booktitle={International conference on machine learning},
  pages={1042--1051},
  year={2019},
  organization={PMLR}
}

@InProceedings{yinWang2020EfficientTabularOPE,
  title = 	 {Asymptotically Efficient Off-Policy Evaluation for Tabular Reinforcement Learning},
  author =       {Yin, Ming and Wang, Yu-Xiang},
  booktitle = 	 {Proceedings of the Twenty Third International Conference on Artificial Intelligence and Statistics},
  pages = 	 {3948--3958},
  year = 	 {2020},
  editor = 	 {Chiappa, Silvia and Calandra, Roberto},
  volume = 	 {108},
  series = 	 {Proceedings of Machine Learning Research},
  month = 	 {26--28 Aug},
  publisher =    {PMLR},
  url = 	 {https://proceedings.mlr.press/v108/yin20b.html}
}

@article{kallusUehara2020DoubleRL,
  author  = {Nathan Kallus and Masatoshi Uehara},
  title   = {Double Reinforcement Learning for Efficient Off-Policy Evaluation in Markov Decision Processes},
  journal = {Journal of Machine Learning Research},
  year    = {2020},
  volume  = {21},
  number  = {167},
  pages   = {1--63},
  url     = {http://jmlr.org/papers/v21/19-827.html}
}

@InProceedings{kallusUehara2020DRL_ICML,
  title = 	 {Double Reinforcement Learning for Efficient and Robust Off-Policy Evaluation},
  author =       {Kallus, Nathan and Uehara, Masatoshi},
  booktitle = 	 {Proceedings of the 37th International Conference on Machine Learning},
  pages = 	 {5078--5088},
  year = 	 {2020},
  editor = 	 {III, Hal Daum{\'e} and Singh, Aarti},
  volume = 	 {119},
  series = 	 {Proceedings of Machine Learning Research},
  month = 	 {13--18 Jul},
  publisher =    {PMLR},
  url = 	 {https://proceedings.mlr.press/v119/kallus20b.html}
}

@inproceedings{kallusUehara2019IntrinsicOPE,
 author = {Kallus, Nathan and Uehara, Masatoshi},
 booktitle = {Advances in Neural Information Processing Systems},
 editor = {H. Wallach and H. Larochelle and A. Beygelzimer and F. d\textquotesingle Alch\'{e}-Buc and E. Fox and R. Garnett},
 pages = {},
 publisher = {Curran Associates, Inc.},
 title = {Intrinsically Efficient, Stable, and Bounded Off-Policy Evaluation for Reinforcement Learning},
 url = {https://proceedings.neurips.cc/paper_files/paper/2019/file/59bcda7c438bad7d2afffe9e2fed00be-Paper.pdf},
 volume = {32},
 year = {2019}
}

@InProceedings{kallusUehara2020PolicyGradients,
  title = 	 {Statistically Efficient Off-Policy Policy Gradients},
  author =       {Kallus, Nathan and Uehara, Masatoshi},
  booktitle = 	 {Proceedings of the 37th International Conference on Machine Learning},
  pages = 	 {5089--5100},
  year = 	 {2020},
  editor = 	 {III, Hal Daum{\'e} and Singh, Aarti},
  volume = 	 {119},
  series = 	 {Proceedings of Machine Learning Research},
  month = 	 {13--18 Jul},
  publisher =    {PMLR},
  url = 	 {https://proceedings.mlr.press/v119/kallus20c.html}
}

@inproceedings{kallusUehara2020DeterministicPolicyDR,
 author = {Kallus, Nathan and Uehara, Masatoshi},
 booktitle = {Advances in Neural Information Processing Systems},
 editor = {H. Larochelle and M. Ranzato and R. Hadsell and M.F. Balcan and H. Lin},
 pages = {10420--10430},
 publisher = {Curran Associates, Inc.},
 title = {Doubly Robust Off-Policy Value and Gradient Estimation for Deterministic Policies},
 url = {https://proceedings.neurips.cc/paper_files/paper/2020/file/75df63609809c7a2052fdffe5c00a84e-Paper.pdf},
 volume = {33},
 year = {2020}
}

@article{kallusUehara2022BreakingHorizonDRL,
  author       = {Nathan Kallus and
                  Masatoshi Uehara},
  title        = {Efficiently Breaking the Curse of Horizon in Off-Policy Evaluation
                  with Double Reinforcement Learning},
  journal      = {Oper. Res.},
  volume       = {70},
  number       = {6},
  pages        = {3282--3302},
  year         = {2022},
  url          = {https://doi.org/10.1287/opre.2021.2249},
  doi          = {10.1287/OPRE.2021.2249},
  bibsource    = {dblp computer science bibliography, https://dblp.org}
}

@article{kallusUehara2024NaturalPolicies,
  author = {Kallus, Nathan and Uehara, Masatoshi},
  title = {Efficient Evaluation of Natural Stochastic Policies in Off-Line Reinforcement Learning},
  journal = {Biometrika},
  volume = {111},
  number = {1},
  pages = {51--69},
  year = {2024},
  doi = {10.1093/biomet/asad059},
  url = {https://academic.oup.com/biomet/article/111/1/51/7284104}
}

@article{ueharaShiKallus2022OPEReview,
  author       = {Masatoshi Uehara and
                  Chengchun Shi and
                  Nathan Kallus},
  title        = {A Review of Off-Policy Evaluation in Reinforcement Learning},
  journal      = {CoRR},
  volume       = {abs/2212.06355},
  year         = {2022},
  url          = {https://doi.org/10.48550/arXiv.2212.06355},
  doi          = {10.48550/ARXIV.2212.06355},
  eprinttype   = {arXiv},
  eprint       = {2212.06355},
  bibsource    = {dblp computer science bibliography, https://dblp.org}
}

@article{vanDerLaanEtAl2025AutomaticDRL,
  title={Semiparametric double reinforcement learning with applications to long-term causal inference},
  author={van der Laan, Lars and Hubbard, David and Tran, Allen and Kallus, Nathan and Bibaut, Aur{\'e}lien},
  journal={arXiv preprint arXiv:2501.06926},
  year={2025}
}

@inproceedings{riedmiller2005neural,
  title={Neural fitted Q iteration--first experiences with a data efficient neural reinforcement learning method},
  author={Riedmiller, Martin},
  booktitle={European conference on machine learning},
  pages={317--328},
  year={2005},
  organization={Springer}
}

@inproceedings{tosatto2017boosted,
  title={Boosted fitted q-iteration},
  author={Tosatto, Samuele and Pirotta, Matteo and d’Eramo, Carlo and Restelli, Marcello},
  booktitle={International Conference on Machine Learning},
  pages={3434--3443},
  year={2017},
  organization={PMLR}
}

@article{vanDerLaanKallusBibaut2025ClassificationIRL,
  author       = {Lars van der Laan and
                  Nathan Kallus and
                  Aur{\'{e}}lien Bibaut},
  title        = {Inverse Reinforcement Learning Using Just Classification and a Few
                  Regressions},
  journal      = {CoRR},
  volume       = {abs/2509.21172},
  year         = {2025},
  url          = {https://doi.org/10.48550/arXiv.2509.21172},
  doi          = {10.48550/ARXIV.2509.21172},
  eprinttype   = {arXiv},
  eprint       = {2509.21172},
  bibsource    = {dblp computer science bibliography, https://dblp.org}
}

@article{vanDerLaanBibautKallus2025EfficientIRL,
  author       = {Lars van der Laan and
                  Aur{\'{e}}lien Bibaut and
                  Nathan Kallus},
  title        = {Efficient Inference for Inverse Reinforcement Learning and Dynamic
                  Discrete Choice Models},
  journal      = {CoRR},
  volume       = {abs/2512.24407},
  year         = {2025},
  url          = {https://doi.org/10.48550/arXiv.2512.24407},
  doi          = {10.48550/ARXIV.2512.24407},
  eprinttype   = {arXiv},
  eprint       = {2512.24407},
  bibsource    = {dblp computer science bibliography, https://dblp.org}
}

@inproceedings{liuEtAl2018InfiniteHorizonOPE,
 author = {Liu, Qiang and Li, Lihong and Tang, Ziyang and Zhou, Dengyong},
 booktitle = {Advances in Neural Information Processing Systems},
 editor = {S. Bengio and H. Wallach and H. Larochelle and K. Grauman and N. Cesa-Bianchi and R. Garnett},
 pages = {},
 publisher = {Curran Associates, Inc.},
 title = {Breaking the Curse of Horizon: Infinite-Horizon Off-Policy Estimation},
 url = {https://proceedings.neurips.cc/paper_files/paper/2018/file/dda04f9d634145a9c68d5dfe53b21272-Paper.pdf},
 volume = {31},
 year = {2018}
}

@InProceedings{hallakMannor2017COPTD,
  title = 	 {Consistent On-Line Off-Policy Evaluation},
  author = 	 {Assaf Hallak and Shie Mannor},
  booktitle = 	 {Proceedings of the 34th International Conference on Machine Learning},
  pages = 	 {1372--1383},
  year = 	 {2017},
  editor = 	 {Precup, Doina and Teh, Yee Whye},
  volume = 	 {70},
  series = 	 {Proceedings of Machine Learning Research},
  month = 	 {06--11 Aug},
  publisher =    {PMLR},
  url = 	 {https://proceedings.mlr.press/v70/hallak17a.html}
}

@article{suttonEtAl2016EmphaticTD,
  author  = {Richard S. Sutton and A. Rupam Mahmood and Martha White},
  title   = {An Emphatic Approach to the Problem of Off-policy Temporal-Difference Learning},
  journal = {Journal of Machine Learning Research},
  year    = {2016},
  volume  = {17},
  number  = {73},
  pages   = {1--29},
  url     = {http://jmlr.org/papers/v17/14-488.html}
}

@inproceedings{geladaBellemare2019CovariateShift,
  author       = {Carles Gelada and
                  Marc G. Bellemare},
  title        = {Off-Policy Deep Reinforcement Learning by Bootstrapping the Covariate
                  Shift},
  booktitle    = {The Thirty-Third {AAAI} Conference on Artificial Intelligence, {AAAI}
                  2019, The Thirty-First Innovative Applications of Artificial Intelligence
                  Conference, {IAAI} 2019, The Ninth {AAAI} Symposium on Educational
                  Advances in Artificial Intelligence, {EAAI} 2019, Honolulu, Hawaii,
                  USA, January 27 - February 1, 2019},
  pages        = {3647--3655},
  publisher    = {{AAAI} Press},
  year         = {2019},
  url          = {https://doi.org/10.1609/aaai.v33i01.33013647},
  doi          = {10.1609/AAAI.V33I01.33013647},
  bibsource    = {dblp computer science bibliography, https://dblp.org}
}

@inproceedings{nachumEtAl2019DualDICE,
 author = {Nachum, Ofir and Chow, Yinlam and Dai, Bo and Li, Lihong},
 booktitle = {Advances in Neural Information Processing Systems},
 editor = {H. Wallach and H. Larochelle and A. Beygelzimer and F. d\textquotesingle Alch\'{e}-Buc and E. Fox and R. Garnett},
 pages = {},
 publisher = {Curran Associates, Inc.},
 title = {DualDICE: Behavior-Agnostic Estimation of Discounted Stationary Distribution Corrections},
 url = {https://proceedings.neurips.cc/paper_files/paper/2019/file/cf9a242b70f45317ffd281241fa66502-Paper.pdf},
 volume = {32},
 year = {2019}
}

@inproceedings{zhangEtAl2020GenDICE,
  author       = {Ruiyi Zhang and
                  Bo Dai and
                  Lihong Li and
                  Dale Schuurmans},
  title        = {GenDICE: Generalized Offline Estimation of Stationary Values},
  booktitle    = {8th International Conference on Learning Representations, {ICLR} 2020,
                  Addis Ababa, Ethiopia, April 26-30, 2020},
  publisher    = {OpenReview.net},
  year         = {2020},
  url          = {https://openreview.net/forum?id=HkxlcnVFwB},
  bibsource    = {dblp computer science bibliography, https://dblp.org}
}

@InProceedings{zhangEtAl2020GradientDICE,
  title = 	 {{G}radient{DICE}: Rethinking Generalized Offline Estimation of Stationary Values},
  author =       {Zhang, Shangtong and Liu, Bo and Whiteson, Shimon},
  booktitle = 	 {Proceedings of the 37th International Conference on Machine Learning},
  pages = 	 {11194--11203},
  year = 	 {2020},
  editor = 	 {III, Hal Daumé and Singh, Aarti},
  volume = 	 {119},
  series = 	 {Proceedings of Machine Learning Research},
  month = 	 {13--18 Jul},
  publisher =    {PMLR},
  url = 	 {https://proceedings.mlr.press/v119/zhang20r.html}
}

@InProceedings{ueharaEtAl2020MWLMQL,
  title = 	 {Minimax Weight and Q-Function Learning for Off-Policy Evaluation},
  author =       {Uehara, Masatoshi and Huang, Jiawei and Jiang, Nan},
  booktitle = 	 {Proceedings of the 37th International Conference on Machine Learning},
  pages = 	 {9659--9668},
  year = 	 {2020},
  editor = 	 {III, Hal Daumé and Singh, Aarti},
  volume = 	 {119},
  series = 	 {Proceedings of Machine Learning Research},
  month = 	 {13--18 Jul},
  publisher =    {PMLR},
  url = 	 {https://proceedings.mlr.press/v119/uehara20a.html}
}

@article{ueharaEtAl2021FiniteSampleMinimax,
  author       = {Masatoshi Uehara and
                  Masaaki Imaizumi and
                  Nan Jiang and
                  Nathan Kallus and
                  Wen Sun and
                  Tengyang Xie},
  title        = {Finite Sample Analysis of Minimax Offline Reinforcement Learning:
                  Completeness, Fast Rates and First-Order Efficiency},
  journal      = {CoRR},
  volume       = {abs/2102.02981},
  year         = {2021},
  url          = {https://arxiv.org/abs/2102.02981},
  eprinttype   = {arXiv},
  eprint       = {2102.02981},
  bibsource    = {dblp computer science bibliography, https://dblp.org}
}

@inproceedings{yangEtAl2020RegularizedLagrangianDICE,
 author = {Yang, Mengjiao and Nachum, Ofir and Dai, Bo and Li, Lihong and Schuurmans, Dale},
 booktitle = {Advances in Neural Information Processing Systems},
 editor = {H. Larochelle and M. Ranzato and R. Hadsell and M.F. Balcan and H. Lin},
 pages = {6551--6561},
 publisher = {Curran Associates, Inc.},
 title = {Off-Policy Evaluation via the Regularized Lagrangian},
 url = {https://proceedings.neurips.cc/paper_files/paper/2020/file/488e4104520c6aab692863cc1dba45af-Paper.pdf},
 volume = {33},
 year = {2020}
}

@inproceedings{daiEtAl2020CoinDICE,
 author = {Dai, Bo and Nachum, Ofir and Chow, Yinlam and Li, Lihong and Szepesvari, Csaba and Schuurmans, Dale},
 booktitle = {Advances in Neural Information Processing Systems},
 editor = {H. Larochelle and M. Ranzato and R. Hadsell and M.F. Balcan and H. Lin},
 pages = {9398--9411},
 publisher = {Curran Associates, Inc.},
 title = {CoinDICE: Off-Policy Confidence Interval Estimation},
 url = {https://proceedings.neurips.cc/paper_files/paper/2020/file/6aaba9a124857622930ca4e50f5afed2-Paper.pdf},
 volume = {33},
 year = {2020}
}

@article{cheEtAl2025AVGDICE,
    title={{AVG-DICE}: {S}tationary Distribution Correction by Regression},
    author={Che, Fengdi and Chan, Bryan and Ma, Chen and Mahmood, A. Rupam},
    journal={Reinforcement Learning Journal},
    volume={6},
    pages={2415--2426},
    year={2025}
}

@article{nachumEtAl2019AlgaeDICE,
  author       = {Ofir Nachum and
                  Bo Dai and
                  Ilya Kostrikov and
                  Yinlam Chow and
                  Lihong Li and
                  Dale Schuurmans},
  title        = {AlgaeDICE: Policy Gradient from Arbitrary Experience},
  journal      = {CoRR},
  volume       = {abs/1912.02074},
  year         = {2019},
  url          = {http://arxiv.org/abs/1912.02074},
  eprinttype   = {arXiv},
  eprint       = {1912.02074},
  bibsource    = {dblp computer science bibliography, https://dblp.org}
}

@inproceedings{kostrikovEtAl2020ValueDICE,
  author       = {Ilya Kostrikov and
                  Ofir Nachum and
                  Jonathan Tompson},
  title        = {Imitation Learning via Off-Policy Distribution Matching},
  booktitle    = {8th International Conference on Learning Representations, {ICLR} 2020,
                  Addis Ababa, Ethiopia, April 26-30, 2020},
  publisher    = {OpenReview.net},
  year         = {2020},
  url          = {https://openreview.net/forum?id=Hyg-JC4FDr},
  bibsource    = {dblp computer science bibliography, https://dblp.org}
}

@InProceedings{leeEtAl2021OptiDICE,
  title = 	 {OptiDICE: Offline Policy Optimization via Stationary Distribution Correction Estimation},
  author =       {Lee, Jongmin and Jeon, Wonseok and Lee, Byungjun and Pineau, Joelle and Kim, Kee-Eung},
  booktitle = 	 {Proceedings of the 38th International Conference on Machine Learning},
  pages = 	 {6120--6130},
  year = 	 {2021},
  editor = 	 {Meila, Marina and Zhang, Tong},
  volume = 	 {139},
  series = 	 {Proceedings of Machine Learning Research},
  month = 	 {18--24 Jul},
  publisher =    {PMLR},
  url = 	 {https://proceedings.mlr.press/v139/lee21f.html}
}

@inproceedings{leeEtAl2022COptiDICE,
  author       = {Jongmin Lee and
                  Cosmin Paduraru and
                  Daniel J. Mankowitz and
                  Nicolas Heess and
                  Doina Precup and
                  Kee{-}Eung Kim and
                  Arthur Guez},
  title        = {COptiDICE: Offline Constrained Reinforcement Learning via Stationary
                  Distribution Correction Estimation},
  booktitle    = {The Tenth International Conference on Learning Representations, {ICLR}
                  2022, Virtual Event, April 25-29, 2022},
  publisher    = {OpenReview.net},
  year         = {2022},
  url          = {https://openreview.net/forum?id=FLA55mBee6Q},
  bibsource    = {dblp computer science bibliography, https://dblp.org}
}

@InProceedings{maEtAl2022SMODICE,
  title = 	 {Versatile Offline Imitation from Observations and Examples via Regularized State-Occupancy Matching},
  author =       {Ma, Yecheng and Shen, Andrew and Jayaraman, Dinesh and Bastani, Osbert},
  booktitle = 	 {Proceedings of the 39th International Conference on Machine Learning},
  pages = 	 {14639--14663},
  year = 	 {2022},
  editor = 	 {Chaudhuri, Kamalika and Jegelka, Stefanie and Song, Le and Szepesvari, Csaba and Niu, Gang and Sabato, Sivan},
  volume = 	 {162},
  series = 	 {Proceedings of Machine Learning Research},
  month = 	 {17--23 Jul},
  publisher =    {PMLR},
  url = 	 {https://proceedings.mlr.press/v162/ma22a.html}
}

@article{kimEtAl2022LobsDICE,
  author       = {Geon{-}Hyeong Kim and
                  Jongmin Lee and
                  Youngsoo Jang and
                  Hongseok Yang and
                  Kee{-}Eung Kim},
  title        = {LobsDICE: Offline Learning from Observation via Stationary
                  Distribution Correction Estimation},
  journal      = {CoRR},
  volume       = {abs/2202.13536},
  year         = {2022},
  url          = {https://arxiv.org/abs/2202.13536},
  eprinttype   = {arXiv},
  eprint       = {2202.13536},
  bibsource    = {dblp computer science bibliography, https://dblp.org}
}

@article{lagoudakisParr2003LSPI,
  author       = {Michail G. Lagoudakis and
                  Ronald Parr},
  title        = {Least-Squares Policy Iteration},
  journal      = {J. Mach. Learn. Res.},
  volume       = {4},
  pages        = {1107--1149},
  year         = {2003},
  url          = {https://jmlr.org/papers/v4/lagoudakis03a.html},
  bibsource    = {dblp computer science bibliography, https://dblp.org}
}

@article{ernstEtAl2005TreeBatchRL,
  author  = {Damien  Ernst and Pierre  Geurts and Louis  Wehenkel},
  title   = {Tree-Based Batch Mode Reinforcement Learning},
  journal = {Journal of Machine Learning Research},
  year    = {2005},
  volume  = {6},
  number  = {18},
  pages   = {503--556},
  url     = {http://jmlr.org/papers/v6/ernst05a.html}
}

@inproceedings{antosEtAl2007FQI,
 author = {Antos, Andr\'{a}s and Szepesv\'{a}ri, Csaba and Munos, R\'{e}mi},
 booktitle = {Advances in Neural Information Processing Systems},
 editor = {J. Platt and D. Koller and Y. Singer and S. Roweis},
 pages = {},
 publisher = {Curran Associates, Inc.},
 title = {Fitted Q-iteration in continuous action-space MDPs},
 url = {https://proceedings.neurips.cc/paper_files/paper/2007/file/da0d1111d2dc5d489242e60ebcbaf988-Paper.pdf},
 volume = {20},
 year = {2007}
}

@book{bogachev2007measure,
  title={Measure Theory},
  author={Bogachev, Vladimir I.},
  year={2007},
  publisher={Springer},
  address={Berlin, Heidelberg},
  doi={10.1007/978-3-540-34514-5},
  url={https://doi.org/10.1007/978-3-540-34514-5}
}

@article{huangChenJiang2023DensityFeatures,
  title={Reinforcement Learning in Low-Rank MDPs with Density Features},
  author={Huang, Audrey and Chen, Jinglin and Jiang, Nan},
  journal={arXiv preprint arXiv:2302.02252},
  year={2023},
  url={https://arxiv.org/abs/2302.02252}
}

@inproceedings{huangJiang2024OccupancyPG,
  title={Occupancy-based Policy Gradient: Estimation, Convergence, and Optimality},
  author={Huang, Audrey and Jiang, Nan},
  booktitle={Advances in Neural Information Processing Systems},
  year={2024},
  url={https://proceedings.neurips.cc/paper_files/paper/2024/file/010c855df402b443e0c16e5b7434e74c-Paper-Conference.pdf}
}

@article{munosSzepesvari2008FVI,
  author  = {R{{\'e}}mi Munos and Csaba Szepesv{{\'a}}ri},
  title   = {Finite-Time Bounds for Fitted Value Iteration},
  journal = {Journal of Machine Learning Research},
  year    = {2008},
  volume  = {9},
  number  = {27},
  pages   = {815--857},
  url     = {http://jmlr.org/papers/v9/munos08a.html}
}

@InProceedings{leEtAl2019BatchPolicyConstraints,
  title = 	 {Batch Policy Learning under Constraints},
  author =       {Le, Hoang and Voloshin, Cameron and Yue, Yisong},
  booktitle = 	 {Proceedings of the 36th International Conference on Machine Learning},
  pages = 	 {3703--3712},
  year = 	 {2019},
  editor = 	 {Chaudhuri, Kamalika and Salakhutdinov, Ruslan},
  volume = 	 {97},
  series = 	 {Proceedings of Machine Learning Research},
  month = 	 {09--15 Jun},
  publisher =    {PMLR},
  url = 	 {https://proceedings.mlr.press/v97/le19a.html}
}

@article{vanDerLaanKallus2025StationaryWeightedFQE,
  author       = {Lars van der Laan and
                  Nathan Kallus},
  title        = {Fitted {Q} Evaluation Without Bellman Completeness via Stationary
                  Weighting},
  journal      = {CoRR},
  volume       = {abs/2512.23805},
  year         = {2025},
  url          = {https://doi.org/10.48550/arXiv.2512.23805},
  doi          = {10.48550/ARXIV.2512.23805},
  eprinttype   = {arXiv},
  eprint       = {2512.23805},
  bibsource    = {dblp computer science bibliography, https://dblp.org}
}

@article{vanDerLaanKallus2025SoftFQI,
  author       = {Lars van der Laan and
                  Nathan Kallus},
  title        = {Stationary Reweighting Yields Local Convergence of Soft Fitted Q-Iteration},
  journal      = {CoRR},
  volume       = {abs/2512.23927},
  year         = {2025},
  url          = {https://doi.org/10.48550/arXiv.2512.23927},
  doi          = {10.48550/ARXIV.2512.23927},
  eprinttype   = {arXiv},
  eprint       = {2512.23927},
  bibsource    = {dblp computer science bibliography, https://dblp.org}
}

@article{pattersonEtAl2022GeneralizedPBE,
  author  = {Andrew Patterson and Adam White and Martha White},
  title   = {A Generalized Projected Bellman Error for Off-policy Value Estimation in Reinforcement Learning},
  journal = {Journal of Machine Learning Research},
  year    = {2022},
  volume  = {23},
  number  = {145},
  pages   = {1--61},
  url     = {http://jmlr.org/papers/v23/21-037.html}
}

@book{puterman1994MDP, title={Markov Decision Processes: Discrete Stochastic Dynamic Programming}, ISBN={9780470316887}, ISSN={1940-6347}, url={http://dx.doi.org/10.1002/9780470316887}, DOI={10.1002/9780470316887}, journal={Wiley Series in Probability and Statistics}, publisher={Wiley}, author={Puterman, Martin L.}, year={1994}, month=Apr }

@book{meynTweedieGlynn2009MarkovChains, title={Markov Chains and Stochastic Stability}, ISBN={9780511626630}, url={http://dx.doi.org/10.1017/CBO9780511626630}, DOI={10.1017/cbo9780511626630}, publisher={Cambridge University Press}, author={Meyn, Sean P. and Tweedie, Richard L.}, year={2009}, month=Apr }

@article{bartlettEtAl2005LocalRademacher,
  author = {Bartlett, Peter L. and Bousquet, Olivier and Mendelson, Shahar},
  title = {Local Rademacher Complexities},
  journal = {The Annals of Statistics},
  volume = {33},
  number = {4},
  pages = {1497--1537},
  year = {2005},
  doi = {10.1214/009053605000000282},
  url = {https://doi.org/10.1214/009053605000000282}
}

@article{bousquet2002bennett,
  title = {A Bennett concentration inequality and its application to suprema of empirical processes},
  author = {Bousquet, Olivier},
  journal = {Comptes Rendus Mathematique},
  volume = {334},
  number = {6},
  pages = {495--500},
  year = {2002},
  publisher = {Elsevier}
}

@article{vanDerVaartWellner2011LocalMaximal,
  author = {van der Vaart, Aad W. and Wellner, Jon A.},
  title = {A Local Maximal Inequality under Uniform Entropy},
  journal = {Electronic Journal of Statistics},
  volume = {5},
  pages = {192--203},
  year = {2011}
}

@article{nicklPotscher2007BracketingMetricEntropy,
  author = {Nickl, Richard and P{\"o}tscher, Benedikt M.},
  title = {Bracketing Metric Entropy Rates and Empirical Central Limit Theorems for Function Classes of {Besov}- and {Sobolev}-Type},
  journal = {Journal of Theoretical Probability},
  volume = {20},
  number = {2},
  pages = {177--199},
  year = {2007}
}

@book{wainwright2019HighDimensionalStatistics,
  author = {Wainwright, Martin J.},
  title = {High-Dimensional Statistics: A Non-Asymptotic Viewpoint},
  series = {Cambridge Series in Statistical and Probabilistic Mathematics},
  volume = {48},
  publisher = {Cambridge University Press},
  year = {2019},
  doi = {10.1017/9781108627771},
  url = {https://doi.org/10.1017/9781108627771}
}

@book{coverThomas2006Elements,
  author = {Cover, Thomas M. and Thomas, Joy A.},
  title = {Elements of Information Theory},
  edition = {Second},
  publisher = {Wiley-Interscience},
  year = {2006},
  isbn = {9780471241959}
}

@book{brezis2011FunctionalAnalysis,
  author = {Brezis, Haim},
  title = {Functional Analysis, Sobolev Spaces and Partial Differential Equations},
  series = {Universitext},
  publisher = {Springer},
  address = {New York, NY},
  year = {2011},
  doi = {10.1007/978-0-387-70914-7}
}

@article{raginsky2014StrongDataProcessing,
  author = {Raginsky, Maxim},
  title = {Strong Data Processing Inequalities and {$\Phi$}-{S}obolev Inequalities for Discrete Channels},
  journal = {CoRR},
  volume = {abs/1411.3575},
  year = {2014},
  url = {https://arxiv.org/abs/1411.3575},
  eprinttype = {arXiv},
  eprint = {1411.3575}
}

@article{dikkala2020minimax,
  title={Minimax estimation of conditional moment models},
  author={Dikkala, Nishanth and Lewis, Greg and Mackey, Lester and Syrgkanis, Vasilis},
  journal={Advances in Neural Information Processing Systems},
  volume={33},
  pages={12248--12262},
  year={2020}
}

@article{bennett2025inference,
  title={Inference on strongly identified functionals of weakly identified functions},
  author={Bennett, Andrew and Kallus, Nathan and Mao, Xiaojie and Newey, Whitney K and Syrgkanis, Vasilis and Uehara, Masatoshi},
  journal={Journal of the Royal Statistical Society Series B: Statistical Methodology},
  pages={qkaf075},
  year={2025},
  publisher={Oxford University Press UK}
}

@article{bennett2023source,
  title={Source condition double robust inference on functionals of inverse problems},
  author={Bennett, Andrew and Kallus, Nathan and Mao, Xiaojie and Newey, Whitney and Syrgkanis, Vasilis and Uehara, Masatoshi},
  journal={arXiv preprint arXiv:2307.13793},
  year={2023}
}

@InProceedings{fujimotoEtAl2021SRDICE,
  title = {A Deep Reinforcement Learning Approach to Marginalized Importance Sampling with the Successor Representation},
  author = {Fujimoto, Scott and Meger, David and Precup, Doina},
  booktitle = {Proceedings of the 38th International Conference on Machine Learning},
  pages = {3518--3529},
  year = {2021},
  editor = {Meila, Marina and Zhang, Tong},
  volume = {139},
  series = {Proceedings of Machine Learning Research},
  month = {18--24 Jul},
  publisher = {PMLR},
  url = {https://proceedings.mlr.press/v139/fujimoto21a.html}
}

@InProceedings{pavseHanna2023ScalingMIS,
  title = {Scaling Marginalized Importance Sampling to High-Dimensional State-Spaces via State Abstraction},
  author = {Pavse, Brahma S. and Hanna, Josiah P.},
  booktitle = {Proceedings of the 37th AAAI Conference on Artificial Intelligence},
  year = {2023},
  month = {February},
  location = {Washington, DC, USA}
}

@article{banerjee2005clustering,
  title={Clustering with Bregman divergences},
  author={Banerjee, Arindam and Merugu, Srujana and Dhillon, Inderjit S and Ghosh, Joydeep},
  journal={Journal of machine learning research},
  volume={6},
  number={Oct},
  pages={1705--1749},
  year={2005}
}

@article{csiszar1975divergence,
  title={I-divergence geometry of probability distributions and minimization problems},
  author={Csisz{\'a}r, Imre},
  journal={The annals of probability},
  pages={146--158},
  year={1975},
  publisher={JSTOR}
}

@inproceedings{baird1995residual,
  author    = {Baird, Leemon},
  title     = {Residual Algorithms: Reinforcement Learning with Function Approximation},
  booktitle = {Proceedings of the Twelfth International Conference on Machine Learning},
  pages     = {30--37},
  year      = {1995},
  publisher = {Morgan Kaufmann}
}

@article{foster2021offline,
  title={Offline reinforcement learning: Fundamental barriers for value function approximation},
  author={Foster, Dylan J and Krishnamurthy, Akshay and Simchi-Levi, David and Xu, Yunzong},
  journal={arXiv preprint arXiv:2111.10919},
  year={2021}
}

@article{amortila2020variant,
  title={A variant of the wang-foster-kakade lower bound for the discounted setting},
  author={Amortila, Philip and Jiang, Nan and Xie, Tengyang},
  journal={arXiv preprint arXiv:2011.01075},
  year={2020}
}

@article{wang2021exponential,
  title={An exponential lower bound for linearly realizable mdp with constant suboptimality gap},
  author={Wang, Yuanhao and Wang, Ruosong and Kakade, Sham},
  journal={Advances in Neural Information Processing Systems},
  volume={34},
  pages={9521--9533},
  year={2021}
}

@inproceedings{chang2022learning,
  title={Learning bellman complete representations for offline policy evaluation},
  author={Chang, Jonathan and Wang, Kaiwen and Kallus, Nathan and Sun, Wen},
  booktitle={International Conference on Machine Learning},
  pages={2938--2971},
  year={2022},
  organization={PMLR}
}

@inproceedings{wang2021statistical,
  title={What are the Statistical Limits of Offline RL with Linear Function Approximation?},
  author={Wang, Ruosong and Foster, Dean and Kakade, Sham M},
  booktitle={International Conference on Learning Representations},
  year={2021},
  url={https://arxiv.org/abs/2010.11895}
}

@article{tsitsiklisVanRoy1997Analysis,
  title={An Analysis of Temporal-Difference Learning with Function Approximation},
  author={Tsitsiklis, John N. and Van Roy, Benjamin},
  journal={IEEE Transactions on Automatic Control},
  volume={42},
  number={5},
  pages={674--690},
  year={1997},
  doi={10.1109/9.580874}
}

@article{hu2025fast,
  title={Fast rates for the regret of offline reinforcement learning},
  author={Hu, Yichun and Kallus, Nathan and Uehara, Masatoshi},
  journal={Mathematics of Operations Research},
  volume={50},
  number={1},
  pages={633--655},
  year={2025},
  publisher={INFORMS}
}

@article{mnih2013playing,
  title={Playing atari with deep reinforcement learning},
  author={Mnih, Volodymyr and Kavukcuoglu, Koray and Silver, David and Graves, Alex and Antonoglou, Ioannis and Wierstra, Daan and Riedmiller, Martin},
  journal={arXiv preprint arXiv:1312.5602},
  year={2013}
}

\appendix

The appendix is organized by proof role. Appendix~\ref{app:technical-tools}
records the empirical-process and concentration tools used repeatedly below.
Appendix~\ref{app:kl-fori-proofs} gives the core KL-\textsc{FORE}
identification and projection lemmas, and Appendix~\ref{app:fitted-kl-proofs}
proves the fitted KL projection bound. Appendix~\ref{app:policy-evaluation-proofs}
then proves the policy-evaluation consequences in
Section~\ref{sec:applications-policy-evaluation}.
Appendix~\ref{app:coverage-truncated-fore-proofs} establishes the
coverage-stopped population theory, while Appendix~\ref{app:clip-finite}
collects the clipped-occupancy and finite-sample results used to
prove Theorem~\ref{thm:clip-finite-fore}.
Appendix~\ref{app:mixed-contraction}
contains the undiscounted KL contraction result under a one-step strong
data-processing condition. Appendix~\ref{app:strong-form-regression-fori}
contains the backward-regression variant of \textsc{FORE}; its \(L^1(\nu)\),
or total-variation, contraction and adjoint-completeness limitation are kept
separate from the KL-\textsc{FORE} proofs. Appendix~\ref{app:experiment-details}
records the numerical constructions used in Section~\ref{sec:experiments}.

\section{Technical tools used in the proofs}
\label{app:technical-tools}

This section records the empirical-process and concentration inequalities used
to prove the finite-sample theory in
Sections~\ref{sec:fitted-kl-fori} and~\ref{sec:coverage-clipped-finite-sample}.
Throughout, \(Z_1,\ldots,Z_n\) are independent observations with common
law \(P\), \(\sigma_1,\ldots,\sigma_n\) are independent Rademacher variables,
and \(P_n=n^{-1}\sum_{i=1}^n\delta_{Z_i}\).

The lemmas are the Ledoux--Talagrand contraction inequality
\citep[Chapter~5]{wainwright2019HighDimensionalStatistics}, Bousquet's version
of Talagrand's maximal inequality for empirical processes
\citep{bousquet2002bennett}, and the scalar Bernstein inequality
\citep[Chapter~2]{wainwright2019HighDimensionalStatistics}.

\begin{lemma}[Rademacher contraction]
\label{lem:tool-rademacher-contraction}
Let \(\mathcal G\) be a class of measurable real-valued functions and let
\(\varphi_i:\mathbb R\to\mathbb R\), \(i=1,\ldots,n\), be \(L\)-Lipschitz
functions with \(\varphi_i(0)=0\). Then, conditionally on
\(Z_1,\ldots,Z_n\),
\[
    E_\sigma\sup_{g\in\mathcal G}
    \left|
        \frac1n\sum_{i=1}^n\sigma_i\varphi_i(g(Z_i))
    \right|
    \le
    2L
    E_\sigma\sup_{g\in\mathcal G}
    \left|
        \frac1n\sum_{i=1}^n\sigma_i g(Z_i)
    \right|.
\]
The same bound holds for a common Lipschitz map \(\varphi\) applied pointwise.
\end{lemma}

\begin{lemma}[Bousquet's inequality]
\label{lem:tool-bousquet}
Let \(\mathcal G\) be a countable class of measurable functions satisfying
\(Pg=0\), \(\|g\|_\infty\le b\), and \(Pg^2\le v\) for all
\(g\in\mathcal G\). Then, for every \(u\ge0\), with probability at least
\(1-e^{-u}\),
\[
    \sup_{g\in\mathcal G}|(P_n-P)g|
    \le
    E\sup_{g\in\mathcal G}|(P_n-P)g|
    +
    \sqrt{
        \frac{2u}{n}
        \left\{
            v+2bE\sup_{g\in\mathcal G}|(P_n-P)g|
        \right\}
    }
    +
    \frac{bu}{3n}.
\]
\end{lemma}

\begin{lemma}[Bernstein's inequality]
\label{lem:tool-bernstein}
Let \(Y_1,\ldots,Y_n\) be independent mean-zero variables with
\(|Y_i|\le b\) almost surely and \(n^{-1}\sum_{i=1}^nE Y_i^2\le v\). Then, for
every \(u\ge0\), with probability at least \(1-e^{-u}\),
\[
    \left|\frac1n\sum_{i=1}^nY_i\right|
    \le
    \sqrt{\frac{2vu}{n}}+\frac{bu}{3n}.
\]
\end{lemma}

The following is a standard local Rademacher-complexity bound based on
Dudley's entropy integral; see \citet{bartlettEtAl2005LocalRademacher} and
\citet[Chapter~14]{wainwright2019HighDimensionalStatistics}.

\begin{lemma}[Localized entropy bound for Rademacher averages]
\label{lem:tool-local-entropy}
Let \(\mathcal G\) be a uniformly bounded class and suppose that, uniformly over
probability measures \(Q\),
\[
    \log N\{\epsilon,\mathcal G,L^2(Q)\}\le H(\epsilon).
\]
Then the localized Rademacher averages used in
\eqref{eq:rad-critical-radius} are bounded, up to a universal constant, by the
corresponding Dudley integral
\[
    \mathcal R_n(\mathcal G,r;P)
    \lesssim
    \frac{1}{\sqrt n}
    \int_0^r \sqrt{1+H(\epsilon)}\,\dd\epsilon ,
\]
with the integral truncated at the uniform envelope.
\end{lemma}

\section{KL projection proofs for \textsc{FORE}}
\label{app:kl-fori-proofs}

For a finite signed measure \(\mu\), write \(\abs{\mu}\) for its total variation
measure.

\subsection{Propagation of absolute continuity}
\label{app:overlap}

\begin{lemma}[Propagation of absolute continuity]
\label{lem:ac-propagation}
Assume \(\nu\Ppi\ll\nu\). If a finite signed measure \(\mu\) satisfies
\(\abs{\mu}\ll\nu\), then \(\abs{\mu\Ppi}\ll\nu\). Consequently, under
Condition~\ref{ass:overlap}, \(d_{\pi,\gamma}\ll\nu\) for every
\(\gamma\in[0,1)\).
\end{lemma}

\begin{proof}
Let \(B\subseteq\sX\) be measurable with \(\nu(B)=0\). Since
\(\nu\Ppi\ll\nu\),
\[
    0=(\nu\Ppi)(B)=\int \Ppi(B\mid x)\,\nu(\dd x).
\]
The integrand is nonnegative, so \(\Ppi(B\mid x)=0\) for \(\nu\)-almost every
\(x\). If \(\abs{\mu}\ll\nu\), the same exceptional set is also
\(\abs{\mu}\)-null, and therefore
\[
    (\abs{\mu}\Ppi)(B)=\int \Ppi(B\mid x)\,\abs{\mu}(\dd x)=0.
\]
For every measurable \(A\), the partition definition of total variation and the
triangle inequality under the integral give
\(\abs{\mu\Ppi}(A)\le(\abs{\mu}\Ppi)(A)\). Hence
\(\abs{\mu\Ppi}(B)=0\) whenever \(B\) is \(\nu\)-null.

Under Condition~\ref{ass:overlap}, \(\dinit\ll\nu\). Applying the first part
inductively with \(\mu=\dinit\Ppi^t\) shows \(\dinit\Ppi^t\ll\nu\) for every
\(t\ge0\). The countable nonnegative mixture
\[
    d_{\pi,\gamma}
    =
    (1-\gamma)\sum_{t=0}^\infty \gamma^t\dinit\Ppi^t
\]
is therefore also absolutely continuous with respect to \(\nu\).
\end{proof}

\subsection{Adjoint Bellman moment identity}
The adjoint Bellman operator
\[
    \Bpig\omega
    =
    (1-\gamma)\omegazero
    +
    \gamma
    \frac{\dd\{(\omega\nu)\Ppi\}}{\dd\nu}
\]
can be written in measure form as
\begin{equation}
\label{eq:adjoint-flow}
    (\Bpig\omega)\nu
    =
    (1-\gamma)\dinit
    +
    \gamma(\omega\nu)\Ppi .
\end{equation}

\begin{proof}[Proof of the adjoint Bellman moment identity \eqref{eq:bellman-moment-update}]
Equation~\eqref{eq:adjoint-flow} gives
\[
\begin{aligned}
    \int f(x)(\Bpig\omega)(x)\,\nu(\dd x)
    &=
    (1-\gamma)\int f(x)\,\dinit(\dd x)
    +
    \gamma\int f(y)\{(\omega\nu)\Ppi\}(\dd y)\\
    &=
    (1-\gamma)E_{\dinit}\{f(X)\}
    +
    \gamma\int
    \left\{
        \int f(y)\Ppi(\dd y\mid x)
    \right\}
    \omega(x)\nu(\dd x)\\
    &=
    (1-\gamma)E_{\dinit}\{f(X)\}
    +
    \gamma E_\nu\{\omega(X)f(X^+)\}.
\end{aligned}
\]
This proves the adjoint Bellman moment identity
\eqref{eq:bellman-moment-update}.
\end{proof}

\begin{proof}[Proof of Lemma~\ref{lem:bellman-kl-loss}]
For any \(h\in\sH\), the terms in
\[
    D_\nu(\Bpig\omega\|\omega_h)
    =
    E_\nu\{(\Bpig\omega)(X)\log(\Bpig\omega)(X)\}
    -
    E_\nu\{(\Bpig\omega)(X)h(X)\}
    +
    \Lambda_\nu(h)
\]
that depend on \(h\) are the final two terms. Applying
\eqref{eq:bellman-moment-update} to \(f=h\) gives
\[
    E_\nu\{(\Bpig\omega)(X)h(X)\}
    =
    (1-\gamma)E_{\dinit}\{h(X)\}
    +
    \gamma E_\nu\{\omega(X)h(X^+)\}.
\]
Thus minimizing \(D_\nu(\Bpig\omega\|\omega_h)\) over \(h\in\sH\) is equivalent
to minimizing the objective stated in Lemma~\ref{lem:bellman-kl-loss}.
\end{proof}

\subsection{KL-projected \textsc{FORE}}
\label{app:kl-projected-proofs}

For the proofs in this subsection, set
\[
    \sH^\circ=\{h-E_\nu\{h(X)\}:h\in\sH\}.
\]
The centering map is continuous and linear on \(L^2(\nu)\), so
\(\sH^\circ\) is convex and compact in \(L^2(\nu)\) under
Condition~\ref{ass:kl-class}. Condition~\ref{ass:kl-bounded} gives
\[
    \sup_{h\in\sH^\circ}\|h\|_\infty\le R .
\]
Moreover \(\omega_h=\omega_{h-E_\nu\{h(X)\}}\), so
\[
    \sW=\{\omega_h:h\in\sH^\circ\}.
\]
Hence every \(\omega\in\sW\) has a centered representative
\(h\in\sH^\circ\) satisfying \(\omega=\omega_h\).

We first record the standard Pythagorean inequality for KL projections onto
normalized exponential families \citep{csiszar1975divergence,banerjee2005clustering}.

\begin{lemma}[Convex KL projection inequality]
\label{lem:kl-projection-inequality}
Assume Conditions~\ref{ass:kl-class} and~\ref{ass:kl-bounded}. Let
\(u\in\Delta_\nu\) satisfy \(E_\nu\{u\log_+u\}<\infty\). Then the map
\(h\mapsto D_\nu(u\|\omega_h)\) attains its minimum over
\(\sH^\circ\). Writing
\(\bar u=\Pi_{\sW}^{\rm KL}u=\omega_{h_u^\star}\) for any minimizer
\(h_u^\star\in\sH^\circ\) and letting \(v=\omega_g\in\sW\) with
\(g\in\sH^\circ\), we have
\[
    D_\nu(\bar u\|v)
    \le
    D_\nu(u\|v)-D_\nu(u\|\bar u).
\]
In particular, \(D_\nu(\bar u\|v)\le D_\nu(u\|v)\).
\end{lemma}

\begin{proof}
Since \(u\in\Delta_\nu\) and \(E_\nu\{u\log_+u\}<\infty\), we have
\(u\log u\in L^1(\nu)\). Since every \(h\in\sH^\circ\) is bounded,
\(h\in L^1(u\nu)\).
Minimizing \(D_\nu(u\|\omega_h)\) is therefore equivalent to minimizing
\[
    F_u(h)=\Lambda_\nu(h)-E_\nu\{u(X)h(X)\}.
\]
The set \(\sH^\circ\) is compact in \(L^2(\nu)\). If
\(h_m\to h\) in \(L^2(\nu)\) with \(h_m,h\in\sH^\circ\), then
\(\|h_m\|_\infty\vee\|h\|_\infty\le R\). Hence
\(|e^{h_m}-e^h|\le e^R|h_m-h|\), and therefore
\(\Lambda_\nu(h_m)\to \Lambda_\nu(h)\). Also
\(h_m\to h\) in \(\nu\)-measure, hence in \(u\nu\)-measure because
\(u\nu\ll\nu\). The uniform bound
\(\|h_m-h\|_\infty\le2R\) then implies
\(E_\nu\{u(X)|h_m-h|(X)\}\to0\).
Thus \(F_u\) is continuous on the compact set \(\sH^\circ\), and it attains
its minimum.

Let \(h_t=(1-t)h_u^\star+tg\). Since \(h_u^\star\) minimizes
\(F_u\) over the convex class \(\sH^\circ\), the right derivative at \(t=0\) is
nonnegative. Since \(h_u^\star\) and \(g\) are bounded by \(R\), dominated
convergence gives
\[
    \frac{\dd}{\dd t}\Lambda_\nu(h_t)
    =
    E_\nu\{\omega_{h_t}(X)(g-h_u^\star)(X)\},
    \qquad 0\le t\le1 .
\]
Therefore
\begin{equation}
\label{eq:kl-projection-first-order}
    E_\nu\{\bar u(X)(g-h_u^\star)(X)\}
    -
    E_\nu\{u(X)(g-h_u^\star)(X)\}
    \ge 0 .
\end{equation}
The normalized log-ratio form gives
\begin{equation}
\label{eq:kl-projection-three-point}
\begin{aligned}
    &D_\nu(u\|v)-D_\nu(u\|\bar u)
    -D_\nu(\bar u\|v)\\
    &\qquad =
    E_\nu\{\bar u(X)(g-h_u^\star)(X)\}
    -
    E_\nu\{u(X)(g-h_u^\star)(X)\}.
\end{aligned}
\end{equation}
Combining \eqref{eq:kl-projection-first-order} and
\eqref{eq:kl-projection-three-point} gives the stated projection inequality.
\end{proof}

\begin{lemma}[KL projection comparison with an external target]
\label{lem:kl-projection-violation}
Assume Conditions~\ref{ass:kl-class} and~\ref{ass:kl-bounded}. Let
\(u\in\Delta_\nu\) satisfy \(E_\nu\{u\log_+u\}<\infty\), let
\(\bar u=\Pi_{\sW}^{\rm KL}u\), and let \(w\in\Delta_\nu\) be positive
\(\nu\)-almost everywhere. Then, for every \(v\in\sW\),
\[
    D_\nu(\bar u\|w)
    \le
    D_\nu(u\|w)
    +
    e^{4R}D_\nu(v\|w).
\]
\end{lemma}

\begin{proof}
The projection exists by Lemma~\ref{lem:kl-projection-inequality}. Fix
\(v\in\sW\). If \(D_\nu(u\|w)=\infty\) or \(D_\nu(v\|w)=\infty\), the claim is
immediate, so assume both quantities are finite. Lemma~\ref{lem:kl-projection-inequality} gives
\[
    D_\nu(\bar u\|v)
    \le
    D_\nu(u\|v)-D_\nu(u\|\bar u).
\]
Therefore
\[
\begin{aligned}
    D_\nu(\bar u\|w)
    &=
    D_\nu(\bar u\|v)
    +
    \int \bar u(x)\log\frac{v(x)}{w(x)}\,\nu(\dd x)\\
    &\le
    D_\nu(u\|v)
    -
    D_\nu(u\|\bar u)
    +
    \int \bar u(x)\log\frac{v(x)}{w(x)}\,\nu(\dd x)\\
    &=
    D_\nu(u\|w)
    +
    \int (u-\bar u)(x)\log\frac{w(x)}{v(x)}\,\nu(\dd x)
    -
    D_\nu(u\|\bar u).
\end{aligned}
\]
Set \(\delta=\log(w/v)\). The variational inequality for KL divergence gives
\[
    E_\nu\{u(X)\delta(X)\}-D_\nu(u\|\bar u)
    \le
    \log E_\nu\{\bar u(X)e^{\delta(X)}\}.
\]
Therefore, using \(\log z\le z-1\),
\begin{equation}
\label{eq:kl-external-target-fenchel}
\begin{aligned}
    \int (u-\bar u)\delta\,\dd\nu-D_\nu(u\|\bar u)
    &\le
    \log E_\nu\{\bar u e^\delta\}-E_\nu\{\bar u\delta\} \\
    &\le
    E_\nu\{\bar u(e^\delta-1-\delta)\}.
\end{aligned}
\end{equation}
For \(\bar u=\omega_{\bar h}\) and \(v=\omega_g\) with
\(\bar h,g\in\sH^\circ\), Condition~\ref{ass:kl-bounded} gives
\(\bar u/v\le e^{4R}\). Since \(e^t-1-t\ge0\) for all \(t\),
\begin{equation}
\label{eq:kl-external-target-ratio}
    E_\nu\{\bar u(e^\delta-1-\delta)\}
    \le
    e^{4R}E_\nu\{v(e^\delta-1-\delta)\}
    =
    e^{4R}D_\nu(v\|w).
\end{equation}
Applying \eqref{eq:kl-external-target-fenchel} and
\eqref{eq:kl-external-target-ratio} to the KL decomposition proves the lemma.
\end{proof}

\begin{proof}[Proof of Theorem~\ref{thm:kl-fori-realizable}]
Fix \(\omega\in\sW\). Lemma~\ref{lem:adjoint-bellman-kl-contraction}, applied
with \(\widetilde\omega=\omegastar\) and using
\(\Bpig\omegastar=\omegastar\), gives
\[
    D_\nu(\Bpig\omega\|\omegastar)
    \le
    \gamma D_\nu(\omega\|\omegastar).
\]
For any \(\omega\in\sW\), choose \(h\in\sH^\circ\) such that
\(\omega=\omega_h\). Then \(\|h\|_\infty\le R\), so
\(\Lambda_\nu(h)=\log E_\nu e^h\in[-R,R]\), and
\(\omega=e^{h-\Lambda_\nu(h)}\le e^{2R}\). Hence, if
\[
    r_\omega
    =
    \frac{\dd\{(\omega\nu)\Ppi\}}{\dd\nu},
\]
then \(r_\omega\le e^{2R}\cpi\) \(\nu\)-almost everywhere. Therefore
\[
    u_\omega:=\Bpig\omega
    =
    (1-\gamma)\omegazero+\gamma r_\omega
\]
satisfies
\[
    u_\omega\log_+u_\omega
    \le
    C_R\{1+\omegazero\log_+\omegazero+\cpi\log_+\cpi\}
\]
for a constant \(C_R<\infty\), because \(t\log_+(at)\le
C_a\{t\log_+t+t\}\) for each fixed \(a<\infty\). Hence
\(E_\nu\{u_\omega\log_+u_\omega\}<\infty\) by
Condition~\ref{ass:kl-population-integrability}, since density ratios integrate
to one and \(t\log t\) is bounded below. Lemma~\ref{lem:kl-projection-violation}
gives, for every \(v\in\sW\),
\[
    D_\nu(\mathsf T_{\sW}^{\rm KL}\omega\|\omegastar)
    \le
    D_\nu(\Bpig\omega\|\omegastar)+e^{4R}D_\nu(v\|\omegastar).
\]
Taking the infimum over \(v\in\sW\) gives
\begin{equation}
\label{eq:kl-projected-one-step}
    D_\nu(\mathsf T_{\sW}^{\rm KL}\omega\|\omegastar)
    \le
    D_\nu(\Bpig\omega\|\omegastar)
    +
    e^{4R}\varepsilon_{\rm KL}
    \le
    \gamma D_\nu(\omega\|\omegastar)+e^{4R}\varepsilon_{\rm KL}.
\end{equation}
Set \(C_{\rm app}=e^{4R}\).
Applying this one-step inequality to
\(\omega^{(k+1)}=\mathsf T_{\sW}^{\rm KL}\omega^{(k)}\) and iterating
\eqref{eq:kl-projected-one-step} yields
\begin{equation}
\label{eq:kl-realizable-iterated}
    D_\nu(\omega^{(K)}\|\omegastar)
    \le
    \gamma^K D_\nu(\omega^{(0)}\|\omegastar)
    +
    C_{\rm app}\frac{1-\gamma^K}{1-\gamma}\varepsilon_{\rm KL}.
\end{equation}
If \(\omegastar\in\sW\), choose \(v=\omegastar\) in the approximation term, so
\(\varepsilon_{\rm KL}=0\). Since \(\Bpig\omegastar=\omegastar\), a KL
projection of \(\Bpig\omegastar\) onto \(\sW\) is \(\omegastar\) itself.
Equation~\eqref{eq:kl-realizable-iterated} with
\(\varepsilon_{\rm KL}=0\) gives the realizable contraction.
\end{proof}

\section{Fitted KL projection bounds}
\label{app:fitted-kl-proofs}

This section proves Theorem~\ref{thm:fitted-kl-fori}. We use the centered class
\(\sH^\circ=\{h-E_\nu\{h(X)\}:h\in\sH\}\) and write
\[
    a_{n,\rm fit}(\delta)
    :=
    \mathfrak r_{n,\rm fit}^2
    +
    \frac{\log(1/\delta)}{n}.
\]

Throughout this section, fix versions of the centered log-ratios that
satisfy Condition~\ref{ass:kl-bounded} on \(\mathcal X_R\) and set them equal
to zero on \(\mathcal X_R^c\). In the setting of
Theorem~\ref{thm:fitted-kl-fori}, \(1\in\mathcal W\), so
Condition~\ref{ass:fitted-kl-coverage-bounded} implies
\(\dinit\ll\nu\) and \(\nup\ll\nu\). This modification therefore leaves
all population and empirical losses unchanged almost surely and provides a
common bounded envelope under every sampling law used below.

\begin{lemma}[Bounded transition densities imply subexponential smoothing]
\label{lem:bounded-kernel-smoothing}
Suppose \(\omegazero\le L_0\) \(\nu\)-almost surely, and suppose that
\(\Ppi(\cdot\mid x)\) admits a jointly measurable density
\(p_\pi(\cdot\mid x)\) relative to \(\nu\) satisfying
\[
  \operatorname*{ess\,sup}_{(x,y)\sim\nu\otimes\nu}
  p_\pi(y\mid x)\le L_P.
\]
Then Condition~\ref{ass:fitted-kl-coverage-bounded} holds with
\[
  K_0\le \frac{L_0}{\log 2},
  \qquad
  K_+\le \frac{L_P}{\log 2}.
\]
\end{lemma}

\begin{proof}
Every \(\omega\in\sW\) is normalized under \(\nu\). Hence, by Tonelli's
theorem,
\begin{equation}
\label{eq:bounded-kernel-successor-density}
  \frac{\mathrm d\{(\omega\nu)\Ppi\}}{\mathrm d\nu}(y)
  =
  \int \omega(x)p_\pi(y\mid x)\,\dd\nu(x)
  \le
  L_P
\end{equation}
for \(\nu\)-almost every \(y\). Finally, if \(0\le Z\le L\), then
\(E\exp\{Z\log(2)/L\}\le2\), so
\(\|Z\|_{\psi_1}\le L/\log 2\). Applying this observation to
\(\omegazero\) and \eqref{eq:bounded-kernel-successor-density} proves the
result.
\end{proof}

We first define the population objective and minimizer targeted by \textsc{FORE}.
For \(f\in\sW\) and \(h\in\sH^\circ\), let
\[
    L_f(h)
    =
    \Lambda_\nu(h)
    -
    (1-\gamma)E_{\dinit}\{h(X)\}
    -
    \gamma E_\nu\{f(X)h(X^+)\},
\]
where \(X^+\mid X\sim\Ppi(\cdot\mid X)\). Let
\[
    h_f^\star\in\argmin_{h\in\sH^\circ}L_f(h),
    \qquad
    \bar u_f=\omega_{h_f^\star},
    \qquad
    \Delta_f(h)=L_f(h)-L_f(h_f^\star).
\]
For \(\omega\in\sW\), write
\(u_\omega=\Bpig\omega\). By the adjoint Bellman moment identity
\eqref{eq:bellman-moment-update}, \(L_\omega(h)\) equals
\(D_\nu(u_\omega\|\omega_h)\) up to a term that does not depend on \(h\),
and \(\bar u_\omega=\Pi_{\sW}^{\rm KL}u_\omega\).
For sample averages, write
\[
    P_{n,X}g=n^{-1}\sum_{i=1}^n g(X_i),
    \qquad
    P_{n,0}g=n^{-1}\sum_{i=1}^n g(X_i^0),
\]
and, for functions \(\varphi\) of a transition pair,
\[
    P_{n,+}\varphi=n^{-1}\sum_{i=1}^n \varphi(X_i,X_i^+).
\]
For a positive input \(f\), define the self-normalized empirical successor
average
\[
    P_{n,f}^{+}g
    =
    \frac{P_{n,+}\{f(X)g(X^+)\}}{P_{n,X}f}.
\]
With this notation,
\[
\begin{aligned}
    \widehat L_f(h)
    &=
    \widehat\Lambda_\nu(h)
    -
    (1-\gamma)P_{n,0}h \\
    &\quad
    -
    \gamma P_{n,f}^{+}h .
\end{aligned}
\]
This empirical loss satisfies \(\widehat L_f(h+c)=\widehat L_f(h)\).
Here \(P_{n,X}\) is the empirical distribution of the transition covariates,
\(P_{n,0}\) is the empirical distribution of the initial sample, and
\(P_{n,+}\) is the empirical distribution of the transition pairs.
The empirical-process events below are uniform over \(f\in\sW\) and
\(h\in\sH^\circ\). This uniformity is what permits their later use at the
data-dependent fitted inputs \(f=\omega_{\rm p}^{(k)}\), without conditioning
on a particular iterate.

\begin{lemma}[Scaling of the fitted critical radius]
\label{lem:kl-critical-radius-scaling}
Assume Conditions~\ref{ass:kl-class} and~\ref{ass:kl-bounded}. For every fixed
\(A<\infty\) and \(b>0\), there is a constant \(L_{A,b}<\infty\), depending
only on \(A\), \(b\), and \(R\), such that
\begin{equation}
\label{eq:kl-critical-radius-scaled-local}
    \mathfrak C_n(Ar)
    \le
    b r^2
    \qquad
    \text{for all } r\ge L_{A,b}\mathfrak r_{n,\rm fit}.
\end{equation}
For every fixed \(A<\infty\), there is a constant \(C_A<\infty\), depending
only on \(A\) and \(R\), such that
\begin{equation}
\label{eq:kl-critical-radius-scaled-global}
    \mathfrak C_n(A)
    \le
    C_A\mathfrak r_{n,\rm fit}.
\end{equation}
\end{lemma}

\begin{proof}
The class \(\mathcal H_\Delta\) is star-shaped. Indeed, if
\(h_\Delta=h_1-h_2\) with \(h_1,h_2\in\sH^\circ\) and \(t\in[0,1]\), then
\[
    t h_\Delta
    =
    \{t h_1+(1-t)h_2\}-h_2
\]
and the convexity of \(\sH^\circ\) implies
\(t h_1+(1-t)h_2\in\sH^\circ\). The product class
\(\mathcal G_\times\) is also star-shaped: if
\(g(x,x^+)=f(x)h_\Delta(x^+)\) and \(t\in[0,1]\), then
\[
    t g(x,x^+)
    =
    f(x)\{t h_\Delta(x^+)\},
\]
where \(t h_\Delta\in\mathcal H_\Delta\). Therefore
\(t g\in\mathcal G_\times\).

Let \(\mathcal G\) be any of the classes entering \(\mathfrak C_n\), and let
\(0<a\le b_0\). For each \(g\in\mathcal G\) with
\(\|g\|_{L^2(P)}\le b_0\), the function \((a/b_0)g\) belongs to
\(\mathcal G\) and has \(L^2(P)\)-norm at most \(a\). Hence
\[
    \mathcal R_n(\mathcal G,b_0;P)
    \le
    \frac{b_0}{a}\mathcal R_n(\mathcal G,a;P).
\]
It follows that \(s\mapsto \mathcal R_n(\mathcal G,s;P)/s\) is nonincreasing
for each such class, and therefore \(s\mapsto\mathfrak C_n(s)/s\) is
nonincreasing. The localized classes are nested in the radius, so
\(\mathfrak C_n\) is nondecreasing.

The bounded envelopes make the fixed-point set in
\eqref{eq:rad-critical-radius} nonempty for large radii: \(\mathfrak C_n(s)\)
is bounded uniformly in \(s\), while \(s^2\to\infty\). Let \(r_\star\) denote
the infimum in \eqref{eq:rad-critical-radius}. Since
\(r_\star\le\mathfrak r_{n,\rm fit}\), the definition of the infimum gives a
radius \(t\le r_\star+\mathfrak r_{n,\rm fit}\le2\mathfrak r_{n,\rm fit}\)
such that \(\mathfrak C_n(t)\le t^2\).
Let \(r\ge L\mathfrak r_{n,\rm fit}\), where \(L\ge1\) will be chosen below.
If \(Ar\ge t\), then
\[
    \mathfrak C_n(Ar)
    \le
    \frac{Ar}{t}\mathfrak C_n(t)
    \le
    Ar t
    \le
    \frac{2A}{L}r^2 .
\]
If \(Ar<t\), monotonicity gives
\[
    \mathfrak C_n(Ar)
    \le
    \mathfrak C_n(t)
    \le
    t^2
    \le
    \frac{4}{L^2}r^2 .
\]
Choosing \(L=L_{A,b}\) large enough so that \(2A/L\le b\) and
\(4/L^2\le b\) proves \eqref{eq:kl-critical-radius-scaled-local}.

To prove \eqref{eq:kl-critical-radius-scaled-global}, use the same
\(t\le2\mathfrak r_{n,\rm fit}\). If
\(A\ge t\), star-shapedness gives
\[
    \mathfrak C_n(A)
    \le
    \frac{A}{t}\mathfrak C_n(t)
    \le
    At
    \le
    2A\mathfrak r_{n,\rm fit}.
\]
If \(A<t\), then
\(\mathfrak C_n(A)\le\mathfrak C_n(t)\le4\mathfrak r_{n,\rm fit}^2\). This is
bounded by a constant times \(\mathfrak r_{n,\rm fit}\) when
\(\mathfrak r_{n,\rm fit}\le1\); when \(\mathfrak r_{n,\rm fit}>1\), the
bounded envelopes give
\(\mathfrak C_n(A)\le C_A\le C_A\mathfrak r_{n,\rm fit}\), which proves
\eqref{eq:kl-critical-radius-scaled-global}.
\end{proof}

\begin{lemma}[Curvature and variance of the fitted KL loss]
\label{lem:kl-loss-curvature}
Assume Conditions~\ref{ass:kl-class}, \ref{ass:kl-bounded}, and
\ref{ass:fitted-kl-coverage-bounded}. There is a finite constant
\(C_{\rm curv}\), depending only on \(R\), such that, for every
\(f\in\sW\), \(h\in\sH^\circ\), and \(t\ge0\), with
\(g=h-h_f^\star\),
\[
\begin{aligned}
    \|g\|_{L^2(\nu)}^2
    +\|e^h-e^{h_f^\star}\|_{L^2(\nu)}^2
    &\le C_{\rm curv}\Delta_f(h),\\
    E_{\dinit}\{g^2(X)\}
    &\le t\|g\|_{L^2(\nu)}^2+8R^2K_0e^{-t/K_0},\\
    E\{f^2(X)g^2(X^+)\}
    &\le e^{2R}\{t\|g\|_{L^2(\nu)}^2
       +8R^2K_+e^{-t/K_+}\},
\end{aligned}
\]
where the final expectation is under \(X\sim\nu\) and
\(X^+\mid X\sim\Ppi(\cdot\mid X)\).
\end{lemma}

\begin{proof}
Put \(g=h-h_f^\star\) and let
\[
    u_f(x)
    =
    (1-\gamma)\frac{\dd\dinit}{\dd\nu}(x)
    +
    \gamma\frac{\dd\{(f\nu)\Ppi\}}{\dd\nu}(x).
\]
Because \(f\in\sW\), \(u_f\) is a density with respect to \(\nu\). Since
\(\sH^\circ\) is convex and \(h_f^\star\) minimizes \(L_f\) over
\(\sH^\circ\), the one-sided directional derivative of
\(t\mapsto L_f(h_f^\star+t\{h-h_f^\star\})\) at \(t=0\) is nonnegative.
Boundedness of \(\sH^\circ\) justifies differentiating under the expectation,
so, with \(\bar u_f=\omega_{h_f^\star}\),
\[
    E_\nu\{\bar u_f(X)g(X)\}
    -
    E_\nu\{u_f(X)g(X)\}
    \ge0 .
\]
Consequently,
\[
\begin{aligned}
    \Delta_f(h)
    &=
    \Lambda_\nu(h)-\Lambda_\nu(h_f^\star)-E_\nu\{u_f(X)g(X)\} \\
    &\ge
    \Lambda_\nu(h)-\Lambda_\nu(h_f^\star)-E_\nu\{\bar u_f(X)g(X)\}
    =
    D_\nu(\bar u_f\|\omega_h).
\end{aligned}
\]
Along
\(h_t=h_f^\star+t g\), the second derivative is
\(\operatorname{Var}_{\omega_{h_t}}\{g(X)\}\). Because
\(h,h_f^\star\in\sH^\circ\), \(E_\nu\{g(X)\}=0\).
Condition~\ref{ass:kl-bounded} gives
\(e^{-2R}\le\omega_{h_t}\le e^{2R}\). Hence
\[
    \operatorname{Var}_{\omega_{h_t}}\{g(X)\}
    =
    \inf_a E_\nu\{\omega_{h_t}(X)(g(X)-a)^2\}
    \ge
    e^{-2R}\inf_a E_\nu\{(g(X)-a)^2\}
    =
    e^{-2R}\|g\|_{L^2(\nu)}^2 .
\]
Therefore
\begin{equation}
\label{eq:kl-loss-curvature-l2}
    D_\nu(\bar u_f\|\omega_h)
    =
    \int_0^1(1-t)\operatorname{Var}_{\omega_{h_t}}\{g(X)\}\,\dd t
    \ge
    \frac12 e^{-2R}\|g\|_{L^2(\nu)}^2 .
\end{equation}
Equation~\eqref{eq:kl-loss-curvature-l2} gives the \(L^2(\nu)\) control. Since
\(|e^u-e^v|\le e^R|u-v|\) for \(u,v\in[-R,R]\), it also controls the
exponential component.

It remains to prove the two transfer bounds. If \(r\ge0\) and \(q\) is
bounded, then, for every \(t\ge0\),
\begin{equation}
\label{eq:smoothed-norm-transfer}
    \int q^2r\,\dd\nu
    \le
    t\|q\|_{L^2(\nu)}^2
    +\|q\|_\infty^2E_\nu(r-t)_+ .
\end{equation}
Indeed, \(r\le t+(r-t)_+\). Moreover, if
\(\|r\|_{\psi_1}\le K\), Markov's inequality and Tonelli's theorem give
\[
    \nu(r>s)\le2e^{-s/K},
    \qquad
    E_\nu(r-t)_+
    =\int_t^\infty\nu(r>s)\,\dd s
    \le
    2Ke^{-t/K}.
\]
Apply \eqref{eq:smoothed-norm-transfer} first with
\(r=\omegazero\), \(q=g\), and \(K=K_0\). For the successor term, put
\(r_f=\dd\{(f\nu)\Ppi\}/\dd\nu\). Since
\(f^2\le e^{2R}f\),
\[
    E\{f^2(X)g^2(X^+)\}
    \le e^{2R}\int g^2r_f\,\dd\nu.
\]
Condition~\ref{ass:fitted-kl-coverage-bounded},
\eqref{eq:smoothed-norm-transfer}, and \(\|g\|_\infty\le2R\) give the
initial-law and successor-law transfer inequalities in the lemma.
\end{proof}

\begin{lemma}[Uniform empirical denominator bound]
\label{lem:kl-denominator-process}
Assume Conditions~\ref{ass:kl-class} and~\ref{ass:kl-bounded}. There is a
constant \(C_{\rm den}\), depending only on \(R\), such that, for every
\(u\ge0\), with probability at least \(1-e^{-u}\),
\[
    \sup_{f\in\sW}|(P_{n,X}-\nu)f|
    \le
    C_{\rm den}
    \left\{
        \mathfrak r_{n,\rm fit}
        +
        \sqrt{\frac{u}{n}}
        +
        \frac{u}{n}
    \right\}.
\]
\end{lemma}

\begin{proof}
Every \(f\in\sW\) can be written as \(\omega_h\) for some
\(h\in\sH^\circ\). Fix \(h_0\in\sH^\circ\). Since
\(\Lambda_\nu(h)=\log E_\nu e^{h(X)}\) and \(\|h\|_\infty\le R\),
\begin{equation}
\label{eq:kl-log-normalizer-lipschitz}
    |\Lambda_\nu(h)-\Lambda_\nu(h_0)|
    \le
    C_R\|h-h_0\|_{L^2(\nu)} .
\end{equation}
Write \(\eta_h=\Lambda_\nu(h)-\Lambda_\nu(h_0)\). Then
\[
    \omega_h-\omega_{h_0}
    =
    e^{-\Lambda_\nu(h)}
    \{e^h-e^{h_0}\}
    +
    \{e^{-\Lambda_\nu(h)}-e^{-\Lambda_\nu(h_0)}\}e^{h_0}.
\]
The first term is indexed by Lipschitz transforms of \(h-h_0\): the maps
\(u\mapsto e^{h_0(x)+u}-e^{h_0(x)}\) are \(C_R\)-Lipschitz on
\([-2R,2R]\) and vanish at \(u=0\). Symmetrization and
Lemma~\ref{lem:tool-rademacher-contraction} therefore give
\[
    E\sup_{h\in\sH^\circ}
    \left|
        \frac1n\sum_{i=1}^n\sigma_i
        e^{-\Lambda_\nu(h)}
        \{e^{h(X_i)}-e^{h_0(X_i)}\}
    \right|
    \le
    C_R\mathcal R_n(\mathcal H_\Delta,C_R;\nu).
\]
For the second term, \eqref{eq:kl-log-normalizer-lipschitz} and
\(\|h-h_0\|_{L^2(\nu)}\le2R\) give
\[
    \sup_{h\in\sH^\circ}
    |e^{-\Lambda_\nu(h)}-e^{-\Lambda_\nu(h_0)}|
    \le C_R .
\]
Moreover,
\[
    E_\sigma
    \left|
        \frac1n\sum_{i=1}^n\sigma_i e^{h_0(X_i)}
    \right|
    \le
    \left\{
        \frac1{n^2}\sum_{i=1}^n e^{2h_0(X_i)}
    \right\}^{1/2}
    \le
    e^R n^{-1/2}.
\]
Combining the two parts gives
\[
    E\sup_{h\in\sH^\circ}
    \left|
        \frac1n\sum_{i=1}^n\sigma_i
        \{\omega_h(X_i)-\omega_{h_0}(X_i)\}
    \right|
    \le
    C_R\mathcal R_n(\mathcal H_\Delta,C_R;\nu)+C_Rn^{-1/2}.
\]
By Lemma~\ref{lem:kl-critical-radius-scaling},
\(\mathcal R_n(\mathcal H_\Delta,C_R;\nu)\le C_R\mathfrak r_{n,\rm fit}\),
after enlarging constants and using \(\mathfrak r_{n,\rm fit}\ge n^{-1/2}\).
Lemma~\ref{lem:tool-bernstein} controls the fixed
function \(\omega_{h_0}\), and Lemma~\ref{lem:tool-bousquet} adds the
deviation term for the supremum. Applying these two inequalities with
\(u+\log2\), and
enlarging constants, gives, with probability at least \(1-e^{-u}\),
\[
    \sup_{h\in\sH^\circ}|(P_{n,X}-\nu)\omega_h|
    \le
    C_R
    \left\{
        \mathfrak r_{n,\rm fit}
        +
        \sqrt{\frac{u}{n}}
        +
        \frac{u}{n}
    \right\}.
\]
This proves the claim.
\end{proof}

\begin{lemma}[Uniform fitted KL empirical-process bound]
\label{lem:kl-uniform-process}
Assume Conditions~\ref{ass:kl-class}, \ref{ass:kl-bounded}, and
\ref{ass:fitted-kl-coverage-bounded}. Let
\(a_{n,\rm fit}(\delta)\) be defined as at the start of this appendix. Then, with probability
at least \(1-\delta\), simultaneously for every \(f\in\sW\) and every
\(h\in\sH^\circ\),
\[
\begin{aligned}
    &\left|
        \{\widehat L_f(h)-L_f(h)\}
        -
        \{\widehat L_f(h_f^\star)-L_f(h_f^\star)\}
    \right| \\
    &\hspace{5em}\le
    \frac14\Delta_f(h)
    +
    C_{\rm env}\log(en)a_{n,\rm fit}(\delta),
\end{aligned}
\]
where
\[
  C_{\rm env}
  \le
  C_0(1+K_0+K_+)^q(1+e^{2R})^p
\]
for universal finite exponents \(p,q\) and a universal finite \(C_0\).
\end{lemma}

\begin{proof}
For \(f\in\sW\) and \(h\in\sH^\circ\), put
\(h^\dagger=h_f^\star\), \(g_h=h-h^\dagger\), and
\[
    \mathbb Z_f(h)
    =
    \{\widehat L_f(h)-L_f(h)\}
    -
    \{\widehat L_f(h^\dagger)-L_f(h^\dagger)\}.
\]
Define
\begin{align}
    \mathbb A_n(g)
    &=(P_{n,0}-\dinit)g, \nonumber\\
    \mathbb B_{n,f}(g)
    &=P_{n,f}^{+}g-E\{f(X)g(X^+)\}, \nonumber\\
    \mathbb C_n(h,h^\dagger)
    &=
    \{\widehat\Lambda_\nu(h)-\Lambda_\nu(h)\}
    -
    \{\widehat\Lambda_\nu(h^\dagger)-\Lambda_\nu(h^\dagger)\}.
    \label{eq:fit-process-pieces}
\end{align}
Using the definitions of \(L_f\) and \(\widehat L_f\),
\begin{equation}
\label{eq:fit-process-decomp}
    \mathbb Z_f(h)
    =
    \mathbb C_n(h,h^\dagger)
    -
    (1-\gamma)\mathbb A_n(g_h)
    -
    \gamma\mathbb B_{n,f}(g_h).
\end{equation}
Put \(L_n=\log(en)\). We first prove the following shell bound. There is a
constant \(C_{\rm sh}\), polynomial in \(1+K_0+K_+\) and \(1+e^{2R}\),
such that, for every \(r\ge n^{-1/2}\) and \(u\ge0\), with probability at
least \(1-5e^{-u}\),
\begin{equation}
\label{eq:fit-shell-bound}
    \sup_{\substack{f\in\sW,\ h\in\sH^\circ:\\
        \Delta_f(h)\le r^2}}
    |\mathbb Z_f(h)|
    \le
    C_{\rm sh}\left\{
        \mathfrak C_n(C_{\rm sh}\sqrt{L_n}\,r)
        +\sqrt{L_n}\,r\mathfrak r_{n,\rm fit}
        +\sqrt{L_n}\,r\sqrt{\frac{u}{n}}
        +
        L_n\frac{u}{n}
    \right\}.
\end{equation}
To prove \eqref{eq:fit-shell-bound}, set
\(t_n=(1+K_0+K_+)L_n\) in
Lemma~\ref{lem:kl-loss-curvature}. On the slice
\(\Delta_f(h)\le r^2\), that lemma and
\(\max\{e^{-t_n/K_0},e^{-t_n/K_+}\}\le(en)^{-1}\) give
\begin{equation}
\label{eq:fit-shell-local-norms}
\begin{aligned}
    \|g_h\|_{L^2(\nu)}&\le C_Rr,\\
    \|g_h\|_{L^2(\dinit)}
    +\{E f^2(X)g_h^2(X^+)\}^{1/2}
    &\le C_{\rm sh}\sqrt{L_n}\,r.
\end{aligned}
\end{equation}
The second line of \eqref{eq:fit-shell-local-norms} uses
\(r\ge n^{-1/2}\), and one may take
\begin{equation}
\label{eq:fit-shell-constant}
 C_{\rm sh}
 \le
 C(1+K_0+K_+)(1+R)^2(1+e^{2R})^4.
\end{equation}
We apply the empirical-process bounds below to centered versions of these
localized classes. Condition~\ref{ass:kl-bounded} supplies a common envelope,
and Lemma~\ref{lem:kl-loss-curvature} bounds each localized variance by
\(C_{\rm sh}L_n r^2\). After symmetrization controls the mean supremum,
Lemma~\ref{lem:tool-bousquet} contributes the deviation terms
\(\sqrt{L_n}r\sqrt{u/n}+L_nu/n\). We use
Lemma~\ref{lem:tool-rademacher-contraction} for Lipschitz transforms and
Lemma~\ref{lem:tool-bernstein} for fixed-function terms.
Moreover \(g_h=h-h_f^\star\in\mathcal H_\Delta\), because both
\(h\) and \(h_f^\star\) belong to \(\sH^\circ\).
For \(\mathbb A_n(g_h)\) in \eqref{eq:fit-process-decomp}, symmetrization
bounds the expectation of the centered localized difference class
\(\{h-h_0:h,h_0\in\sH^\circ\}\), and Lemma~\ref{lem:tool-bousquet} gives
\[
    \sup_{\substack{f\in\sW,\ h\in\sH^\circ:\\
        \Delta_f(h)\le r^2}}
    |(P_{n,0}-\dinit)g_h|
    \le
    C\left\{
        \mathcal R_n(
          \mathcal H_\Delta,C_{\rm sh}\sqrt{L_n}\,r;\dinit
        )
        +C_{\rm sh}\sqrt{L_n}\,r\sqrt{\frac{u}{n}}
        +C_{\rm sh}L_n\frac{u}{n}
    \right\}
\]
with probability at least \(1-e^{-u}\), uniformly over \(f\in\sW\).

For \(\mathbb B_{n,f}(g_h)\) in \eqref{eq:fit-process-decomp}, write
\(P_{n,X} f=n^{-1}\sum_i f(X_i)\). Since \(f\in\sW\),
\(e^{-2R}\le f\le e^{2R}\), and hence \(P_{n,X}f\ge e^{-2R}\) deterministically.
Also \(E_\nu f=1\). Therefore, for each \(g_h\),
\begin{equation}
\label{eq:fit-B-decomp}
    P_{n,f}^{+}g_h-E\{f(X)g_h(X^+)\}
    =
    \frac{
        (P_{n,+}-Q_{\nu,\pi})\{f(X)g_h(X^+)\}
        -
        E\{f(X)g_h(X^+)\}(P_{n,X}f-1)
    }{P_{n,X}f}.
\end{equation}
On the shell, the first numerator in \eqref{eq:fit-B-decomp} is indexed by
functions in \(\mathcal G_\times\) with \(L^2(Q_{\nu,\pi})\)-norm at most
\(C_{\rm sh}\sqrt{L_n}\,r\).
Symmetrization bounds the expectation by the localized Rademacher complexity,
and Lemma~\ref{lem:tool-bousquet} therefore gives
\[
\begin{aligned}
    &\sup_{\substack{f\in\sW,\ h\in\sH^\circ:\\
        \Delta_f(h)\le r^2}}
    |(P_{n,+}-Q_{\nu,\pi})\{f(X)g_h(X^+)\}|\\
    &\qquad\le
    C\left\{
        \mathcal R_n\bigl(
          \mathcal G_\times,C_{\rm sh}\sqrt{L_n}\,r;Q_{\nu,\pi}
        \bigr)
        +C_{\rm sh}\sqrt{L_n}\,r\sqrt{\frac{u}{n}}
        +C_{\rm sh}L_n\frac{u}{n}
    \right\}
\end{aligned}
\]
with probability at least \(1-e^{-u}\). In the second numerator in
\eqref{eq:fit-B-decomp}, the curvature bound gives
\(|E\{f(X)g_h(X^+)\}|\le C_{\rm sh}\sqrt{L_n}\,r\) on the shell, while uniform boundedness gives
\(|E\{f(X)g_h(X^+)\}|\le C_R\). Thus the multiplier may be taken as
\(C_{\rm sh}(\sqrt{L_n}r\wedge C_R)\).
Lemma~\ref{lem:kl-denominator-process} gives
\[
    \sup_{f\in\sW}|(P_{n,X}-\nu)f|
    \le C\left\{
        \mathfrak r_{n,\rm fit}
        +
        \sqrt{\frac{u}{n}}
        +
        \frac{u}{n}
    \right\}
\]
with probability at least \(1-e^{-u}\). Hence the part of
\eqref{eq:fit-B-decomp} containing \(P_{n,X}f-1\) is bounded by
\begin{equation}
\label{eq:fit-B-denominator-bound}
    C\left\{
        \sqrt{L_n}\,r\mathfrak r_{n,\rm fit}
        +\sqrt{L_n}\,r\sqrt{\frac{u}{n}}
        +L_n\frac{u}{n}
    \right\},
\end{equation}
after enlarging \(C\). Combining \eqref{eq:fit-B-denominator-bound} with the
first-numerator bound in \eqref{eq:fit-B-decomp} gives
\begin{equation}
\label{eq:fit-B-bound}
    \sup_{\substack{f\in\sW,\ h\in\sH^\circ:\\
        \Delta_f(h)\le r^2}}
    |\mathbb B_{n,f}(g_h)|
    \le
    C\left\{
        \mathcal R_n(
          \mathcal G_\times,C_{\rm sh}\sqrt{L_n}\,r;Q_{\nu,\pi}
        )
        +\sqrt{L_n}\,r\mathfrak r_{n,\rm fit}
        +\sqrt{L_n}\,r\sqrt{\frac{u}{n}}
        +L_n\frac{u}{n}
    \right\}.
\end{equation}

To bound \(\mathbb C_n(h,h^\dagger)\) in \eqref{eq:fit-process-decomp}, set
\(h_t=h^\dagger+t g_h\) and \(\omega_t=\omega_{h_t}\) for
\(t\in[0,1]\). Since \(\sH^\circ\) is convex, \(h_t\in\sH^\circ\) and
\(\omega_t\in\sW\). Differentiating along this path gives
\begin{equation}
\label{eq:fit-C-path}
\begin{aligned}
    |\mathbb C_n(h,h^\dagger)|
    &\qquad\le
    \int_0^1
    \left|
        \frac{P_{n,X}\{\omega_t g_h\}}{P_{n,X}\omega_t}
        -
        E_\nu\{\omega_t(X)g_h(X)\}
    \right|\,dt .
\end{aligned}
\end{equation}
Indeed, \(d\widehat\Lambda_\nu(h_t)/dt=
P_{n,X}\{\omega_tg_h\}/P_{n,X}\omega_t\), because multiplying
\(\exp(h_t)\) by the population normalizing constant cancels in the empirical
ratio, while \(d\Lambda_\nu(h_t)/dt=E_\nu\{\omega_tg_h\}\).
For each \(t\),
\begin{equation}
\label{eq:fit-C-decomp}
    \frac{P_{n,X}\{\omega_t g_h\}}{P_{n,X}\omega_t}
    -
    E_\nu\{\omega_t(X)g_h(X)\}
    =
    \frac{
        (P_{n,X}-\nu)(\omega_t g_h)
        -
        E_\nu\{\omega_t(X)g_h(X)\}(P_{n,X}\omega_t-1)
    }{P_{n,X}\omega_t}.
\end{equation}
Because \(\omega_t\in\sW\), \(P_{n,X}\omega_t\ge e^{-2R}\) deterministically.
On the slice \(\Delta_f(h)\le r^2\), Lemma~\ref{lem:kl-loss-curvature} gives
\(\|g_h\|_{L^2(\nu)}\le Cr\). Since \(\omega_t\le e^{2R}\),
\[
    \{E_\nu\omega_t^2(X)g_h^2(X)\}^{1/2}
    +
    |E_\nu\{\omega_t(X)g_h(X)\}|
    \le Cr .
\]
Thus the first numerator in \eqref{eq:fit-C-decomp} is indexed by
\(\mathcal G_\times\) under \(Q_{\nu,\Delta}\), with \(L^2(Q_{\nu,\Delta})\)
norm at most \(Cr\). Symmetrization bounds the expectation by the localized
Rademacher complexity, and Lemma~\ref{lem:tool-bousquet} gives
\[
    \sup_{\substack{f\in\sW,\ h\in\sH^\circ:\\
        \Delta_f(h)\le r^2}}
    \sup_{t\in[0,1]}
    |(P_{n,X}-\nu)(\omega_t g_h)|
    \le
    C\left\{
        \mathcal R_n(\mathcal G_\times,Cr;Q_{\nu,\Delta})
        +r\sqrt{\frac{u}{n}}
        +\frac{u}{n}
    \right\}
\]
with probability at least \(1-e^{-u}\). The second numerator in
\eqref{eq:fit-C-decomp} is bounded by
\(C(r\wedge C_R)\sup_{f\in\sW}|(P_{n,X}-\nu)f|\), which is controlled by
Lemma~\ref{lem:kl-denominator-process}. Since
\(P_{n,X}\omega_t\ge e^{-2R}\), \(\mathbb C_n(h,h^\dagger)\) is bounded by
\[
    C\left\{
        \mathcal R_n(\mathcal G_\times,Cr;Q_{\nu,\Delta})
        +r\mathfrak r_{n,\rm fit}
        +r\sqrt{\frac{u}{n}}
        +\frac{u}{n}
    \right\}.
\]
Combining the bounds for \(\mathbb A_n\), \(\mathbb B_{n,f}\), and
\(\mathbb C_n\) in \eqref{eq:fit-process-decomp}, intersecting the component
events, and applying a union bound gives \eqref{eq:fit-shell-bound}.

Set \(A_n=C_{\rm sh}\sqrt{L_n}\) and
\(b=(64C_{\rm sh})^{-1}\). The proof of
Lemma~\ref{lem:kl-critical-radius-scaling}, applied with \(A=A_n\), gives a
constant \(L_{\rm sh}\ge1\) such that
\[
  C_{\rm sh}\mathfrak C_n(A_nr)\le r^2/64,
  \qquad
  r\ge L_{\rm sh}\mathfrak r_{n,\rm fit}.
\]
Indeed, the proof of that lemma permits the explicit choice
\[
  L_{\rm sh}
  =
  1\vee\frac{2A_n}{b}\vee\frac{2}{\sqrt b},
\]
so, for the present \(A_n\) and \(b\),
\[
  L_{\rm sh}
  \le
  C(1+C_{\rm sh})^2\sqrt{L_n}.
\]
After increasing the universal constant in this choice, put
\(\bar r_n=L_{\rm sh}\mathfrak r_{n,\rm fit}\). Then the two linear terms in
\eqref{eq:fit-shell-bound}, followed by Young's inequality for the confidence
term, give, for \(r\ge\bar r_n\),
\begin{equation}
\label{eq:fit-shell-peeled}
  \sup_{\substack{f\in\sW,\ h\in\sH^\circ:\\
      \Delta_f(h)\le r^2}}
  |\mathbb Z_f(h)|
  \le
  \frac{r^2}{16}
  +
  C_{\rm env}L_n\frac{u}{n},
\end{equation}
where \(C_{\rm env}\le C(1+C_{\rm sh})^4\). In particular,
\[
  \bar r_n^2
  \le
  C_{\rm env}L_n\mathfrak r_{n,\rm fit}^2.
\]

The boundedness of \(\sH^\circ\) and \(\sW\) implies
\(\sup_{f\in\sW,h\in\sH^\circ}\Delta_f(h)\le C_R\).
Apply \eqref{eq:fit-shell-peeled} to the inner set
\(\Delta_f(h)\le\bar r_n^2\) and to the nonempty dyadic shells
\[
  2^j\bar r_n^2<\Delta_f(h)\le2^{j+1}\bar r_n^2,
  \qquad j\ge0,
\]
using \(r_j=2^{(j+1)/2}\bar r_n\) and
\(u_j=\log(10/\delta)+(j+1)\log2\). The component failure probabilities are
summable. On shell \(j\),
\[
  \frac{r_j^2}{16}
  =
  2^{j-3}\bar r_n^2
  \le
  \frac18\Delta_f(h).
\]
Moreover, \(j+1\le2^{j+1}\) and
\(\mathfrak r_{n,\rm fit}^2\ge n^{-1}\). By increasing the fixed polynomial
factor in \(L_{\rm sh}\), if necessary,
\[
  C_{\rm env}L_n\frac{(j+1)\log2}{n}
  \le
  2^{j-3}\bar r_n^2
  \le
  \frac18\Delta_f(h).
\]
The inner set contributes at most a constant multiple of
\(\bar r_n^2+C_{\rm env}L_n\log(10/\delta)/n\). Hence, on an event of
probability at least \(1-\delta\), simultaneously for every \(f,h\),
\[
    |\mathbb Z_f(h)|
    \le
    \frac14\Delta_f(h)
    +
    C_{\rm env}L_n
    \left\{
        \mathfrak r_{n,\rm fit}^2+\frac{\log(1/\delta)}{n}
    \right\}.
\]
Finally, \eqref{eq:fit-shell-constant} and
\(1+R\le2(1+e^{2R})\) show that \(C_{\rm env}\) has the polynomial dependence
stated in the lemma.
\end{proof}

\begin{lemma}[Uniform empirical-normalizer bound]
\label{lem:kl-uniform-normalizer}
Assume Conditions~\ref{ass:kl-class} and~\ref{ass:kl-bounded}. There is a
constant \(C_{\rm norm}\), depending only on \(R\), such that, with probability
at least \(1-\delta\),
\[
    \sup_{h\in\sH^\circ}
    |\widehat\Lambda_\nu(h)-\Lambda_\nu(h)|
    \le
    C_{\rm norm}
    \left\{
        \mathfrak r_{n,\rm fit}+
        \sqrt{\frac{\log(1/\delta)}{n}}
    \right\}.
\]
\end{lemma}

\begin{proof}
Choose any \(h_0\in\sH^\circ\). Condition~\ref{ass:kl-bounded} gives
\(e^{-R}\le e^h\le e^R\) for all \(h\in\sH^\circ\). Thus both
\(P_{n,X} e^h\) and \(E_\nu e^h\) lie in
\([e^{-R},e^R]\), and
\[
    \sup_{h\in\sH^\circ}|\widehat\Lambda_\nu(h)-\Lambda_\nu(h)|
    \le
    C_R\sup_{h\in\sH^\circ}|(P_{n,X}-\nu)e^h| .
\]
Moreover,
\[
    \sup_{h\in\sH^\circ}|(P_{n,X}-\nu)e^h|
    \le
    |(P_{n,X}-\nu)e^{h_0}|
    +
    \sup_{h\in\sH^\circ}|(P_{n,X}-\nu)(e^h-e^{h_0})|.
\]
The first term is at most \(C_R\sqrt{u/n}\) with probability at least
\(1-e^{-u}\) by Lemma~\ref{lem:tool-bernstein}. For the second term, define
\(\mathcal F_0=\{e^h-e^{h_0}:h\in\sH^\circ\}\). For each
\(h\in\sH^\circ\), the difference \(h-h_0\) belongs to
\(\mathcal H_\Delta\) and satisfies \(\|h-h_0\|_{L^2(\nu)}\le2R\).
Since the maps \(u\mapsto e^{h_0(x)+u}-e^{h_0(x)}\) are
\(C_R\)-Lipschitz on \([-2R,2R]\) and vanish at \(u=0\), the contraction
inequality in Lemma~\ref{lem:tool-rademacher-contraction} gives
\[
    E\sup_{g\in\mathcal F_0}|(P_{n,X}-\nu)g|
    \le
    C_R\mathcal R_n(\mathcal H_\Delta,2R;\nu).
\]
By Lemma~\ref{lem:kl-critical-radius-scaling}, the right-hand side is at most
\(C_R\mathfrak r_{n,\rm fit}\), after enlarging constants.
Lemma~\ref{lem:tool-bousquet}, applied to the bounded class
\(\mathcal F_0\), adds \(C_R\sqrt{u/n}+C_Ru/n\). Hence, with probability at
least \(1-2e^{-u}\),
\[
    \sup_{h\in\sH^\circ}|\widehat\Lambda_\nu(h)-\Lambda_\nu(h)|
    \le
    C_R\left\{\mathfrak r_{n,\rm fit}+\sqrt{\frac{u}{n}}+\frac{u}{n}\right\}.
\]
Taking \(u=\log(2/\delta)\) and using \(u/n\le u^{1/2}/n^{1/2}\) after
enlarging the constant when \(u\le n\) gives the stated uniform normalizer
bound. If \(u>n\), the deterministic bound
\(\sup_{h\in\sH^\circ}|\widehat\Lambda_\nu(h)-\Lambda_\nu(h)|\le 2R\) gives
the same conclusion after another enlargement of the constant.
\end{proof}

\begin{lemma}[Empirical normalization is a scalar KL perturbation]
\label{lem:kl-scalar-normalization}
Let \(h\in\sH^\circ\), \(\omega_h=e^{h-\Lambda_\nu(h)}\), and
\(\widehat\omega_h=e^{h-\widehat\Lambda_\nu(h)}\). If
\(\ell_h=\widehat\Lambda_\nu(h)-\Lambda_\nu(h)\), then
\[
    \widehat\omega_h=e^{-\ell_h}\omega_h
\]
and
\[
    D_\nu^{\rm gen}(\widehat\omega_h\|\omegastar)
    =
    e^{-\ell_h}D_\nu(\omega_h\|\omegastar)
    +
    e^{-\ell_h}(-\ell_h)-e^{-\ell_h}+1.
\]
Consequently, because \(|\ell_h|\le2R\),
\[
    D_\nu^{\rm gen}(\widehat\omega_h\|\omegastar)
    \le
    e^{|\ell_h|}D_\nu(\omega_h\|\omegastar)
    +
    C_R\ell_h^2,
    \qquad
    |e^{-\ell_h}-1|^2\le C_R\ell_h^2 .
\]
\end{lemma}

\begin{proof}
The identity \(\widehat\omega_h=e^{-\ell_h}\omega_h\) follows directly from the
definitions. Since
\(\int\omega_h\,\dd\nu=\int\omegastar\,\dd\nu=1\),
\[
\begin{aligned}
    D_\nu^{\rm gen}(e^{-\ell_h}\omega_h\|\omegastar)
    &=
    \int
    e^{-\ell_h}\omega_h
    \log\frac{e^{-\ell_h}\omega_h}{\omegastar}\,\dd\nu
    -e^{-\ell_h}+1 \\
    &=
    e^{-\ell_h}D_\nu(\omega_h\|\omegastar)
    +
    e^{-\ell_h}(-\ell_h)-e^{-\ell_h}+1 .
\end{aligned}
\]
The functions \(u\mapsto e^{-u}(-u)-e^{-u}+1\) and
\(u\mapsto e^{-u}-1\) have first derivative zero and finite second derivative on
\([-2R,2R]\). Taylor's theorem on this compact interval gives the two bounds.
\end{proof}

\begin{lemma}[One-step KL-projected \textsc{FORE} recursion]
\label{lem:kl-projected-recursion}
Assume Conditions~\ref{ass:overlap}, \ref{ass:kl-class},
\ref{ass:kl-population-integrability}, and~\ref{ass:kl-bounded}. Then there is
a finite constant \(C_{\rm app}\), depending only on \(R\), such that, for
every \(\omega\in\sW\),
\[
    D_\nu\left(
        \Pi_{\sW}^{\rm KL}\Bpig\omega
        \middle\|\omegastar
    \right)
    \le
    \gamma D_\nu(\omega\|\omegastar)
    +
    C_{\rm app}\varepsilon_{\rm KL}.
\]
\end{lemma}

\begin{proof}
Lemma~\ref{lem:adjoint-bellman-kl-contraction}, applied with
\(\widetilde\omega=\omegastar\) and using
\(\Bpig\omegastar=\omegastar\), gives
\[
    D_\nu(\Bpig\omega\|\omegastar)
    \le
    \gamma D_\nu(\omega\|\omegastar)
\]
under Condition~\ref{ass:overlap}. As in the proof of
Theorem~\ref{thm:kl-fori-realizable}, \(u_\omega=\Bpig\omega\) satisfies
\(E_\nu\{u_\omega\log_+u_\omega\}<\infty\). Lemma~\ref{lem:kl-projection-violation}
therefore gives, after taking the infimum over \(v\in\sW\),
\[
    D_\nu\left(
        \Pi_{\sW}^{\rm KL}\Bpig\omega
        \middle\|\omegastar
    \right)
    \le
    D_\nu(\Bpig\omega\|\omegastar)
    +
    e^{4R}\varepsilon_{\rm KL},
\]
where \(e^{4R}\) depends only on \(R\). Substituting the adjoint Bellman
contraction bound into this projection comparison proves the lemma.
\end{proof}

\begin{lemma}[Uniform fitted-loss excess risk]
\label{lem:kl-erm-excess}
On the event in Lemma~\ref{lem:kl-uniform-process}, the exact empirical
minimizer
\(\widehat h_{k+1}\in\argmin_{h\in\sH^\circ}\widehat L_{\widehat\omega^{(k)}}(h)\)
obeys, for every \(k=0,\ldots,K-1\) such that
\(\omega_{\rm p}^{(k)}\in\sW\),
\[
    L_{\omega_{\rm p}^{(k)}}(\widehat h_{k+1})
    -
    \inf_{h\in\sH^\circ}L_{\omega_{\rm p}^{(k)}}(h)
    \le
    C_{\rm env}\log(en)
    a_{n,\rm fit}(\delta),
\]
where
\(\omega_{\rm p}^{(k)}=\widehat\omega^{(k)}/E_\nu\widehat\omega^{(k)}\) and
\(C_{\rm env}\) has the polynomial dependence stated in
Lemma~\ref{lem:kl-uniform-process}.
\end{lemma}

\begin{proof}
Fix \(k\), write \(f=\omega_{\rm p}^{(k)}\),
\(\widehat h=\widehat h_{k+1}\), and \(h^\star=h_f^\star\). Exact ERM gives
\(\widehat L_{\widehat\omega^{(k)}}(\widehat h)
-\widehat L_{\widehat\omega^{(k)}}(h^\star)\le0\). By self-normalization of
the input weights, \(\widehat L_{\widehat\omega^{(k)}}=\widehat L_f\), so
\(\widehat L_f(\widehat h)-\widehat L_f(h^\star)\le0\).
No conditioning on \(\widehat\omega^{(k)}\) is required here: the event in
Lemma~\ref{lem:kl-uniform-process} holds simultaneously for every
deterministic \(f\in\sW\), and the theorem proof verifies that the random
input \(\omega_{\rm p}^{(k)}\) belongs to \(\sW\).
Thus
\[
\begin{aligned}
    \Delta_f(\widehat h)
    &\le
    \left|
        \{\widehat L_f(\widehat h)-L_f(\widehat h)\}
        -
        \{\widehat L_f(h^\star)-L_f(h^\star)\}
    \right| \\
    &\le
    \frac14\Delta_f(\widehat h)
    +
    C_{\rm env}\log(en)a_{n,\rm fit}(\delta).
\end{aligned}
\]
Moving the first term to the left and absorbing the numerical factor into
\(C_{\rm env}\)
proves the stated excess-loss bound.
\end{proof}

\begin{lemma}[Lower-envelope comparison]
\label{lem:kl-lower-tail-comparison}
Assume Conditions~\ref{ass:kl-population-integrability} and~\ref{ass:fitted-kl-lower-tail}.
Let \(M<\infty\) and \(\tau\in(0,1]\), and put
\(\omega_{\star,\tau}=\omegastar\vee\tau\). If \(0\le a\le M\)
\(\nu\)-almost everywhere, then
\[
    \left|
    D_\nu^{\rm gen}(a\|\omegastar)
    -
    D_\nu^{\rm gen}(a\|\omega_{\star,\tau})
    \right|
    \le
    A\left(1+\frac{M}{\alpha}\right)\tau^\alpha .
\]
\end{lemma}

\begin{proof}
Since \(\omega_{\star,\tau}=\omegastar\) on \(\{\omegastar\ge\tau\}\),
\begin{equation}
\label{eq:kl-lower-envelope-divergence}
    D_\nu^{\rm gen}(a\|\omegastar)
    -
    D_\nu^{\rm gen}(a\|\omega_{\star,\tau})
    =
    \int_{\{\omegastar<\tau\}}
    \left\{
        a\log\frac{\tau}{\omegastar}
        +\omegastar-\tau
    \right\}\,\dd\nu .
\end{equation}
Condition~\ref{ass:kl-population-integrability} gives \(\omegastar>0\) \(\nu\)-almost
everywhere. Therefore, Tonelli's theorem and
Condition~\ref{ass:fitted-kl-lower-tail} give
\begin{equation}
\label{eq:kl-lower-envelope-log}
\begin{aligned}
    \int_{\{\omegastar<\tau\}}
    \log\frac{\tau}{\omegastar}\,\dd\nu
    &=
    \int_0^\infty
    \nu\{0<\omegastar<\tau e^{-s}\}\,\dd s  \\
    &\le
    A\tau^\alpha\int_0^\infty e^{-\alpha s}\,\dd s
    =
    A\alpha^{-1}\tau^\alpha .
\end{aligned}
\end{equation}
Also,
\begin{equation}
\label{eq:kl-lower-envelope-mass}
    \int_{\{\omegastar<\tau\}}(\tau-\omegastar)\,\dd\nu
    \le
    \tau\nu\{0<\omegastar<\tau\}
    \le
    A\tau^{\alpha+1}
    \le
    A\tau^\alpha .
\end{equation}
Taking absolute values in \eqref{eq:kl-lower-envelope-divergence} and applying
\eqref{eq:kl-lower-envelope-log} and \eqref{eq:kl-lower-envelope-mass},
together with \(a\le M\), proves the lemma.
\end{proof}

\begin{lemma}[Approximate KL projection with a lower-envelope target]
\label{lem:kl-projection-perturbation}
Assume Conditions~\ref{ass:overlap}, \ref{ass:kl-class},
\ref{ass:kl-population-integrability}, and~\ref{ass:kl-bounded}. Fix \(\tau\in(0,1]\), put
\(\omega_{\star,\tau}=\omegastar\vee\tau\), fix \(\omega\in\sW\), and let
\(\widetilde h\in\sH^\circ\) and \(\widetilde\omega=\omega_{\widetilde h}\). If
\[
    L_\omega(\widetilde h)-\inf_{h\in\sH^\circ}L_\omega(h)\le\Delta,
\]
then, for every \(\lambda>0\),
\[
    D_\nu^{\rm gen}(\widetilde\omega\|\omega_{\star,\tau})
    \le
    (1+\lambda)D_\nu^{\rm gen}(\bar u_\omega\|\omega_{\star,\tau})
    +
    C_{\rm pert}\left\{
        1+\lambda^{-1}\left(1+\log\frac1\tau\right)
    \right\}\Delta,
\]
where \(C_{\rm pert}<\infty\) depends only on \(R\).
\end{lemma}

\begin{proof}
Because \(L_\omega(h)\) differs from
\(D_\nu(u_\omega\|\omega_h)\) by an additive constant independent of \(h\),
\[
    D_\nu(u_\omega\|\widetilde\omega)
    -
    D_\nu(u_\omega\|\bar u_\omega)
    =
    L_\omega(\widetilde h)-L_\omega(h_\omega^\star)
    \le
    \Delta .
\]
By Lemma~\ref{lem:kl-projection-inequality}, with
\(u=u_\omega\), \(\bar u=\bar u_\omega\), and \(v=\widetilde\omega\),
\begin{equation}
\label{eq:kl-approx-projection-reverse}
    D_\nu(\bar u_\omega\|\widetilde\omega)
    \le
    D_\nu(u_\omega\|\widetilde\omega)
    -
    D_\nu(u_\omega\|\bar u_\omega)
    \le
    \Delta .
\end{equation}
Since \(\bar u_\omega,\widetilde\omega\in[e^{-2R},e^{2R}]\), the ratio
\(\widetilde\omega/\bar u_\omega\) lies in \([e^{-4R},e^{4R}]\). On this
compact interval, the functions
\((v-1)^2\), \(v\log v-v+1\), and \(-\log v+v-1\) all vanish only at
\(v=1\), have positive second derivative at \(v=1\), and are continuous away
from \(v=1\). Hence their ratios are bounded above and below by constants
depending only on \(R\). Applying this pointwise comparison with
\(v=\widetilde\omega/\bar u_\omega\),
\eqref{eq:kl-approx-projection-reverse} implies
\begin{equation}
\label{eq:kl-approx-projection-comparison}
    D_\nu(\widetilde\omega\|\bar u_\omega)
    +
    \int
    \frac{(\widetilde\omega-\bar u_\omega)^2}{\bar u_\omega}\,\dd\nu
    \le
    C_R\Delta .
\end{equation}
Let \(s=\omega_{\star,\tau}\). With \(r=\bar u_\omega/s\), we have
\(0\le r\le e^{2R}/\tau\). Let \(\phi(r)=r\log r-r+1\), with the conventions
\(0\log0=0\) and \(0(\log0)^2=0\). There is a finite constant \(C_R\), depending only on \(R\), such
that, for all \(0\le r\le e^{2R}/\tau\),
\[
    r(\log r)^2
    \le
    C_R\left(1+\log\frac1\tau\right)\phi(r).
\]
Hence
\begin{equation}
\label{eq:kl-approx-target-log-moment}
\begin{aligned}
    \int \bar u_\omega
    \left\{\log\frac{\bar u_\omega}{s}\right\}^2\,\dd\nu
    &=
    \int s r(\log r)^2\,\dd\nu \\
    &\le
    C_R\left(1+\log\frac1\tau\right)
    D_\nu^{\rm gen}(\bar u_\omega\|s).
\end{aligned}
\end{equation}
Using the identity
\begin{equation}
\label{eq:kl-approx-divergence-decomposition}
    D_\nu^{\rm gen}(\widetilde\omega\|s)
    =
    D_\nu^{\rm gen}(\bar u_\omega\|s)
    +
    D_\nu(\widetilde\omega\|\bar u_\omega)
    +
    \int(\widetilde\omega-\bar u_\omega)
    \log\frac{\bar u_\omega}{s}\,\dd\nu,
\end{equation}
By Cauchy--Schwarz, \eqref{eq:kl-approx-projection-comparison}, and
\eqref{eq:kl-approx-target-log-moment}, the cross term satisfies
\begin{equation}
\label{eq:kl-approx-cross-term}
\begin{aligned}
    \left|
        \int(\widetilde\omega-\bar u_\omega)
        \log\frac{\bar u_\omega}{s}\,\dd\nu
    \right|
    &\le
    \left\{
        \int\frac{(\widetilde\omega-\bar u_\omega)^2}{\bar u_\omega}\,\dd\nu
    \right\}^{1/2}
    \left\{
        \int \bar u_\omega
        \left(\log\frac{\bar u_\omega}{s}\right)^2
        \,\dd\nu
    \right\}^{1/2} \\
    &\le
    C_R
    \sqrt{
        \left(1+\log\frac1\tau\right)
        \Delta D_\nu^{\rm gen}(\bar u_\omega\|s)
    } .
\end{aligned}
\end{equation}
Combining \eqref{eq:kl-approx-divergence-decomposition},
\eqref{eq:kl-approx-projection-comparison}, and
\eqref{eq:kl-approx-cross-term}, and applying Young's inequality with
\(L_\tau=1+\log(1/\tau)\), gives
\[
    D_\nu^{\rm gen}(\widetilde\omega\|s)
    \le
    (1+\lambda)D_\nu^{\rm gen}(\bar u_\omega\|s)
    +
    C_R\{1+\lambda^{-1}L_\tau\}\Delta .
\]
This proves the stated inequality after enlarging \(C_R\).
\end{proof}

\begin{proof}[Proof of Theorem~\ref{thm:fitted-kl-fori}]
The proof separates empirical normalization from the population KL geometry.
Throughout the proof \(C_{\rm env}\) denotes a finite constant with the
polynomial dependence stated in Theorem~\ref{thm:fitted-kl-fori}; fixed
dependence on \(A\) and \(\alpha\) is absorbed into this constant. Set
\(\rho=(1+\gamma)/2\), \(L_n=\log(en)\), and write
\[
    a_n
    =
    L_na_{n,\rm fit}(\delta),
    \qquad
    A_{\rm lt}
    =
    A\left(1+\frac{e^{2R}}{\alpha}\right),
    \qquad
    \zeta_n
    =
    2R\wedge
    C_{\rm norm}
    \left\{
        \mathfrak r_{n,\rm fit}+
        \sqrt{\frac{\log(1/\delta)}{n}}
    \right\}.
\]
Apply Lemmas~\ref{lem:kl-uniform-process}
and~\ref{lem:kl-uniform-normalizer} with failure probabilities \(\delta/2\)
each. Since \(\log(2/\delta)\le \log(1/\delta)+\log2\), replacing
\(\delta\) by \(\delta/2\) only enlarges the universal constants multiplying
\(a_{n,\rm fit}(\delta)\). Work on the intersection of these two events, which
has probability at least \(1-\delta\). Both events are uniform over the
log-ratio class, so they may be evaluated at the random iterates constructed by
the algorithm; no union bound over \(k\) is needed. For each \(k\), define the
normalizing constant and the corresponding population-normalized ratio
\[
    \omega_{\rm p}^{(k)}
    =
    \frac{\widehat\omega^{(k)}}{E_\nu\widehat\omega^{(k)}},
    \qquad
    \widehat c_k
    =
    E_\nu\widehat\omega^{(k)} .
\]
Since \(\widehat\omega^{(0)}\equiv1\), we have
\(\omega_{\rm p}^{(0)}=\widehat\omega^{(0)}\) and \(\widehat c_0=1\).
Set \(\ell_0=0\).
For each fitted iterate, choose a centered representative \(\widehat h_k\). This
is valid because, for any constant \(c\), replacing \(h\) by \(h+c\) changes
neither \(\exp\{h-\Lambda_\nu(h)\}\) nor
\(\exp\{h-\widehat\Lambda_\nu(h)\}\); see
Appendix~\ref{app:kl-projected-proofs}.
Thus, for \(k\ge1\),
\[
    \omega_{\rm p}^{(k)}
    =
    \exp\{\widehat h_k-\Lambda_\nu(\widehat h_k)\}
    \in\sW ,
\]
while \(\omega_{\rm p}^{(0)}=1\in\sW\) because \(0\in\sH\). Hence all fitted
inputs belong to \(\sW\).
For \(1\le k\le K\),
\[
    \ell_k
    :=
    \widehat\Lambda_\nu(\widehat h_k)-\Lambda_\nu(\widehat h_k)
    =
    -\log\widehat c_k,
    \qquad
    \widehat\omega^{(k)}
    =
    e^{-\ell_k}\omega_{\rm p}^{(k)} ,
\]
so Lemma~\ref{lem:kl-uniform-normalizer} gives
\[
    \max_{1\le k\le K}
    |\ell_k|
    \le \zeta_n,
    \qquad
    \max_{1\le k\le K}
    \left|E_\nu\widehat\omega^{(k)}-1\right|
    \le C_{\rm norm}\zeta_n,
\]
where the second inequality follows from \(\widehat c_k=e^{-\ell_k}\) and
\(|e^u-1|\le C_R|u|\) on
\([-2R,2R]\). By Lemma~\ref{lem:kl-erm-excess}, for
\(k=0,\ldots,K-1\),
\[
    L_{\omega_{\rm p}^{(k)}}(\widehat h_{k+1})
    -
    \inf_{h\in\sH^\circ}L_{\omega_{\rm p}^{(k)}}(h)
    \le
    C_{\rm env}a_n .
\]
Fix \(\tau\in(0,1]\), put
\(\omega_{\star,\tau}=\omegastar\vee\tau\) and
\(L_\tau=1+\log(1/\tau)\). Applying
Lemma~\ref{lem:kl-projection-perturbation} with
\(\Delta=C_{\rm env}a_n\) gives, for every \(\lambda>0\),
\begin{equation}
\label{eq:fit-approx-projection-step}
    D_\nu^{\rm gen}(\omega_{\rm p}^{(k+1)}\|\omega_{\star,\tau})
    \le
    (1+\lambda)
    D_\nu^{\rm gen}\left(
        \Pi_{\sW}^{\rm KL}\Bpig\omega_{\rm p}^{(k)}
        \middle\|\omega_{\star,\tau}
    \right)
    +
    C_{\rm env}\{1+\lambda^{-1}L_\tau\}a_n .
\end{equation}
Since every element of \(\sW\) is bounded by \(e^{2R}\),
Lemma~\ref{lem:kl-lower-tail-comparison} gives
\begin{equation}
\label{eq:fit-lower-envelope-comparisons}
\begin{aligned}
    D_\nu(\omega_{\rm p}^{(k+1)}\|\omegastar)
    &\le
    D_\nu^{\rm gen}(\omega_{\rm p}^{(k+1)}\|\omega_{\star,\tau})
    +A_{\rm lt}\tau^\alpha,\\
    D_\nu^{\rm gen}\left(
        \Pi_{\sW}^{\rm KL}\Bpig\omega_{\rm p}^{(k)}
        \middle\|\omega_{\star,\tau}
    \right)
    &\le
    D_\nu\left(
        \Pi_{\sW}^{\rm KL}\Bpig\omega_{\rm p}^{(k)}
        \middle\|\omegastar
    \right)
    +A_{\rm lt}\tau^\alpha .
\end{aligned}
\end{equation}
By Lemma~\ref{lem:kl-projected-recursion},
\begin{equation}
\label{eq:fit-projected-recursion-step}
    D_\nu\left(
        \Pi_{\sW}^{\rm KL}\Bpig\omega_{\rm p}^{(k)}
        \middle\|\omegastar
    \right)
    \le
    \gamma D_\nu(\omega_{\rm p}^{(k)}\|\omegastar)
    +
    C_{\rm app}\varepsilon_{\rm KL}.
\end{equation}
Combining \eqref{eq:fit-approx-projection-step},
\eqref{eq:fit-lower-envelope-comparisons}, and
\eqref{eq:fit-projected-recursion-step} gives
\begin{equation}
\label{eq:fit-pre-rho-recursion}
\begin{aligned}
    D_\nu(\omega_{\rm p}^{(k+1)}\|\omegastar)
    &\le
    (1+\lambda)\gamma
    D_\nu(\omega_{\rm p}^{(k)}\|\omegastar)
    +(1+\lambda)C_{\rm app}\varepsilon_{\rm KL}\\
    &\qquad
    +C_{\rm env}\{1+\lambda^{-1}L_\tau\}a_n
    +C_{\rm env}(1+\lambda)A_{\rm lt}\tau^\alpha .
\end{aligned}
\end{equation}
Choose
\(\lambda_\rho=1\) if \(\gamma=0\), and otherwise choose
\[
    \lambda_\rho=
    1\wedge \frac{\rho-\gamma}{2\gamma}.
\]
Then \((1+\lambda_\rho)\gamma\le\rho\). Indeed, if
\(\lambda_\rho=(\rho-\gamma)/(2\gamma)\), then
\((1+\lambda_\rho)\gamma=(\rho+\gamma)/2\le\rho\); if
\(\lambda_\rho=1\), then \(\rho\ge3\gamma\), so
\((1+\lambda_\rho)\gamma=2\gamma\le\rho\). Moreover,
\begin{equation}
\label{eq:fit-lambda-bounds}
    1+\lambda_\rho\le 2,
    \qquad
    1+\lambda_\rho^{-1}\le \frac{C}{\rho-\gamma},
\end{equation}
with the same conclusion when \(\gamma=0\). Substituting
\(\lambda=\lambda_\rho\) in \eqref{eq:fit-pre-rho-recursion} and using
\eqref{eq:fit-lambda-bounds} gives
\begin{equation}
\label{eq:fit-pop-recursion}
    D_\nu(\omega_{\rm p}^{(k+1)}\|\omegastar)
    \le
    \rho D_\nu(\omega_{\rm p}^{(k)}\|\omegastar)
    +
    C_{\rm env}\varepsilon_{\rm KL}
    +
    \frac{C_{\rm env}}{\rho-\gamma}L_\tau a_n
    +
    C_{\rm env}A_{\rm lt}\tau^\alpha .
\end{equation}
Iterating \eqref{eq:fit-pop-recursion} gives
\begin{equation}
\label{eq:fit-pop-recursion-iterated}
    D_\nu(\omega_{\rm p}^{(K)}\|\omegastar)
    \le
    \rho^K D_\nu(\omega_{\rm p}^{(0)}\|\omegastar)
    +
    C_{\rm env}\frac{1-\rho^K}{1-\rho}\varepsilon_{\rm KL}
    +
    \frac{C_{\rm env}}{\rho-\gamma}
    \frac{1-\rho^K}{1-\rho}L_\tau a_n
    +
    \frac{C_{\rm env}}{1-\rho}A_{\rm lt}\tau^\alpha .
\end{equation}

For \(\widehat\omega^{(K)}=e^{-\ell_K}\omega_{\rm p}^{(K)}\),
Lemmas~\ref{lem:kl-uniform-normalizer} and
\ref{lem:kl-scalar-normalization} give
\begin{equation}
\label{eq:fit-final-normalization}
    D_\nu^{\rm gen}(\widehat\omega^{(K)}\|\omegastar)
    \le
    e^{\zeta_n}
    D_\nu(\omega_{\rm p}^{(K)}\|\omegastar)
    +
    C_R\zeta_n^2 .
\end{equation}
Since \(\zeta_n\le2R\) and
\begin{equation}
\label{eq:fit-zeta-square}
    \zeta_n^2
    \le
    C\left\{
        \mathfrak r_{n,\rm fit}^2+\frac{\log(1/\delta)}{n}
    \right\}
    \le
    Ca_n,
\end{equation}
substituting \eqref{eq:fit-pop-recursion-iterated} into
\eqref{eq:fit-final-normalization}, using \eqref{eq:fit-zeta-square} and
\(\omega_{\rm p}^{(0)}=\widehat\omega^{(0)}\), and enlarging \(C_{\rm env}\)
gives
\begin{equation}
\label{eq:fit-final-preoptimization}
    D_\nu^{\rm gen}(\widehat\omega^{(K)}\|\omegastar)
    \le
    C_{\rm env}\rho^K D_\nu^{\rm gen}(\widehat\omega^{(0)}\|\omegastar)
    +
    \frac{C_{\rm env}}{1-\rho}\varepsilon_{\rm KL}
    +
    \frac{C_{\rm env}}{(\rho-\gamma)(1-\rho)}L_\tau a_n
    +
    \frac{C_{\rm env}}{1-\rho}A_{\rm lt}\tau^\alpha .
\end{equation}
Since \(\rho-\gamma\le1\), the two lower-envelope terms in
\eqref{eq:fit-final-preoptimization} are bounded by
\begin{equation}
\label{eq:fit-lower-envelope-objective}
    \frac{C_{\rm env}}{(\rho-\gamma)(1-\rho)}
    \left\{L_\tau a_n+A_{\rm lt}\tau^\alpha\right\}.
\end{equation}
If \(A_{\rm lt}>a_n\), take
\(\tau=(a_n/A_{\rm lt})^{1/\alpha}\); otherwise take \(\tau=1\). Then
\[
    L_\tau a_n+A_{\rm lt}\tau^\alpha
    \le
    2a_n\left[
    1+\frac1\alpha
    \max\left\{0,\log\frac{A_{\rm lt}}{a_n}\right\}
    \right].
\]
Since \(a_n\ge\mathfrak r_{n,\rm fit}^2\ge n^{-1}\),
\[
    \max\left\{0,\log\frac{A_{\rm lt}}{a_n}\right\}
    \le
    \max\{0,\log(A_{\rm lt}n)\}.
\]
After absorbing the fixed dependence on \(A_{\rm lt}\) and \(\alpha\) into
\(C_{\rm env}\), the optimized lower-envelope bound is at most
\begin{equation}
\label{eq:fit-lower-envelope-optimized}
    C_{\rm env}\log^2(en)
    \left\{
        \mathfrak r_{n,\rm fit}^2
        +
        \frac{\log(1/\delta)}{n}
    \right\}.
\end{equation}
With \(\rho=(1+\gamma)/2\), we have
\((1-\rho)^{-1}=2(1-\gamma)^{-1}\) and
\(\{(\rho-\gamma)(1-\rho)\}^{-1}=4(1-\gamma)^{-2}\). Applying
\eqref{eq:fit-lower-envelope-optimized} to
\eqref{eq:fit-lower-envelope-objective} proves the stated generalized KL
bound.

It remains to record the constant dependence. Write \(C_{\rm proc}\) for the
constant produced by the uniform empirical-process argument. The bounds in
Lemmas~\ref{lem:kl-uniform-process}--\ref{lem:kl-projection-perturbation}
may then be summarized as
\[
\begin{aligned}
    C_{\rm sh}
    &\le (1+K_0+K_+)P_{\rm sh}(1+e^{2R}),\\
    C_{\rm proc}
    &\le P_{\rm proc}(1+C_{\rm sh}),\\
    C_{\rm norm}\vee C_{\rm pert}\vee C_{\rm app}
    &\le P_{\rm alg}(1+e^{2R}),
\end{aligned}
\]
where \(P_{\rm sh}\), \(P_{\rm proc}\), and \(P_{\rm alg}\) are fixed
polynomials with universal coefficients. Since
\(1+R\le2(1+e^{2R})\), finite products of these constants satisfy
\[
    C_{\rm env}
    \le
    C_0(A,\alpha)
    (1+K_0+K_+)^q(1+e^{2R})^p
\]
for universal finite exponents \(p,q\). These constants are independent of
\(n\), \(\delta\), and \(K\); the sample-size and horizon factors remain
explicit in the theorem.
\end{proof}

\begin{corollary}[Entropy-integral control of the fitted critical radius]
\label{cor:kl-entropy-integral}
Assume the conditions of Theorem~\ref{thm:fitted-kl-fori}. For
\(\epsilon>0\), let
\[
    \mathfrak H_{\sH}(\epsilon)
    =
    \sup_Q
    \log N\{\epsilon,\sH^\circ,L^2(Q)\},
    \qquad
    \mathcal J_{\sH}(r)
    =
    \int_0^r \sqrt{1+\mathfrak H_{\sH}(\epsilon)}\,\dd\epsilon,
\]
where the supremum is over probability distributions on the state--action
space. Suppose \(\mathcal J_{\sH}(2R)<\infty\). Define
\[
    \mathfrak r_{n,{\rm ent}}
    =
    n^{-1/2}
    \vee
    \inf\left\{r>0:
    \frac{\mathcal J_{\sH}(r)}{\sqrt n}
    \le
    r^2
    \right\}.
\]
Write
\[
    a_{n,{\rm ent}}(\delta)
    =
    \mathfrak r_{n,{\rm ent}}^2+\frac{\log(1/\delta)}{n}.
\]
Then the fitted critical radius in \eqref{eq:rad-critical-radius} satisfies
\(\mathfrak r_{n,\rm fit}\le C_R\mathfrak r_{n,{\rm ent}}\).
Consequently, with probability at least \(1-\delta\),
\begin{align*}
    D_\nu^{\rm gen}(\widehat\omega^{(K)}\|\omegastar)
    \le{}&
    C_{\rm env}\left(\frac{1+\gamma}{2}\right)^K
    D_\nu^{\rm gen}(\widehat\omega^{(0)}\|\omegastar)\\
    &+
    \frac{C_{\rm env}}{1-\gamma}\varepsilon_{\rm KL}\\
    &+
    \frac{C_{\rm env}}{(1-\gamma)^2}
    \log^2(en)a_{n,{\rm ent}}(\delta),
\end{align*}
where \(C_{\rm env}\) has the same dependencies as in
Theorem~\ref{thm:fitted-kl-fori}.
\end{corollary}

\begin{proof}
Throughout the proof, \(C_R\) denotes a finite constant depending only on \(R\).
For any probability distribution \(P\) on the
state-action space,
\[
    \log N\{\epsilon,\mathcal H_\Delta,L^2(P)\}
    \le
    2\mathfrak H_{\sH}(\epsilon/2),
    \qquad \epsilon>0.
\]
Lemma~\ref{lem:tool-local-entropy} therefore gives, uniformly over
\(P\in\{\nu,\dinit\}\),
\[
    \mathcal R_n(\mathcal H_\Delta,r;P)
    \le
    \frac{C_R\mathcal J_{\sH}(C_Rr)}{\sqrt n}.
\]
Thus the two \(\mathcal H_\Delta\) terms in \(\mathfrak C_n(r)\) are each
bounded by \(C_R\mathcal J_{\sH}(C_Rr)/\sqrt n\).

It remains to control the product class. Since
\(\sW=\{\omega_a:a\in\sH^\circ\}\), the definition of
\(\mathcal G_\times\) gives
\[
    \mathcal G_\times
    =
    \left\{
    (x,x^+)\mapsto
    \omega_a(x)b_\Delta(x^+):a\in\sH^\circ,\ b_\Delta\in\mathcal H_\Delta
    \right\},
    \qquad
    \omega_a=e^{a-\Lambda_\nu(a)} .
\]
Let \(Q_{\nu,\pi}\) be the distribution of \((X,X^+)\). Its marginals are \(\nu\) and
\(\nup\). For \(a,a'\in\sH^\circ\) and
\(b_\Delta,b'_\Delta\in\mathcal H_\Delta\), the Lipschitz property of
\(\Lambda_\nu\) on the bounded class and the boundedness of
\(\mathcal H_\Delta\) give
\[
\begin{aligned}
    &\bigl\|
        \omega_a(X)b_\Delta(X^+)
        -
        \omega_{a'}(X)b'_\Delta(X^+)
    \bigr\|_{L^2(Q_{\nu,\pi})} \\
    &\qquad\le
    C_R\left\{
        \|a-a'\|_{L^2(\nu)}
        +
        \|b_\Delta-b'_\Delta\|_{L^2(\nup)}
    \right\}.
\end{aligned}
\]
Hence
\[
    \log N\{\epsilon,\mathcal G_\times,L^2(Q_{\nu,\pi})\}
    \le
    C_R+3\mathfrak H_{\sH}(\epsilon/C_R),
    \qquad \epsilon>0.
\]
Applying Lemma~\ref{lem:tool-local-entropy} to the product class yields
\[
    \mathcal R_n(\mathcal G_\times,r;Q_{\nu,\pi})
    \le
    \frac{C_R\mathcal J_{\sH}(C_Rr)}{\sqrt n}.
\]
Replacing \(Q_{\nu,\pi}\) by \(Q_{\nu,\Delta}\) changes the second marginal
from \(\nup\) to \(\nu\), so the same argument gives the same bound for
\(\mathcal R_n(\mathcal G_\times,r;Q_{\nu,\Delta})\).

Thus all components of \(\mathfrak C_n\) satisfy
\begin{equation}
\label{eq:kl-entropy-complexity-bound}
    \mathfrak C_n(s)
    \le
    \frac{C_R\mathcal J_{\sH}(C_Rs)}{\sqrt n}
    \qquad\text{for every }s>0.
\end{equation}
Because covering numbers decrease as the radius increases,
\(r\mapsto \mathcal J_{\sH}(r)/r\) is nonincreasing. By the definition of
\(\mathfrak r_{n,{\rm ent}}\) as an infimum, there is
\(t\le2\mathfrak r_{n,{\rm ent}}\) such that
\(\mathcal J_{\sH}(t)/\sqrt n\le t^2\). Let
\(s=L\mathfrak r_{n,{\rm ent}}\), where \(L\ge1\) will be chosen large enough
depending only on \(C_R\). If \(C_Rs\ge t\), then
\[
    \frac{\mathcal J_{\sH}(C_Rs)}{\sqrt n}
    \le
    \frac{C_Rs}{t}\frac{\mathcal J_{\sH}(t)}{\sqrt n}
    \le
    C_Rs t
    \le
    \frac{2C_R^2}{L}s^2 .
\]
If \(C_Rs<t\), monotonicity gives
\[
    \frac{\mathcal J_{\sH}(C_Rs)}{\sqrt n}
    \le
    \frac{\mathcal J_{\sH}(t)}{\sqrt n}
    \le
    t^2
    \le
    \frac{4}{L^2}s^2 .
\]
Combining these two cases with \eqref{eq:kl-entropy-complexity-bound}, and
choosing \(L\) large enough, gives
\(\mathfrak C_n(s)\le s^2\). Therefore
\(\inf\{r>0:\mathfrak C_n(r)\le r^2\}\le L\mathfrak r_{n,{\rm ent}}\). Since
\(L\ge1\) and \(\mathfrak r_{n,{\rm ent}}\ge n^{-1/2}\), the leading
\(n^{-1/2}\) term in \eqref{eq:rad-critical-radius} is also bounded by
\(L\mathfrak r_{n,{\rm ent}}\). Hence
\(\mathfrak r_{n,\rm fit}\le L\mathfrak r_{n,{\rm ent}}\). Renaming \(L\) as
part of the constant \(C_R\), substituting
\(\mathfrak r_{n,\rm fit}\le L\mathfrak r_{n,{\rm ent}}\) into
\(a_{n,\rm fit}(\delta)\), and applying Theorem~\ref{thm:fitted-kl-fori}
proves Corollary~\ref{cor:kl-entropy-integral}.
\end{proof}

\begin{corollary}[Finite-dimensional fitted \textsc{FORE} rate]
\label{cor:kl-finite-dimensional}
Assume the conditions of Theorem~\ref{thm:fitted-kl-fori}. If, in addition,
\(\{h-E_\nu\{h(X)\}:h\in\sH\}\) is contained in a \(d\)-dimensional linear span
with \(d\ge1\), then the Rademacher critical radius in
\eqref{eq:rad-critical-radius} satisfies
\[
    \mathfrak r_{n,\rm fit}\le C_R\sqrt{\frac{d\log(en)}{n}},
\]
where \(C_R\) depends only on \(R\). Consequently,
\[
    a_{n,\rm fit}(\delta)
    \le
    C_R\frac{d\log(en)+\log(1/\delta)}{n},
\]
so \(a_{n,\rm fit}(\delta)\) has the stated finite-dimensional order.
\end{corollary}

\begin{proof}
Fix any \(P\in\{\nu,\dinit\}\) and any center \(h_0\in\sH^\circ\). The
localized difference class
\[
    \{h-h_0:h\in\sH^\circ,\ \|h-h_0\|_{L^2(P)}\le r\}
\]
is contained in a \(d\)-dimensional linear space, has \(L^2(P)\)-radius \(r\),
and has a bounded envelope depending only on \(R\). Its covering numbers obey
\[
    \log N\{\epsilon,
        \{h-h_0:h\in\sH^\circ,\|h-h_0\|_{L^2(P)}\le r\},
        L^2(P)
    \}
    \le
    d\log\left(\frac{C_R r}{\epsilon}\right),
    \qquad 0<\epsilon\le C_R r .
\]
The localized entropy bound in Lemma~\ref{lem:tool-local-entropy}, applied to
this finite-dimensional class, gives
\[
    \mathcal R_n(\mathcal H_\Delta,r;P)
    \le
    C_R r\sqrt{\frac{d}{n}}
\]
uniformly over \(P\in\{\nu,\dinit\}\). Hence
the two \(\mathcal H_\Delta\) terms in \(\mathfrak C_n(r)\) are each bounded
by \(C_R r\sqrt{d/n}\).

The product class is contained in the bounded parametric class
\[
    \{(x,x^+)\mapsto
    \omega_a(x)b_\Delta(x^+):
    a\in\sH^\circ,\ b_\Delta\in\mathcal H_\Delta\}.
\]
Let \(Q\) denote either \(Q_{\nu,\pi}\) or \(Q_{\nu,\Delta}\), and write
\(\mathcal G_\times(r;Q)\) for the \(L^2(Q)\)-localized product class. If
\(\omega_a b_\Delta\in\mathcal G_\times(r;Q)\), then
\(\|b_\Delta\|_{L^2(Q_2)}\le C_R r\), where \(Q_2\) is the second marginal of
\(Q\), because \(\omega_a\ge e^{-2R}\). For two products,
\[
\begin{aligned}
    &\|\omega_a b_\Delta-\omega_{a'}b'_\Delta\|_{L^2(Q)}\\
    &\qquad\le
    C_R\left\{
        \|a-a'\|_{L^2(Q_1)}
        +
        \|b_\Delta-b'_\Delta\|_{L^2(Q_2)}
    \right\},
\end{aligned}
\]
where \(Q_1\) is the first marginal. This uses the Lipschitz property of
\(a\mapsto \omega_a=\exp\{a-\Lambda_\nu(a)\}\) on \(\sH^\circ\) under
Condition~\ref{ass:kl-bounded}, together with the bounded envelope of
\(\mathcal H_\Delta\). The finite-dimensional covering bound therefore gives
\[
    \log N\{\epsilon,\mathcal G_\times(r;Q),L^2(Q)\}
    \le
    C d\log\left(\frac{C_R}{\epsilon}\right),
    \qquad 0<\epsilon\le C_R r .
\]
Since the critical radius is at least \(n^{-1/2}\),
Lemma~\ref{lem:tool-local-entropy} applied to this localized product class
gives, for the relevant radii,
\[
    \mathcal R_n(\mathcal G_\times,r;Q)
    \le
    C_Rr\sqrt{\frac{d\log(en)}{n}} .
\]
Therefore the fixed-point inequality in \eqref{eq:rad-critical-radius} holds
whenever \(r\ge C_R\sqrt{d\log(en)/n}\). Since \(d\ge1\),
\[
    n^{-1/2}\le \sqrt{\frac{d\log(en)}{n}},
\]
so the initial \(n^{-1/2}\) term in \eqref{eq:rad-critical-radius} is no
larger than this radius.
Substituting the fitted critical-radius bound into the definition of
\(a_{n,\rm fit}(\delta)\) gives the finite-dimensional rate in
Corollary~\ref{cor:kl-finite-dimensional}.
\end{proof}

\section{Policy-evaluation proofs}
\label{app:policy-evaluation-proofs}

\subsection{Discounted-occupancy contraction for FQE}
\label{app:discounted-occupancy-contraction}

Throughout this appendix, write
\[
    \|f\|_\star^2
    =
    E_{d_{\pi,\gamma}}\{f(X)^2\}
    =
    E_\nu\{\omegastar(X)f(X)^2\}.
\]
For a nonnegative weight \(\omega\), write
\[
    \|f\|_\omega^2
    =
    E_\nu\{\omega(X)f(X)^2\}.
\]

\begin{lemma}[Weighted projection existence]
\label{lem:weighted-projection-existence}
Assume \(r\in L^2(d_{\pi,\gamma})\), and let \(\sQ\) be nonempty, closed, and
convex in \(L^2(d_{\pi,\gamma})\). Then, for every \(Q\in\sQ\), the Bellman
target \(\mathcal T^\pi Q\) belongs to \(L^2(d_{\pi,\gamma})\), and the oracle
projection \(\Pi_{\sQ,\omegastar}\mathcal T^\pi Q\) exists and is unique. For a
nonnegative weight \(\omega\), if \(\mathcal T^\pi Q\in L^2(\omega\,d\nu)\) and
\(\sQ\) is closed in \(L^2(\omega\,d\nu)\), then
\(\Pi_{\sQ,\omega}\mathcal T^\pi Q\) exists and is unique.
\end{lemma}

\begin{proof}
Fix \(Q\in\sQ\). If \(\gamma=0\), then \(\mathcal T^\pi Q=r\), so
\(\mathcal T^\pi Q\in L^2(d_{\pi,\gamma})\) by assumption.
If \(\gamma>0\), conditional Jensen's inequality and the discounted occupancy
identity imply
\[
    E_{d_{\pi,\gamma}}\{(\Ppi Q)(X)^2\}
    \le
    E_{d_{\pi,\gamma}\Ppi}\{Q(X)^2\}
    \le
    \gamma^{-1}E_{d_{\pi,\gamma}}\{Q(X)^2\}.
\]
It follows that \(\Ppi Q\in L^2(d_{\pi,\gamma})\), and hence
\(\mathcal T^\pi Q=r+\gamma\Ppi Q\in L^2(d_{\pi,\gamma})\).

Since \(\sQ\) is closed and convex in \(L^2(d_{\pi,\gamma})\), the Hilbert
projection theorem
\citep{brezis2011FunctionalAnalysis} gives
existence and uniqueness of the oracle projection. The same argument gives
existence and uniqueness of the \(\omega\)-weighted projection whenever
\(\mathcal T^\pi Q\in L^2(\omega\,d\nu)\) and \(\sQ\) is closed in
\(L^2(\omega\,d\nu)\).
\end{proof}

\begin{lemma}[FQE Bellman contraction under discounted occupancy]
\label{lem:discounted-occupancy-q-contraction}
Let \(\gamma\in[0,1)\). For any measurable \(Q_1,Q_2\),
\begin{equation}
\label{eq:fqe-bellman-contraction}
    \|\mathcal T^\pi Q_1-\mathcal T^\pi Q_2\|_\star
    \le
    \sqrt{\gamma}\,\|Q_1-Q_2\|_\star .
\end{equation}
Consequently, for the oracle projected Bellman operator,
\begin{equation}
\label{eq:fqe-projected-contraction}
    \|\mathcal T_{\sQ,\star}Q_1
      -\mathcal T_{\sQ,\star}Q_2\|_\star
    \le
    \sqrt{\gamma}\,\|Q_1-Q_2\|_\star .
\end{equation}
\end{lemma}

\begin{proof}
If \(\gamma=0\), then \(\mathcal T^\pi Q_1=\mathcal T^\pi Q_2=r\), so the
bound in \eqref{eq:fqe-bellman-contraction} holds. Assume \(\gamma>0\), and
write \(\Delta=Q_1-Q_2\).
Since the reward cancels,
\[
    \mathcal T^\pi Q_1-\mathcal T^\pi Q_2
    =
    \gamma \Ppi\Delta .
\]
By conditional Jensen's inequality,
\[
    |(\Ppi\Delta)(X)|^2
    \le
    E\{\Delta(X^+)^2\mid X\}.
\]
Integrating with respect to \(d_{\pi,\gamma}\) gives
\[
    \|\Ppi\Delta\|_\star^2
    \le
    E_{d_{\pi,\gamma}\Ppi}\{\Delta(X)^2\}.
\]
The discounted occupancy identity
\[
    d_{\pi,\gamma}
    =
    (1-\gamma)\dinit+\gamma d_{\pi,\gamma}\Ppi
\]
implies \(\gamma d_{\pi,\gamma}\Ppi\le d_{\pi,\gamma}\). Therefore
\[
    E_{d_{\pi,\gamma}\Ppi}\{\Delta(X)^2\}
    \le
    \frac{1}{\gamma}E_{d_{\pi,\gamma}}\{\Delta(X)^2\}
    =
    \frac{1}{\gamma}\|\Delta\|_\star^2.
\]
Combining the conditional Jensen bound with the discounted occupancy
inequality gives
\[
    \|\mathcal T^\pi Q_1-\mathcal T^\pi Q_2\|_\star^2
    =
    \gamma^2\|\Ppi\Delta\|_\star^2
    \le
    \gamma\|\Delta\|_\star^2.
\]
Taking square roots proves \eqref{eq:fqe-bellman-contraction}.
Equation~\eqref{eq:fqe-projected-contraction} follows from the
nonexpansiveness of Hilbert-space projection onto a closed convex set
\citep{brezis2011FunctionalAnalysis}:
\[
\begin{aligned}
    \|\mathcal T_{\sQ,\star}Q_1
      -\mathcal T_{\sQ,\star}Q_2\|_\star
    &=
    \|\Pi_{\sQ,\omegastar}\mathcal T^\pi Q_1
      -\Pi_{\sQ,\omegastar}\mathcal T^\pi Q_2\|_\star \\
    &\le
    \|\mathcal T^\pi Q_1-\mathcal T^\pi Q_2\|_\star \\
    &\le
    \sqrt{\gamma}\,\|Q_1-Q_2\|_\star .
\end{aligned}
\]
\end{proof}

\begin{lemma}[FQE projected fixed-point bias]
\label{lem:discounted-fqe-projection-bias}
Assume the oracle-projection conditions of
Lemma~\ref{lem:weighted-projection-existence} and the contraction conditions of
Lemma~\ref{lem:discounted-occupancy-q-contraction}. Let \(Q_{\sQ,\star}\) be
the fixed point of \(\mathcal T_{\sQ,\star}\). Then
\[
    \|Q_{\sQ,\star}-Q^\pi\|_\star
    \le
    \frac{1}{1-\sqrt{\gamma}}
    \inf_{q\in\sQ}\|q-Q^\pi\|_\star .
\]
\end{lemma}

\begin{proof}
By Lemmas~\ref{lem:weighted-projection-existence}
and~\ref{lem:discounted-occupancy-q-contraction}, the oracle projected Bellman
operator is a contraction on \(\sQ\). The Banach fixed-point theorem
\citep{brezis2011FunctionalAnalysis} therefore gives a unique fixed point
\(Q_{\sQ,\star}\). The bias bound uses only this fixed-point identity. Let
\(q^\circ=\Pi_{\sQ,\omegastar}Q^\pi\). Since
\(Q^\pi=\mathcal T^\pi Q^\pi\),
\[
\begin{aligned}
    \|Q_{\sQ,\star}-q^\circ\|_\star
    &=
    \|\Pi_{\sQ,\omegastar}\mathcal T^\pi Q_{\sQ,\star}
      -\Pi_{\sQ,\omegastar}Q^\pi\|_\star \\
    &\le
    \|\mathcal T^\pi Q_{\sQ,\star}-\mathcal T^\pi Q^\pi\|_\star \\
    &\le
    \sqrt{\gamma}\,
    \|Q_{\sQ,\star}-Q^\pi\|_\star .
\end{aligned}
\]
The triangle inequality gives
\[
\begin{aligned}
    \|Q_{\sQ,\star}-Q^\pi\|_\star
    &\le
    \|Q_{\sQ,\star}-q^\circ\|_\star
    +
    \|q^\circ-Q^\pi\|_\star \\
    &\le
    \sqrt{\gamma}\,
    \|Q_{\sQ,\star}-Q^\pi\|_\star
    +
    \inf_{q\in\sQ}\|q-Q^\pi\|_\star .
\end{aligned}
\]
Rearranging proves Lemma~\ref{lem:discounted-fqe-projection-bias}.
\end{proof}

\begin{lemma}[Linear or affine FQE projected fixed-point bias]
\label{lem:discounted-fqe-linear-projection-bias}
Let \(Q_{\sQ,\star}\) be the fixed point of
\(\mathcal T_{\sQ,\star}\). Suppose, in addition to the conditions of
Lemma~\ref{lem:discounted-fqe-projection-bias}, that \(\sQ\) is a closed
affine subspace of \(L^2(d_{\pi,\gamma})\). Then
\[
    \|Q_{\sQ,\star}-Q^\pi\|_\star
    \le
    \frac{1}{\sqrt{1-\gamma}}
    \inf_{q\in\sQ}\|q-Q^\pi\|_\star .
\]
\end{lemma}

\begin{proof}
Let \(\Pi_\star=\Pi_{\sQ,\omegastar}\), and set
\[
    q^\circ=\Pi_\star Q^\pi,
    \qquad
    e=Q_{\sQ,\star}-Q^\pi,
    \qquad
    v=Q_{\sQ,\star}-q^\circ .
\]
Because \(\sQ\) is a closed affine subspace of \(L^2(d_{\pi,\gamma})\),
the Hilbert projection theorem gives
\[
    \langle Q^\pi-q^\circ, q-q^\circ\rangle_\star=0
    \qquad
    \text{for every }q\in\sQ .
\]
Taking \(q=Q_{\sQ,\star}\) and applying the Pythagorean identity gives
\begin{equation}
\label{eq:fqe-affine-pythagorean}
    \|e\|_\star^2
    =
    \|v\|_\star^2+\|q^\circ-Q^\pi\|_\star^2 .
\end{equation}
Since \(Q^\pi=\mathcal T^\pi Q^\pi\) and
\(Q_{\sQ,\star}=\Pi_\star\mathcal T^\pi Q_{\sQ,\star}\),
\begin{equation}
\label{eq:fqe-affine-contraction}
\begin{aligned}
    \|v\|_\star
    &=
    \|\Pi_\star\mathcal T^\pi Q_{\sQ,\star}
      -\Pi_\star\mathcal T^\pi Q^\pi\|_\star \\
    &\le
    \|\mathcal T^\pi Q_{\sQ,\star}-\mathcal T^\pi Q^\pi\|_\star
    \le
    \sqrt{\gamma}\,\|e\|_\star ,
\end{aligned}
\end{equation}
where the first inequality is nonexpansiveness of Hilbert projection and the
second is Lemma~\ref{lem:discounted-occupancy-q-contraction}. Combining the
inequalities \eqref{eq:fqe-affine-pythagorean} and
\eqref{eq:fqe-affine-contraction} yields
\[
    \|e\|_\star^2
    \le
    \gamma\|e\|_\star^2+\|q^\circ-Q^\pi\|_\star^2 .
\]
Rearranging and using
\(\|q^\circ-Q^\pi\|_\star=\inf_{q\in\sQ}\|q-Q^\pi\|_\star\) proves
Lemma~\ref{lem:discounted-fqe-linear-projection-bias}.
\end{proof}

\subsection{Target-functional and weight-conversion bounds}
\label{app:target-functional-proof}

\begin{lemma}[Generalized KL controls bounded functionals]
\label{lem:gen-kl-bounded-functional}
Let \(a\ge0\) and \(b\ge0\) be measurable functions with
\(E_\nu a\le M_a<\infty\) and \(E_\nu b\le M_b<\infty\). Then, for every
bounded measurable \(g\),
\[
    \left|
        E_\nu\{(a(X)-b(X))g(X)\}
    \right|
    \le
    \|g\|_\infty
    \{2(M_a+M_b)D_\nu^{\rm gen}(a\|b)\}^{1/2}.
\]
\end{lemma}

\begin{proof}
Let \(\phi(t)=t\log t-t+1\). The scalar inequality
\((t-1)^2/(t+1)\le2\phi(t)\), \(t\ge0\), gives
\[
    \int\frac{(a-b)^2}{a+b}\,\dd\nu
    \le
    2D_\nu^{\rm gen}(a\|b),
\]
with the integrand taken as zero on \(\{a+b=0\}\). By Cauchy--Schwarz,
\[
\begin{aligned}
    \|a-b\|_{L^1(\nu)}^2
    &\le
    \left\{\int(a+b)\,\dd\nu\right\}
    \left\{\int\frac{(a-b)^2}{a+b}\,\dd\nu\right\} \\
    &\le
    2(M_a+M_b)D_\nu^{\rm gen}(a\|b).
\end{aligned}
\]
Multiplying by \(\|g\|_\infty\) proves the claim.
\end{proof}

For nonnegative weights \(a\) and \(b\) with \(b>0\) \(\nu\)-almost
everywhere, write
\[
    \chi_\star(a,b)
    =
    \left\{
        E_\nu\frac{\{a(X)-b(X)\}^2}{b(X)}
    \right\}^{1/2}.
\]

\begin{lemma}[Generalized KL controls target chi-square under one-sided bounds]
\label{lem:kl-to-chi-one-sided}
Suppose \(a\) and \(b\) are nonnegative functions satisfying
\(a(x)\le M<\infty\) and \(b(x)\ge m>0\) for \(\nu\)-almost every \(x\).
Then there is a finite constant \(C_\chi=C_\chi(m,M)\) such that
\[
    \chi_\star(a,b)
    \le
    C_\chi\{D_\nu^{\rm gen}(a\|b)\}^{1/2}.
\]
If \(a\) and \(b\) both integrate to one under \(\nu\), then
\(D_\nu^{\rm gen}(a\|b)=D_\nu(a\|b)\).
\end{lemma}

\begin{proof}
Let \(\phi(t)=t\log t-t+1\). Since \(0\le a/b\le M/m\),
\[
    c_{m,M}
    :=
    \inf_{0\le t\le M/m}
    \frac{\phi(t)}{(t-1)^2}
\]
is strictly positive, where the ratio is interpreted as \(1/2\) at \(t=1\).
Thus \((t-1)^2\le c_{m,M}^{-1}\phi(t)\) on this interval. By the definition of
the generalized KL divergence,
\[
    D_\nu^{\rm gen}(a\|b)
    =
    E_\nu\left[b(X)\phi\{a(X)/b(X)\}\right].
\]
Therefore
\[
\begin{aligned}
    E_\nu\frac{\{a(X)-b(X)\}^2}{b(X)}
    &=
    E_\nu\left[b(X)\left\{\frac{a(X)}{b(X)}-1\right\}^2\right] \\
    &\le
    c_{m,M}^{-1}D_\nu^{\rm gen}(a\|b).
\end{aligned}
\]
Taking square roots proves the lemma with \(C_\chi=c_{m,M}^{-1/2}\).
\end{proof}

\begin{lemma}[\textsc{FORE} weight bounds from log-ratio bounds]
\label{lem:fori-weight-bounds-from-log-conditions}
Assume Condition~\ref{ass:kl-bounded}. Then, with probability one over the
offline sample,
\[
    e^{-2R}
    \le
    \omega_{\rm fit}(x)
    \le
    e^{2R}
    \qquad
    \nu\text{-a.e.}
\]
\end{lemma}

\begin{proof}
Let \(h\in\sH\), and write \(h_c=h-E_\nu\{h(X)\}\). By
Condition~\ref{ass:kl-bounded}, \(e^{-R}\le e^{h_c(x)}\le e^R\) for
\(\nu\)-almost every \(x\). If \(\widehat\omega_h\) denotes the empirical
normalization of \(e^h\) over any offline-data block, then, on the
probability-one event that the block lies in this full-measure set,
\[
    \widehat\omega_h(x)
    =
    \frac{e^{h_c(x)}}{n^{-1}\sum_{i=1}^n e^{h_c(X_i)}} .
\]
The denominator lies in \([e^{-R},e^R]\), so
\[
    e^{-2R}\le \widehat\omega_h(x)\le e^{2R}
    \qquad \nu\text{-a.e.}
\]
The fitted output \(\omega_{\rm fit}\) has this form for some fitted
\(h\in\sH\), which gives the assertion.
\end{proof}

\begin{proof}[Proof of Corollary~\ref{cor:fori-target-functional}]
By Lemma~\ref{lem:fori-weight-bounds-from-log-conditions},
\(E_\nu\omega_{\rm fit}\le e^{2R}\). Applying
Lemma~\ref{lem:gen-kl-bounded-functional} with
\(a=\omega_{\rm fit}\), \(b=\omegastar\), \(M_a=e^{2R}\), and \(M_b=1\)
gives, for every bounded measurable \(g\),
\[
    \left|
        E_\nu\{(\omega_{\rm fit}(X)-\omegastar(X))g(X)\}
    \right|
    \le
    C_R\|g\|_\infty
    \{D_\nu^{\rm gen}(\omega_{\rm fit}\|\omegastar)\}^{1/2}.
\]
Theorem~\ref{thm:fitted-kl-fori}, with \(K=K_\omega\), gives a
high-probability bound on
\(D_\nu^{\rm gen}(\omega_{\rm fit}\|\omegastar)\). Taking square roots and using
the definition of \(\mathcal E_{\rm FORE}\) gives, on an event of probability at least
\(1-\delta\),
\begin{equation}
\label{eq:fori-tv-rate}
    \|\omega_{\rm fit}-\omegastar\|_{L^1(\nu)}
    \le
    C_R\mathcal E_{\rm FORE}.
\end{equation}
For every bounded \(g\),
\[
\begin{aligned}
    \left|
        E_\nu\{\omega_{\rm fit}(X)g(X)\}
        -
        \Psi_\pi(g)
    \right|
    &=
    \left|
        E_\nu\{(\omega_{\rm fit}(X)-\omegastar(X))g(X)\}
    \right| \\
    &\le
    \|g\|_\infty
    \|\omega_{\rm fit}-\omegastar\|_{L^1(\nu)} \\
    &\le
    C_{\rm eval}\|g\|_\infty\mathcal E_{\rm FORE}.
\end{aligned}
\]
Taking the supremum over \(\|g\|_\infty\le1\) proves
Corollary~\ref{cor:fori-target-functional}.
\end{proof}

\begin{lemma}[Fitted generalized KL error controls target chi-square]
\label{lem:fori-chi-rate}
Suppose the conditions of Theorem~\ref{thm:fitted-kl-fori} and
Condition~\ref{ass:value-kl-lower} hold. Then, with probability at least
\(1-\delta\), there is a finite constant \(C_\chi\), depending only on the
constants in Conditions~\ref{ass:kl-bounded},
\ref{ass:fitted-kl-coverage-bounded}, and~\ref{ass:value-kl-lower}, such that
\begin{equation}
\label{eq:fori-chi-rate}
    \chi_\star(\omega_{\rm fit},\omegastar)
    \le
    C_\chi\mathcal E_{\rm FORE}.
\end{equation}
\end{lemma}

\begin{proof}
Condition~\ref{ass:value-kl-lower} implies
Condition~\ref{ass:fitted-kl-lower-tail} with
\(\alpha=1\) and \(A=1\vee m_\star^{-1}\). Indeed, the lower-tail event is
empty when \(t<m_\star\), while for \(t\ge m_\star\) its probability is at
most \(1\le t/m_\star\). Thus
Theorem~\ref{thm:fitted-kl-fori} may be applied with lower-tail constants
depending only on \(m_\star\).

By Lemma~\ref{lem:fori-weight-bounds-from-log-conditions},
\(\omega_{\rm fit}\le e^{2R}\) \(\nu\)-almost everywhere. By
Condition~\ref{ass:value-kl-lower}, \(\omegastar\ge m_\star\)
\(\nu\)-almost everywhere. Lemma~\ref{lem:kl-to-chi-one-sided}, applied with
\(a=\omega_{\rm fit}\) and \(b=\omegastar\), gives
\[
    \chi_\star(\omega_{\rm fit},\omegastar)
    \le
    C_\chi
    \{D_\nu^{\rm gen}(\omega_{\rm fit}\|\omegastar)\}^{1/2}.
\]
Theorem~\ref{thm:fitted-kl-fori}, with \(K=K_\omega\), bounds the generalized
KL term on an event of probability at least \(1-\delta\). Taking square roots
and using the definition of \(\mathcal E_{\rm FORE}\) proves
\eqref{eq:fori-chi-rate}, after enlarging \(C_\chi\).
\end{proof}

\subsection{Weight perturbation and \textsc{FORE}-weighted FQE}
\label{app:fori-weighted-fqe-proof}

\begin{lemma}[Weight-induced perturbation of the FQE projection]
\label{lem:fori-weight-projection-perturbation}
Suppose the projections \(\mathcal T_{\sQ,\omega}Q\) and
\(\mathcal T_{\sQ,\star}Q\) are well defined for the \(Q\in\sQ\) under
consideration. Suppose also that there is a constant \(c_\omega>0\) such that
\(\|h\|_\omega^2\ge c_\omega\|h\|_\star^2\) for every \(h\in\sQ-\sQ\).
Then
\[
    \|\mathcal T_{\sQ,\omega}Q
      -\mathcal T_{\sQ,\star}Q\|_\star
    \le
    \frac{
        \varepsilon_{\rm Bell}\chi_\star(\omega,\omegastar)
    }{
        c_\omega
    } .
\]
\end{lemma}

\begin{proof}
Fix \(Q\in\sQ\), and set
\[
    m=\mathcal T^\pi Q,
    \qquad
    q_\star=\mathcal T_{\sQ,\star}Q,
    \qquad
    q_\omega=\mathcal T_{\sQ,\omega}Q,
    \qquad
    h=q_\omega-q_\star .
\]
Let \(r_Q=m-q_\star\). The projection optimality condition
\citep{brezis2011FunctionalAnalysis} for
\(q_\omega=\Pi_{\sQ,\omega}m\), evaluated at \(q_\star\in\sQ\), gives
\[
    E_\nu\{\omega(X)(m(X)-q_\omega(X))(q_\star(X)-q_\omega(X))\}
    \le 0.
\]
Since \(m-q_\omega=r_Q-h\), this implies
\[
    \|h\|_\omega^2
    \le
    E_\nu\{\omega(X)r_Q(X)h(X)\}.
\]
Similarly, the projection optimality condition for
\(q_\star=\Pi_{\sQ,\omegastar}m\), evaluated at \(q_\omega\in\sQ\), gives
\[
    E_\nu\{\omegastar(X)r_Q(X)h(X)\}\le 0.
\]
Therefore
\[
\begin{aligned}
    \|h\|_\omega^2
    &\le
    E_\nu\{\omega(X)r_Q(X)h(X)\} \\
    &=
    E_\nu\{\omegastar(X)r_Q(X)h(X)\}
    +
    E_\nu\{(\omega(X)-\omegastar(X))r_Q(X)h(X)\} \\
    &\le
    \left|
    E_\nu\{(\omega(X)-\omegastar(X))r_Q(X)h(X)\}
    \right|.
\end{aligned}
\]
The assumed one-sided comparison gives
\[
    \|h\|_\omega^2
    \ge
    c_\omega\|h\|_\star^2 .
\]
Because \(h=q_\omega-q_\star\) belongs to \(\sQ-\sQ\), Cauchy--Schwarz under
\(d_{\pi,\gamma}=\omegastar\nu\) and the definition of \(\varepsilon_{\rm Bell}\)
give
\[
    \left|
    E_\nu\{(\omega-\omegastar)r_Qh\}
    \right|
    \le
    \chi_\star(\omega,\omegastar)\,
    \varepsilon_{\rm Bell}\|h\|_\star .
\]
Combining the norm comparison with the residual bound gives
\begin{equation}
\label{eq:fqe-weight-perturbation-final}
    c_\omega\|h\|_\star^2
    \le
    \varepsilon_{\rm Bell}\chi_\star(\omega,\omegastar)\|h\|_\star .
\end{equation}
If \(\|h\|_\star=0\), the asserted inequality is immediate. Otherwise, dividing
\eqref{eq:fqe-weight-perturbation-final} by \(\|h\|_\star\) proves the lemma.
\end{proof}

\begin{proof}[Proof of Theorem~\ref{thm:fori-weighted-fqe-main}]
By Lemma~\ref{lem:fori-weight-bounds-from-log-conditions} and
Condition~\ref{ass:value-kl-lower}, \(\omega_{\rm fit}\) is bounded above and
below, and \(\omegastar\) is bounded below. Let \(R\) and \(M_\star\) be the
constants in Conditions~\ref{ass:kl-bounded}
and~\ref{ass:fori-target-upper}, respectively.
The proof of Lemma~\ref{lem:fori-weight-bounds-from-log-conditions} gives
\(\omega_{\rm fit}\ge e^{-2R}\). Since
\(\omegastar\le M_\star\), for every \(h\in\sQ-\sQ\),
\[
    \|h\|_{\omega_{\rm fit}}^2
    =
    E_\nu\{\omega_{\rm fit}(X)h(X)^2\}
    \ge
    \frac{e^{-2R}}{M_\star}\|h\|_\star^2.
\]
The same bounds, together with Condition~\ref{ass:value-kl-lower}, make
\(\|\cdot\|_{\omega_{\rm fit}}\) and \(\|\cdot\|_\star\) equivalent. Hence
\(\sQ\) is closed in \(L^2(\omega_{\rm fit}\,\dd\nu)\), and
\(\mathcal T^\pi Q\in L^2(\omega_{\rm fit}\,\dd\nu)\) whenever
\(Q\in\sQ\).
Lemma~\ref{lem:weighted-projection-existence} gives the required integrability
and existence of the oracle and \(\omega_{\rm fit}\)-weighted projected Bellman
operators.
On the event in
Lemma~\ref{lem:fori-chi-rate}, which has probability at least \(1-\delta\), the
weight-conversion bound \eqref{eq:fori-chi-rate} holds.
Lemma~\ref{lem:fori-weight-projection-perturbation} and
\eqref{eq:fori-chi-rate} imply that, for every \(Q\in\sQ\),
\[
    \|\mathcal T_{\sQ,\omega_{\rm fit}}Q
      -\mathcal T_{\sQ,\star}Q\|_\star
    \le
    \varepsilon_{\rm F},
    \qquad
    \varepsilon_{\rm F}
    :=
    C_\chi\varepsilon_{\rm Bell}\mathcal E_{\rm FORE},
\]
where \(C_\chi\) is enlarged by a factor depending only on the constants in
Conditions~\ref{ass:kl-bounded}, \ref{ass:value-kl-lower},
and~\ref{ass:fori-target-upper}.
Combining this perturbation bound with
Lemma~\ref{lem:discounted-occupancy-q-contraction} yields
\begin{equation}
\label{eq:fqe-oracle-one-step}
\begin{aligned}
    \|Q^{(j+1)}-Q_{\sQ,\star}\|_\star
    &=
    \|\mathcal T_{\sQ,\omega_{\rm fit}}Q^{(j)}
      -\mathcal T_{\sQ,\star}Q_{\sQ,\star}\|_\star \\
    &\le
    \|\mathcal T_{\sQ,\star}Q^{(j)}
      -\mathcal T_{\sQ,\star}Q_{\sQ,\star}\|_\star
    +
    \|\mathcal T_{\sQ,\omega_{\rm fit}}Q^{(j)}
      -\mathcal T_{\sQ,\star}Q^{(j)}\|_\star \\
    &\le
    \sqrt{\gamma}\,
    \|Q^{(j)}-Q_{\sQ,\star}\|_\star
    +
    \varepsilon_{\rm F} .
\end{aligned}
\end{equation}
Iterating \eqref{eq:fqe-oracle-one-step} gives
\begin{equation}
\label{eq:fqe-oracle-iterated}
    \|Q^{(K_Q)}-Q_{\sQ,\star}\|_\star
    \le
    \gamma^{K_Q/2}\|Q^{(0)}-Q_{\sQ,\star}\|_\star
    +
    \frac{1-\gamma^{K_Q/2}}{1-\sqrt{\gamma}}\varepsilon_{\rm F}.
\end{equation}
Combining \eqref{eq:fqe-oracle-iterated} with the triangle inequality and
Lemma~\ref{lem:discounted-fqe-projection-bias} gives
\begin{equation}
\label{eq:fqe-total-error}
\begin{aligned}
    \|Q^{(K_Q)}-Q^\pi\|_\star
    \le\;&
    \gamma^{K_Q/2}\|Q^{(0)}-Q_{\sQ,\star}\|_\star \\
    &+
    \frac{1-\gamma^{K_Q/2}}{1-\sqrt{\gamma}}\varepsilon_{\rm F}
    +
    \frac{1}{1-\sqrt{\gamma}}
    \inf_{q\in\sQ}\|q-Q^\pi\|_\star .
\end{aligned}
\end{equation}
Substituting the definition of \(\varepsilon_{\rm F}\) into
\eqref{eq:fqe-total-error} proves the theorem.
\end{proof}

\subsection{Doubly robust value identity}
\label{app:fori-drl-proof}

\begin{proof}[Proof of Theorem~\ref{thm:fori-drl-main}]
Since \(\|r\|_\infty<\infty\), \(r\in L^2(d_{\pi,\gamma})\), and hence
\(\mathcal T^\pi 0=r\) belongs to \(L^2(d_{\pi,\gamma})\).
Lemma~\ref{lem:discounted-occupancy-q-contraction} therefore makes
\(\mathcal T^\pi\) a contraction on \(L^2(d_{\pi,\gamma})\). The Banach
fixed-point theorem gives \(Q^\pi\in L^2(d_{\pi,\gamma})\). Because
\(d_{\pi,\gamma}\) contains the component \((1-\gamma)\dinit\),
\(Q\in L^2(d_{\pi,\gamma})\) implies \(E_{\dinit}|Q(X)|<\infty\). Moreover,
Lemma~\ref{lem:discounted-occupancy-q-contraction}, applied with
\(Q_1=Q\) and \(Q_2=Q^\pi\), gives
\[
    \|\mathcal T^\pi Q-Q^\pi\|_\star
    =
    \|\mathcal T^\pi Q-\mathcal T^\pi Q^\pi\|_\star
    \le
    \sqrt{\gamma}\,\|Q-Q^\pi\|_\star .
\]
Thus \(\mathcal T^\pi Q-Q\in L^2(d_{\pi,\gamma})\).
Taking \(f=Q\) in the adjoint Bellman moment identity
\eqref{eq:weak-occupancy-balance} and conditioning on \(X\) gives
\[
    (1-\gamma)E_{\dinit}\{Q(X)\}
    =
    E_\nu\{\omegastar(X)[Q(X)-\gamma \Ppi Q(X)]\}.
\]
Therefore
\[
\begin{aligned}
    \Psi_{\rm DR}(\omega,Q)
    &=
    E_\nu\{\omegastar(X)[Q(X)-\gamma \Ppi Q(X)]\} \\
    &\quad+
    E_\nu\{\omega(X)[r(X)+\gamma \Ppi Q(X)-Q(X)]\}.
\end{aligned}
\]
Subtracting \(V_\pi(r)=E_\nu\{\omegastar(X)r(X)\}\) gives
\[
\begin{aligned}
    \Psi_{\rm DR}(\omega,Q)-V_\pi(r)
    &=
    E_\nu\{\omega(X)[r(X)+\gamma \Ppi Q(X)-Q(X)]\} \\
    &\quad-
    E_\nu\{\omegastar(X)[r(X)+\gamma \Ppi Q(X)-Q(X)]\} \\
    &=
    E_\nu\!\left[
        \{\omega(X)-\omegastar(X)\}
        \{\mathcal T^\pi Q(X)-Q(X)\}
    \right].
\end{aligned}
\]
This is the standard doubly robust identity. Taking
\(\omega=\omega_{\rm fit}\) and applying Cauchy--Schwarz with respect to
\(d_{\pi,\gamma}=\omegastar\nu\) gives
\begin{equation}
\label{eq:dr-weight-residual-bound}
\begin{aligned}
    \left|\Psi_{\rm DR}(\omega_{\rm fit},Q)-V_\pi(r)\right|
    &\le
    \left\{
        E_\nu\frac{\{\omega_{\rm fit}(X)-\omegastar(X)\}^2}{\omegastar(X)}
    \right\}^{1/2}
    \|\mathcal T^\pi Q-Q\|_\star \\
    &=
    \chi_\star(\omega_{\rm fit},\omegastar)
    \|\mathcal T^\pi Q-Q\|_\star .
\end{aligned}
\end{equation}
Combining \eqref{eq:dr-weight-residual-bound} with
\eqref{eq:fori-chi-rate} proves the bound in
Theorem~\ref{thm:fori-drl-main}.
\end{proof}

\begin{proof}[Proof of Corollary~\ref{cor:fori-weighted-dr-main}]
For \(Q=Q^{(K_Q)}\), Lemma~\ref{lem:discounted-occupancy-q-contraction} gives
\[
\begin{aligned}
    \|\mathcal T^\pi Q-Q\|_\star
    &\le
    \|\mathcal T^\pi Q-\mathcal T^\pi Q^\pi\|_\star
    +
    \|Q-Q^\pi\|_\star \\
    &\le
    (1+\sqrt{\gamma})\|Q-Q^\pi\|_\star .
\end{aligned}
\]
On the event \eqref{eq:fori-chi-rate}, which has probability at least
\(1-\delta\), Theorem~\ref{thm:fori-weighted-fqe-main} and the assumption that
the finite-iteration term is negligible give
\begin{equation}
\label{eq:weighted-dr-q-error}
    \|Q^{(K_Q)}-Q^\pi\|_\star
    \le
    C\left[
        \frac{\varepsilon_{\rm Bell}\mathcal E_{\rm FORE}}
        {1-\sqrt{\gamma}}
        +
        \frac{1}{1-\sqrt{\gamma}}
        \inf_{q\in\sQ}\|q-Q^\pi\|_\star
    \right].
\end{equation}
Applying Theorem~\ref{thm:fori-drl-main} with
\eqref{eq:weighted-dr-q-error} and using
\(1+\sqrt{\gamma}\le2\) proves the corollary, after enlarging
\(C_{\rm DR}\).
\end{proof}

\section{Population theory for coverage-stopped \textsc{FORE}}
\label{app:coverage-truncated-fore-proofs}

This section establishes the population properties of the coverage-stopped
occupancy ratio and its clipped fixed point. We first derive contraction and
the stopped-trajectory representation, then develop the generalized KL
projection inequalities used to prove
Theorem~\ref{thm:coverage-truncated-recursion} and
Proposition~\ref{prop:coverage-truncated-gate}.

Throughout this section, \(D^{\rm gen}(\mu\|\eta)\) denotes the generalized KL
divergence between finite nonnegative measures. If
\(\lambda\) dominates \(\mu\) and \(\eta\), and
\(p=\dd\mu/\dd\lambda\), \(q=\dd\eta/\dd\lambda\), then
\[
    D^{\rm gen}(\mu\|\eta)
    =
    \int
    \left\{
        p\log\frac{p}{q}-p+q
    \right\}\,\dd\lambda ,
\]
with the usual conventions. The value is independent of the dominating
measure. When \(\mu=f\nu\) and \(\eta=g\nu\), this definition reduces to
\(D_\nu^{\rm gen}(f\|g)\).

For a finite nonnegative measure \(\mu=\mu_{\rm ac}+\mu_\perp\), where the
decomposition is relative to \(\nu\), write
\[
    [\mu]_{{\rm ac},\tau_u}
    =
    \left(
        \frac{\dd\mu_{\rm ac}}{\dd\nu}\wedge\tau_u
    \right)\nu .
\]
We also write \([\mu]_{{\rm ac},\infty}=\mu_{\rm ac}\).

\subsection{Coverage-stopped and clipped operators}

\begin{lemma}[Clipping the absolutely continuous component is nonexpansive]
\label{lem:ac-cap-nonexpansive}
Let \(\mu_1\) and \(\mu_2\) be finite nonnegative measures. For every
\(\tau_u\in(0,\infty]\),
\[
    \left\|
        \frac{\dd[\mu_1]_{{\rm ac},\tau_u}}{\dd\nu}
        -
        \frac{\dd[\mu_2]_{{\rm ac},\tau_u}}{\dd\nu}
    \right\|_{L^1(\nu)}
    \le
    |\mu_1-\mu_2|(\sX),
\]
and
\[
    D^{\rm gen}\!\left(
        [\mu_1]_{{\rm ac},\tau_u}
        \big\|
        [\mu_2]_{{\rm ac},\tau_u}
    \right)
    \le
    D^{\rm gen}(\mu_1\|\mu_2).
\]
\end{lemma}

\begin{proof}
Let \(\lambda=\nu+\mu_1+\mu_2\), and write
\(p_i=\dd\mu_i/\dd\lambda\) and \(q=\dd\nu/\dd\lambda\). First suppose
\(\tau_u<\infty\). The density of
\([\mu_i]_{{\rm ac},\tau_u}\) with respect to \(\lambda\) is
\[
    \tilde p_i
    =
    p_i\wedge(\tau_u q),
\]
with \(\tilde p_i=0\) on \(\{q=0\}\). The scalar map
\(a\mapsto a\wedge c\) is \(1\)-Lipschitz for each fixed \(c\ge0\). Hence
\[
    \int|\tilde p_1-\tilde p_2|\,\dd\lambda
    \le
    \int|p_1-p_2|\,\dd\lambda
    =
    |\mu_1-\mu_2|(\sX).
\]
This proves the stated \(L^1(\nu)\) bound.

Let \(\phi(a,b)=a\log(a/b)-a+b\). For every \(a,b,c\ge0\),
\[
    \phi(a\wedge c,b\wedge c)\le \phi(a,b).
\]
If \(a,b\le c\), this is equality. If \(a,b\ge c\), the left side is
\(\phi(c,c)=0\). If \(a\le c\le b\), then
\(\phi(a,c)\le\phi(a,b)\) because \(y\mapsto\phi(a,y)\) is nondecreasing
for \(y\ge a\). If \(b\le c\le a\), then
\(\phi(c,b)\le\phi(a,b)\) because \(x\mapsto\phi(x,b)\) is nondecreasing
for \(x\ge b\). Applying the scalar inequality pointwise with
\(c=\tau_u q\) and integrating proves the generalized KL bound for finite
\(\tau_u\). For \(\tau_u=\infty\), the density of
\([\mu_i]_{{\rm ac},\infty}\) relative to \(\lambda\) is
\(p_i\mathbf 1\{q>0\}\). Both conclusions follow by integrating over
\(\{q>0\}\), because the corresponding integrands are nonnegative on
\(\{q=0\}\).
\end{proof}

\begin{lemma}[Contraction and uniqueness of the coverage-stopped occupancy ratio]
\label{lem:coverage-retained-fixed-point}
The map \(\Bpigcov\) sends the nonnegative cone of \(L^1(\nu)\)
into itself and satisfies
\[
    \|\Bpigcov\omega_1-\Bpigcov\omega_2\|_{L^1(\nu)}
    \le
    \gamma\|\omega_1-\omega_2\|_{L^1(\nu)}
\]
for all nonnegative \(\omega_1,\omega_2\in L^1(\nu)\). Consequently,
\(\Bpigcov\) has a unique nonnegative fixed point
\(\omega_{\rm cov}\in L^1(\nu)\), and
\(E_\nu\omega_{\rm cov}\le1\).
\end{lemma}

\begin{proof}
For nonnegative \(\omega\in L^1(\nu)\), put
\[
    \mu_\omega
    =
    (1-\gamma)\dinit+\gamma(\omega\nu)\Ppi.
\]
The absolutely continuous component of \(\mu_\omega\) has mass at most
\(\mu_\omega(\sX)=1-\gamma+\gamma E_\nu\omega\), so
\(\Bpigcov\omega\in L^1(\nu)\). By
Lemma~\ref{lem:ac-cap-nonexpansive} with \(\tau_u=\infty\),
\[
\begin{aligned}
    \|\Bpigcov\omega_1-\Bpigcov\omega_2\|_{L^1(\nu)}
    &\le
    |\mu_{\omega_1}-\mu_{\omega_2}|(\sX)\\
    &\le
    \gamma\|\omega_1-\omega_2\|_{L^1(\nu)}.
\end{aligned}
\]
The nonnegative cone of \(L^1(\nu)\) is complete, so the Banach fixed-point
theorem gives a unique fixed point. Taking total masses in
\eqref{eq:coverage-retained-fixed-point} yields
\[
    E_\nu\omega_{\rm cov}
    \le
    1-\gamma+\gamma E_\nu\omega_{\rm cov},
\]
and hence \(E_\nu\omega_{\rm cov}\le1\).
\end{proof}

Under Condition~\ref{ass:overlap},
\(\Bpigcov=\Bpig\) on the nonnegative cone of \(L^1(\nu)\).
Uniqueness therefore gives \(\omega_{\rm cov}=\omegastar\).

\begin{lemma}[Contraction and uniqueness of the clipped fixed point]
\label{lem:coverage-clipped-fixed-point}
For each \(\tau_u\in(0,\infty)\), the map \(\Bpigcap{\tau_u}\) sends the
nonnegative cone of \(L^1(\nu)\) into
\(\{\omega\in L^1(\nu):0\le\omega\le\tau_u\ \nu\text{-a.e.}\}\) and
satisfies
\[
    \|\Bpigcap{\tau_u}\omega_1-\Bpigcap{\tau_u}\omega_2\|_{L^1(\nu)}
    \le
    \gamma\|\omega_1-\omega_2\|_{L^1(\nu)}
\]
for all nonnegative \(\omega_1,\omega_2\in L^1(\nu)\). Consequently,
\(\Bpigcap{\tau_u}\) has a unique fixed point \(\omega_{\tau_u}\in[0,\tau_u]\),
and \(E_\nu\omega_{\tau_u}\le1\).
\end{lemma}

\begin{proof}
The map sends nonnegative functions into \([0,\tau_u]\) by construction. For
\(i=1,2\), set
\[
    \mu_i
    =
    (1-\gamma)\dinit+\gamma(\omega_i\nu)\Ppi .
\]
By Lemma~\ref{lem:ac-cap-nonexpansive},
\[
\begin{aligned}
    \|\Bpigcap{\tau_u}\omega_1-\Bpigcap{\tau_u}\omega_2\|_{L^1(\nu)}
    &\le
    |\mu_1-\mu_2|(\sX) \\
    &=
    \gamma\left|
        \{(\omega_1-\omega_2)\nu\}\Ppi
    \right|(\sX) \\
    &\le
    \gamma\|\omega_1-\omega_2\|_{L^1(\nu)} .
\end{aligned}
\]
The set
\(\{\omega\in L^1(\nu):0\le\omega\le\tau_u\ \nu\text{-a.e.}\}\) is complete
under \(L^1(\nu)\). The Banach fixed-point theorem gives existence and
uniqueness of \(\omega_{\tau_u}\).
Taking total masses in the fixed-point identity gives
\[
    E_\nu\omega_{\tau_u}
    \le
    1-\gamma+\gamma E_\nu\omega_{\tau_u},
\]
which proves \(E_\nu\omega_{\tau_u}\le1\).
\end{proof}

\begin{lemma}[Bias of the clipped fixed point relative to the coverage-stopped occupancy ratio]
\label{lem:coverage-retained-cap-bias}
For \(0<\tau_1\le\tau_2<\infty\),
\[
    0\le\omega_{\tau_1}\le\omega_{\tau_2}\le\omega_{\rm cov}
    \qquad \nu\text{-a.e.}
\]
Moreover, for every \(\tau_u\in(0,\infty)\),
\[
    \|\omega_{\tau_u}-\omega_{\rm cov}\|_{L^1(\nu)}
    \le
    \frac{E_\nu\{(\omega_{\rm cov}-\tau_u)_+\}}{1-\gamma}.
\]
Consequently, \(\omega_{\tau_u}\uparrow\omega_{\rm cov}\)
\(\nu\)-almost everywhere and in \(L^1(\nu)\) as \(\tau_u\to\infty\).
\end{lemma}

\begin{proof}
The maps \(\Bpigcap{\tau_u}\) and \(\Bpigcov\) preserve pointwise
order on the nonnegative cone, and
\[
    \Bpigcap{\tau_1}\omega
    \le
    \Bpigcap{\tau_2}\omega
    \le
    \Bpigcov\omega
\]
for every nonnegative \(\omega\). Starting each Picard iteration at zero and
passing to its \(L^1(\nu)\) limit therefore gives the stated ordering of the
fixed points.

Using the two fixed-point identities and
Lemma~\ref{lem:coverage-clipped-fixed-point},
\[
\begin{aligned}
    \|\omega_{\tau_u}-\omega_{\rm cov}\|_{L^1(\nu)}
    &\le
    \|\Bpigcap{\tau_u}\omega_{\tau_u}
      -\Bpigcap{\tau_u}\omega_{\rm cov}\|_{L^1(\nu)}
    +
    \|\Bpigcap{\tau_u}\omega_{\rm cov}
      -\Bpigcov\omega_{\rm cov}\|_{L^1(\nu)}\\
    &\le
    \gamma\|\omega_{\tau_u}-\omega_{\rm cov}\|_{L^1(\nu)}
    +E_\nu\{(\omega_{\rm cov}-\tau_u)_+\},
\end{aligned}
\]
because
\(\Bpigcap{\tau_u}\omega_{\rm cov}=\omega_{\rm cov}\wedge\tau_u\).
Rearranging proves the bound. Its right-hand side tends to zero because
\(\omega_{\rm cov}\in L^1(\nu)\). The fixed-point ordering then gives the
almost-everywhere monotone convergence.
\end{proof}

\begin{lemma}[Clipped-target generalized KL controls coverage-stopped occupancy-ratio \(L^1\) error]
\label{lem:coverage-retained-gen-kl-l1}
For every \(\tau_u\in(0,\infty)\) and every nonnegative \(\omega\) satisfying
\(E_\nu\omega\le M<\infty\),
\[
\begin{aligned}
    \|\omega-\omega_{\rm cov}\|_{L^1(\nu)}
    \le{}&
    \left\{
        2(M+1)D_\nu^{\rm gen}(\omega\|\omega_{\tau_u})
    \right\}^{1/2}
    +
    \frac{E_\nu\{(\omega_{\rm cov}-\tau_u)_+\}}{1-\gamma}.
\end{aligned}
\]
\end{lemma}

\begin{proof}
Lemma~\ref{lem:gen-kl-bounded-functional}, applied with
\(a=\omega\), \(b=\omega_{\tau_u}\), and a measurable version of
\(g=\operatorname{sign}(\omega-\omega_{\tau_u})\), gives
\[
    \|\omega-\omega_{\tau_u}\|_{L^1(\nu)}
    \le
    \{2(M+1)D_\nu^{\rm gen}(\omega\|\omega_{\tau_u})\}^{1/2},
\]
because Lemma~\ref{lem:coverage-clipped-fixed-point} gives
\(E_\nu\omega_{\tau_u}\le1\). The triangle inequality and
Lemma~\ref{lem:coverage-retained-cap-bias} prove the result.
\end{proof}

\begin{lemma}[Clipping bias under subexponential tails]
\label{lem:coverage-retained-subexp-cap-bias}
If \(\|\omega_{\rm cov}\|_{\psi_1}\le K_{\rm cov}\), then, for every
\(\tau_u>0\),
\begin{equation}
\label{eq:coverage-subexp-cap-bias}
    E_\nu\{(\omega_{\rm cov}-\tau_u)_+\}
    \le
    2K_{\rm cov}\exp(-\tau_u/K_{\rm cov}).
\end{equation}
In particular, if \(\tau_{u,n}=1\vee A\log(en)\), then
\begin{equation}
\label{eq:coverage-subexp-log-cap-bias}
    \frac{E_\nu\{(\omega_{\rm cov}-\tau_{u,n})_+\}}{1-\gamma}
    \le
    \frac{2K_{\rm cov}}{1-\gamma}(en)^{-A/K_{\rm cov}}.
\end{equation}
\end{lemma}

\begin{proof}
The definition of the \(\psi_1\) norm and Markov's inequality give
\[
    \nu\{\omega_{\rm cov}>t\}
    \le
    2\exp(-t/K_{\rm cov}),
    \qquad t>0.
\]
Therefore, Tonelli's theorem yields
\[
\begin{aligned}
    E_\nu\{(\omega_{\rm cov}-\tau_u)_+\}
    &=
    \int_{\tau_u}^{\infty}\nu\{\omega_{\rm cov}>t\}\,\dd t\\
    &\le
    2K_{\rm cov}\exp(-\tau_u/K_{\rm cov}).
\end{aligned}
\]
Since \(\tau_{u,n}\ge A\log(en)\),
\eqref{eq:coverage-subexp-cap-bias} implies
\eqref{eq:coverage-subexp-log-cap-bias}.
\end{proof}

\begin{lemma}[Lower-tail transfer under recursive clipping]
\label{lem:coverage-retained-lower-tail-transfer}
Suppose \(\gamma\in[0,1)\),
\(\|\omega_{\rm cov}\|_{\psi_1}\le K_{\rm cov}\), and
\[
    \omega_{\rm cov}>0
    \quad \nu\text{-a.e.},
    \qquad
    \nu\{0<\omega_{\rm cov}\le t\}
    \le
    A_{\rm cov}t^{\alpha_{\rm cov}},
    \qquad 0<t\le1.
\]
Then there exists \(A_0<\infty\), depending only on
\(A_{\rm cov}\), \(\alpha_{\rm cov}\), \(K_{\rm cov}\), and \(\gamma\),
such that, for every \(\tau_u\ge1\),
\[
    \omega_{\tau_u}>0
    \quad \nu\text{-a.e.},
    \qquad
    \nu\{0<\omega_{\tau_u}\le t\}
    \le
    A_0t^{\alpha_0},
    \qquad 0<t\le1,
    \qquad
    \alpha_0=\frac{\alpha_{\rm cov}}{2}.
\]
Moreover, \(A_0\) may be chosen so that
\[
    \log_+A_0
    \le
    C_{\rm lt}+\log\frac{e}{1-\gamma},
\]
where \(C_{\rm lt}<\infty\) depends only on
\(A_{\rm cov}\), \(\alpha_{\rm cov}\), and \(K_{\rm cov}\).
\end{lemma}

\begin{proof}
Put \(b=\omega_{\rm cov}\). The coverage-stopped operator has the affine
representation
\[
    \Bpigcov f=g_0+\gamma\mathcal Kf,
\]
where \(g_0\ge0\) is the density of the absolutely continuous component of
\((1-\gamma)\dinit\), and \(\mathcal K\) is a positive linear map satisfying
\(\|\mathcal Kf\|_{L^1(\nu)}\le\|f\|_{L^1(\nu)}\) for every nonnegative
\(f\). This follows from linearity of the Lebesgue decomposition and the fact
that \(\Ppi\) is a Markov kernel.

Fix \(\tau_u\ge1\). For \(L\ge\tau_u\), let
\(R_L=\{b\le L\}\), and let \(b_L\) be the fixed point of
\[
    f\longmapsto \mathbf 1_{R_L}\Bpigcov f.
\]
This map is monotone. Positivity of \(\mathcal K\) gives
\(|\mathcal K(f-g)|\le\mathcal K|f-g|\), so the map is also a
\(\gamma\)-contraction on the nonnegative cone of \(L^1(\nu)\). Moreover,
\(\mathbf 1_{R_L}\Bpigcov b=\mathbf 1_{R_L}b\le b\), so monotone Picard
iteration from zero gives \(0\le b_L\le b\). The two fixed-point identities
and the affine representation yield
\[
    b-b_L
    =
    \mathbf 1_{R_L^c}b
    +\gamma\mathbf 1_{R_L}\mathcal K(b-b_L).
\]
All terms are nonnegative. Integrating and using the \(L^1(\nu)\)
nonexpansivity of \(\mathcal K\) gives
\begin{equation}
\label{eq:coverage-killed-fixed-point-tail}
    \|b-b_L\|_{L^1(\nu)}
    \le
    \frac{E_\nu\{b\mathbf 1(b>L)\}}{1-\gamma}.
\end{equation}

Set \(r=\tau_u/L\). Because \(b_L\) is supported on \(R_L\) and
\(b_L\le b\), we have \(rb_L\le\tau_u\). Also,
\[
    \Bpigcov(rb_L)
    =
    r\Bpigcov b_L+(1-r)g_0
    \ge
    r\Bpigcov b_L
    \ge
    rb_L,
\]
where the last inequality is an equality on \(R_L\) and is immediate on its
complement. Hence
\(\Bpigcap{\tau_u}(rb_L)\ge rb_L\). Monotone Picard iteration of the clipped
operator from this subsolution therefore gives
\begin{equation}
\label{eq:coverage-capped-killed-lower-bound}
    \omega_{\tau_u}
    \ge
    \frac{\tau_u}{L}b_L
    \qquad \nu\text{-a.e.}
\end{equation}

The subexponential assumption and Tonelli's theorem give
\[
\begin{aligned}
    E_\nu\{b\mathbf 1(b>L)\}
    &=
    L\nu(b>L)+E_\nu(b-L)_+\\
    &\le
    2(L+K_{\rm cov})e^{-L/K_{\rm cov}}.
\end{aligned}
\]
Fix \(0<t\le1\), put \(q=1+\alpha_{\rm cov}/2\), and choose
\[
    L
    =
    \tau_u+qK_{\rm cov}\log(1/t),
    \qquad
    u
    =
    \frac{tL}{\tau_u}.
\]
By \eqref{eq:coverage-capped-killed-lower-bound},
\(\{\omega_{\tau_u}\le t\}\subseteq\{b_L\le u\}\). Since \(b>0\)
\(\nu\)-almost everywhere,
\begin{equation}
\label{eq:coverage-lower-tail-set-bound}
\begin{aligned}
    \nu\{b_L\le u\}
    &\le
    \nu\{0<b\le2u\}
    +
    \nu\{b-b_L>u\}\\
    &\le
    (A_{\rm cov}\vee1)(2u)^{\alpha_{\rm cov}}
    +
    \frac{2(L+K_{\rm cov})e^{-L/K_{\rm cov}}}
         {(1-\gamma)u},
\end{aligned}
\end{equation}
where the second line uses \eqref{eq:coverage-killed-fixed-point-tail} and
extends the assumed lower-tail bound trivially to arguments larger than one.
The elementary inequality
\(t^{1/2}\log(1/t)\le2/e\) gives
\begin{equation}
\label{eq:coverage-lower-tail-u-bound}
    u
    \le
    \left(1+\frac{2qK_{\rm cov}}{e}\right)t^{1/2}.
\end{equation}
Furthermore,
\begin{equation}
\label{eq:coverage-lower-tail-remainder}
\begin{aligned}
    \frac{2(L+K_{\rm cov})e^{-L/K_{\rm cov}}}
         {(1-\gamma)u}
    &=
    \frac{2(L+K_{\rm cov})\tau_u e^{-\tau_u/K_{\rm cov}}}
         {(1-\gamma)L}
    t^{q-1}\\
    &\le
    \frac{2K_{\rm cov}}{1-\gamma}
    t^{\alpha_{\rm cov}/2},
\end{aligned}
\end{equation}
because
\((L+K_{\rm cov})\tau_u/L\le\tau_u+K_{\rm cov}\) and
\((x+K_{\rm cov})e^{-x/K_{\rm cov}}\le K_{\rm cov}\) for \(x\ge0\).
Combining \eqref{eq:coverage-lower-tail-set-bound},
\eqref{eq:coverage-lower-tail-u-bound}, and
\eqref{eq:coverage-lower-tail-remainder} proves the asserted lower-tail
inequality with
\[
    A_0
    =
    (A_{\rm cov}\vee1)
    \left\{2\left(1+\frac{2qK_{\rm cov}}{e}\right)\right\}^{\alpha_{\rm cov}}
    +
    \frac{2K_{\rm cov}}{1-\gamma}.
\]
Because \((1-\gamma)^{-1}\ge1\), this choice is at most
\(C(1-\gamma)^{-1}\), where \(C\) depends only on
\(A_{\rm cov}\), \(\alpha_{\rm cov}\), and \(K_{\rm cov}\). Taking
\(\log_+\) proves the stated bound on \(A_0\).
Letting \(t\downarrow0\) also gives
\(\nu\{\omega_{\tau_u}=0\}=0\), completing the proof.
\end{proof}

\begin{lemma}[Stopped-trajectory representation of the coverage-stopped occupancy ratio]
\label{lem:coverage-truncated-retained-flow}
Let \(\omega_{\rm cov}\) be the fixed point in
\eqref{eq:coverage-retained-fixed-point}, put
\(\eta_{\rm cov}=\omega_{\rm cov}\nu\), and
let
\[
    \mu_{\rm cov}=(1-\gamma)\dinit+\gamma\eta_{\rm cov}\Ppi .
\]
Then \(\eta_{\rm cov}\le\mu_{\rm cov}\). Define
\[
    \mu_{{\rm cov},\perp}=\mu_{\rm cov}-\eta_{\rm cov}.
\]
There exists a measurable set \(C_{\rm cov}\) such that
\(\eta_{\rm cov}(C_{\rm cov}^c)=0\) and
\(\mu_{{\rm cov},\perp}(C_{\rm cov})=0\). Let
\[
    a_{\rm cov}(x)=\mathbf 1\{x\in C_{\rm cov}\}.
\]
If \((X_t)_{t\ge0}\) follows the target-policy kernel \(\Ppi\) from initial
law \(\dinit\), set
\[
    T_{\rm cov}=\inf\{t\ge0:a_{\rm cov}(X_t)=0\},
\]
where \(\inf\emptyset=\infty\).
Then, for every measurable \(B\),
\begin{equation}
\label{eq:coverage-stopped-trajectory-representation}
    \eta_{\rm cov}(B)
    =
    (1-\gamma)\mathbb E
    \left[
        \sum_{t=0}^\infty
        \gamma^t\mathbf 1\{T_{\rm cov}>t,\ X_t\in B\}
    \right].
\end{equation}
\end{lemma}

\begin{proof}
By the fixed-point equation,
\(\eta_{\rm cov}=[\mu_{\rm cov}]_{{\rm ac},\infty}\). Hence
\(\eta_{\rm cov}\) and \(\mu_{{\rm cov},\perp}\) are the absolutely continuous
and singular components, respectively, of the Lebesgue decomposition of
\(\mu_{\rm cov}\) relative to \(\nu\). They are mutually singular, so a
measurable set \(C_{\rm cov}\) with the stated properties exists. Consequently,
\[
    \eta_{\rm cov}=a_{\rm cov}\mu_{\rm cov}.
\]

For \(t\ge0\), define the finite measure
\[
    \eta_t(B)=\mathbb P(T_{\rm cov}>t,\ X_t\in B).
\]
Then
\begin{equation}
\label{eq:coverage-stopped-measure-recursion}
    \eta_0(B)=\int_B a_{\rm cov}(x)\,\dinit(\dd x),
    \qquad
    \eta_{t+1}(B)=\int_B a_{\rm cov}(y)\,(\eta_t\Ppi)(\dd y).
\end{equation}
Let \(\bar\eta=(1-\gamma)\sum_{t=0}^\infty\gamma^t\eta_t\). The series is
finite because each \(\eta_t\) has total mass at most one. Summing
\eqref{eq:coverage-stopped-measure-recursion} gives
\[
    \bar\eta(B)
    =
    (1-\gamma)\int_B a_{\rm cov}(x)\,\dinit(\dd x)
    +
    \gamma\int_B a_{\rm cov}(y)\,(\bar\eta\Ppi)(\dd y).
\]
For a finite measure \(\xi\), write \(a_{\rm cov}\xi\) for the measure
\(B\mapsto\int_B a_{\rm cov}\,\dd\xi\). Equivalently,
\begin{equation}
\label{eq:coverage-stopped-discounted-fixed-point}
    \bar\eta
    =
    a_{\rm cov}\{(1-\gamma)\dinit+\gamma\bar\eta\Ppi\}.
\end{equation}
For this fixed \(a_{\rm cov}\), the map
\(\xi\mapsto a_{\rm cov}\{(1-\gamma)\dinit+\gamma\xi\Ppi\}\) is a
\(\gamma\)-contraction in total variation, since \(0\le a_{\rm cov}\le1\) and
\[
    \left|
        a_{\rm cov}\gamma\{(\xi-\xi')\Ppi\}
    \right|(\sX)
    \le
    \gamma|\xi-\xi'|(\sX).
\]
The measure \(\eta_{\rm cov}\) is also a fixed point of the map in
\eqref{eq:coverage-stopped-discounted-fixed-point} because
\(\eta_{\rm cov}=a_{\rm cov}\mu_{\rm cov}\). Therefore
\(\bar\eta=\eta_{\rm cov}\), which proves
\eqref{eq:coverage-stopped-trajectory-representation}.
Taking total masses and summing the geometric series gives
\(E_\nu\omega_{\rm cov}=1-\mathbb E(\gamma^{T_{\rm cov}})\). Moreover, the stopped
sum is pathwise dominated by the full discounted occupancy sum. Integrating
the upper and lower reward bounds over the difference between these two
measures gives the value interval stated in
Section~\ref{sec:coverage-clipped-target}.
\end{proof}

\subsection{Generalized KL projection inequalities}

\begin{lemma}[Generalized KL projection inequality]
\label{lem:genkl-projection-inequality}
Let \(\sH_{\rm clip}\) be convex. Fix a bounded nonnegative measurable
\(u\), and suppose
\(\bar h\in\sH_{\rm clip}\) minimizes
\[
    h\mapsto E_\nu e^{h(X)}-E_\nu\{u(X)h(X)\}
\]
over \(\sH_{\rm clip}\). Let \(\bar u=e^{\bar h}\). Then, for every
\(v=e^g\in\sW_{\rm clip}\),
\[
    D_\nu^{\rm gen}(\bar u\|v)
    \le
    D_\nu^{\rm gen}(u\|v)-D_\nu^{\rm gen}(u\|\bar u).
\]
In particular, \(D_\nu^{\rm gen}(\bar u\|v)\le D_\nu^{\rm gen}(u\|v)\).
\end{lemma}

\begin{proof}
For \(g\in\sH_{\rm clip}\), set \(h_t=(1-t)\bar h+tg\). Convexity gives
\(h_t\in\sH_{\rm clip}\) for \(t\in[0,1]\). Since \(\bar h\) minimizes the objective,
the right derivative at \(t=0\) is nonnegative:
\begin{equation}
\label{eq:genkl-projection-first-order}
    E_\nu\{(\bar u(X)-u(X))(g-\bar h)(X)\}\ge0 .
\end{equation}
Since \(\bar u=e^{\bar h}\) and \(v=e^g\),
\begin{equation}
\label{eq:genkl-projection-three-point}
\begin{aligned}
    &D_\nu^{\rm gen}(u\|v)
      -D_\nu^{\rm gen}(u\|\bar u)
      -D_\nu^{\rm gen}(\bar u\|v) \\
    &\qquad =
    E_\nu\{(u(X)-\bar u(X))(\bar h-g)(X)\}.
\end{aligned}
\end{equation}
Equations~\eqref{eq:genkl-projection-first-order} and
\eqref{eq:genkl-projection-three-point} prove the projection inequality.
\end{proof}

\begin{lemma}[Generalized KL projection comparison with an external target]
\label{lem:genkl-projection-residual}
Let \(\sH_{\rm clip}\) be convex, and suppose
\(\sup_{h,g\in\sH_{\rm clip}}\|h-g\|_\infty\le R\). Fix a bounded
nonnegative measurable \(u\) and a nonnegative \(w\in L^1(\nu)\). If
\(\bar u=\Pi_{\sW_{\rm clip}}^{\rm genKL}u\), then, for every
\(v\in\sW_{\rm clip}\),
\[
    D_\nu^{\rm gen}(\bar u\|w)
    \le
    D_\nu^{\rm gen}(u\|w)
    +
    e^{R}D_\nu^{\rm gen}(v\|w).
\]
\end{lemma}

\begin{proof}
If \(D_\nu^{\rm gen}(u\|w)\) or \(D_\nu^{\rm gen}(v\|w)\) is infinite, there
is nothing to prove. Assume both divergences are finite. Let
\(\bar u=e^{\bar h}\) and \(v=e^g\), with \(\bar h,g\in\sH_{\rm clip}\). By
Lemma~\ref{lem:genkl-projection-inequality},
\[
    D_\nu^{\rm gen}(\bar u\|v)
    \le
    D_\nu^{\rm gen}(u\|v)-D_\nu^{\rm gen}(u\|\bar u).
\]
For nonnegative functions \(a,v,w\),
\begin{equation}
\label{eq:genkl-change-target}
    D_\nu^{\rm gen}(a\|w)-D_\nu^{\rm gen}(a\|v)
    =
    E_\nu\!\left\{a(X)\log\frac{v(X)}{w(X)}\right\}
    +
    E_\nu\{w(X)-v(X)\}.
\end{equation}
Applying \eqref{eq:genkl-change-target} first with \(a=\bar u\) and then with
\(a=u\), and using Lemma~\ref{lem:genkl-projection-inequality}, gives
\begin{equation}
\label{eq:genkl-external-target-decomposition}
    D_\nu^{\rm gen}(\bar u\|w)
    \le
    D_\nu^{\rm gen}(u\|w)
    +
    E_\nu\{(u-\bar u)\delta\}
    -
    D_\nu^{\rm gen}(u\|\bar u),
\end{equation}
where \(\delta=\log(w/v)\). The scalar Fenchel inequality gives
\[
    E_\nu\{u\delta\}-D_\nu^{\rm gen}(u\|\bar u)
    \le
    E_\nu\{\bar u(e^\delta-1)\}.
\]
Subtracting \(E_\nu\{\bar u\delta\}\) from both sides yields
\begin{equation}
\label{eq:genkl-external-target-fenchel}
    E_\nu\{(u-\bar u)\delta\}
    -
    D_\nu^{\rm gen}(u\|\bar u)
    \le
    E_\nu\{\bar u(e^\delta-1-\delta)\}.
\end{equation}
The log-diameter condition gives \(\bar u/v\le e^R\). Since
\(e^t-1-t\ge0\),
\begin{equation}
\label{eq:genkl-external-target-ratio}
    E_\nu\{\bar u(e^\delta-1-\delta)\}
    \le
    e^R
    E_\nu\{v(e^\delta-1-\delta)\}
    =
    e^R D_\nu^{\rm gen}(v\|w).
\end{equation}
Combining \eqref{eq:genkl-external-target-decomposition},
\eqref{eq:genkl-external-target-fenchel}, and
\eqref{eq:genkl-external-target-ratio} proves the stated bound.
\end{proof}

\subsection{Population recursion and moment identification}

\begin{proof}[Proof of Theorem~\ref{thm:coverage-truncated-recursion}]
We first record the one-step generalized KL contraction used by the projected
recursion. For \(i=1,2\), define the Bellman image measures
\[
    \mu_i
    =
    (1-\gamma)\dinit+\gamma(\omega_i\nu)\Ppi .
\]
Lemma~\ref{lem:ac-cap-nonexpansive}, joint convexity of generalized KL
divergence, and data processing for Markov kernels yield
\[
\begin{aligned}
    D_\nu^{\rm gen}(\Bpigcap{\tau_u}\omega_1\|\Bpigcap{\tau_u}\omega_2)
    &=
    D^{\rm gen}\!\left(
        [\mu_1]_{{\rm ac},\tau_u}
        \big\|
        [\mu_2]_{{\rm ac},\tau_u}
    \right)\\
    &\le
    D^{\rm gen}(\mu_1\|\mu_2)\\
    &\le
    \gamma
    D^{\rm gen}\{(\omega_1\nu)\Ppi\|(\omega_2\nu)\Ppi\}\\
    &\le
    \gamma D_\nu^{\rm gen}(\omega_1\|\omega_2).
\end{aligned}
\]

Condition~\ref{ass:kl-class} makes \(\sH\) compact in \(L^2(\nu)\), and
\(\sH_{\rm clip}\) is an \(L^2(\nu)\)-closed subset. Moreover, for
\(0\le u\le\tau_u\), the map
\[
    h\mapsto E_\nu e^h-E_\nu(uh)
\]
is continuous on \(\sH_{\rm clip}\), because both \(u\) and \(e^h\)
are uniformly bounded there. Hence every generalized KL projection used in
the recursion is attained.

Condition~\ref{ass:kl-class} implies that \(\sH_{\rm clip}\) is convex, so
Lemmas~\ref{lem:genkl-projection-inequality}
and~\ref{lem:genkl-projection-residual} apply.
By definition of \(\sH_{\rm clip}\),
\(\sup_{h,g\in\sH_{\rm clip}}\|h-g\|_\infty
\le\log(\tau_u/\tau_\ell)\). For the projected
recursion, Lemma~\ref{lem:genkl-projection-residual}, with
\(R=\log(\tau_u/\tau_\ell)\), \(u=\Bpigcap{\tau_u}\omega\), \(w=\omega_{\tau_u}\), and
arbitrary \(v\in\sW_{\rm clip}\), gives
\[
    D_\nu^{\rm gen}(\mathsf T_{\sW_{\rm clip}}^{\rm genKL}\omega\|\omega_{\tau_u})
    \le
    D_\nu^{\rm gen}(\Bpigcap{\tau_u}\omega\|\omega_{\tau_u})
    +
    \frac{\tau_u}{\tau_\ell}
    D_\nu^{\rm gen}(v\|\omega_{\tau_u}).
\]
Taking the infimum over \(v\in\sW_{\rm clip}\), then using the generalized KL
contraction with \(\omega_2=\omega_{\tau_u}\) and
\(\Bpigcap{\tau_u}\omega_{\tau_u}=\omega_{\tau_u}\), gives
\begin{equation}
\label{eq:coverage-projected-one-step}
    D_\nu^{\rm gen}(\mathsf T_{\sW_{\rm clip}}^{\rm genKL}\omega\|\omega_{\tau_u})
    \le
    \gamma D_\nu^{\rm gen}(\omega\|\omega_{\tau_u})
    +
    \frac{\tau_u}{\tau_\ell}
    \varepsilon_{\rm ratio}(\tau_u).
\end{equation}
Iterating \eqref{eq:coverage-projected-one-step} gives the projected-recursion
bound. Since
\(\omega^{(K)}\in\sW_{\rm clip}\), we have
\(E_\nu\omega^{(K)}\le\tau_u\). Lemma~\ref{lem:coverage-retained-gen-kl-l1}
therefore gives
\[
    \|\omega^{(K)}-\omega_{\rm cov}\|_{L^1(\nu)}
    \le
    \{2(\tau_u+1)D_\nu^{\rm gen}
    (\omega^{(K)}\|\omega_{\tau_u})\}^{1/2}
    +
    \frac{E_\nu\{(\omega_{\rm cov}-\tau_u)_+\}}{1-\gamma}.
\]
This is the asserted \(L^1(\nu)\) inequality.
\end{proof}

\begin{proof}[Proof of Proposition~\ref{prop:coverage-truncated-gate}]
Let
\[
    \mu_\omega
    =
    (1-\gamma)\dinit+\gamma(\omega\nu)\Ppi
    =
    (\Bpigcov\omega)\nu+\mu_{\omega,\perp}
\]
be the Lebesgue decomposition relative to \(\nu\). The retention-indicator
objective can be rewritten as
\[
    \tau_u
    +
    \int c(x)\{(\Bpigcov\omega)(x)-\tau_u\}\,\nu(\dd x)
    +
    \int c(x)\,\mu_{\omega,\perp}(\dd x),
\]
using \(E_\nu1=1\). Thus any minimizer satisfies
\begin{equation}
\label{eq:retention-indicator-characterization}
    c_{\omega,\tau_u}^\star=1
    \quad \nu\text{-a.e. on }\{\Bpigcov\omega<\tau_u\},
    \qquad
    c_{\omega,\tau_u}^\star=0
    \quad \nu\text{-a.e. on }\{\Bpigcov\omega>\tau_u\}.
\end{equation}
It may take either value \(\nu\)-almost everywhere on
\(\{\Bpigcov\omega=\tau_u\}\), and it vanishes
\(\mu_{\omega,\perp}\)-almost everywhere.

For any bounded measurable \(h\), define
\[
    M_{\omega,c}(h)
    =
    (1-\gamma)E_{\dinit}\{c(X)h(X)\}
    +
    \gamma E_\nu\{\omega(X)c(X^+)h(X^+)\}
    +
    \tau_u E_\nu\{(1-c(X))h(X)\}.
\]
Equation~\eqref{eq:retention-indicator-characterization} gives
\begin{equation}
\label{eq:retention-clipped-moment}
\begin{aligned}
    M_{\omega,c_{\omega,\tau_u}^\star}(h)
    &=
    \int c_{\omega,\tau_u}^\star(x)h(x)\,\mu_\omega(\dd x)
    +
    \tau_u\int \{1-c_{\omega,\tau_u}^\star(x)\}h(x)\,\nu(\dd x)\\
    &=
    \int
    \{(\Bpigcov\omega)(x)\wedge\tau_u\}h(x)\,\nu(\dd x).
\end{aligned}
\end{equation}
The first equality expands the definition of \(M_{\omega,c}\). The second uses
\(c_{\omega,\tau_u}^\star=0\) on \(\mu_{\omega,\perp}\) and the pointwise values of
the retention indicator on the absolutely continuous component. Thus,
\(M_{\omega,c_{\omega,\tau_u}^\star}(h)\) equals the moment of the clipped density.

Finally, let
\(u=\Bpigcap{\tau_u}\omega=(\Bpigcov\omega)\wedge\tau_u\).
By the assumed nonemptiness and Condition~\ref{ass:kl-class},
\(\sH_{\rm clip}\) is a compact subset of \(L^2(\nu)\). Since \(u\) is
bounded, the projection objective is
continuous on this set and therefore attains its minimum.
For \(h\in\sH_{\rm clip}\),
\[
    D_\nu^{\rm gen}(u\|e^h)
    =
    E_\nu e^{h(X)}
    -
    E_\nu\{u(X)h(X)\}
    +
    C_{\omega,\tau_u},
\]
where \(C_{\omega,\tau_u}\) does not depend on \(h\). Substituting
\eqref{eq:retention-clipped-moment} shows that minimizing
\(D_\nu^{\rm gen}(u\|e^h)\) over \(h\in\sH_{\rm clip}\) is equivalent to minimizing
\(E_\nu e^h-M_{\omega,c_{\omega,\tau_u}^\star}(h)\), which is the objective in
Proposition~\ref{prop:coverage-truncated-gate}. This proves the proposition.
\end{proof}

\section{Fixed-level fitted theory for coverage-stopped \textsc{FORE}}
\label{app:clip-finite}

This section gives the exact-ERM finite-sample bound for
Algorithm~\ref{alg:coverage-truncated-fore} at a fixed upper clipping level.
Theorem~\ref{thm:clip-finite-fore} applies this result at
\(\tau_u=\tau_{u,n}\) and accounts separately for the clipping bias.

\subsection{Finite-sample losses and critical radii}

For analysis at a fixed upper clipping level, we use the lower-tail bound
\begin{equation}
\label{eq:coverage-clipped-fixed-lower-tail}
    \omega_{\tau_u}>0
    \quad \nu\text{-a.e.},
    \qquad
    \nu\{0<\omega_{\tau_u}\le t\}
    \le
    A_0t^{\alpha_0},
    \qquad 0<t\le1.
\end{equation}
Conditions~\ref{ass:clip-finite-cov-upper-tail}
and~\ref{ass:clip-finite-lower-tail} imply
\eqref{eq:coverage-clipped-fixed-lower-tail} uniformly over \(\tau_u\ge1\),
with \(\alpha_0=\alpha_{\rm cov}/2\), by
Lemma~\ref{lem:coverage-retained-lower-tail-transfer}. For a fixed clipping
level, \eqref{eq:coverage-clipped-fixed-lower-tail} may instead be assumed
directly.

As in the fitted-KL analysis of Appendix~\ref{app:fitted-kl-proofs}, we
control each empirical minimization uniformly over the ratio class. Because
the clipped recursion is unnormalized, the losses below use direct
sample averages of \(e^h\), without empirical log-normalizers or
self-normalized successor averages. After defining the empirical losses and
critical radii, we analyze retention-indicator estimation, projection
estimation, and the fitted recursion in turn.

For the finite-sample statements, let
\[
    P_{n,X}g=\frac1n\sum_{i=1}^n g(X_i),
    \qquad
    P_{n,0}g=\frac1n\sum_{i=1}^n g(X_i^0),
\]
where \(X_1^0,\ldots,X_n^0\) are i.i.d. from \(\dinit\) and independent of the
transition sample. For functions \(\varphi\) of a transition pair, write
\[
    P_{n,+}\varphi
    =
    \frac1n\sum_{i=1}^n \varphi(X_i,X_i^+),
    \qquad
    X_i^+\mid X_i\sim\Ppi(\cdot\mid X_i).
\]
Throughout this section, write
\(R_{\rm clip}=\log(\tau_u\vee\tau_\ell^{-1})\). For a fixed clipping
level \(\tau_u\), abbreviate
\[
    \varepsilon_{\rm ratio}
    :=\varepsilon_{\rm ratio}(\tau_u),
\]
and let \(\varepsilon_{\rm cls}\) denote the approximation error of the
learned coverage classifier in
Section~\ref{sec:coverage-clipped-finite-sample}, with
\(\tau_{u,n}\) replaced by \(\tau_u\).

For the retention-indicator ERM, define
\[
    L_f^{\rm ret}(c)
    =
    (1-\gamma)E_{\dinit}\{c(X)\}
    +
    \gamma E_{Q_{\nu,\pi}}\{f(X)c(X^+)\}
    +
    \tau_u E_\nu\{1-c(X)\},
\]
and its empirical analogue
\[
    \widehat L_f^{\rm ret}(c)
    =
    (1-\gamma)P_{n,0}c
    +
    \gamma P_{n,+}\{f(X)c(X^+)\}
    +
    \tau_u P_{n,X}(1-c).
\]
For the projection ERM, define
\[
\begin{aligned}
    L_{f,c}^{\rm proj}(h)
    =
    &E_\nu e^{h(X)}
    -
    (1-\gamma)E_{\dinit}\{c(X)h(X)\} \\
    &-
    \gamma E_{Q_{\nu,\pi}}\{f(X)c(X^+)h(X^+)\}
    -
    \tau_u E_\nu\{(1-c(X))h(X)\},
\end{aligned}
\]
and
\[
\begin{aligned}
    \widehat L_{f,c}^{\rm proj}(h)
    =
    &P_{n,X}e^h
    -
    (1-\gamma)P_{n,0}(ch) \\
    &-
    \gamma P_{n,+}\{f(X)c(X^+)h(X^+)\}
    -
    \tau_u P_{n,X}\{(1-c)h\}.
\end{aligned}
\]

\begin{lemma}[Attainment of fitted population projections]
\label{lem:clip-finite-proj-attainment}
Assume Condition~\ref{ass:clip-finite-projection-compact}. Then, for every
\(f\in\sW_{\rm clip}\) and \(c\in\mathcal C\), the population
loss \(L_{f,c}^{\rm proj}\) attains its minimum over
\(\sH_{\rm clip}\).
\end{lemma}

\begin{proof}
For \(h,g\in\sH_{\rm clip}\), the exponential map is
\(\tau_u\)-Lipschitz on
\([\log\tau_\ell,\log\tau_u]\). Since
\(0\le f\le\tau_u\) and \(0\le c\le1\),
\[
\begin{aligned}
    |L_{f,c}^{\rm proj}(h)-L_{f,c}^{\rm proj}(g)|
    \le{}&
    2\tau_u\|h-g\|_{L^1(\nu)}
    +(1-\gamma)\|h-g\|_{L^1(\dinit)}\\
    &+
    \gamma\tau_u\|h-g\|_{L^1(\nup)}.
\end{aligned}
\]
Each measure on the right is dominated by \(3\bar\nu_\pi\).
Consequently, convergence in \(L^2(\bar\nu_\pi)\) implies convergence of
all three terms on the right. Thus \(L_{f,c}^{\rm proj}\) is continuous on the
compact set \(\sH_{\rm clip}\), and therefore attains its
minimum.
\end{proof}

We next define the loss-difference classes and their critical radii. Each loss
difference has one component for the initial sample and another for the
transition sample. We include the \(\nu\)-terms in the transition component
because \(\nu\) is the first-coordinate marginal of \(Q_{\nu,\pi}\). For the
retention-indicator step, let
\[
\begin{aligned}
    \mathcal L^\Delta_{{\rm ret},0}
    &:={}
    \{x\mapsto t(1-\gamma)d(x):
        d=c_1-c_2,\ c_1,c_2\in\mathcal C,\ 0\le t\le1\},\\
    \mathcal L^\Delta_{{\rm ret},Q}
    &:={}
    \{(x,x^+)\mapsto t\{\gamma f(x)d(x^+)-\tau_ud(x)\}:
        f\in\sW_{\rm clip},\ d=c_1-c_2,\
        c_1,c_2\in\mathcal C,\ 0\le t\le1\}.
\end{aligned}
\]
At \(t=1\), the expectations of the two components under \(\dinit\) and
\(Q_{\nu,\pi}\) sum to
\(L_f^{\rm ret}(c_1)-L_f^{\rm ret}(c_2)\).
Define the corresponding localized complexity by
\[
    \mathfrak C_{n,\rm ret}(r)
    :=
    \max\left\{
        \mathcal R_n(\mathcal L^\Delta_{{\rm ret},0},r;\dinit),
        \mathcal R_n(\mathcal L^\Delta_{{\rm ret},Q},r;Q_{\nu,\pi})
    \right\}.
\]
For the margin exponent in
Condition~\ref{ass:clip-finite-selector-margin}, set
\[
    p_{\rm ret}
    :=2+\frac{2}{\alpha_{\rm mar}},
    \qquad
    \beta_{\rm ret}
    :=\frac{\alpha_{\rm mar}}{\alpha_{\rm mar}+1}
    =\frac{2}{p_{\rm ret}},
    \qquad
    q_{\rm ret}
    :=\frac{p_{\rm ret}}{2(p_{\rm ret}-1)}
    =\frac{\alpha_{\rm mar}+1}{\alpha_{\rm mar}+2}.
\]
The retention-indicator critical radius and error are
\[
\begin{aligned}
    \mathfrak r_{n,\rm ret}
    &:={}
    n^{-1/\{2(p_{\rm ret}-1)\}}
    \vee
    \inf\left\{
        r>0:
        \mathfrak C_{n,\rm ret}(r)\le r^{p_{\rm ret}}
    \right\},\\
    a_{n,\rm ret}(\delta)
    &:={}
    \mathfrak r_{n,\rm ret}^{p_{\rm ret}}
    +
    \left\{\frac{\log(6/\delta)}{n}\right\}^{q_{\rm ret}}
    +
    \frac{\log(6/\delta)}{n}.
\end{aligned}
\]

For the projection step, let
\[
\begin{aligned}
    \mathcal L^\Delta_{{\rm proj},0}
    &:={}
    \{x\mapsto-t(1-\gamma)c(x)g(x):
        c\in\mathcal C,\ g=h_1-h_2,\
        h_1,h_2\in\sH_{\rm clip},\ 0\le t\le1\},\\
    \mathcal L^\Delta_{{\rm proj},Q}
    &:={}
    \left\{
        \begin{aligned}
        (x,x^+)\mapsto t\bigl[&
            e^{h_1(x)}-e^{h_2(x)}
            -\gamma f(x)c(x^+)g(x^+)\\[-2pt]
            &-\tau_u\{1-c(x)\}g(x)\bigr]:\\[-2pt]
        &f\in\sW_{\rm clip},\ c\in\mathcal C,\\[-2pt]
        &g=h_1-h_2,\ h_1,h_2\in\sH_{\rm clip},\ 0\le t\le1
        \end{aligned}
    \right\}.
\end{aligned}
\]
At \(t=1\), the expectations of these components sum to
\(L_{f,c}^{\rm proj}(h_1)-L_{f,c}^{\rm proj}(h_2)\).
The corresponding localized complexity is
\[
    \mathfrak C_{n,\rm proj}(r)
    :=
    \max\left\{
        \mathcal R_n(\mathcal L^\Delta_{{\rm proj},0},r;\dinit),
        \mathcal R_n(\mathcal L^\Delta_{{\rm proj},Q},r;Q_{\nu,\pi})
    \right\}.
\]
The joint localized complexity and its quadratic critical radius are
\[
\begin{aligned}
    \mathfrak C_{n,\rm clip}(r)
    &:={}
    \mathfrak C_{n,\rm ret}(r)
    \vee\mathfrak C_{n,\rm proj}(r),\\
    \mathfrak r_{n,\rm clip}
    &:={}
    n^{-1/2}
    \vee
    \inf\left\{
        r>0:
        \mathfrak C_{n,\rm clip}(r)\le r^2
    \right\}.
\end{aligned}
\]
Let \(\kappa_{\rm clip}\ge1\) be a fixed constant, depending only on
\(\tau_\ell,\tau_u,A_{\rm mar}\), and \(\alpha_{\rm mar}\), large enough for
the localization bound in
Lemma~\ref{lem:clip-finite-proj-localization}. Define
\[
\begin{aligned}
    v_{\rm ret}(s)
    &:=s^{\beta_{\rm ret}}+s,\\
    \mathfrak r_{n,\rm proj}
    &:=n^{-1/2}\vee
    \inf\left\{r>0:
        \mathfrak C_{n,\rm proj}(r)\le r^2
    \right\},\\
    a_{n,\rm proj}(s,\delta)
    &:={}
    \mathfrak r_{n,\rm proj}^2
    +\mathfrak C_{n,\rm proj}\left(
        \kappa_{\rm clip}\sqrt{v_{\rm ret}(s)}
    \right)
    +\sqrt{
        \frac{v_{\rm ret}(s)\log(8/\delta)}{n}
    }
    +\frac{\log(8/\delta)}{n}.
\end{aligned}
\]
The unbarred \(a\)-quantities contain sampling error only, whereas the barred
quantities also include \(\varepsilon_{\rm cls}\). The quantity
\(b_{n,\rm clip}\) is the confidence-free baseline used in the lower-tail
multiplier. Define these combined rates and the associated lower-tail
quantities by
\[
\begin{aligned}
    a_{n,\rm clip}(\delta)
    &:={}
    a_{n,\rm proj}\bigl(a_{n,\rm ret}(\delta),\delta\bigr)
    +a_{n,\rm ret}(\delta),\\
    \bar a_{n,\rm ret}(\delta)
    &:={}
    a_{n,\rm ret}(\delta)+\varepsilon_{\rm cls},\\
    \bar a_{n,\rm clip}(\delta)
    &:={}
    a_{n,\rm proj}\bigl(\bar a_{n,\rm ret}(\delta),\delta\bigr)
    +\bar a_{n,\rm ret}(\delta),\\
    b_{n,\rm clip}
    &:={}
    \mathfrak r_{n,\rm proj}^2
    +\mathfrak C_{n,\rm proj}\left[
        \kappa_{\rm clip}
        \left\{v_{\rm ret}
            (\mathfrak r_{n,\rm ret}^{p_{\rm ret}})
        \right\}^{1/2}
    \right]
    +\mathfrak r_{n,\rm ret}^{p_{\rm ret}},\\
    A_{{\rm clip},{\rm lt}}
    &:={}
    A_0\left(1+\frac{\tau_u}{\alpha_0}\right),\\
    \mathfrak m_{n,\alpha_0}^{\rm clip}
    &:={}
    1+\frac{1}{\alpha_0}
    \max\left\{
        0,
        \log\left(
            \frac{A_{{\rm clip},{\rm lt}}}
                 {b_{n,\rm clip}}
        \right)
    \right\}.
\end{aligned}
\]

\begin{lemma}[Scaling of the retention and projection critical radii]
\label{lem:clip-finite-critical-radius-scaling}
Assume Condition~\ref{ass:kl-class}.
For every fixed \(0<A<\infty\) and \(b>0\), there are constants
\(L_{A,b,{\rm ret}}<\infty\) and \(L_{A,b,{\rm proj}}<\infty\), depending
only on \(A\), \(b\), and \(p_{\rm ret}\), such that
\begin{equation}
\label{eq:clip-critical-radius-ret-scaling}
    \mathfrak C_{n,\rm ret}(Ar)
    \le
    br^{p_{\rm ret}}
    \qquad
    \text{for all }r\ge L_{A,b,{\rm ret}}\mathfrak r_{n,\rm ret},
\end{equation}
and
\begin{equation}
\label{eq:clip-critical-radius-proj-scaling}
    \mathfrak C_{n,\rm proj}(Ar)
    \le br^2
    \qquad
    \text{for all }r\ge L_{A,b,{\rm proj}}\mathfrak r_{n,\rm proj}.
\end{equation}
For every fixed \(A\ge1\), there is \(C_A<\infty\) such that
\begin{equation}
\label{eq:clip-proj-error-scaling}
    a_{n,\rm proj}(As,\delta)
    \le C_A a_{n,\rm proj}(s,\delta)
    \qquad\text{for all }s\ge0.
\end{equation}
\end{lemma}

\begin{proof}
The four loss-difference classes are star-shaped. The argument in
Lemma~\ref{lem:kl-critical-radius-scaling} therefore shows that
\(r\mapsto\mathfrak C_{n,\rm ret}(r)/r\) and
\(r\mapsto\mathfrak C_{n,\rm proj}(r)/r\) are nonincreasing.

For the retention-indicator radius, choose
\(t\le2\mathfrak r_{n,\rm ret}\) such that
\(\mathfrak C_{n,\rm ret}(t)\le t^{p_{\rm ret}}\). If \(Ar\ge t\), then
\[
    \mathfrak C_{n,\rm ret}(Ar)
    \le
    \frac{Ar}{t}\mathfrak C_{n,\rm ret}(t)
    \le
    Ar t^{p_{\rm ret}-1} .
\]
If \(Ar<t\), monotonicity gives
\(\mathfrak C_{n,\rm ret}(Ar)\le t^{p_{\rm ret}}\). Taking
\(r\ge L_{A,b,{\rm ret}}\mathfrak r_{n,\rm ret}\), with
\(L_{A,b,{\rm ret}}\) large enough, makes both bounds at most
\(br^{p_{\rm ret}}\).

For the projection radius, choose \(t\le2\mathfrak r_{n,\rm proj}\) such that
\[
    \mathfrak C_{n,\rm proj}(t)\le t^2.
\]
If \(Ar\ge t\), star-shaped scaling gives
\[
    \mathfrak C_{n,\rm proj}(Ar)
    \le \frac{Ar}{t}\mathfrak C_{n,\rm proj}(t)
    \le Ar t.
\]
If \(Ar<t\), monotonicity instead gives
\(\mathfrak C_{n,\rm proj}(Ar)\le t^2\). Choosing
\(L_{A,b,{\rm proj}}\) sufficiently large makes both bounds at most
\(br^2\) whenever
\(r\ge L_{A,b,{\rm proj}}\mathfrak r_{n,\rm proj}\).

Finally, \(v_{\rm ret}(As)\le A v_{\rm ret}(s)\) for \(A\ge1\).
Star-shaped scaling therefore gives
\[
    \mathfrak C_{n,\rm proj}\left(
        \kappa_{\rm clip}\sqrt{v_{\rm ret}(As)}
    \right)
    \le
    \sqrt A\,
    \mathfrak C_{n,\rm proj}\left(
        \kappa_{\rm clip}\sqrt{v_{\rm ret}(s)}
    \right).
\]
In addition,
\[
    \sqrt{\frac{v_{\rm ret}(As)\log(8/\delta)}{n}}
    \le
    \sqrt A
    \sqrt{\frac{v_{\rm ret}(s)\log(8/\delta)}{n}}.
\]
The remaining terms in \(a_{n,\rm proj}\) do not depend on \(s\), which
proves \eqref{eq:clip-proj-error-scaling}.
\end{proof}

\begin{lemma}[Comparison with the joint critical radius]
\label{lem:clip-finite-rate-comparison}
If
\[
    \mathfrak r_{n,\rm clip}
    \vee\frac{\log(1/\delta)}{n}
    \le1,
\]
then
\begin{equation}
\label{eq:clip-rate-critical-radii}
    \mathfrak r_{n,\rm proj}
    \le C\mathfrak r_{n,\rm clip},
    \qquad
    \mathfrak r_{n,\rm ret}
    \le C\mathfrak r_{n,\rm clip}^{
        \alpha_{\rm mar}/(\alpha_{\rm mar}+2)},
\end{equation}
and
\begin{equation}
\label{eq:clip-rate-retention}
    a_{n,\rm ret}(\delta)
    \le
    C\left\{
        \mathfrak r_{n,\rm clip}^{2q_{\rm ret}}
        +\left\{\frac{\log(1/\delta)}{n}\right\}^{q_{\rm ret}}
    \right\}.
\end{equation}
Moreover,
\begin{equation}
\label{eq:clip-rate-statistical}
    a_{n,\rm clip}(\delta)
    \le
    C\mathcal E_{n,\rm stat}(\delta).
\end{equation}
If, in addition, \(\varepsilon_{\rm cls}\le1\), then
\begin{equation}
\label{eq:clip-rate-barred}
    \bar a_{n,\rm clip}(\delta)
    \le
    C\left\{
        \mathcal E_{n,\rm stat}(\delta)
        +\varepsilon_{\rm cls}
    \right\}.
\end{equation}
For every \(n\ge1\),
\begin{equation}
\label{eq:clip-rate-lower-tail-factor}
    \mathfrak m_{n,\alpha_0}^{\rm clip}
    \le
    1+\frac{1}{\alpha_0}
    \left[
        \log(en)
        +\log_+\left\{
            A_0\left(1+\frac{\tau_u}{\alpha_0}\right)
        \right\}
    \right].
\end{equation}
Here and below, \(C<\infty\) may depend on the fixed constants in the stated
finite-sample conditions other than \(A_0\), and on
\(\tau_\ell,\tau_u\), but not on \(n\) or \(\delta\). All dependence on
\(A_0\) in \eqref{eq:clip-rate-lower-tail-factor} is explicit.
\end{lemma}

\begin{proof}
Write
\[
    r=\mathfrak r_{n,\rm clip},
    \qquad
    x=\frac{\log(1/\delta)}{n},
    \qquad
    a=\frac{\alpha_{\rm mar}}{\alpha_{\rm mar}+2}
      =\frac{1}{p_{\rm ret}-1}.
\]
Choose \(t\le2r\) such that
\(\mathfrak C_{n,\rm clip}(t)\le t^2\). Since
\(\mathfrak C_{n,\rm proj}\le\mathfrak C_{n,\rm clip}\) and
\(r\ge n^{-1/2}\), the definition of the projection critical radius gives
\begin{equation}
\label{eq:clip-rate-proj-radius}
    \mathfrak r_{n,\rm proj}\le2r.
\end{equation}

For the retention-indicator radius, let \(u=Lr^a\), where \(L\ge2\) is a fixed
constant. Since \(r\le1\), we have \(u\ge t\). Star-shaped scaling gives
\[
    \mathfrak C_{n,\rm ret}(u)
    \le
    \frac{u}{t}\mathfrak C_{n,\rm clip}(t)
    \le ut
    \le2ur.
\]
Because \(u^{p_{\rm ret}-1}=L^{p_{\rm ret}-1}r\), choosing \(L\) large
enough makes \(2ur\le u^{p_{\rm ret}}\). Moreover,
\(r\ge n^{-1/2}\) implies
\(r^a\ge n^{-1/\{2(p_{\rm ret}-1)\}}\). The definition of the retention-indicator
critical radius therefore gives
\begin{equation}
\label{eq:clip-rate-ret-radius}
    \mathfrak r_{n,\rm ret}\le Cr^a.
\end{equation}
Equations~\eqref{eq:clip-rate-proj-radius} and
\eqref{eq:clip-rate-ret-radius} prove
\eqref{eq:clip-rate-critical-radii}.

Because \(r^2\ge n^{-1}\), for \(j\in\{6,8\}\),
\[
    \frac{\log(j/\delta)}{n}
    =
    x+\frac{\log j}{n}
    \le
    x+C r^2.
\]
Since \(ap_{\rm ret}=2q_{\rm ret}\), the definition of
\(a_{n,\rm ret}(\delta)\) and \eqref{eq:clip-rate-ret-radius} give
\begin{equation}
\label{eq:clip-rate-ret-bound}
    a_{n,\rm ret}(\delta)
    \le
    C\left\{
        r^{2q_{\rm ret}}+x^{q_{\rm ret}}
    \right\}.
\end{equation}
Here we used \(r\vee x\le1\) and \(q_{\rm ret}\le1\) to absorb the
linear confidence term.
Thus, \eqref{eq:clip-rate-retention} holds. Set
\(s=a_{n,\rm ret}(\delta)\). Under the stated small-radius condition, this
bound and the identities
\(q_{\rm ret}\beta_{\rm ret}=a\) and
\(2q_{\rm ret}\beta_{\rm ret}=2a\) give
\begin{equation}
\label{eq:clip-rate-vret}
    \sqrt{v_{\rm ret}(s)}
    \le
    C\left\{r^a+x^{a/2}\right\}.
\end{equation}

The star-shaped scaling argument also shows that, for every \(R>0\),
\begin{equation}
\label{eq:clip-rate-proj-complexity}
    \mathfrak C_{n,\rm proj}(R)
    \le
    \mathfrak C_{n,\rm clip}(R)
    \le
    4r^2+2Rr.
\end{equation}
Applying \eqref{eq:clip-rate-proj-complexity} with
\(R=\kappa_{\rm clip}\sqrt{v_{\rm ret}(s)}\), and using
\(2q_{\rm ret}=1+a\), yields
\begin{equation}
\label{eq:clip-rate-aclip}
\begin{aligned}
    a_{n,\rm clip}(\delta)
    \le C\bigl\{&
        r^{1+a}
        +r x^{a/2}
        +r^a x^{1/2}
        +x^{(1+a)/2}
    \bigr\}\\
    ={}&
    C\left(r+x^{1/2}\right)
    \left(r^a+x^{a/2}\right).
\end{aligned}
\end{equation}
Indeed, the projection confidence level satisfies
\(\log(8/\delta)/n\le x+Cr^2\), so its square-root term contributes at
most the four summands in \eqref{eq:clip-rate-aclip}. The retention-indicator error contributes
\(r^{1+a}+x^{(1+a)/2}\), while
\(r^2\le r^{1+a}\) and
\(x\le x^{(1+a)/2}\) under the small-radius condition. This proves the
bound \eqref{eq:clip-rate-statistical}.

Now suppose \(\varepsilon_{\rm cls}\le1\), and abbreviate
\(e=\varepsilon_{\rm cls}\). Since
\(\bar a_{n,\rm ret}(\delta)=s+e\), concavity of
\(u\mapsto u^{\beta_{\rm ret}}\) gives
\begin{equation}
\label{eq:clip-rate-vret-sum}
    v_{\rm ret}(s+e)
    \le
    v_{\rm ret}(s)+v_{\rm ret}(e).
\end{equation}
Apply the bound
\eqref{eq:clip-rate-proj-complexity} at
\(R=\kappa_{\rm clip}\sqrt{v_{\rm ret}(s+e)}\). Together with
\(\sqrt{u+v}\le\sqrt u+\sqrt v\), the definition of
\(\bar a_{n,\rm clip}(\delta)\), and the established bound for
\(a_{n,\rm clip}(\delta)\), this yields
\begin{equation}
\label{eq:clip-rate-bar-pre-young}
    \bar a_{n,\rm clip}(\delta)
    \le
    C\mathcal E_{n,\rm stat}(\delta)
    +C\left(r+x^{1/2}\right)\sqrt{v_{\rm ret}(e)}
    +Ce.
\end{equation}
Because \(0\le e\le1\),
\(v_{\rm ret}(e)\le2e^{\beta_{\rm ret}}\). Young's inequality with
conjugate exponents
\[
    \frac{2}{\beta_{\rm ret}}
    \quad\text{and}\quad
    \frac{2}{2-\beta_{\rm ret}}
    =2q_{\rm ret}
    =1+a
\]
therefore gives
\begin{equation}
\label{eq:clip-rate-young}
    \left(r+x^{1/2}\right)e^{\beta_{\rm ret}/2}
    \le
    Ce+C\left(r+x^{1/2}\right)^{1+a}.
\end{equation}
Since \(0<a\le1\), concavity also gives
\begin{equation}
\label{eq:clip-rate-stat-concavity}
    \left(r+x^{1/2}\right)^{1+a}
    \le
    \left(r+x^{1/2}\right)
    \left(r^a+x^{a/2}\right)
    =\mathcal E_{n,\rm stat}(\delta).
\end{equation}
Combining \eqref{eq:clip-rate-bar-pre-young},
\eqref{eq:clip-rate-young}, and \eqref{eq:clip-rate-stat-concavity} proves
\eqref{eq:clip-rate-barred}.

Finally, \(b_{n,\rm clip}\ge
\mathfrak r_{n,\rm proj}^2\ge n^{-1}\). Thus
\[
    \mathfrak m_{n,\alpha_0}^{\rm clip}
    \le
    1+\frac{1}{\alpha_0}
    \max\{0,\log(A_{{\rm clip},{\rm lt}}n)\}
    \le
    1+\frac{1}{\alpha_0}
    \left[
        \log(en)
        +\log_+\left\{
            A_0\left(1+\frac{\tau_u}{\alpha_0}\right)
        \right\}
    \right],
\]
which proves \eqref{eq:clip-rate-lower-tail-factor}.
\end{proof}

\begin{lemma}[VC-class critical-radius bounds]
\label{lem:clip-finite-vc-radius}
Suppose that \(\mathcal C\) has VC dimension
\(d_{\mathcal C}\ge1\) and that
\(\sH_{\rm clip}\) is VC-subgraph with dimension
\(d_{\mathcal H}\ge1\). Then
\begin{equation}
\label{eq:clip-vc-quadratic-radii}
    \mathfrak r_{n,\rm clip}
    \vee
    \mathfrak r_{n,\rm proj}
    \le
    C_{\tau_\ell,\tau_u}
    \left\{
        \frac{
            (d_{\mathcal C}+d_{\mathcal H})\log(en)
        }{n}
    \right\}^{1/2}.
\end{equation}
Moreover,
\begin{equation}
\label{eq:clip-vc-retention-radius}
    \mathfrak r_{n,\rm ret}
    \le
    C_{\tau_\ell,\tau_u,\alpha_{\rm mar}}
    \left\{
        \frac{
            (d_{\mathcal C}+d_{\mathcal H})\log(en)
        }{n}
    \right\}^{
        \alpha_{\rm mar}/\{2(\alpha_{\rm mar}+2)\}
    }.
\end{equation}
Consequently, for every \(0<\delta<1\),
\begin{equation}
\label{eq:clip-vc-statistical-rate}
    \mathcal E_{n,\rm stat}(\delta)
    \le
    C_{\tau_\ell,\tau_u,\alpha_{\rm mar}}
    \left\{
        \frac{
            (d_{\mathcal C}+d_{\mathcal H})\log(en)
            +\log(1/\delta)
        }{n}
    \right\}^{q_{\rm ret}}.
\end{equation}
\end{lemma}

\begin{proof}
Write \(d=d_{\mathcal C}+d_{\mathcal H}\), and let
\[
    \mathcal D_{\mathcal C}=\mathcal C-\mathcal C,
    \qquad
    \mathcal D_{\mathcal H}
    =\sH_{\rm clip}-\sH_{\rm clip}.
\]
Uniform VC entropy bounds and the Lipschitz property of
\(h\mapsto e^h\) on \([\log\tau_\ell,\log\tau_u]\) give
\[
\begin{aligned}
    &\log N\{\epsilon,\mathcal C,L^2(Q)\}
    \vee
    \log N\{\epsilon,\mathcal D_{\mathcal C},L^2(Q)\}
    \vee
    \log N\{\epsilon,\mathcal D_{\mathcal H},L^2(Q)\}\\
    &\qquad\vee
    \log N\{\epsilon,\sW_{\rm clip},L^2(Q)\}
    \le
    C d\log\left(\frac{C_{\tau_\ell,\tau_u}}{\epsilon}\right)
\end{aligned}
\]
uniformly over probability measures \(Q\). Here the bounds for the difference
classes follow by taking products of two covering nets for the corresponding
base class.

For uniformly bounded functions,
\[
    \|ab-a'b'\|_{L^2(Q)}
    \le
    \|a\|_\infty\|b-b'\|_{L^2(Q)}
    +\|b'\|_\infty\|a-a'\|_{L^2(Q)}.
\]
For a function class \(\mathcal G\), write
\(\operatorname{star}(\mathcal G)=\{tg:g\in\mathcal G,\ 0\le t\le1\}\);
coordinate subscripts indicate composition with the corresponding coordinate
of \((x,x^+)\). The definitions of the loss classes give
\[
\begin{aligned}
    \mathcal L^\Delta_{{\rm ret},0}
    &\subseteq \operatorname{star}(\mathcal D_{\mathcal C}),\\
    \mathcal L^\Delta_{{\rm ret},Q}
    &\subseteq \operatorname{star}\{
        \gamma(\sW_{\rm clip})_1\mathcal D_{\mathcal C,2}
        -\tau_u\mathcal D_{\mathcal C,1}\},\\
    \mathcal L^\Delta_{{\rm proj},0}
    &\subseteq \operatorname{star}(
        \mathcal C\mathcal D_{\mathcal H}),\\
    \mathcal L^\Delta_{{\rm proj},Q}
    &\subseteq \operatorname{star}\{
        (\sW_{\rm clip}-\sW_{\rm clip})_1
        -\gamma(\sW_{\rm clip})_1\mathcal C_2\mathcal D_{\mathcal H,2}
        -\tau_u(1-\mathcal C_1)\mathcal D_{\mathcal H,1}\}.
\end{aligned}
\]
Applying the product inequality repeatedly, together with the covering-number
bounds for sums and for the scalar \(t\in[0,1]\), therefore gives, for each of
the four loss-difference classes,
\[
    \log N\{\epsilon,\mathcal L,L^2(Q)\}
    \le
    C d\log\left(
        \frac{C_{\tau_\ell,\tau_u}}{\epsilon}
    \right),
    \qquad
    0<\epsilon\le C_{\tau_\ell,\tau_u},
\]
uniformly over probability measures \(Q\).

Lemma~\ref{lem:tool-local-entropy} now gives, for
\(n^{-1/2}\le r\le B_{\tau_\ell,\tau_u}\), where
\(B_{\tau_\ell,\tau_u}\ge1\) is a common envelope,
\begin{equation}
\label{eq:clip-vc-local-complexity}
    \mathfrak C_{n,\rm clip}(r)
    \le
    C_{\tau_\ell,\tau_u}
    r\left\{
        \frac{d\log(en)}{n}
    \right\}^{1/2}.
\end{equation}
Let
\[
    r_0
    =
    L_{\tau_\ell,\tau_u}
    \left\{\frac{d\log(en)}{n}\right\}^{1/2}
\]
for a sufficiently large constant \(L_{\tau_\ell,\tau_u}\). If
\(r_0\le B_{\tau_\ell,\tau_u}\),
\eqref{eq:clip-vc-local-complexity} gives
\(\mathfrak C_{n,\rm clip}(r_0)\le r_0^2\). If
\(r_0>B_{\tau_\ell,\tau_u}\), the global envelope bound instead gives
\[
    \mathfrak C_{n,\rm clip}(r_0)
    \le B_{\tau_\ell,\tau_u}
    \le r_0^2.
\]
Finally, \(r_0\ge n^{-1/2}\), so the definition of
\(\mathfrak r_{n,\rm clip}\) proves its bound in
\eqref{eq:clip-vc-quadratic-radii}. The same
quadratic fixed-point argument applies to
\(\mathfrak C_{n,\rm proj}\le\mathfrak C_{n,\rm clip}\) and gives the
remaining bound in \eqref{eq:clip-vc-quadratic-radii}.

For the retention-indicator radius, set
\[
    a=\frac{\alpha_{\rm mar}}{\alpha_{\rm mar}+2}
      =\frac{1}{p_{\rm ret}-1},
    \qquad
    u_0
    =
    L_{\tau_\ell,\tau_u,\alpha_{\rm mar}}
    \left\{\frac{d\log(en)}{n}\right\}^{a/2}.
\]
If \(u_0\le B_{\tau_\ell,\tau_u}\),
\eqref{eq:clip-vc-local-complexity} gives
\[
    \mathfrak C_{n,\rm ret}(u_0)
    \le
    C_{\tau_\ell,\tau_u}
    u_0\left\{\frac{d\log(en)}{n}\right\}^{1/2}
    \le
    u_0^{p_{\rm ret}}
\]
when \(L_{\tau_\ell,\tau_u,\alpha_{\rm mar}}\) is sufficiently large. If
\(u_0>B_{\tau_\ell,\tau_u}\), the global envelope bound gives the same
fixed-point inequality because
\[
    \mathfrak C_{n,\rm ret}(u_0)
    \le B_{\tau_\ell,\tau_u}
    <u_0
    \le u_0^{p_{\rm ret}}.
\]
Since
\(d\log(en)\ge1\), \(u_0\) also dominates the deterministic term in the
definition of \(\mathfrak r_{n,\rm ret}\). This proves the retention-indicator radius
bound \eqref{eq:clip-vc-retention-radius}.

Set
\[
    z
    =
    \frac{d\log(en)+\log(1/\delta)}{n}.
\]
Equation~\eqref{eq:clip-vc-quadratic-radii} and the definition of
\(\mathcal E_{n,\rm stat}(\delta)\) give
\[
    \mathcal E_{n,\rm stat}(\delta)
    \le
    C z^{1/2}z^{a/2}
    =
    C z^{q_{\rm ret}},
\]
because \(q_{\rm ret}=(1+a)/2\). This proves
\eqref{eq:clip-vc-statistical-rate}.
\end{proof}

\subsection{Retention-indicator estimation}

\begin{lemma}[Retention-indicator regret identity and margin control]
\label{lem:clip-finite-sel-regret}
Fix \(f\in\sW_{\rm clip}\), write
\[
    \mu_f=(1-\gamma)\dinit+\gamma(f\nu)\Ppi
    =(\Bpigcov f)\nu+\mu_{f,\perp},
\]
and let \(c_{f,\tau_u}^\star\) be an oracle retention indicator satisfying
the pointwise characterization in
Section~\ref{sec:coverage-clipped-fitted}. Define
\[
    \Delta_f^{\rm ret}(c)
    =
    L_f^{\rm ret}(c)
    -
    L_f^{\rm ret}(c_{f,\tau_u}^\star).
\]
Then, for every binary measurable \(c\),
\begin{equation}
\label{eq:retention-regret-identity}
\begin{aligned}
    \Delta_f^{\rm ret}(c)
    ={}&
    E_\nu\left[
        |(\Bpigcov f)(X)-\tau_u|
        1\{c(X)\ne c_{f,\tau_u}^\star(X)\}
    \right]
    +\int c\,\dd\mu_{f,\perp}.
\end{aligned}
\end{equation}
If Condition~\ref{ass:clip-finite-selector-margin} also holds, then
\begin{equation}
\label{eq:retention-margin-control}
    \nu\{c\ne c_{f,\tau_u}^\star\}
    \le
    C_{\rm ret}
    \{\Delta_f^{\rm ret}(c)\}^{\alpha_{\rm mar}/(\alpha_{\rm mar}+1)},
\end{equation}
where \(C_{\rm ret}<\infty\) depends only on
\(A_{\rm mar}\) and \(\alpha_{\rm mar}\).
Moreover,
\begin{equation}
\label{eq:retention-mu-disagreement}
    \int(c-c_{f,\tau_u}^\star)^2\,\dd\mu_f
    \le
    \tau_u\nu\{c\ne c_{f,\tau_u}^\star\}
    +\Delta_f^{\rm ret}(c).
\end{equation}
\end{lemma}

\begin{proof}
By the Lebesgue decomposition of \(\mu_f\), the retention-indicator loss is
\begin{equation}
\label{eq:retention-loss-decomposition}
    L_f^{\rm ret}(c)
    =
    \tau_u
    +\int c(x)\{(\Bpigcov f)(x)-\tau_u\}\,\nu(\dd x)
    +\int c\,\dd\mu_{f,\perp}.
\end{equation}
The oracle retention indicator equals one on \(\{\Bpigcov f<\tau_u\}\), equals zero on
\(\{\Bpigcov f>\tau_u\}\), and vanishes \(\mu_{f,\perp}\)-almost everywhere. Its
value on the \(\nu\)-tie set \(\{\Bpigcov f=\tau_u\}\) is immaterial. Subtracting
\eqref{eq:retention-loss-decomposition} evaluated at
\(c=c_{f,\tau_u}^\star\) gives \eqref{eq:retention-regret-identity}.

Let \(A_c=\{c\ne c_{f,\tau_u}^\star\}\). For \(0<s\le1\),
\[
    \nu(A_c)
    \le
    \nu\{|\Bpigcov f-\tau_u|\le s\}
    +
    \nu\{A_c,\ |\Bpigcov f-\tau_u|>s\}
    \le
    A_{\rm mar}s^{\alpha_{\rm mar}}
    +
    s^{-1}\Delta_f^{\rm ret}(c).
\]
If \(\Delta_f^{\rm ret}(c)=0\), letting \(s\downarrow0\) in this set bound
and using the margin condition gives \(\nu(A_c)=0\).
Otherwise, take
\[
    s=
    \left\{
        \frac{\Delta_f^{\rm ret}(c)}
        {A_{\rm mar}\alpha_{\rm mar}}
    \right\}^{1/(\alpha_{\rm mar}+1)}
    \wedge1.
\]
If the untruncated choice exceeds one, then \(\nu(A_c)\le1\), and the same
bound follows after enlarging the constant. The margin condition also implies
\(\nu\{\Bpigcov f=\tau_u\}=0\); hence retention-indicator disagreement on ties is \(\nu\)-null.

We next control retention-indicator disagreement under the full Bellman measure. Let
\(A_-=\{c=0,c_{f,\tau_u}^\star=1\}\) and
\(A_+=\{c=1,c_{f,\tau_u}^\star=0\}\). On \(A_-\), we have \(\Bpigcov f\le\tau_u\),
whereas on \(A_+\),
\(\Bpigcov f=\tau_u+(\Bpigcov f-\tau_u)\). Since the oracle retention indicator vanishes on the
singular component,
\[
\begin{aligned}
    \int(c-c_{f,\tau_u}^\star)^2\,\dd\mu_f
    &=\int_{A_-\cup A_+}\Bpigcov f\,\dd\nu
      +\int c\,\dd\mu_{f,\perp}\\
    &\le \tau_u\nu(A_-\cup A_+)
      +\int_{A_+}(\Bpigcov f-\tau_u)\,\dd\nu
      +\int c\,\dd\mu_{f,\perp}\\
    &\le \tau_u\nu\{c\ne c_{f,\tau_u}^\star\}
      +\Delta_f^{\rm ret}(c).
\end{aligned}
\]
This proves \eqref{eq:retention-mu-disagreement}.
\end{proof}

\begin{lemma}[Uniform retention-indicator excess risk]
\label{lem:clip-finite-sel-erm}
Assume Condition~\ref{ass:clip-finite-selector-margin}. For each
\(f\in\sW_{\rm clip}\), let
\[
    \widehat c_f
    \in
    \argmin_{c\in\mathcal C}\widehat L_f^{\rm ret}(c)
\]
be an exact empirical minimizer. Then, with probability at least
\(1-\delta\),
\[
    \sup_{f\in\sW_{\rm clip}}
    \Delta_f^{\rm ret}(\widehat c_f)
    \le
    C_{\rm ret,erm}\bar a_{n,\rm ret}(\delta),
\]
where \(C_{\rm ret,erm}<\infty\) depends only on
\(\tau_u,A_{\rm mar}\), and \(\alpha_{\rm mar}\).
\end{lemma}

\begin{proof}
Fix \(\eta>0\). For each \(f\in\sW_{\rm clip}\), choose
\(c_{f,\eta}^\circ\in\mathcal C\) such that
\[
    \Delta_f^{\rm ret}(c_{f,\eta}^\circ)
    \le
    \inf_{c\in\mathcal C}\Delta_f^{\rm ret}(c)+\eta
    \le
    \varepsilon_{\rm cls}+\eta.
\]
Such a retention indicator exists by the definition of the infimum; no attainment
condition on \(\mathcal C\) is needed. For \(c\in\mathcal C\), define
\[
    \mathbb Z_{f,\eta}^{\rm ret}(c)
    =
    \{\widehat L_f^{\rm ret}(c)-L_f^{\rm ret}(c)\}
    -
    \{\widehat L_f^{\rm ret}(c_{f,\eta}^\circ)
        -L_f^{\rm ret}(c_{f,\eta}^\circ)\}.
\]
Put \(d=c-c_{f,\eta}^\circ\). Since both retention indicators belong to
\(\mathcal C\), the initial-law loss difference belongs to
\(\mathcal L^\Delta_{{\rm ret},0}\), while the combined transition-sample
loss difference
\[
    (x,x^+)\mapsto \gamma f(x)d(x^+)-\tau_ud(x)
\]
belongs to \(\mathcal L^\Delta_{{\rm ret},Q}\).

Suppose that
\[
    \Delta_f^{\rm ret}(c)
    \vee
    \Delta_f^{\rm ret}(c_{f,\eta}^\circ)
    \le r^{p_{\rm ret}}.
\]
The triangle inequality for retention-indicator disagreement and
Lemma~\ref{lem:clip-finite-sel-regret} give
\[
\begin{aligned}
    \nu\{c\ne c_{f,\eta}^\circ\}
    &\le
    \nu\{c\ne c_{f,\tau_u}^\star\}
    +\nu\{c_{f,\eta}^\circ\ne c_{f,\tau_u}^\star\}\\
    &\le
    C\left[
        \{\Delta_f^{\rm ret}(c)\}^{\beta_{\rm ret}}
        +\{\Delta_f^{\rm ret}(c_{f,\eta}^\circ)\}^{\beta_{\rm ret}}
    \right]
    \le Cr^2,
\end{aligned}
\]
where \(p_{\rm ret}\beta_{\rm ret}=2\) yields the bound \(Cr^2\). Applying
\eqref{eq:retention-mu-disagreement} separately to \(c\) and
\(c_{f,\eta}^\circ\) yields
\[
\begin{aligned}
    \int(c-c_{f,\eta}^\circ)^2\,\dd\mu_f
    &\le
    2\int(c-c_{f,\tau_u}^\star)^2\,\dd\mu_f
    +2\int(c_{f,\eta}^\circ-c_{f,\tau_u}^\star)^2\,\dd\mu_f\\
    &\le Cr^2+Cr^{p_{\rm ret}}
    \le Cr^2
\end{aligned}
\]
for \(0<r\le1\). Therefore
\[
\begin{aligned}
    \|(1-\gamma)d\|_{L^2(\dinit)}^2
    &\le (1-\gamma)\int d^2\,\dd\mu_f
    \le Cr^2,\\
    \|\gamma f(X)d(X^+)-\tau_ud(X)\|_{L^2(Q_{\nu,\pi})}^2
    &\le
    2\|\gamma f(X)d(X^+)\|_{L^2(Q_{\nu,\pi})}^2
    +2\tau_u^2\|d\|_{L^2(\nu)}^2\\
    &\le
    2\tau_u\int d^2\,\dd\mu_f
    +2\tau_u^2\nu\{c\ne c_{f,\eta}^\circ\}
    \le Cr^2.
\end{aligned}
\]
The transition-sample bound uses \(f^2\le\tau_u f\), the fact that the first-coordinate
marginal of \(Q_{\nu,\pi}\) is \(\nu\), and the definition of \(\mu_f\).
Thus the two empirical-process terms have \(L^2(\dinit)\)- and
\(L^2(Q_{\nu,\pi})\)-radii of order \(r\), respectively. This conclusion does
not require either distribution to be dominated by \(\nu\). For \(r>1\), the
same bound follows after enlarging the constant because both classes have
bounded envelopes.

Symmetrization and Lemma~\ref{lem:tool-bousquet} now imply that, for every
\(u\ge0\), with probability at least \(1-2e^{-u}\),
\[
    \sup_{\substack{f\in\sW_{\rm clip},\,c\in\mathcal C:\\
        \Delta_f^{\rm ret}(c)\vee
        \Delta_f^{\rm ret}(c_{f,\eta}^\circ)
        \le r^{p_{\rm ret}}}}
    |\mathbb Z_{f,\eta}^{\rm ret}(c)|
    \le
    C\left\{
        \mathfrak C_{n,\rm ret}(Cr)
        +
        r\sqrt{\frac{u}{n}}
        +
        \frac{u}{n}
    \right\}.
\]
Take \(r\) at least a sufficiently large constant multiple of
\(\mathfrak r_{n,\rm ret}\). Lemma~\ref{lem:clip-finite-critical-radius-scaling}
then gives
\(\mathfrak C_{n,\rm ret}(Cr)\le cr^{p_{\rm ret}}\), for a numerical
\(c>0\) chosen small enough. Young's inequality gives
\[
    r\sqrt{\frac{u}{n}}
    \le
    c r^{p_{\rm ret}}
    +
    C\left(\frac{u}{n}\right)^{q_{\rm ret}}.
\]
For the peeling argument, write
\[
    \bar b_{\rm ret,\eta}
    =a_{n,\rm ret}(\delta)+\varepsilon_{\rm cls}+\eta.
\]
Choose a sufficiently large fixed
constant \(C_0\), and define
\[
\begin{aligned}
    \mathcal S_0
    &=\{(f,c):\Delta_f^{\rm ret}(c)\le C_0\bar b_{\rm ret,\eta}\},\\
    \mathcal S_j
    &=\{(f,c):2^{j-1}C_0\bar b_{\rm ret,\eta}
        <\Delta_f^{\rm ret}(c)
        \le2^jC_0\bar b_{\rm ret,\eta}\},\qquad j\ge1.
\end{aligned}
\]
For \(j\ge0\), set
\[
    r_j=(2^jC_0\bar b_{\rm ret,\eta})^{1/p_{\rm ret}},
    \qquad
    u_j=\log(6/\delta)+2j\log2.
\]
Since
\(\Delta_f^{\rm ret}(c_{f,\eta}^\circ)
\le\varepsilon_{\rm cls}+\eta\le\bar b_{\rm ret,\eta}\),
the comparator lies in every localization ball used for a nonempty shell.
Moreover,
\(\bar b_{\rm ret,\eta}\ge
\mathfrak r_{n,\rm ret}^{p_{\rm ret}}\), so choosing
\(C_0\) large enough makes the critical-radius bound applicable at every
\(r_j\). The deterministic term in the definition of
\(\mathfrak r_{n,\rm ret}\) ensures that
\(\bar b_{\rm ret,\eta}\ge n^{-q_{\rm ret}}\), while the definition of
\(a_{n,\rm ret}(\delta)\) gives
\(\bar b_{\rm ret,\eta}\ge
\{\log(6/\delta)/n\}^{q_{\rm ret}}\) and
\(\bar b_{\rm ret,\eta}\ge\log(6/\delta)/n\). Consequently,
\[
    \left(\frac{u_j}{n}\right)^{q_{\rm ret}}
    +\frac{u_j}{n}
    \le C(1+j^{q_{\rm ret}}+j)\bar b_{\rm ret,\eta}
    \le C2^j\bar b_{\rm ret,\eta}.
\]
The constants in the critical-radius and Young inequalities can therefore be
chosen so that the localized deviation is bounded by
\(C\bar b_{\rm ret,\eta}\) on
\(\mathcal S_0\), and by
\(\Delta_f^{\rm ret}(c)/4+C\bar b_{\rm ret,\eta}\) on every
\(\mathcal S_j\), \(j\ge1\). Finally,
\[
    2\sum_{j=0}^\infty e^{-u_j}
    =\frac{4\delta}{9}
    \le\delta.
\]
A union bound over the shells thus gives an event of probability at least
\(1-\delta\) on which
\[
    |\mathbb Z_{f,\eta}^{\rm ret}(c)|
    \le
    \frac14\Delta_f^{\rm ret}(c)
    +
    C\bar b_{\rm ret,\eta}
\]
simultaneously for all \(f\in\sW_{\rm clip}\) and \(c\in\mathcal C\).

Exact ERM gives
\(\widehat L_f^{\rm ret}(\widehat c_f)
\le \widehat L_f^{\rm ret}(c_{f,\eta}^\circ)\). Applying the uniform deviation
bound with \(c=\widehat c_f\) gives
\[
    \Delta_f^{\rm ret}(\widehat c_f)
    \le
    \Delta_f^{\rm ret}(c_{f,\eta}^\circ)
    +
    \frac14\Delta_f^{\rm ret}(\widehat c_f)
    +
    C\bar b_{\rm ret,\eta}.
\]
Moving the fractional term to the left and taking the supremum over \(f\)
give a bound by \(C\bar b_{\rm ret,\eta}\). Apply the peeling argument with
\(\eta=a_{n,\rm ret}(\delta)\). Then
\(\bar b_{\rm ret,\eta}\le2\bar a_{n,\rm ret}(\delta)\), which proves the
stated uniform retention-indicator regret bound.
\end{proof}

\subsection{Projection estimation}

\begin{lemma}[Quadratic curvature of the projection loss]
\label{lem:clip-finite-proj-curvature}
Assume Conditions~\ref{ass:kl-class} and
\ref{ass:clip-finite-projection-compact}. For fixed
\(f\in\sW_{\rm clip}\) and \(c\in\mathcal C\),
Lemma~\ref{lem:clip-finite-proj-attainment} ensures that the population
minimizer set is nonempty. Choose \(h^\star_{f,c}\) from this set and define
\[
    h^\star_{f,c}
    \in
    \argmin_{h\in\sH_{\rm clip}}L_{f,c}^{\rm proj}(h),
    \qquad
    \Delta^{\rm proj}_{f,c}(h)
    =
    L_{f,c}^{\rm proj}(h)-L_{f,c}^{\rm proj}(h^\star_{f,c}).
\]
Then, for every \(h\in\sH_{\rm clip}\),
\begin{equation}
\label{eq:clip-proj-curvature}
    \Delta^{\rm proj}_{f,c}(h)
    \ge
    \frac{\tau_\ell}{2}
    \|h-h^\star_{f,c}\|_{L^2(\nu)}^2.
\end{equation}
\end{lemma}

\begin{proof}
Let \(g=h-h^\star_{f,c}\) and
\(h_t=h^\star_{f,c}+tg\), \(0\le t\le1\). Condition~\ref{ass:kl-class}
and the pointwise bounds defining \(\sH_{\rm clip}\) imply that this
set is convex. Hence \(h_t\in\sH_{\rm clip}\) for every
\(t\in[0,1]\). The linear part of \(L_{f,c}^{\rm proj}\) has zero second
derivative along this path, whereas
\[
    \frac{\dd^2}{\dd t^2}E_\nu e^{h_t(X)}
    =
    E_\nu\{e^{h_t(X)}g^2(X)\}
    \ge
    \tau_\ell\|g\|_{L^2(\nu)}^2 .
\]
The one-sided derivative at \(t=0\) is nonnegative because
\(h^\star_{f,c}\) minimizes the loss. Taylor's formula with integral
remainder therefore gives
\[
\begin{aligned}
    \Delta^{\rm proj}_{f,c}(h)
    &\ge
    \int_0^1(1-t)
    E_\nu\{e^{h_t(X)}g^2(X)\}\,\dd t\\
    &\ge
    \frac{\tau_\ell}{2}\|g\|_{L^2(\nu)}^2 .
\end{aligned}
\]
This is the stated curvature bound.
\end{proof}

\begin{lemma}[Localization of projection loss differences]
\label{lem:clip-finite-proj-localization}
Assume Conditions~\ref{ass:kl-class},
\ref{ass:clip-finite-selector-margin}, and
\ref{ass:clip-finite-projection-compact}. There is a finite constant
\(\kappa_{\rm clip}\), depending only on
\(\tau_\ell,\tau_u,A_{\rm mar}\), and \(\alpha_{\rm mar}\), with the
following property. If
\[
    \Delta_f^{\rm ret}(c)\le s,
    \qquad
    \Delta^{\rm proj}_{f,c}(h)\le t,
\]
and \(g=h-h^\star_{f,c}\), then
\begin{equation}
\label{eq:clip-proj-loss-localization}
\begin{aligned}
    \|(1-\gamma)cg\|_{L^2(\dinit)}^2
    &\le \kappa_{\rm clip}^2\{t+v_{\rm ret}(s)\},\\
    \|\zeta_{f,c,h}\|_{L^2(Q_{\nu,\pi})}^2
    &\le \kappa_{\rm clip}^2\{t+v_{\rm ret}(s)\}.
\end{aligned}
\end{equation}
Here
\[
    \zeta_{f,c,h}(x,x^+)
    =e^{h(x)}-e^{h^\star_{f,c}(x)}
    -\gamma f(x)c(x^+)g(x^+)
    -\tau_u\{1-c(x)\}g(x).
\]
\end{lemma}

\begin{proof}
Let \(c^\star=c_{f,\tau_u}^\star\), and write
\(\mu_f=(\Bpigcov f)\nu+\mu_{f,\perp}\). By
Lemma~\ref{lem:clip-finite-proj-curvature},
\[
    \|g\|_{L^2(\nu)}^2\le 2t/\tau_\ell.
\]
Moreover, \(\|g\|_\infty\le2R_{\rm clip}\). Since \(c^\star\) vanishes on the
singular component and \(\Bpigcov f\le\tau_u\) wherever \(c^\star=1\),
\begin{equation}
\label{eq:clip-proj-retained-moment}
\begin{aligned}
    \int c g^2\,\dd\mu_f
    &\le \int c^\star g^2\,\dd\mu_f
       +4R_{\rm clip}^2\int(c-c^\star)^2\,\dd\mu_f\\
    &\le \tau_u\|g\|_{L^2(\nu)}^2
       +4R_{\rm clip}^2\left[
          \tau_u\nu\{c\ne c^\star\}
          +\Delta_f^{\rm ret}(c)
       \right]\\
    &\le C\{t+s^{\beta_{\rm ret}}+s\}.
\end{aligned}
\end{equation}
The second line uses Lemma~\ref{lem:clip-finite-sel-regret}, and the third uses
its margin bound. In particular, the retention-indicator regret includes any singular
mass retained by \(c\), so this bound does not require \(\mu_f\ll\nu\).

The initial-law loss difference satisfies
\[
    \|(1-\gamma)cg\|_{L^2(\dinit)}^2
    \le (1-\gamma)\int c g^2\,\dd\mu_f.
\]
Because \(f^2\le\tau_u f\), the successor part of the transition loss
satisfies
\begin{equation}
\label{eq:clip-proj-successor-moment}
\begin{aligned}
    &\|\gamma f(X)c(X^+)g(X^+)\|_{L^2(Q_{\nu,\pi})}^2\\
    &\qquad\le
    \gamma\tau_u\left[
        \gamma E_{Q_{\nu,\pi}}\{f(X)c(X^+)g^2(X^+)\}
    \right]
    \le \tau_u\int c g^2\,\dd\mu_f.
\end{aligned}
\end{equation}
The offline-state part is bounded by
\(\tau_u^2\|g\|_{L^2(\nu)}^2\). The mean-value theorem and the pointwise
upper bound on \(e^h\) also give
\[
    |e^h-e^{h^\star_{f,c}}|\le\tau_u|g|,
\]
so the exponential part has squared \(L^2(\nu)\)-norm at most
\(\tau_u^2\|g\|_{L^2(\nu)}^2\). The first-coordinate marginal of
\(Q_{\nu,\pi}\) is \(\nu\). Hence the inequality
\((a+b+c)^2\le3(a^2+b^2+c^2)\), together with
\eqref{eq:clip-proj-curvature}, \eqref{eq:clip-proj-retained-moment}, and
\eqref{eq:clip-proj-successor-moment}, yields
\begin{equation}
\label{eq:clip-proj-zeta-moment}
    \|\zeta_{f,c,h}\|_{L^2(Q_{\nu,\pi})}^2
    \le C\{t+v_{\rm ret}(s)\}.
\end{equation}
Enlarging \(\kappa_{\rm clip}\) proves
\eqref{eq:clip-proj-loss-localization}.
\end{proof}

\begin{lemma}[Uniform projection excess risk]
\label{lem:clip-finite-proj-erm}
Assume Conditions~\ref{ass:kl-class},
\ref{ass:clip-finite-selector-margin}, and
\ref{ass:clip-finite-projection-compact}. Fix \(s\ge0\). For every
\(f\in\sW_{\rm clip}\) and \(c\in\mathcal C\) satisfying
\(\Delta_f^{\rm ret}(c)\le s\), let
\[
    \widehat h_{f,c}
    \in
    \argmin_{h\in\sH_{\rm clip}}\widehat L_{f,c}^{\rm proj}(h)
\]
be an exact empirical minimizer. Then, with probability at least
\(1-\delta\),
\[
    \sup_{\substack{f\in\sW_{\rm clip},\,c\in\mathcal C:\\
        \Delta_f^{\rm ret}(c)\le s}}
    \left\{
        L_{f,c}^{\rm proj}(\widehat h_{f,c})
        -
        \inf_{h\in\sH_{\rm clip}}L_{f,c}^{\rm proj}(h)
    \right\}
    \le
    C_{\rm proj,erm}a_{n,\rm proj}(s,\delta),
\]
where \(C_{\rm proj,erm}<\infty\) depends only on
\(\tau_u,\tau_\ell^{-1},A_{\rm mar}\), and \(\alpha_{\rm mar}\).
\end{lemma}

\begin{proof}
For \(f\in\sW_{\rm clip}\), \(c\in\mathcal C\), and \(h\in\sH_{\rm clip}\), put
\(h^\dagger=h^\star_{f,c}\), \(g=h-h^\dagger\), and
\[
    \mathbb Z^{\rm proj}_{f,c}(h)
    =
    \{\widehat L_{f,c}^{\rm proj}(h)-L_{f,c}^{\rm proj}(h)\}
    -
    \{\widehat L_{f,c}^{\rm proj}(h^\dagger)-L_{f,c}^{\rm proj}(h^\dagger)\}.
\]
The centered loss difference decomposes as
\begin{equation}
\label{eq:clip-proj-centered-loss-decomposition}
\begin{aligned}
    \mathbb Z^{\rm proj}_{f,c}(h)
    &=-(1-\gamma)(P_{n,0}-\dinit)(cg)
      +(P_{n,+}-Q_{\nu,\pi})\zeta_{f,c,h},
\end{aligned}
\end{equation}
where \(\zeta_{f,c,h}\) is defined in
Lemma~\ref{lem:clip-finite-proj-localization}. Identity
\eqref{eq:clip-proj-centered-loss-decomposition} uses the fact that
\(P_{n,X}\) and \(\nu\) are the first-coordinate marginals of \(P_{n,+}\)
and \(Q_{\nu,\pi}\), respectively.
Suppose also that \(\Delta_f^{\rm ret}(c)\le s\) and
\(\Delta^{\rm proj}_{f,c}(h)\le r^2\). By
Lemma~\ref{lem:clip-finite-proj-localization}, their respective
\(L^2(\dinit)\)- and \(L^2(Q_{\nu,\pi})\)-radii are bounded by
\(\kappa_{\rm clip}\sqrt{r^2+v_{\rm ret}(s)}\).

The initial-law difference belongs to
\(\mathcal L^\Delta_{{\rm proj},0}\), and the exact transition-loss difference
\(\zeta_{f,c,h}\) belongs to \(\mathcal L^\Delta_{{\rm proj},Q}\), with
\(h_1=h\), \(h_2=h^\dagger\), and \(t=1\). Thus these classes contain the
full projection-loss differences for the initial and transition samples,
respectively.

Symmetrization and Lemma~\ref{lem:tool-bousquet} therefore imply that, for
every \(u\ge0\), with probability at least \(1-2e^{-u}\),
\[
    \sup_{\substack{f\in\sW_{\rm clip},\,c\in\mathcal C,\,h\in\sH_{\rm clip}:\\
        \Delta_f^{\rm ret}(c)\le s,\\
        \Delta^{\rm proj}_{f,c}(h)\le r^2}}
    |\mathbb Z^{\rm proj}_{f,c}(h)|
    \le
    C\left\{
        \mathfrak C_{n,\rm proj}\left(
            \kappa_{\rm clip}\sqrt{r^2+v_{\rm ret}(s)}
        \right)
        +
        \sqrt{\frac{\{r^2+v_{\rm ret}(s)\}u}{n}}
        +
        \frac{u}{n}
    \right\}.
\]
The complexity term depends on two localization scales. If
\(r^2\ge v_{\rm ret}(s)\), then
\begin{equation}
\label{eq:clip-proj-complexity-large-radius}
    \kappa_{\rm clip}\sqrt{r^2+v_{\rm ret}(s)}
    \le \sqrt2\kappa_{\rm clip}r,
\end{equation}
and Lemma~\ref{lem:clip-finite-critical-radius-scaling} bounds the complexity
by a sufficiently small multiple of \(r^2\) whenever \(r\) is a sufficiently
large fixed multiple of \(\mathfrak r_{n,\rm proj}\). If
\(r^2<v_{\rm ret}(s)\), star-shaped scaling instead gives
\begin{equation}
\label{eq:clip-proj-complexity-small-radius}
    \mathfrak C_{n,\rm proj}\left(
        \kappa_{\rm clip}\sqrt{r^2+v_{\rm ret}(s)}
    \right)
    \le
    \sqrt2\mathfrak C_{n,\rm proj}\left(
        \kappa_{\rm clip}\sqrt{v_{\rm ret}(s)}
    \right).
\end{equation}
In the second case, this quantity is bounded by the corresponding term in
\(a_{n,\rm proj}(s,\delta)\). For the concentration term,
\[
    \sqrt{\frac{\{r^2+v_{\rm ret}(s)\}u}{n}}
    \le
    r\sqrt{\frac{u}{n}}
    +\sqrt{\frac{v_{\rm ret}(s)u}{n}}.
\]
Young's inequality bounds the first term on the right by a small multiple of
\(r^2\) plus \(Cu/n\).

To complete the peeling argument, write
\(b_{\rm proj}=a_{n,\rm proj}(s,\delta)\), choose a sufficiently
large fixed constant \(C_0\), and define
\[
\begin{aligned}
    \mathcal T_0
    &=\{(f,c,h):\Delta_f^{\rm ret}(c)\le s,
        \Delta^{\rm proj}_{f,c}(h)\le C_0b_{\rm proj}\},\\
    \mathcal T_j
    &=\{(f,c,h):\Delta_f^{\rm ret}(c)\le s,
        2^{j-1}C_0b_{\rm proj}
        <\Delta^{\rm proj}_{f,c}(h)
        \le2^jC_0b_{\rm proj}\},\qquad j\ge1.
\end{aligned}
\]
For \(j\ge0\), set
\[
    r_j^2=2^jC_0b_{\rm proj},
    \qquad
    u_j=\log(8/\delta)+2j\log2.
\]
The definition of \(b_{\rm proj}\) gives
\begin{equation}
\label{eq:clip-proj-baseline-components}
\begin{aligned}
    b_{\rm proj}\ge
    \mathfrak r_{n,\rm proj}^2,
    &\qquad
    b_{\rm proj}\ge
    \mathfrak C_{n,\rm proj}\left(
        \kappa_{\rm clip}\sqrt{v_{\rm ret}(s)}
    \right),\\
    b_{\rm proj}\ge
    \sqrt{\frac{v_{\rm ret}(s)\log(8/\delta)}{n}},
    &\qquad
    b_{\rm proj}\ge\frac{\log(8/\delta)}{n}.
\end{aligned}
\end{equation}
Since \(\log(8/\delta)>1\),
\eqref{eq:clip-proj-baseline-components} also implies
\(b_{\rm proj}\ge\sqrt{v_{\rm ret}(s)/n}\) and
\(b_{\rm proj}\ge n^{-1}\). Hence
\[
    \sqrt{\frac{v_{\rm ret}(s)u_j}{n}}
    +\frac{u_j}{n}
    \le C(1+\sqrt j+j)b_{\rm proj}
    \le C2^jb_{\rm proj}.
\]
For each shell, \eqref{eq:clip-proj-complexity-large-radius} and
\eqref{eq:clip-proj-complexity-small-radius} bound the complexity by either a small
multiple of \(r_j^2\) or at most \(C b_{\rm proj}\). After increasing
\(C_0\), these bounds and Young's inequality therefore yield a bound of
\(Cb_{\rm proj}\) on \(\mathcal T_0\), and a bound of
\(\Delta^{\rm proj}_{f,c}(h)/4+Cb_{\rm proj}\) on every
\(\mathcal T_j\), \(j\ge1\). The shell failure probabilities satisfy
\[
    2\sum_{j=0}^\infty e^{-u_j}
    =\frac{\delta}{3}
    \le\delta.
\]
Thus a union bound gives an event of probability at least \(1-\delta\) on
which
\[
    |\mathbb Z^{\rm proj}_{f,c}(h)|
    \le
    \frac14\Delta^{\rm proj}_{f,c}(h)
    +
    C a_{n,\rm proj}(s,\delta)
\]
simultaneously for all \(f,c,h\) satisfying
\(\Delta_f^{\rm ret}(c)\le s\).

Exact ERM gives
\(\widehat L_{f,c}^{\rm proj}(\widehat h_{f,c})
\le \widehat L_{f,c}^{\rm proj}(h^\star_{f,c})\). Evaluating the uniform
deviation bound at \(h=\widehat h_{f,c}\) yields
\[
    \Delta^{\rm proj}_{f,c}(\widehat h_{f,c})
    \le
    \frac14\Delta^{\rm proj}_{f,c}(\widehat h_{f,c})
    +
    C a_{n,\rm proj}(s,\delta).
\]
Rearranging proves the result.
\end{proof}

\begin{lemma}[Retention-indicator perturbation of the projection loss]
\label{lem:clip-finite-selector-perturbation}
Assume Conditions~\ref{ass:kl-class},
\ref{ass:clip-finite-selector-margin}, and
\ref{ass:clip-finite-projection-compact}. Fix
\(f\in\sW_{\rm clip}\), let
\(c_{f,\tau_u}^\star\) be the oracle retention indicator, and let \(c\in\mathcal C\). Then
\begin{equation}
\label{eq:clip-retention-projection-perturbation}
    \sup_{h\in\sH_{\rm clip}}
    |L_{f,c}^{\rm proj}(h)-L_{f,c_{f,\tau_u}^\star}^{\rm proj}(h)|
    \le
    \log(\tau_u\vee\tau_\ell^{-1})\Delta_f^{\rm ret}(c).
\end{equation}
Consequently, on the retention-indicator ERM event in
Lemma~\ref{lem:clip-finite-sel-erm} and the projection event in
Lemma~\ref{lem:clip-finite-proj-erm} with
\(s=C_{\rm ret,erm}\bar a_{n,\rm ret}(\delta)\), if
\(\widehat c_f\) is the fitted retention indicator and
\(\widehat h_f\) minimizes \(\widehat L_{f,\widehat c_f}^{\rm proj}\) over
\(\sH_{\rm clip}\), then
\begin{equation}
\label{eq:clip-retention-projection-excess}
    L_{f,c_{f,\tau_u}^\star}^{\rm proj}(\widehat h_f)
    -
    \inf_{h\in\sH_{\rm clip}}L_{f,c_{f,\tau_u}^\star}^{\rm proj}(h)
    \le
    C\bar a_{n,\rm clip}(\delta).
\end{equation}
\end{lemma}

\begin{proof}
Write \(\mu_f=(\Bpigcov f)\nu+\mu_{f,\perp}\). For
\(h\in\sH_{\rm clip}\), the two projection losses differ by
\[
\begin{aligned}
    L_{f,c}^{\rm proj}(h)-L_{f,c_{f,\tau_u}^\star}^{\rm proj}(h)
    ={}&
    -\int h(x)\{c(x)-c_{f,\tau_u}^\star(x)\}
        \{(\Bpigcov f)(x)-\tau_u\}\,\nu(\dd x)\\
    &-\int h(x)c(x)\,\mu_{f,\perp}(\dd x),
\end{aligned}
\]
where the second term uses
\(c_{f,\tau_u}^\star=0\), \(\mu_{f,\perp}\)-almost everywhere.
Taking absolute values, using
\(\|h\|_\infty\le\log(\tau_u\vee\tau_\ell^{-1})\), and applying
Lemma~\ref{lem:clip-finite-sel-regret} give
\eqref{eq:clip-retention-projection-perturbation}.

Let \(\widehat h_f\) minimize the empirical loss with the fitted retention indicator
\(\widehat c_f\). On the retention-indicator ERM event,
\(\Delta_f^{\rm ret}(\widehat c_f)\le
C\bar a_{n,\rm ret}(\delta)\). Lemma~\ref{lem:clip-finite-proj-erm}, applied
with \(s=C\bar a_{n,\rm ret}(\delta)\), gives
\[
    L_{f,\widehat c_f}^{\rm proj}(\widehat h_f)
    -
    \inf_{h\in\sH_{\rm clip}}L_{f,\widehat c_f}^{\rm proj}(h)
    \le
    C a_{n,\rm proj}
        \bigl(C\bar a_{n,\rm ret}(\delta),\delta\bigr).
\]
Replacing \(\widehat c_f\) by \(c_{f,\tau_u}^\star\) at both the fitted point and the
infimum changes the excess loss by at most twice the uniform loss difference,
because
\[
    \left|
    \inf_{h\in\sH_{\rm clip}}L_{f,\widehat c_f}^{\rm proj}(h)
    -
    \inf_{h\in\sH_{\rm clip}}L_{f,c_{f,\tau_u}^\star}^{\rm proj}(h)
    \right|
    \le
    \sup_{h\in\sH_{\rm clip}}
    |L_{f,\widehat c_f}^{\rm proj}(h)-L_{f,c_{f,\tau_u}^\star}^{\rm proj}(h)|.
\]
Equation~\eqref{eq:clip-proj-error-scaling} gives
\[
    a_{n,\rm proj}
    \bigl(C\bar a_{n,\rm ret}(\delta),\delta\bigr)
    \le
    C a_{n,\rm proj}
    \bigl(\bar a_{n,\rm ret}(\delta),\delta\bigr).
\]
Combining this comparison with the ERM bounds and the retention-indicator
perturbation bound gives
\[
    L_{f,c_{f,\tau_u}^\star}^{\rm proj}(\widehat h_f)
    -\inf_{h\in\sH_{\rm clip}}
        L_{f,c_{f,\tau_u}^\star}^{\rm proj}(h)
    \le
    C\left[
        a_{n,\rm proj}
            \bigl(\bar a_{n,\rm ret}(\delta),\delta\bigr)
        +R_{\rm clip}\bar a_{n,\rm ret}(\delta)
    \right].
\]
Because \(R_{\rm clip}\) depends only on the fixed envelopes, the right-hand
side is bounded by \(C\bar a_{n,\rm clip}(\delta)\). This proves the
bound \eqref{eq:clip-retention-projection-excess}.
\end{proof}

\begin{lemma}[Approximate generalized KL projection toward the clipped fixed point]
\label{lem:clip-finite-approx-projection}
Assume Conditions~\ref{ass:kl-class} and
\ref{ass:clip-finite-projection-compact}, and suppose
\eqref{eq:coverage-clipped-fixed-lower-tail} holds. Fix
\(f\in\sW_{\rm clip}\), and let
\[
    \bar\omega_f
    =
    \mathsf T_{\sW_{\rm clip}}^{\rm genKL}f
    =
    \Pi_{\sW_{\rm clip}}^{\rm genKL}(\Bpigcap{\tau_u}f).
\]
Suppose \(\widetilde h\in\sH_{\rm clip}\),
\(\widetilde\omega=e^{\widetilde h}\), and
\[
    L_{f,c_{f,\tau_u}^\star}^{\rm proj}(\widetilde h)
    -
    \inf_{h\in\sH_{\rm clip}}L_{f,c_{f,\tau_u}^\star}^{\rm proj}(h)
    \le
    \Delta,
    \qquad
    \Delta\ge\bar a_{n,\rm clip}(\delta).
\]
Then, for every \(0<\lambda\le1\),
\[
    D_\nu^{\rm gen}(\widetilde\omega\|\omega_{\tau_u})
    \le
    (1+\lambda)
    D_\nu^{\rm gen}(\bar\omega_f\|\omega_{\tau_u})
    +
    C_{\tau_\ell,\tau_u,\rm pert}
    \{1+\lambda^{-1}\mathfrak m_{n,\alpha_0}^{\rm clip}\}\Delta,
\]
where \(C_{\tau_\ell,\tau_u,\rm pert}<\infty\) depends only on
\(\tau_\ell\) and \(\tau_u\); all lower-tail dependence is contained in
\(\mathfrak m_{n,\alpha_0}^{\rm clip}\).
\end{lemma}

\begin{proof}
Let \(u_f=\Bpigcap{\tau_u}f\). By Proposition~\ref{prop:coverage-truncated-gate},
\(L_{f,c_{f,\tau_u}^\star}^{\rm proj}(h)\) differs from
\(D_\nu^{\rm gen}(u_f\|e^h)\) by a term independent of \(h\). Therefore
\[
    D_\nu^{\rm gen}(u_f\|\widetilde\omega)
    -
    D_\nu^{\rm gen}(u_f\|\bar\omega_f)
    \le
    \Delta .
\]
Applying Lemma~\ref{lem:genkl-projection-inequality} with
\(u=u_f\), \(\bar u=\bar\omega_f\), and
\(v=\widetilde\omega\), gives
\[
    D_\nu^{\rm gen}(\bar\omega_f\|\widetilde\omega)
    \le
    \Delta .
\]
Since \(\widetilde\omega,\bar\omega_f\in[\tau_\ell,\tau_u]\), the scalar
comparisons between \(a\log(a/b)-a+b\), \(b\log(b/a)-b+a\), and
\((a-b)^2/b\) on the compact interval
\([\frac{\tau_\ell}{\tau_u},\frac{\tau_u}{\tau_\ell}]\) imply
\begin{equation}
\label{eq:clip-approx-reverse-and-chi}
    D_\nu^{\rm gen}(\widetilde\omega\|\bar\omega_f)
    +
    \int
    \frac{(\widetilde\omega-\bar\omega_f)^2}{\bar\omega_f}\,\dd\nu
    \le
    C_{\tau_\ell,\tau_u}\Delta .
\end{equation}
Let
\[
    \eta_{n,\rm clip}(\delta)
    =
    1\wedge
    \left\{
        \frac{\bar a_{n,\rm clip}(\delta)}{A_{{\rm clip},{\rm lt}}}
    \right\}^{1/\alpha_0},
    \qquad
    \omega_\eta=\omega_{\tau_u}\vee\eta_{n,\rm clip}(\delta).
\]
Define
\[
    B_{n,\tau_\ell,\tau_u}(\delta)
    =
    1+\frac{1}{\alpha_0}
    \max\left\{
        0,
        \log\left(
            \frac{A_{{\rm clip},{\rm lt}}}{\bar a_{n,\rm clip}(\delta)}
        \right)
    \right\}.
\]
The layer-cake calculation in Lemma~\ref{lem:kl-lower-tail-comparison}, with
\(\omega_{\tau_u}\) in place of \(\omegastar\) and upper envelope \(\tau_u\), gives
\begin{equation}
\label{eq:clip-approx-lower-envelope}
    \left|
    D_\nu^{\rm gen}(a\|\omega_{\tau_u})
    -
    D_\nu^{\rm gen}(a\|\omega_\eta)
    \right|
    \le
    A_{{\rm clip},{\rm lt}}\eta_{n,\rm clip}(\delta)^{\alpha_0}
    \le
    \bar a_{n,\rm clip}(\delta)
    \le
    \Delta
\end{equation}
for \(a=\widetilde\omega\) and \(a=\bar\omega_f\). Also,
\[
    1+\log\{1/\eta_{n,\rm clip}(\delta)\}
    \le
    C B_{n,\tau_\ell,\tau_u}(\delta).
\]
Since \(\tau_\ell\le\bar\omega_f\le\tau_u\) and
\(\omega_\eta\ge\eta_{n,\rm clip}(\delta)\), the scalar inequality
\[
    r(\log r)^2
    \le
    C_{\tau_\ell,\tau_u} B_{n,\tau_\ell,\tau_u}(\delta)\{r\log r-r+1\},
    \qquad
    0<r\le \tau_u/\eta_{n,\rm clip}(\delta),
\]
implies
\begin{equation}
\label{eq:clip-approx-log-square}
    \int \bar\omega_f
    \left(\log\frac{\bar\omega_f}{\omega_\eta}\right)^2\,\dd\nu
    \le
    C_{\tau_\ell,\tau_u} B_{n,\tau_\ell,\tau_u}(\delta)
    D_\nu^{\rm gen}(\bar\omega_f\|\omega_\eta).
\end{equation}
Using the identity
\begin{equation}
\label{eq:clip-approx-divergence-decomposition}
    D_\nu^{\rm gen}(\widetilde\omega\|\omega_\eta)
    =
    D_\nu^{\rm gen}(\bar\omega_f\|\omega_\eta)
    +
    D_\nu^{\rm gen}(\widetilde\omega\|\bar\omega_f)
    +
    \int(\widetilde\omega-\bar\omega_f)
    \log\frac{\bar\omega_f}{\omega_\eta}\,\dd\nu,
\end{equation}
Cauchy--Schwarz and \eqref{eq:clip-approx-log-square} give
\begin{equation}
\label{eq:clip-approx-cross-term}
\begin{aligned}
    &\left|
    \int(\widetilde\omega-\bar\omega_f)
    \log\frac{\bar\omega_f}{\omega_\eta}\,\dd\nu
    \right| \\
    &\qquad\le
    C_{\tau_\ell,\tau_u}
    \left\{
        \Delta
        B_{n,\tau_\ell,\tau_u}(\delta)
        D_\nu^{\rm gen}(\bar\omega_f\|\omega_\eta)
    \right\}^{1/2}.
\end{aligned}
\end{equation}
Young's inequality with parameter \(\lambda\) gives
\begin{equation}
\label{eq:clip-approx-smoothed-recursion}
    D_\nu^{\rm gen}(\widetilde\omega\|\omega_\eta)
    \le
    (1+\lambda)
    D_\nu^{\rm gen}(\bar\omega_f\|\omega_\eta)
    +
    C_{\tau_\ell,\tau_u}\{1+\lambda^{-1}B_{n,\tau_\ell,\tau_u}(\delta)\}\Delta .
\end{equation}
Applying \eqref{eq:clip-approx-lower-envelope} to both divergences in
\eqref{eq:clip-approx-smoothed-recursion}
and absorbing their additive \(\Delta\)-terms into the constant gives
\[
\begin{aligned}
    D_\nu^{\rm gen}(\widetilde\omega\|\omega_{\tau_u})
    \le{}&
    (1+\lambda)
    D_\nu^{\rm gen}(\bar\omega_f\|\omega_{\tau_u})\\
    &+
    C_{\tau_\ell,\tau_u}
    \{1+\lambda^{-1}B_{n,\tau_\ell,\tau_u}(\delta)\}\Delta.
\end{aligned}
\]
Since
\(\bar a_{n,\rm ret}(\delta)\ge
\mathfrak r_{n,\rm ret}^{p_{\rm ret}}\), monotonicity of
\(v_{\rm ret}\) and \(\mathfrak C_{n,\rm proj}\) gives
\[
    \bar a_{n,\rm clip}(\delta)
    \ge b_{n,\rm clip}.
\]
The definitions of the two lower-tail factors therefore imply
\(B_{n,\tau_\ell,\tau_u}(\delta)
\le\mathfrak m_{n,\alpha_0}^{\rm clip}\), which proves the stated bound.
\end{proof}

In what follows, write
\[
    \Lambda_{\rm clip}
    :=
    1+\tau_u+\tau_\ell^{-1}.
\]
\begin{lemma}[Polynomial dependence on the clipping envelopes]
\label{lem:clip-finite-polynomial-envelope}
Fix \(A_{\rm mar},\alpha_0\), and \(\alpha_{\rm mar}\). In the applications
below, every constant arising in
Lemmas~\ref{lem:clip-finite-critical-radius-scaling}--
\ref{lem:clip-finite-approx-projection} that depends on the clipping levels
may be chosen to be at most
\[
    C_0\Lambda_{\rm clip}^{p}
\]
for finite \(C_0,p\) depending only on these three fixed constants. Under the
VC assumptions of Lemma~\ref{lem:clip-finite-vc-radius}, the constants in its
three bounds have the same property.
\end{lemma}

\begin{proof}
The basic envelope and curvature factors satisfy
\begin{equation}
\label{eq:clip-envelope-basic-factors}
\begin{aligned}
    1+R_{\rm clip}
    &\le 2\Lambda_{\rm clip},\\
    B_{\rm ret}
    &\le 2\tau_u\le2\Lambda_{\rm clip},\\
    B_{\rm proj}
    &\le C\tau_u(1+R_{\rm clip})
      \le C\Lambda_{\rm clip}^2,\\
    \tau_\ell^{-1}
    &\le\Lambda_{\rm clip},
    \qquad
    \frac{\tau_u}{\tau_\ell}
    \le\Lambda_{\rm clip}^2,
\end{aligned}
\end{equation}
where \(B_{\rm ret}\) and \(B_{\rm proj}\) are common envelopes for the two
loss-difference classes. Lemmas~\ref{lem:clip-finite-sel-regret} and
\ref{lem:clip-finite-proj-localization}, together with
\(f^2\le\tau_u f\), therefore give
\[
    \kappa_{\rm clip}
    \vee C_{\rm ret,erm}
    \vee C_{\rm proj,erm}
    \le C\Lambda_{\rm clip}^{p_1}
\]
for a finite exponent \(p_1\) depending only on \(\alpha_{\rm mar}\). The
critical-radius, concentration, peeling, and Young inequalities use only fixed
sums, products, maxima, and powers of these quantities. The resulting ERM and
retention-perturbation constants are therefore polynomial in
\(\Lambda_{\rm clip}\).

For the generalized KL comparisons, put
\(L=\tau_u/\tau_\ell\le\Lambda_{\rm clip}^2\). For
\(\phi(r)=r\log r-r+1\), strong convexity on \([L^{-1},L]\) gives
\[
    \phi(r)\ge \frac{(r-1)^2}{2L}.
\]
Thus the reverse-KL and chi-square comparison constants are polynomial in
\(\Lambda_{\rm clip}\). Moreover,
\[
    \frac{r(\log r)^2}{\phi(r)}
    \le C\{1+\log_+ r\},
    \qquad r>0,
\]
with the ratio defined by continuity at \(r=1\). At
\(r\le\tau_u/\eta_{n,\rm clip}(\delta)\), the right-hand side is bounded by
\(C(1+R_{\rm clip})B_{n,\tau_\ell,\tau_u}(\delta)\). The factor
\(1+R_{\rm clip}\) is polynomial in \(\Lambda_{\rm clip}\). Finally,
\(A_{{\rm clip},{\rm lt}}=A_0(1+\tau_u/\alpha_0)\) appears only through the
explicit lower-tail factors \(B_{n,\tau_\ell,\tau_u}(\delta)\) and
\(\mathfrak m_{n,\alpha_0}^{\rm clip}\). Thus the remaining lower-tail comparison
and approximate-projection constants are polynomial in
\(\Lambda_{\rm clip}\), without further dependence on \(A_0\).

For VC classes, the covering bounds acquire only the factors in
\eqref{eq:clip-envelope-basic-factors}
and their logarithms. Since \(\log\Lambda_{\rm clip}\le\Lambda_{\rm clip}\),
the local-entropy and fixed-point calculations preserve polynomial dependence.
Taking the maximum of the finitely many exponents proves the lemma.
\end{proof}

\subsection{Fitted recursion and finite-sample bounds}

\begin{theorem}[Fixed-level fitted coverage-stopped \textsc{FORE} bound]
\label{thm:clip-finite-fore-sharp}
Let \(\gamma\in[0,1)\), and fix
\(0<\tau_\ell\le1\le\tau_u<\infty\). Assume
Condition~\ref{ass:kl-class}, the lower-tail bound
\eqref{eq:coverage-clipped-fixed-lower-tail}, and the fixed-level versions of
Conditions~\ref{ass:clip-finite-selector-margin} and
\ref{ass:clip-finite-projection-compact}, with clipping level \(\tau_u\).
Suppose also that
\(1\in\sW_{\rm clip}\). Let
\(\{\widehat\omega^{(k)}\}_{k=0}^K\) be the exact-ERM fitted
coverage-stopped \textsc{FORE} iterates of
Algorithm~\ref{alg:coverage-truncated-fore}. Then,
for every \(0<\delta<1\), with probability at least \(1-\delta\), for
\(\rho=(1+\gamma)/2\),
\begin{equation}
\label{eq:clip-fixed-level-fitted-bound}
\begin{aligned}
    D_\nu^{\rm gen}(\widehat\omega^{(K)}\|\omega_{\tau_u})
    \le{}&
    C_{\tau_\ell,\tau_u}\rho^K
    D_\nu^{\rm gen}(\widehat\omega^{(0)}\|\omega_{\tau_u})\\
    &+
    \frac{C_{\tau_\ell,\tau_u}}{1-\gamma}
    \frac{\tau_u}{\tau_\ell}
    \varepsilon_{\rm ratio}\\
    &+
    \frac{C_{\tau_\ell,\tau_u}}{(1-\gamma)^2}
    \mathfrak m_{n,\alpha_0}^{\rm clip}
    \bar a_{n,\rm clip}(\delta),
\end{aligned}
\end{equation}
where \(C_{\tau_\ell,\tau_u}<\infty\) depends only on the two envelopes
and the fixed constants in the stated conditions other than \(A_0\). It may
be chosen so that
\begin{equation}
\label{eq:clip-fixed-level-envelope}
    C_{\tau_\ell,\tau_u}
    \le C_0\Lambda_{\rm clip}^{p}
\end{equation}
for finite \(C_0,p\) depending only on
\(A_{\rm mar},\alpha_0\), and \(\alpha_{\rm mar}\).
\end{theorem}

\begin{proof}[Proof of Theorem~\ref{thm:clip-finite-fore-sharp}]
Let
\(s_\delta=C_{\rm ret,erm}\bar a_{n,\rm ret}(\delta/2)\). Apply
Lemma~\ref{lem:clip-finite-sel-erm} with failure probability \(\delta/2\),
and apply Lemma~\ref{lem:clip-finite-proj-erm} with
\(s=s_\delta\) and the same failure probability. Work on the intersection of
these events, which has probability at least \(1-\delta\). Both events are
uniform over \(f\in\sW_{\rm clip}\), so they apply to the random fitted
inputs \(\widehat\omega^{(k)}\). The confidence terms at level \(\delta/2\)
are bounded by fixed multiples of those at level \(\delta\). Indeed,
\(\log(12/\delta)\le2\log(6/\delta)\) and
\(\log(16/\delta)\le2\log(8/\delta)\). It follows that
\(a_{n,\rm ret}(\delta/2)\le C a_{n,\rm ret}(\delta)\) and
\(\bar a_{n,\rm ret}(\delta/2)
\le C\bar a_{n,\rm ret}(\delta)\).
Equation~\eqref{eq:clip-proj-error-scaling}, together with
\(\log(16/\delta)\le2\log(8/\delta)\), gives
\[
    a_{n,\rm proj}(s_\delta,\delta/2)
    \le
    C a_{n,\rm proj}
    \bigl(\bar a_{n,\rm ret}(\delta),\delta\bigr).
\]

The initialization belongs to \(\sW_{\rm clip}\) because
\(1\in\sW_{\rm clip}\). Every later iterate also belongs to this
class because it has the form \(e^{\widehat h}\), with
\(\widehat h\in\sH_{\rm clip}\).

Fix \(k<K\) and set \(f=\widehat\omega^{(k)}\). Let
\(\widehat c_f\) and \(\widehat h_f\) be the retention-indicator and
projection ERMs chosen by Algorithm~\ref{alg:coverage-truncated-fore}. By
Lemma~\ref{lem:clip-finite-selector-perturbation},
\[
    L_{f,c_{f,\tau_u}^\star}^{\rm proj}(\widehat h_f)
    -
    \inf_{h\in\sH_{\rm clip}}L_{f,c_{f,\tau_u}^\star}^{\rm proj}(h)
    \le
    C\bar a_{n,\rm clip}(\delta).
\]
Set \(\Delta=C\bar a_{n,\rm clip}(\delta)\), increasing \(C\) if
necessary so that \(\Delta\ge\bar a_{n,\rm clip}(\delta)\), and apply
Lemma~\ref{lem:clip-finite-approx-projection}. For every \(0<\lambda\le1\),
\[
\begin{aligned}
    D_\nu^{\rm gen}(\widehat\omega^{(k+1)}\|\omega_{\tau_u})
    &\le
    (1+\lambda)
    D_\nu^{\rm gen}(
        \mathsf T_{\sW_{\rm clip}}^{\rm genKL}f
        \|\omega_{\tau_u}) \\
    &\quad
    +
    C_{\tau_\ell,\tau_u}
    \{1+\lambda^{-1}\mathfrak m_{n,\alpha_0}^{\rm clip}\}
    \bar a_{n,\rm clip}(\delta).
\end{aligned}
\]
The proof of Theorem~\ref{thm:coverage-truncated-recursion} gives the one-step
inequality for every nonnegative \(f\) with \(0\le f\le\tau_u\). In particular,
it applies to all \(f\in\sW_{\rm clip}\):
\[
    D_\nu^{\rm gen}(
        \mathsf T_{\sW_{\rm clip}}^{\rm genKL}f
        \|\omega_{\tau_u})
    \le
    \gamma D_\nu^{\rm gen}(f\|\omega_{\tau_u})
    +
    \frac{\tau_u}{\tau_\ell}\varepsilon_{\rm ratio}.
\]
Combining these one-step bounds gives
\[
\begin{aligned}
    D_\nu^{\rm gen}(\widehat\omega^{(k+1)}\|\omega_{\tau_u})
    &\le
    (1+\lambda)\gamma
    D_\nu^{\rm gen}(\widehat\omega^{(k)}\|\omega_{\tau_u})
    +
    2\frac{\tau_u}{\tau_\ell}\varepsilon_{\rm ratio} \\
    &\quad
    +
    C_{\tau_\ell,\tau_u}
    \{1+\lambda^{-1}\mathfrak m_{n,\alpha_0}^{\rm clip}\}
    \bar a_{n,\rm clip}(\delta),
\end{aligned}
\]
after using \(1+\lambda\le2\).
Choose \(\lambda=1\) if \(\gamma=0\), and otherwise choose
\[
    \lambda
    =
    1\wedge \frac{\rho-\gamma}{2\gamma}.
\]
Then \((1+\lambda)\gamma\le\rho\), and
\(1+\lambda^{-1}\le C(1-\gamma)^{-1}\). Since
\(\mathfrak m_{n,\alpha_0}^{\rm clip}\ge1\), it follows that
\[
    1+\lambda^{-1}\mathfrak m_{n,\alpha_0}^{\rm clip}
    \le
    C(1-\gamma)^{-1}\mathfrak m_{n,\alpha_0}^{\rm clip}.
\]
Hence
\begin{equation}
\label{eq:clip-fitted-fixed-level-recursion}
    D_\nu^{\rm gen}(\widehat\omega^{(k+1)}\|\omega_{\tau_u})
    \le
    \rho
    D_\nu^{\rm gen}(\widehat\omega^{(k)}\|\omega_{\tau_u})
    +
    C_{\tau_\ell,\tau_u}\frac{\tau_u}{\tau_\ell}\varepsilon_{\rm ratio}
    +
    \frac{C_{\tau_\ell,\tau_u}}{1-\gamma}
    \mathfrak m_{n,\alpha_0}^{\rm clip}
    \bar a_{n,\rm clip}(\delta).
\end{equation}
Iterating \eqref{eq:clip-fitted-fixed-level-recursion} and using
\((1-\rho)^{-1}=2(1-\gamma)^{-1}\) gives
\[
\begin{aligned}
    D_\nu^{\rm gen}(\widehat\omega^{(K)}\|\omega_{\tau_u})
    &\le
    \rho^K
    D_\nu^{\rm gen}(\widehat\omega^{(0)}\|\omega_{\tau_u})
    +
    \frac{C_{\tau_\ell,\tau_u}}{1-\gamma}\frac{\tau_u}{\tau_\ell}\varepsilon_{\rm ratio} \\
    &\quad
    +
    \frac{C_{\tau_\ell,\tau_u}}{(1-\gamma)^2}
    \mathfrak m_{n,\alpha_0}^{\rm clip}
    \bar a_{n,\rm clip}(\delta).
\end{aligned}
\]
Absorbing fixed multiplicative constants into \(C_{\tau_\ell,\tau_u}\)
proves \eqref{eq:clip-fixed-level-fitted-bound}.
Lemma~\ref{lem:clip-finite-polynomial-envelope} and the
finite sums and products used in the recursion show that this constant is at
most polynomial in \(\Lambda_{\rm clip}\), as asserted in
\eqref{eq:clip-fixed-level-envelope}.
\end{proof}

\begin{corollary}[Generalized KL bound under logarithmic clipping]
\label{cor:clip-finite-fore-kl}
Suppose the assumptions of Theorem~\ref{thm:clip-finite-fore} hold. Then,
for every \(0<\delta<1\), with probability at least \(1-\delta\), for
\(\rho=(1+\gamma)/2\),
\begin{equation}
\label{eq:clip-logarithmic-kl-bound}
\begin{aligned}
D_\nu^{\rm gen}(\widehat\omega^{(K)}\|\omega_{\tau_{u,n}})
\le{}&
C_n \rho^K
D_\nu^{\rm gen}(\widehat\omega^{(0)}\|\omega_{\tau_{u,n}})
+
\frac{C_n}{1-\gamma}
\frac{\tau_{u,n}}{\tau_\ell}
\varepsilon_{\rm ratio}(\tau_{u,n}) \\
&+
\frac{C_n}{(1-\gamma)^2}
\left\{
    \mathcal E_{n,\rm stat}(\delta)
    +\varepsilon_{\rm cls}(\tau_{u,n})
\right\},
\end{aligned}
\end{equation}
where finite constants \(C_0,q\), independent of \(n\), may be chosen so that
\begin{equation}
\label{eq:clip-logarithmic-constant}
    C_n
    \le
    C_0\{1+\log(en)+\tau_\ell^{-1}\}^{q}
    \left\{1+\log\frac{e}{1-\gamma}\right\}.
\end{equation}
Here \(C_0\) and \(q\) depend only on \(A\), \(K_{\rm cov}\), and the fixed
constants in the stated conditions.
\end{corollary}

\begin{proof}[Proof of Corollary~\ref{cor:clip-finite-fore-kl}]
Set \(\tau_u=\tau_{u,n}\), and use the fixed-level abbreviations
\(\varepsilon_{\rm ratio}=\varepsilon_{\rm ratio}(\tau_{u,n})\) and
\(\varepsilon_{\rm cls}=\varepsilon_{\rm cls}(\tau_{u,n})\).
First suppose that
\[
    \mathfrak r_{n,\rm clip}
    \vee\frac{\log(1/\delta)}{n}
    \vee\varepsilon_{\rm cls}
    \le1.
\]
The explicit lower-tail multiplier bound in
Lemma~\ref{lem:clip-finite-rate-comparison} and
\(\tau_{u,n}=1\vee A\log(en)\) give
\begin{equation}
\label{eq:clip-logarithmic-lower-tail-stat}
    \mathfrak m_{n,\alpha_0}^{\rm clip}
    \bar a_{n,\rm clip}(\delta)
    \le
    C\left\{
        1+\alpha_0^{-1}
        \left[
            \log(en)+\log\frac{e}{1-\gamma}
        \right]
    \right\}
    \left\{
        \mathcal E_{n,\rm stat}(\delta)
        +\varepsilon_{\rm cls}
    \right\}.
\end{equation}
Here Lemma~\ref{lem:coverage-retained-lower-tail-transfer} bounds
\(\log_+A_0\) by a fixed constant plus
\(\log\{e/(1-\gamma)\}\). Substituting
\eqref{eq:clip-logarithmic-lower-tail-stat} into
Theorem~\ref{thm:clip-finite-fore-sharp} and absorbing its logarithmic factor
into \(C_n\) gives the stated generalized KL bound in this case.

It remains to consider the complementary case. The lower-tail bound
\eqref{eq:coverage-clipped-fixed-lower-tail} and the layer-cake identity give
\begin{equation}
\label{eq:clip-log-integrability}
    \int\log_+\{1/\omega_{\tau_u}\}\,\dd\nu
    =
    \int_0^\infty
        \nu\{0<\omega_{\tau_u}<e^{-t}\}\,\dd t
    \le
    \int_0^\infty
        \min\{1,A_0e^{-\alpha_0t}\}\,\dd t
    \le
    \frac{1+\log_+A_0}{\alpha_0}.
\end{equation}
Every fitted iterate takes values in \([\tau_\ell,\tau_u]\), while
\(0<\omega_{\tau_u}\le\tau_u\), \(\nu\)-almost everywhere. The integral
bound \eqref{eq:clip-log-integrability} and the scalar inequality
\[
    a\log(a/b)-a+b
    \le
    \tau_u\{1+\log\tau_u+\log_+(1/b)\},
    \qquad 0<a,b\le\tau_u,
\]
therefore imply
\[
    D_\nu^{\rm gen}(\widehat\omega^{(K)}\|\omega_{\tau_u})
    \le
    C_{\tau_\ell,\tau_u,\alpha_0}\{1+\log_+A_0\}.
\]
If \(\mathfrak r_{n,\rm clip}>1\),
\(\log(1/\delta)/n>1\), or \(\varepsilon_{\rm cls}>1\), then
\[
    \mathcal E_{n,\rm stat}(\delta)+\varepsilon_{\rm cls}>1.
\]
Thus the final error term in the generalized KL bound controls this crude
bound after increasing its constant; the remaining terms are nonnegative.
The crude bound used here is polynomial in \(\Lambda_{\rm clip}\) and
linear in \(1+\log_+A_0\). Together with
Lemma~\ref{lem:clip-finite-polynomial-envelope} and
\[
    \Lambda_{\rm clip}
    =1+\tau_{u,n}+\tau_\ell^{-1}
    \le
    C_A\{1+\log(en)+\tau_\ell^{-1}\},
\]
this proves the generalized KL inequality with
\[
    C_n
    \le
    C_0\{1+\log(en)+\tau_\ell^{-1}\}^{q}
    \left\{1+\log\frac{e}{1-\gamma}\right\},
\]
for finite constants \(C_0,q\) of the stated form.
\end{proof}

\begin{proof}[Proof of Theorem~\ref{thm:clip-finite-fore}]
Work on the event in Corollary~\ref{cor:clip-finite-fore-kl}. Every fitted
iterate is at most \(\tau_{u,n}\), so
Lemma~\ref{lem:coverage-retained-gen-kl-l1} gives
\[
\begin{aligned}
    \|\widehat\omega^{(K)}-\omega_{\rm cov}\|_{L^1(\nu)}
    \le{}&
    \left\{
        2(\tau_{u,n}+1)
        D_\nu^{\rm gen}
        (\widehat\omega^{(K)}\|\omega_{\tau_{u,n}})
    \right\}^{1/2}\\
    &+
    \frac{E_\nu\{(\omega_{\rm cov}-\tau_{u,n})_+\}}{1-\gamma}.
\end{aligned}
\]
Substitute the generalized KL bound and apply
\(\sqrt{x+y}\le\sqrt{x}+\sqrt{y}\). After factoring
\((1-\gamma)^{-1}\) from the approximation terms, the inequality
\(\sqrt{x}+\sqrt{y}\le\{2(x+y)\}^{1/2}\) combines them into the approximation
term in Theorem~\ref{thm:clip-finite-fore}; the universal factor \(\sqrt{2}\)
is absorbed into \(C_n\). The resulting factors
\[
    (\tau_{u,n}+1)^{1/2},
    \qquad
    (\tau_{u,n}/\tau_\ell)^{1/2}
\]
are polynomial in
\(1+\log(en)+\tau_\ell^{-1}\). They may therefore be absorbed by
enlarging \(C_n\), without changing its asserted form; the logarithmic
horizon factor in Corollary~\ref{cor:clip-finite-fore-kl} enters under a
square root.

Finally, Lemma~\ref{lem:coverage-retained-subexp-cap-bias} yields
\[
    \frac{E_\nu\{(\omega_{\rm cov}-\tau_{u,n})_+\}}{1-\gamma}
    \le
    \frac{2K_{\rm cov}}{1-\gamma}(en)^{-A/K_{\rm cov}}
    =o\!\left(\frac{1}{(1-\gamma)\sqrt n}\right),
\]
where the last equality uses \(A>K_{\rm cov}/2\). To compare this term
with the statistical error, put
\(a=\alpha_{\rm mar}/(\alpha_{\rm mar}+2)\). By definition,
\[
    \mathcal E_{n,\rm stat}(\delta)
    \ge
    \mathfrak r_{n,\rm clip}^{1+a}.
\]
If \(\mathfrak r_{n,\rm clip}\le1\), then the right-hand side is at least
\(\mathfrak r_{n,\rm clip}^2\ge n^{-1}\); if
\(\mathfrak r_{n,\rm clip}>1\), it is larger than one. Hence
\(\{\mathcal E_{n,\rm stat}(\delta)\}^{1/2}\ge n^{-1/2}\). Since
\(n^{1/2}(en)^{-A/K_{\rm cov}}\) is bounded over \(n\ge1\), enlarging
\(C_0\) absorbs the clipping bias into the statistical term and proves the
stated \(L^1(\nu)\) inequality.
\end{proof}

\section{Undiscounted KL contraction under strong KL data processing}
\label{app:mixed-contraction}

The main text obtains a strict KL recursion from discounting. When
\(\gamma=1\), the common initial-distribution component disappears, and
ordinary data processing gives only nonexpansiveness
\citep{coverThomas2006Elements}. This section therefore states the additional
mixing requirement directly as a one-step KL strong data-processing inequality
for the target-policy kernel \citep{raginsky2014StrongDataProcessing}.

For \(\gamma=1\), the population adjoint Bellman update satisfies
\[
    (\mathsf B_1^\pi\omega)\nu
    =
    (\omega\nu)P_\pi .
\]
Let
\[
    \Delta_\nu
    =
    \left\{
        \omega\in L^1(\nu):
        \omega\ge 0\ \nu\text{-a.e.},\quad
        \int \omega\,\dd\nu=1
    \right\}
\]
denote the set of \(\nu\)-densities of probability distributions. The one-step
KL strong data-processing condition used below is the following.
\begin{enumerate}[label=\textbf{(A\arabic*)}, ref=A\arabic*, resume=forecond]
\item \label{ass:one-step-kl-sdpi}
\textit{One-step KL strong data processing.} There exists
\(\alpha\in[0,1)\) such that, for all probability distributions
\(\mu,\xi\ll\nu\) with \(0<D_{\rm KL}(\mu\|\xi)<\infty\),
\[
    D_{\rm KL}(\mu P_\pi\|\xi P_\pi)
    \le
    \alpha D_{\rm KL}(\mu\|\xi).
\]
\end{enumerate}
A simple sufficient condition is a one-step Doeblin minorization: if there
exist \(\epsilon>0\) and a probability distribution \(\lambda\) such that
\[
    P_\pi(\cdot\mid x)\ge \epsilon\lambda(\cdot)
    \qquad \text{for all }x,
\]
then Condition~\ref{ass:one-step-kl-sdpi} holds with
\(\alpha\le 1-\epsilon\). Indeed, writing
\(P_\pi=\epsilon\lambda+(1-\epsilon)R\) for the residual Markov kernel \(R\),
joint convexity of KL divergence and data processing
\citep{coverThomas2006Elements} yield, for any probability measures
\(\rho\) and \(\eta\),
\[
    D_{\mathrm{KL}}(\rho P_\pi\,\|\,\eta P_\pi)
    \le (1-\epsilon)D_{\mathrm{KL}}(\rho R\,\|\,\eta R)
    \le (1-\epsilon)D_{\mathrm{KL}}(\rho\,\|\,\eta).
\]
Minorization conditions of this type are standard sufficient conditions for
uniform ergodicity in Markov-chain theory
\citep{meynTweedieGlynn2009MarkovChains}.

\begin{theorem}[Undiscounted adjoint KL contraction]
\label{thm:mixed-adjoint-contraction}
Assume Conditions~\ref{ass:overlap} and~\ref{ass:one-step-kl-sdpi}. The map
\(\mathsf B_1^\pi\) maps \(\Delta_\nu\) into itself, and, for any
\(\omega,\xi\in\Delta_\nu\) with
\(D_\nu(\omega\|\xi)<\infty\),
\begin{equation}
\label{eq:undiscounted-adjoint-contraction}
    D_\nu\left(
        \mathsf B_1^\pi\omega
        \big\|
        \mathsf B_1^\pi\xi
    \right)
    \le
    \alpha D_\nu(\omega\|\xi).
\end{equation}
If there exists a stationary distribution \(d_{\pi,1}\ll\nu\), with ratio
\(\omega_{\pi,1}=\dd d_{\pi,1}/\dd\nu\), then
\begin{equation}
\label{eq:undiscounted-adjoint-geometric}
    D_\nu\left(
        (\mathsf B_1^\pi)^K\omega
        \big\|
        \omega_{\pi,1}
    \right)
    \le
    \alpha^K D_\nu(\omega\|\omega_{\pi,1})
    \qquad
    \text{for all }\omega\in\Delta_\nu
    \text{ with }D_\nu(\omega\|\omega_{\pi,1})<\infty .
\end{equation}
This stationary distribution is unique among stationary distributions
\(\tilde d=\tilde\omega\nu\) satisfying
\(D_\nu(\tilde\omega\|\omega_{\pi,1})<\infty\).
\end{theorem}

\begin{proof}[Proof of Theorem~\ref{thm:mixed-adjoint-contraction}]
Condition~\ref{ass:overlap} and Lemma~\ref{lem:ac-propagation} imply that
\((\omega\nu)P_\pi\ll\nu\) whenever \(\omega\in\Delta_\nu\). Hence
\(\mathsf B_1^\pi\omega\) is well defined as a \(\nu\)-density. Since
\((\omega\nu)P_\pi\) is a probability distribution,
\(\mathsf B_1^\pi\Delta_\nu\subseteq\Delta_\nu\). For
\(\omega,\xi\in\Delta_\nu\), applying
Condition~\ref{ass:one-step-kl-sdpi} to the probability measures
\(\omega\nu\) and \(\xi\nu\) gives
\[
\begin{aligned}
    D_\nu\left(
        \mathsf B_1^\pi\omega
        \big\|
        \mathsf B_1^\pi\xi
    \right)
    &=
    D_{\rm KL}\{(\omega\nu)P_\pi\|(\xi\nu)P_\pi\} \\
    &\le
    \alpha D_{\rm KL}(\omega\nu\|\xi\nu)
    =
    \alpha D_\nu(\omega\|\xi).
\end{aligned}
\]
If \(d_{\pi,1}=\omega_{\pi,1}\nu\) is stationary, then
\(\mathsf B_1^\pi\omega_{\pi,1}=\omega_{\pi,1}\). Iterating
\eqref{eq:undiscounted-adjoint-contraction} with
\(\xi=\omega_{\pi,1}\) gives
\eqref{eq:undiscounted-adjoint-geometric}.
If \(\tilde\omega\nu\) is another stationary distribution and
\(D_\nu(\tilde\omega\|\omega_{\pi,1})<\infty\), then
\[
    D_\nu(\tilde\omega\|\omega_{\pi,1})
    =
    D_\nu\left(
        \mathsf B_1^\pi\tilde\omega
        \big\|
        \mathsf B_1^\pi\omega_{\pi,1}
    \right)
    \le
    \alpha D_\nu(\tilde\omega\|\omega_{\pi,1}).
\]
Since \(\alpha<1\), the finite divergence must be zero. Hence
\(\tilde d=d_{\pi,1}\), proving uniqueness in the stated class.
\end{proof}

\begin{proposition}[Projected undiscounted KL recursion]
\label{prop:projected-undiscounted-kl}
Assume \(\gamma=1\), Condition~\ref{ass:overlap},
Conditions~\ref{ass:kl-class}, \ref{ass:kl-bounded},
and~\ref{ass:one-step-kl-sdpi}. Suppose there
exists a stationary ratio
\(\omega_{\pi,1}\in\Delta_\nu\) with \(\omega_{\pi,1}>0\) \(\nu\)-a.e.\ and
\(\log\omega_{\pi,1}\in L^1(\nu)\). Assume that
\(\mathsf B_1^\pi\omega\) is bounded \(\nu\)-a.e.\ for every
\(\omega\in\sW\). For \(\omega\in\sW\), let
\[
    u_\omega
    =
    \mathsf B_1^\pi\omega,
    \qquad
    \mathsf T_{\sW,1}^{\rm KL}\omega
    =
    \Pi_{\sW}^{\rm KL}u_\omega .
\]
Define the projection violation
\[
    \operatorname{viol}_{{\rm KL},1}(\omega)
    =
    \inf_{v\in\sW}
    \left|
        \int
        \{u_\omega(x)-\mathsf T_{\sW,1}^{\rm KL}\omega(x)\}
        \log\frac{\omega_{\pi,1}(x)}{v(x)}
        \,\nu(\dd x)
    \right|.
\]
Then
\begin{equation}
\label{eq:projected-undiscounted-one-step}
    D_\nu\left(
        \mathsf T_{\sW,1}^{\rm KL}\omega
        \big\|
        \omega_{\pi,1}
    \right)
    \le
    \alpha
    D_\nu(\omega\|\omega_{\pi,1})
    +
    \operatorname{viol}_{{\rm KL},1}(\omega).
\end{equation}
Consequently, if
\(\sup_{\omega\in\sW}\operatorname{viol}_{{\rm KL},1}(\omega)\le\varepsilon_1\),
the projected iterates
\(\omega^{(j+1)}=\mathsf T_{\sW,1}^{\rm KL}\omega^{(j)}\) satisfy
\begin{equation}
\label{eq:projected-undiscounted-iterated}
    D_\nu(\omega^{(J)}\|\omega_{\pi,1})
    \le
    \alpha^J D_\nu(\omega^{(0)}\|\omega_{\pi,1})
    +
    \frac{1-\alpha^J}{1-\alpha}\varepsilon_1 .
\end{equation}
If \(\omega_{\pi,1}\in\sW\), then \(\varepsilon_1=0\),
\(\mathsf T_{\sW,1}^{\rm KL}\omega_{\pi,1}=\omega_{\pi,1}\), and
\[
    D_\nu\left(
        \mathsf T_{\sW,1}^{\rm KL}\omega
        \big\|
        \omega_{\pi,1}
    \right)
    \le
    \alpha D_\nu(\omega\|\omega_{\pi,1}).
\]
\end{proposition}

\begin{proof}
Let \(\bar u=\mathsf T_{\sW,1}^{\rm KL}\omega\) and \(u=u_\omega\). The
convex projection inequality in Lemma~\ref{lem:kl-projection-inequality},
applied to the bounded image \(u\), gives \(D_\nu(\bar u\|v)\le D_\nu(u\|v)\)
for every \(v\in\sW\). Therefore
\[
\begin{aligned}
    D_\nu(\bar u\|\omega_{\pi,1})
    &=
    D_\nu(\bar u\|v)
    +
    \int \bar u(x)\log\frac{v(x)}{\omega_{\pi,1}(x)}\,\nu(\dd x)\\
    &\le
    D_\nu(u\|v)
    +
    \int \bar u(x)\log\frac{v(x)}{\omega_{\pi,1}(x)}\,\nu(\dd x)\\
    &=
    D_\nu(u\|\omega_{\pi,1})
    +
    \int \{u(x)-\bar u(x)\}
        \log\frac{\omega_{\pi,1}(x)}{v(x)}\,\nu(\dd x).
\end{aligned}
\]
Bounding the last integral by its absolute value and then taking the infimum
over \(v\in\sW\) yields
\begin{equation}
\label{eq:mixed-projection-violation}
    D_\nu(\bar u\|\omega_{\pi,1})
    \le
    D_\nu(u\|\omega_{\pi,1})
    +
    \operatorname{viol}_{{\rm KL},1}(\omega).
\end{equation}
Theorem~\ref{thm:mixed-adjoint-contraction}, with
\(\xi=\omega_{\pi,1}\), gives
\begin{equation}
\label{eq:mixed-projected-one-step}
    D_\nu(u\|\omega_{\pi,1})
    \le
    \alpha D_\nu(\omega\|\omega_{\pi,1}).
\end{equation}
Combining \eqref{eq:mixed-projection-violation} and
\eqref{eq:mixed-projected-one-step} proves the one-step inequality in the
proposition, \eqref{eq:projected-undiscounted-one-step}. Iterating that
inequality gives \eqref{eq:projected-undiscounted-iterated}. If
\(\omega_{\pi,1}\in\sW\), choose \(v=\omega_{\pi,1}\) in the violation term.
Since \(u_{\omega_{\pi,1}}=\omega_{\pi,1}\), its KL projection is
\(\omega_{\pi,1}\), and the realizable contraction follows.
\end{proof}

\section{Backward-regression variant of \textsc{FORE}}
\label{app:strong-form-regression-fori}

The backward-regression variant is a fitted-regression version of the
adjoint Bellman recursion. It estimates the initial density ratio and
the one-step target-coverage ratio, repeatedly fits the backward conditional
mean in the backward-regression factorization of the adjoint Bellman operator,
and uses that regression to form the next ratio iterate. It is the
density-ratio analogue of fitted \(Q\)-evaluation: FQE regresses Bellman targets
for value functions, whereas this variant regresses the adjoint Bellman image
for density ratios.

This variant also makes explicit the role of adjoint Bellman completeness. The
main KL-projected method in Section~\ref{sec:kl-fori-theory} works with
adjoint Bellman moment identities and projects the resulting density in KL. It
is positive and normalized by construction, and its population analysis does
not require a regression class to contain every adjoint Bellman image. The
backward-regression variant below is more direct, but its projected population
error vanishes only under \emph{adjoint Bellman completeness}, the
density-ratio analogue of Bellman completeness in FQE. For this reason, we
recommend using the \textsc{FORE} algorithm studied in the main text.

\subsection{Backward-regression adjoint factorization}
\label{sec:strong-form-factorization}

Recall that
\[
    \nup=\nu\Ppi,
    \qquad
    \cpi=\frac{\dd\nup}{\dd\nu}.
\]
Let \(T_\pi f(x)=E\{f(X^+)\mid X=x\}\) be the forward transition operator. For
\(\omega\in L^1(\nu)\), define \(T_\pi'\omega\) as the \(\nu\)-density of the
pushed-forward signed measure \((\omega\nu)\Ppi\). When the functions are
square-integrable, \(T_\pi'\) agrees with the \(L^2(\nu)\)-adjoint:
\[
    \langle \omega,T_\pi f\rangle_\nu
    =
    \langle T_\pi'\omega,f\rangle_\nu .
\]
Define the backward conditional-mean operator
\[
    \Mpi\omega(x)
    =
    E\{\omega(X)\mid X^+=x\},
\]
where \(X\sim\nu\) and \(X^+\mid X\sim\Ppi(\cdot\mid X)\).

\begin{lemma}[Backward-regression factorization of the adjoint transition]
\label{lem:reverse-regression}
Suppose Condition~\ref{ass:overlap} holds. For every
\(\omega\in L^1(\nu)\), we have \(\Mpi\omega\in L^1(\nup)\), and
\begin{equation}
\label{eq:adjoint-transition-regression}
    T_\pi'\omega
    =
    \cpi\Mpi\omega
    \qquad \nu\text{-a.e.}
\end{equation}
Equivalently, \(\cpi\Mpi\omega\) is the \(\nu\)-density of
\((\omega\nu)\Ppi\).
\end{lemma}

\begin{proof}
Let \(J\) be the joint distribution of \((X,X^+)\) generated by \(X\sim\nu\)
and \(X^+\mid X\sim\Ppi(\cdot\mid X)\). Its second marginal is
\(\nup=\nu\Ppi\). Since \(\omega\in L^1(\nu)\),
\[
    \E_J[\abs{\omega(X)}]=\LoneNu{\omega}<\infty,
\]
so the conditional expectation
\(\Mpi\omega(x)=\E[\omega(X)\mid X^+=x]\) belongs to \(L^1(\nup)\).
Conditional Jensen's inequality gives
\[
    \LoneNuPlus{\Mpi\omega}
    =
    \E_J[\abs{\E[\omega(X)\mid X^+]}]
    \le
    \E_J[\abs{\omega(X)}]
    =
    \LoneNu{\omega}.
\]
For any measurable \(B\subseteq\sX\),
\[
\begin{aligned}
    \int_B \cpi(x)\Mpi\omega(x)\,\nu(\dd x)
    &=
    \int_B \Mpi\omega(x)\,\nup(\dd x)\\
    &=
    \E_J[\Mpi\omega(X^+)\1\{X^+\in B\}]\\
    &=
    \E_J[\omega(X)\1\{X^+\in B\}]\\
    &=
    \int \Ppi(B\mid x)\omega(x)\,\nu(\dd x)\\
    &=
    ((\omega\nu)\Ppi)(B).
\end{aligned}
\]
Thus \(\cpi\Mpi\omega\) is the \(\nu\)-density of
\((\omega\nu)\Ppi\). Adding the initial measure gives
\[
    (1-\gamma)\dinit+\gamma(\omega\nu)\Ppi
    \ll
    \nu,
    \qquad
    \frac{\dd[(1-\gamma)\dinit+\gamma(\omega\nu)\Ppi]}{\dd\nu}
    =
    (1-\gamma)\omegazero+\gamma\cpi\Mpi\omega.
\]
\end{proof}

Combining Lemma~\ref{lem:reverse-regression} with
\eqref{eq:adjoint-bellman-fixed-point} gives the backward-regression adjoint Bellman
equation
\begin{equation}
\label{eq:adjoint-bellman}
    \omegastar
    =
    (1-\gamma)\omegazero+\gamma\cpi\Mpi\omegastar .
\end{equation}
Equivalently, the measure-level update has the density
representation
\[
    \Bpig\omega
    =
    (1-\gamma)\omegazero+\gamma\cpi\Mpi\omega .
\]
The regression algorithm uses this representation.

\subsection{Backward-regression fitted adjoint Bellman iteration}
\label{sec:fori-algorithm}

Given offline transitions \((S_i,A_i,S_i')_{i=1}^n\), let
\(X_i=(S_i,A_i)\). For each transition, draw
\(A_i^+\sim\pi(\cdot\mid S_i')\) and set \(X_i^+=(S_i',A_i^+)\). Thus
\(\{X_i\}_{i=1}^n\) is an offline data sample from \(\nu\), while
\(\{X_i^+\}_{i=1}^n\) is a sample from the one-step target-policy successor distribution
\(\nup=\nu\Ppi\).

Given first-stage estimates \(\omegahatzero\) and \(\chat\), the backward-regression
variant estimates the discounted occupancy ratio by fitted adjoint
Bellman iteration. Starting from an initial estimate \(\omegahat^{(0)}\),
iteration \(k\) performs a backward conditional-mean regression of
\(\omegahat^{(k)}(X_i)\) on the successor covariates \(X_i^+\).
With squared-error regression class
\(\mathcal M_k\), define
\[
    \mhat_k
    \in
    \argmin_{m\in\mathcal M_k}
    \frac1n\sum_{i=1}^n
    \left\{
        \omegahat^{(k)}(X_i)-m(X_i^+)
    \right\}^2 .
\]
This regression estimates the backward conditional mean
\(x\mapsto E\{\omegahat^{(k)}(X)\mid X^+=x\}\). The fitted adjoint Bellman
update is then
\[
    \omegahat^{(k+1)}(x)
    =
    (1-\gamma)\omegahatzero(x)
    +
    \gamma\chat(x)\mhat_k(x).
\]
After the first-stage ratio estimates are fixed, each iteration is a
supervised backward conditional-mean regression. A clipped and empirically
normalized update uses a level \(M<\infty\) and sets
\[
    \tilde\omega^{(k+1)}(x)
    =
    \min\{\max\{\omegahat^{(k+1)}(x),0\},M\},
    \qquad
    \omegahat^{(k+1)}(x)
    \leftarrow
    \frac{\tilde\omega^{(k+1)}(x)}
    {n^{-1}\sum_{i=1}^n \tilde\omega^{(k+1)}(X_i)} .
\]
Algorithm~\ref{alg:strong-form-regression-fori} summarizes the procedure.

\begin{algorithm}[!htb]
\caption{Backward-regression \textsc{FORE}}
\label{alg:strong-form-regression-fori}
\begin{algorithmic}[1]
\Require Offline transitions \(\{X_i=(S_i,A_i),S_i'\}_{i=1}^n\), target policy
\(\pi\), target initial distribution \(\dinit=\mu_0\pi\), discount \(\gamma\), iteration
count \(K\), and regression classes \(\mathcal M_k\)
\State Draw \(A_i^+\sim\pi(\cdot\mid S_i')\) and set \(X_i^+=(S_i',A_i^+)\)
for \(i=1,\ldots,n\)
\State Estimate
\(\omegahatzero\approx\omegazero=\dd\dinit/\dd\nu\)
\State Estimate
\(\chat\approx\cpi=\dd\nup/\dd\nu\)
\For{\(k=0,\ldots,K-1\)}
    \State Fit backward conditional-mean regression \(\mhat_k
        \in
        \argmin_{m\in\mathcal M_k}
        \frac1n\sum_{i=1}^n
        \left\{\omegahat^{(k)}(X_i)-m(X_i^+)\right\}^2\)
    \State Set \(\omegahat^{(k+1)}(x)
        =
        (1-\gamma)\omegahatzero(x)+\gamma\chat(x)\mhat_k(x)\)
    \State Clip and normalize \(\omegahat^{(k+1)}\) if the optional clipping
    step is used
\EndFor
\Ensure Discounted occupancy-ratio estimate \(\omegahat^{(K)}\)
\end{algorithmic}
\end{algorithm}

\subsection{First-stage density-ratio estimation}
\label{sec:first-stage-ratios}

The fitted iteration requires estimates of two density ratios: the initial
ratio
\[
    \omegazero = \frac{\dd\dinit}{\dd\nu},
    \qquad
    \dinit(\dd s,\dd a)=\mu_0(\dd s)\pi(\dd a\mid s),
\]
and the one-step ratio
\[
    \cpi = \frac{\dd\nup}{\dd\nu}.
\]
The initial distribution \(\mu_0\) specifies the starting-state population whose target
occupancy is being evaluated. Given samples \(S_j^0\sim\mu_0\), drawing
\(A_j^0\sim\pi(\cdot\mid S_j^0)\) gives
\(X_j^0=(S_j^0,A_j^0)\sim\dinit\). These numerator samples, together with the
offline data sample \(\{X_i\}_{i=1}^n\sim\nu\), can be used to estimate
\(\omegazero\). Similarly, the successor pairs
\(X_i^+=(S_i',A_i^+)\), with \(A_i^+\sim\pi(\cdot\mid S_i')\), are sampled from
\(\nup=\nu\Ppi\). Hence \(\{X_i^+\}_{i=1}^n\), together with the same offline data
sample, can be used to estimate \(\cpi\).

Any density-ratio learner that targets these two ratios can be used to
construct \(\omegahatzero\) and \(\chat\), including methods based on
classification, \(f\)-divergence, or Bregman-risk objectives. These estimates
are then held fixed throughout the adjoint Bellman iteration. For numerical
stability, implementations may clip extreme ratios or apply post-hoc moment
calibration.

\subsection{Population contraction}
\label{sec:population-contraction}

The backward-regression representation in Lemma~\ref{lem:reverse-regression}
identifies the population update underlying
Algorithm~\ref{alg:strong-form-regression-fori}. The next result states its
\(L^1(\nu)\) contraction and fixed point.

\begin{theorem}[Backward-regression \textsc{FORE} contraction]
\label{thm:regression-population-contraction}
Assume Condition~\ref{ass:overlap} and let \(\gamma\in[0,1)\). Then
\(\Bpig:L^1(\nu)\to L^1(\nu)\) is well defined and satisfies
\[
    \LoneNu{\Bpig\omega-\Bpig\omega'}
    \le
    \gamma\LoneNu{\omega-\omega'}
    \qquad
    \text{for all }\omega,\omega'\in L^1(\nu).
\]
Moreover, \(d_{\pi,\gamma}\ll\nu\), and
\(\omegastar=\dd d_{\pi,\gamma}/\dd\nu\) is the unique fixed point of \(\Bpig\)
in \(L^1(\nu)\). Hence, for any \(\omega\in L^1(\nu)\),
\[
    \LoneNu{(\Bpig)^K\omega-\omegastar}
    \le
    \gamma^K\LoneNu{\omega-\omegastar}.
\]
\end{theorem}
\begin{proof}
Lemma~\ref{lem:reverse-regression} identifies
\(\Bpig\omega-\Bpig\omega'\) as the \(\nu\)-density of the signed measure
\(\gamma\{(\omega-\omega')\nu\}\Ppi\). Nonexpansiveness of Markov kernels for
finite signed measures \citep{coverThomas2006Elements} gives
\[
    \LoneNu{\Bpig\omega-\Bpig\omega'}
    \le
    \gamma |(\omega-\omega')\nu|(\sX)
    =
    \gamma\LoneNu{\omega-\omega'}.
\]
The discounted occupancy recursion
\(d_{\pi,\gamma}=(1-\gamma)\dinit+\gamma d_{\pi,\gamma}P_\pi\) shows that
\(\omegastar\) is a fixed point after taking Radon--Nikodym derivatives.
The contraction gives uniqueness in \(L^1(\nu)\), and iterating it yields the
geometric bound.
\end{proof}

\subsection{Projection error and adjoint Bellman completeness}
\label{sec:projected-fori}

Theorem~\ref{thm:regression-population-contraction} describes the ideal
population iteration, in which the backward conditional mean
\(P_{\pi,\nu}^{\leftarrow}\omega\) is evaluated exactly.
Algorithm~\ref{alg:strong-form-regression-fori} instead estimates this
backward conditional mean within a supervised-learning class. We isolate the
population effect of this projection.

For a regression class \(\mathcal M\), let \(\Pi_{\mathcal M}^{+}g\) denote an
\(L^2(\nu_\pi^+)\)-projection of \(g\) onto \(\mathcal M\). The corresponding
population FORE update is
\[
    \widetilde{\mathsf B}_\gamma^\pi \omega
    =
    (1-\gamma)\omegazero
    +
    \gamma\cpi \Pi_{\mathcal M}^{+}P_{\pi,\nu}^{\leftarrow}\omega .
\]
Thus, \(\widetilde{\mathsf B}_\gamma^\pi\) differs from the exact adjoint Bellman update
only by replacing the exact backward conditional mean with its population
projection onto \(\mathcal M\).

For a set \(\mathcal W\subseteq L^1(\nu)\) of possible iterates, define the
\textit{inherent adjoint Bellman error}
\[
    b_{\mathcal M}(\mathcal W)
    =
    \sup_{\omega\in\mathcal W}
    \inf_{m\in\mathcal M}
    \left\|
    m-P_{\pi,\nu}^{\leftarrow}\omega
    \right\|_{L^2(\nu_\pi^+)} .
\]

We say that \(\mathcal M\) is \textit{adjoint Bellman complete} over
\(\mathcal W\) if \(P_{\pi,\nu}^{\leftarrow}\omega\in\mathcal M\) for every
\(\omega\in\mathcal W\) \citep{ueharaEtAl2021FiniteSampleMinimax}. This is the
direct analogue of Bellman completeness for FQE: the regression class must
contain the one-step adjoint Bellman image of every iterate encountered by the
fitted procedure. In this case, \(b_{\mathcal M}(\mathcal W)=0\).

\begin{lemma}[Population perturbation from backward-regression projection]
\label{lem:population-projected-fori}
Assume Condition~\ref{ass:overlap} and let \(\gamma\in[0,1)\). Let
\(\widetilde\omega^{(k+1)}
=
\widetilde{\mathsf B}_\gamma^\pi\widetilde\omega^{(k)}\), and set
\(\mathcal W_K=\{\widetilde\omega^{(0)},\ldots,\widetilde\omega^{(K-1)}\}\).
Then
\begin{equation}
\label{eq:backward-projection-iterated-bound}
    \left\|
    \widetilde\omega^{(K)}-\omegastar
    \right\|_{L^1(\nu)}
    \le
    \gamma^K
    \left\|
    \widetilde\omega^{(0)}-\omegastar
    \right\|_{L^1(\nu)}
    +
    \frac{\gamma(1-\gamma^K)}{1-\gamma}
    b_{\mathcal M}(\mathcal W_K).
\end{equation}
Consequently, if \(\mathcal M\) is adjoint Bellman complete over
\(\mathcal W_K\), then
\begin{equation}
\label{eq:backward-projection-complete-bound}
    \left\|
    \widetilde\omega^{(K)}-\omegastar
    \right\|_{L^1(\nu)}
    \le
    \gamma^K
    \left\|
    \widetilde\omega^{(0)}-\omegastar
    \right\|_{L^1(\nu)} .
\end{equation}
\end{lemma}

\begin{proof}[Proof of Lemma~\ref{lem:population-projected-fori}]
For each \(k\), add and subtract the exact population update
\(\Bpig\widetilde\omega^{(k)}\). The contraction in
Theorem~\ref{thm:regression-population-contraction} gives
\[
\begin{aligned}
    \left\|
    \widetilde\omega^{(k+1)}-\omegastar
    \right\|_{L^1(\nu)}
    &\le
    \gamma
    \left\|
    \widetilde\omega^{(k)}-\omegastar
    \right\|_{L^1(\nu)}
    \\
    &\quad+
    \left\|
    \widetilde{\mathsf B}_\gamma^\pi\widetilde\omega^{(k)}
    -
    \Bpig\widetilde\omega^{(k)}
    \right\|_{L^1(\nu)} .
\end{aligned}
\]
The perturbation term is
\[
\begin{aligned}
    \left\|
    \widetilde{\mathsf B}_\gamma^\pi\widetilde\omega^{(k)}
    -
    \Bpig\widetilde\omega^{(k)}
    \right\|_{L^1(\nu)}
    &=
    \gamma
    \left\|
    \cpi
    \left\{
    \Pi_{\mathcal M}^{+}P_{\pi,\nu}^{\leftarrow}\widetilde\omega^{(k)}
    -
    P_{\pi,\nu}^{\leftarrow}\widetilde\omega^{(k)}
    \right\}
    \right\|_{L^1(\nu)}
    \\
    &=
    \gamma
    \left\|
    \Pi_{\mathcal M}^{+}P_{\pi,\nu}^{\leftarrow}\widetilde\omega^{(k)}
    -
    P_{\pi,\nu}^{\leftarrow}\widetilde\omega^{(k)}
    \right\|_{L^1(\nu_\pi^+)}
    \\
    &\le
    \gamma
    \left\|
    \Pi_{\mathcal M}^{+}P_{\pi,\nu}^{\leftarrow}\widetilde\omega^{(k)}
    -
    P_{\pi,\nu}^{\leftarrow}\widetilde\omega^{(k)}
    \right\|_{L^2(\nu_\pi^+)}
    \\
    &=
    \gamma
    \inf_{m\in\mathcal M}
    \left\|
    m-P_{\pi,\nu}^{\leftarrow}\widetilde\omega^{(k)}
    \right\|_{L^2(\nu_\pi^+)}
    \\
    &\le
    \gamma b_{\mathcal M}(\mathcal W_K),
\end{aligned}
\]
where the second equality uses \(\nu_\pi^+(\dd x)=\cpi(x)\nu(\dd x)\), the
next inequality uses that \(\nu_\pi^+\) is a probability distribution, and the final
equality is the defining property of the \(L^2(\nu_\pi^+)\)-projection. Hence
\begin{equation}
\label{eq:backward-projection-one-step}
    \left\|
    \widetilde\omega^{(k+1)}-\omegastar
    \right\|_{L^1(\nu)}
    \le
    \gamma
    \left\|
    \widetilde\omega^{(k)}-\omegastar
    \right\|_{L^1(\nu)}
    +
    \gamma b_{\mathcal M}(\mathcal W_K).
\end{equation}
Iterating \eqref{eq:backward-projection-one-step} gives
\begin{equation}
\label{eq:backward-projection-geometric-sum}
    \left\|
    \widetilde\omega^{(K)}-\omegastar
    \right\|_{L^1(\nu)}
    \le
    \gamma^K
    \left\|
    \widetilde\omega^{(0)}-\omegastar
    \right\|_{L^1(\nu)}
    +
    \gamma\sum_{j=0}^{K-1}\gamma^j b_{\mathcal M}(\mathcal W_K).
\end{equation}
Evaluating the geometric sum in
\eqref{eq:backward-projection-geometric-sum} gives
\eqref{eq:backward-projection-iterated-bound}.
Under adjoint Bellman completeness over \(\mathcal W_K\), the inherent adjoint
Bellman error is zero, giving
\eqref{eq:backward-projection-complete-bound}.
\end{proof}

Lemma~\ref{lem:population-projected-fori} makes the limitation of this
\textsc{FORE} variant explicit. Without adjoint Bellman completeness, the
backward-regression iteration converges only up to the inherent adjoint
Bellman error.

\section{Numerical experiment details}
\label{app:experiment-details}

This appendix records the exact constructions used for
Section~\ref{sec:experiments}. The reported finite-sample summaries use common
random draws across estimators at a fixed sample size and repetition index.

\subsection{Baird-style finite MRP}
\label{app:baird-style-details}

The state space is
\(\sX=\{u_1,\ldots,u_6,\ell\}\). The six states \(u_j\) are symmetric upper
states and \(\ell\) is the lower state. We use \(\gamma=0.95\),
\[
    \nu(u_j)=0.95/6,\qquad \nu(\ell)=0.05,
    \qquad
    d_0(u_j)=1/6,\qquad d_0(\ell)=0 .
\]
The target transition matrix has the aggregate form
\[
    P(u_j,u_m)=0.05/6,\quad P(u_j,\ell)=0.95,\qquad
    P(\ell,u_m)=0.20/6,\quad P(\ell,\ell)=0.80 .
\]
The scalar value feature is
\[
    \phi(u_j)=0.1,\qquad \phi(\ell)=1 .
\]
Rewards are defined by \(r=\phi-\gamma P\phi\), so
\[
    r(u_j)=-0.80725,\qquad r(\ell)=0.221,
    \qquad Q^\pi=\phi,\qquad V_\pi=0.1 .
\]
This construction keeps the six-upper/one-lower star layout and chooses the
target transition, offline data distribution, and rewards to isolate fitted policy
evaluation.

For linear FQE with \(q_\beta=\beta\phi\), the population update has the
one-dimensional form
\[
    \beta_{k+1}=1-\lambda+\lambda\beta_k,\qquad
    \lambda
    =
    \gamma
    \frac{E_\nu\{\phi(X)\phi(X^+)\}}{E_\nu\{\phi(X)^2\}}
    =
    2.1031722689 .
\]
Thus \(Q^\pi\in\sQ\), but the projected Bellman recursion expands errors under
the Bellman projection induced by the offline data distribution. The
discounted occupancy ratio is constant on the
upper states and equal to
\[
    \omegastar(u_j)=0.2211217321,\qquad
    \omegastar(\ell)=15.7986870897 .
\]
It is represented exactly by the normalized log-linear class
\[
    \omega_\theta(x)
    =
    \frac{\exp\{\theta\,1(x=\ell)\}}
         {E_\nu[\exp\{\theta\,1(X=\ell)\}]} .
\]
The fixed point is \(\theta^\star=4.7432986067\). The derivative of the
scalar \textsc{FORE} update at \(\theta^\star\) has magnitude \(0.1425\).
Projecting FQE under the target occupancy distribution gives scalar multiplier
\(0.8009962427\).

With tabular FQE, the value feature matrix is the \(7\times 7\) identity. The
Bellman image remains in the fitted value class, the projected Bellman operator
has contraction multiplier \(\gamma=0.95\), and the population iterates converge to
the exact value function.

The tabular finite-dimensional objectives for DualDICE, MWL, and MQL are
solved exactly in this finite example. The Baird-style panel in the main text
focuses on the population recursions, where the separation between offline-data
projection and target-occupancy projection is algebraic.

\subsection{Linear-Gaussian policy evaluation}
\label{app:gaussian-details}

The continuous example uses \(X=(S,A)\in\mathbb R^2\). Offline samples are
drawn from \(\nu=N(0,\Sigma_b)\), where
\[
    \Sigma_b
    =
    \begin{pmatrix}
    1.5 & 0 \\
    0 & 0.4
    \end{pmatrix}.
\]
Under the target policy,
\[
    S^+=0.7S+0.5A+\varepsilon_s,\qquad
    A^+=-0.8S^+ + \varepsilon_a,
\]
where \(\varepsilon_s\sim N(0,0.4)\) and
\(\varepsilon_a\sim N(0,0.25)\) are independent. The initial distribution is the
stationary Gaussian distribution of this target transition, with covariance
\[
    \Sigma_\star
    =
    \begin{pmatrix}
    0.508242 & -0.406593\\
    -0.406593 & 0.575275
    \end{pmatrix}.
\]
Thus the target density ratio relative to \(\nu\) is exponential quadratic:
\[
    \omega_\star(x)
    \propto
    \exp\{\theta^\star_1 s^2+\theta^\star_2sa+\theta^\star_3a^2\},
    \qquad
    \theta^\star=(-1.9304504505,\,-3.2,\,-0.75).
\]
Let
\[
    h_\star(x)
    =
    \theta^\star_1s^2+\theta^\star_2sa+\theta^\star_3a^2 .
\]
The normalized log-linear ratio class used by \textsc{FORE}, MWL, and
DualDICE has sufficient statistics \((h_\star(x),s,a)\). The target ratio has
coefficient \((1,0,0)\) in this class. The class is not closed under the target
transition or the corresponding adjoint Bellman update.

For the main finite-sample experiment, the discount is \(\gamma=\gamma_0\).
The value feature is
\[
    q(x)=a^2+0.7241380519,
\]
and the reward is set to
\[
    r(x)=q(x)-\gamma E\{q(X^+)\mid X=x\}.
\]
It follows that \(Q^\pi=q\). Since the initial distribution is stationary under the
target policy, the true policy value is
\[
    V_\pi=E_{\Sigma_\star}\{q(X)\}=1.299413 .
\]
Linear FQE, \textsc{FORE}-reweighted FQE, and MQL use the value class
\[
    \{ \beta_0q+\beta_1s+\beta_2a:\beta\in\mathbb R^3\},
\]
whose true coefficient is \((1,0,0)\). The class is not Bellman complete,
because \(P_\pi q\) contains the missing quadratic terms \(s^2\) and \(sa\).

The corresponding population multipliers are as follows. Under the offline
data distribution, the linear FQE multiplier is \(1.220598\). The population
\textsc{FORE} recursion has effective contraction multiplier \(0.085500\). When
the FQE projection distribution is replaced by the target occupancy distribution recovered by
\textsc{FORE}, the dominant projected FQE multiplier is \(0.682476\).

Finite-sample experiments use
\[
    n\in\{500,1000,2000,5000,10000\}
\]
offline transitions and \(300\) independent repetitions at each sample size.
For \textsc{FORE}, each update solves the empirical moment equation for the
three log-ratio sufficient statistics. \textsc{FORE}-reweighted FQE uses the
same value class as linear FQE, but replaces the empirical least-squares
weights by the fitted \textsc{FORE} density weights from the same sample.

MWL, MQL, and DualDICE use a random-Fourier RBF critic with \(128\) features
and an intercept term. The RBF bandwidths, critic ridge penalties, and MWL
density shrinkage coefficient are fixed once using an independent
offline-data population calculation and then held fixed across all sample sizes
and repetitions. MWL and MQL use bandwidth \(2.2\) and ridge penalty \(0.1\).
MWL uses density shrinkage \(10^{-3}\). DualDICE uses bandwidth \(0.4\), ridge
penalty \(10^{-4}\), and no density shrinkage. The DualDICE potential ridge is
used only to stabilize the finite critic solve.

For the discount sweep in Figure~\ref{fig:gaussian-gamma-sweep}, we use
\[
    \gamma\in
    \{0.5,0.6,0.7,0.8,0.85,0.9,0.93,0.95,0.97,0.98,0.99,
      0.995,0.997,0.999\}.
\]
The sample size is \(n=5000\), with \(500\) repetitions at each discount. The
transition distribution, initial distribution, offline data distribution, value class, ratio class, and critic
classes are the same as in the main experiment. At each discount, rewards are
set to
\[
    r_\gamma(x)=q(x)-\gamma E\{q(X^+)\mid X=x\}.
\]
Linear FQE and \textsc{FORE}-reweighted FQE use \(500\) fitted-\(Q\) updates
for every repetition and discount. \textsc{FORE} iterates until the relative
change in the log-ratio parameter is below \(10^{-7}\), with a maximum of
\(100\) updates. The linear FQE panel reports the fixed-iteration error at
every discount, including discounts for which the empirical projected Bellman
recursion is noncontractive.

\subsection{Coverage-stopped occupancy under insufficient data coverage}
\label{app:stopped-fore-details}

\paragraph{Data-generating process.}
Let \(C\) be uniform on eight contexts and let \(T\in\{0,1\}\) denote the
initial and hub stages, with probabilities \(1-\gamma\) and \(\gamma\),
respectively. The target action is zero. On a covered context--stage pair, the
behavior policy selects this action with probability \(q=0.25\); on an
uncovered pair, it selects it with probability zero. The alternative action is
therefore observed even in contexts where target flow is singular relative to
the offline occupancy. Successor states retain the current context and enter
the hub stage.

For each \(p\in\{0,0.25,0.5,0.75,1\}\), an evenly spread subset of \(8p\)
contexts is designated covered. Under initial support failure, action zero is
supported at the initial stage only in these contexts and is supported at the
hub in every context. Under successor support failure, it is supported
initially in every context and at the hub only in the selected contexts. The
exact coverage-stopped occupancy ratio is \(1/q=4\) on retained target-action
pairs and zero elsewhere. Its total mass is therefore
\[
    m_{\mathrm{cov}}(p)
    =
    \begin{cases}
      p, & \text{initial support failure},\\
      (1-\gamma)+\gamma p, & \text{successor support failure}.
    \end{cases}
\]
We set \(\gamma=0.95\). The four rewards are the constant reward, indicators
of coverage-stopped initial and hub occupancy, and a bounded
context-dependent reward that increases linearly from \(0.1\) to \(1\) across
the eight contexts.

\paragraph{Training and evaluation.}
For each support-failure location, covered-context fraction, and repetition,
we draw one training sample and an independent source-distribution test
sample. Training sizes are
\(n\in\{2{,}000,10{,}000\}\), and each test sample contains \(50{,}000\) rows.
Every estimator is fitted once on the full training sample. Rewards, the exact
coverage-stopped occupancy ratio, its mass, and the coverage-stopped values are
used only for evaluation. There is no renormalization of the fitted
coverage-stopped occupancy.

Coverage-stopped \textsc{FORE} uses two hidden layers of width \(64\), \(120\)
outer iterations, and \(50\) warm-started gradient steps for each coverage
classifier and ratio update. Its fitted ratio is restricted to
\([10^{-6},20]\). The estimand itself is not clipped: its maximum positive
ratio is \(4\), so the upper clipping level is inactive in the population.
Standard \textsc{FORE} uses the same hidden-layer widths, \(60\) outer
iterations, and \(50\) gradient steps per update. The post-hoc baseline is
\(\min\{\widehat\omega_{\mathrm{std}},20\}\) and reuses the standard fit.
The standard forms of DualDICE and MWL are omitted because they target a
normalized occupancy rather than the coverage-stopped occupancy defined in
Section~\ref{sec:coverage-clipped-target}.

We report independent-test \(L^1(\nu)\) ratio error, absolute coverage-stopped
occupancy-mass error, and mean absolute coverage-stopped value error over the
four rewards. The constant-reward value equals the coverage-stopped occupancy
mass and therefore checks directly that the fitted measure has not been
normalized to one. We also report coverage-classification errors, whether any
fit reaches the upper clipping level, and runtime.

\begin{table}[H]
\centering
\caption{Coverage-stopped occupancy experiment. Entries are median
independent-test errors over \(300\) runs at each sample size: \(30\)
repetitions, two locations of support failure, and five covered-context
fractions. All methods are evaluated against the same coverage-stopped
occupancy ratio.}
\label{tab:stopped-fore-summary}
\begin{tabular}{@{}r l r r r@{}}
\toprule
\(n\) & Method & Ratio \(L^1(\nu)\) & Mass error & Mean value error \\
\midrule
\(2{,}000\) & Coverage-stopped \textsc{FORE} & \(0.0783\) & \(0.0276\) & \(0.0197\) \\
             & Standard \textsc{FORE}          & \(1.3187\) & \(0.7762\) & \(0.4545\) \\
             & Post-hoc clipped                & \(0.7722\) & \(0.3327\) & \(0.1953\) \\
\addlinespace
\(10{,}000\) & Coverage-stopped \textsc{FORE} & \(0.0324\) & \(0.0090\) & \(0.0067\) \\
              & Standard \textsc{FORE}         & \(1.1542\) & \(0.4753\) & \(0.3145\) \\
              & Post-hoc clipped               & \(0.6262\) & \(0.1643\) & \(0.1244\) \\
\bottomrule
\end{tabular}
\end{table}

\paragraph{Coverage classification and ratio errors.}
At \(n=10{,}000\), the learned coverage classifier retains no uncovered
initial or successor context. In one run it removes one covered initial
context, corresponding to \(0.172\%\) of the test rows. At \(n=2{,}000\), one
run with successor support failure retains one of two uncovered successor
contexts; its coverage-stopped occupancy-ratio \(L^1(\nu)\) error is \(0.114\),
compared with approximately \(1.45\times 10^{10}\) for standard
\textsc{FORE}. In \(64\) small-sample runs, the classifier removes at least one
covered initial context, but never more than \(0.332\%\) of the test rows. No
coverage-stopped \textsc{FORE} fit reaches the upper clipping level. By
contrast, standard \textsc{FORE} has \(65\) of \(300\) small-sample ratio
errors above \(10^6\). Post-hoc clipping removes these extreme values but not
the error induced by targeting the normalized occupancy.

\paragraph{Runtime.}
We compare runtime on one CPU thread using the \(n=10{,}000\) training
schedule, both locations of support failure at \(p=0.5\), and three
repetitions. Median runtimes are \(80.8\) seconds for coverage-stopped
\textsc{FORE} and \(21.1\) seconds for standard \textsc{FORE}; maxima are
\(111.9\) and \(26.8\) seconds, respectively. Post-hoc clipping reuses the
standard fit.

\end{document}